\documentclass[runningheads]{llncs}

\usepackage{eccv}

\usepackage{eccvabbrv}
\usepackage{graphicx}
\usepackage{booktabs}
\usepackage{multirow}
\usepackage[accsupp]{axessibility}

\newcommand{\algoname}{\textsc{FDHO}}
\DeclareMathOperator{\sinc}{sinc}
\DeclareMathOperator{\sign}{sign}

\usepackage[export]{adjustbox}

\usepackage{soul}
\usepackage{hyperref}
\usepackage{orcidlink}

\usepackage{xr}
\begin{document}

\title{Spectral Gating via Damped Oscillations for Adaptive Implicit Neural Representations} 

\titlerunning{Spectral Gating via Damped Oscillations}

\author{
Alex Costanzino \inst{1} \orcidlink{0000-0001-9859-8482} 
\and
Pierluigi Zama Ramirez\thanks{Partially supported by the IRIDE-B program, Ca' Foscari University of Venice.} \inst{2} \orcidlink{0000-0001-7734-5064}
\and
Giuseppe Lisanti \inst{1} \orcidlink{0000-0002-0785-9972} 
\and
Luigi Di Stefano \inst{1} \orcidlink{0000-0001-6014-6421}
}

\authorrunning{A.~Costanzino et al.}

\institute{
CVLab, University of Bologna \and Ca' Foscari University of Venice
}

\maketitle

\begin{abstract}
    Implicit Neural Representations (INRs) have been proven successful in encoding continuous signals through coordinate-based networks, yet facing a spectral dilemma: periodic activations capture fine details but act as all-pass filters that memorise noise, while spatially compact activations regularise effectively but suffer from low-frequency bias.
    Existing attempts to resolve this trade-off introduce computational overhead or tuning frailty.
    We propose to model each neuron's activation as the steady-state response of a sinusoidally-forced damped harmonic oscillator, whose amplitude naturally governs the network's spectral selectivity during training.
    By jointly optimising the oscillator parameters alongside the network weights, our method adapts to the target signal's spectral content without explicit regularisation. 
    Initialised in the stopband, the network exhibits a coarse-to-fine learning curriculum that progressively expands its spectral gate, capturing low-frequency structures first and high-frequency details only when justified by the reconstruction objective.
    Comprehensive experiments show that our approach consistently achieves state-of-the-art or competitive results against established INRs, while requiring no task-specific tuning of any hyperparameters.
    
    Project Page available at \url{https://alex-costanzino.github.io/fdho/}.
    

\end{abstract}

\section{Introduction}
\label{sec:introduction}

    Implicit Neural Representations (INRs) have emerged as a powerful paradigm for encoding continuous signals, such as audio, images, shapes, and physical fields, into the weights of coordinate-based neural networks~\cite{sitzmann2019siren, mildenhall2020nerf, park2019deepsdf}. 
    Their appeal lies in being resolution-independent, differentiable, and compact, making them natural fits for tasks ranging from novel view synthesis~\cite{barron2021mipnerf} to solving partial differential equations~\cite{sitzmann2019siren}. 
    At the heart of any INR lies the choice of the activation function, which governs the network's spectral expressivity, \ie, its ability to faithfully represent both smooth global structure and fine-grained local details.
    This choice confronts practitioners with what we term the spectral dilemma.
    Periodic activations, exemplified by SIREN~\cite{sitzmann2019siren}, endow the network with the capacity to represent arbitrarily high frequencies.
    However, their infinite bandwidth makes them susceptible to noise memorisation, treating random perturbations as legitimate signal content.
    Non-periodic alternatives, such as Gaussian activations~\cite{ramasinghe2022beyond}, provide spatial compactness and inherent regularisation but introduce a low-frequency bias that limits their ability to capture oscillatory or fine-grained content.
    Intermediate approaches exist -- Gabor wavelets in WIRE~\cite{saragadam2023wire} offer joint space-frequency localisation, FINER~\cite{liu2024finer} introduces variable-periodic functions, and Fourier Reparametrisation~\cite{shi2024improved} addresses spectral bias through weight decomposition -- but each comes with its own trade-offs in computational cost, sensitivity to hyperparameters, or restricted representational capacity. 
    More recently, MIRE~\cite{jayasundara2024mire} proposes learning per-layer activations selected from a dictionary of predefined functions, and IGA-INR~\cite{shi2024iga} adjusts gradient dynamics through inductive gradient correction, both methods highlighting that the interaction between activation design and training dynamics is central to INR performance.
    Despite this progress, no existing method provides a mechanism to adapt the bandwidth robustly with respect to noise.

    Inspired by classical mechanics, we model each neuron as the steady-state response of a sinusoidally-forced damped harmonic oscillator.
    The resulting activation takes the form of an amplitude-modulated sinusoidal carrier. 
    Such modulation depends on three physical quantities: the forcing frequency $\omega$, the natural frequency $\omega_n$, and the damping factor $\xi$.
    In particular, the amplitude is the transfer function of a second-order linear system, a well-characterised band-pass filter whose centre frequency, bandwidth, and roll-off are determined by $\omega_n$ and $\xi$.
    We refer to this amplitude-modulated bandwidth control as \emph{spectral gating}: the oscillator parameters determine a passband, and only frequency content from the forcing within this band receives non-negligible amplitude gain.
    When these parameters are made learnable, the network acquires the ability to adapt its bandwidth during training. 
    In particular, with our formulation, the optimisation cost to admit higher frequencies is lower for coherent signal content than for random noise (\cref{sec:gradient_analysis}, \cref{cor:gating}).
    Moreover, a coarse-to-fine learning curriculum emerges as we initialise the oscillator parameters in the stopband, so that the network begins with a narrow bandwidth.
    As training progresses and the loss landscape demands richer frequency content, the bandwidth widens, producing a natural coarse-to-fine learning progression without explicit scheduling (\cref{sec:multilayer_short}).
    Finally, the damping factor $\xi$ provides a continuous, learnable axis connecting distinct spectral regimes: as $\xi \to 0$, the activation recovers the unregulated bandwidth of a pure SIREN-style sinusoid, as it increases, the function smoothly transitions toward a more spatially compact representation.
    Hence, rather than committing to a fixed point in the INR's design space, the network can learn where to operate along this axis.

    We validate these claims through extensive experiments and comparisons with established INRs, achieving state-of-the-art or competitive results across all evaluated settings, while requiring no task-specific tuning of any hyperparameters.

\section{Related Work}
\label{sec:related_work}

    \paragraph{Activation Design for INRs.}
        The choice of activation function is central to INR performance.
        ReLU-based networks are known to suffer from spectral bias toward low frequencies~\cite{tancik2020fourfeat}.
        SIREN~\cite{sitzmann2019siren} addressed this with sinusoidal activations, enabling high-frequency representation but acting as an all-pass filter with no intrinsic spatial localisation or frequency regularisation, which leaves it prone to noise memorisation and sensitive to initialisation.
        These limitations motivated a line of activation-level remedies: Gaussian activations~\cite{ramasinghe2022beyond} provide smoothness and an implicit low-frequency bias, while WIRE~\cite{saragadam2023wire} combines spatial and frequency localisation through Gabor wavelets.
        FINER~\cite{liu2024finer} and MIRE~\cite{jayasundara2024mire} instead retain the sinusoidal form but add per-layer spectral control through variable-periodic activations and learned function selection, respectively.
        These methods address the spectral dilemma through activation shape, but none provide a formal mechanism explaining why certain frequencies are favoured or suppressed during training.
        Our approach differs by deriving the activation from a physical system whose transfer function yields provable spectral gating properties.
  
    \paragraph{Spectral Bias and Training Dynamics.}
        A parallel line of work addresses spectral limitations not through the activation itself but through the training procedure or input encoding.
        Fourier feature mappings~\cite{tancik2020fourfeat} lift coordinates into a high-dimensional sinusoidal basis, enabling ReLU-based networks to learn high frequencies. 
        Their NTK~\cite{jacot2018neural} analysis shows that the feature bandwidth directly controls the kernel's spectral support.
        Multiplicative filter networks~\cite{fathony2021multiplicative} achieve frequency control through element-wise products of Gabor filters across layers.
        BACON~\cite{lindell2021bacon} enforces a band-limited hierarchy via additive multi-scale outputs.
        TUNER~\cite{novello2025tuner} instead keeps the sinusoidal activation fixed and controls the spectrum at initialisation, bounding the network's frequency support to stabilise training and prevent overfitting.
        Fourier Reparameterisation~\cite{shi2024improved} decomposes network weights in the frequency domain to mitigate spectral bias, and IGA-INR~\cite{shi2024iga} corrects gradient dynamics through inductive gradient adjustment.
        Our work is complementary: rather than modifying the input encoding, weight parameterisation, or gradient update rule, we embed spectral control directly into the activation's differentiable structure, where it arises as an implicit consequence of optimisation.

    \paragraph{Theoretical analysis of INRs.} 
        Theoretical understanding of INRs remains limited: \cite{tancik2020fourfeat} analyses the NTK of Fourier features to explain their accelerated convergence; \cite{yuce2022structured} shows that SIREN layers compose via Bessel functions; \cite{novello2025tuner} develops an amplitude–phase expansion of sinusoidal MLPs, characterising how layer composition generates new frequencies and bounding the network's representable spectrum; \cite{ramasinghe2022beyond} derives Lipschitz bounds and NTK eigenspectra for Gaussian activations; \cite{saragadam2023wire} analyses Gabor wavelet support to motivate WIRE's frequency-space trade-off; and \cite{sitzmann2019siren} establishes the distribution of activations and gradients through depth for SIREN.
        However, these analyses characterise representational properties -- what each activation can express -- rather than optimisation dynamics -- how gradient descent navigates the parameter space.
        In particular, while TUNER governs the spectrum statically by bounding frequency support at initialisation, none of these works provides a gradient-level mechanism for noise rejection or adaptive bandwidth control.
        Conversely, our contribution features an explicit gating condition that allows us to reject noise while fitting genuine signal components as training progresses.

\section{Formulation}
\label{sec:formulation}

    We derive our activation function from the steady-state analysis of a forced damped harmonic oscillator (\cref{sec:activation_definition}), establish its formal properties (\cref{sec:properties}, proofs in \cref{sec:properties_full}), analyse how its learnable parameters create implicit spectral regularisation through gradient dynamics (\cref{sec:gradient_analysis}, proofs in \cref{sec:gradient_analysis_full}), characterise the multi-layer frequency composition via Bessel-function expansions (\cref{sec:multilayer_short} and \cref{sec:multilayer}), and discuss its neural tangent kernel structure (\cref{sec:ntk}).

    \subsection{Activation Definition}
        \label{sec:activation_definition}
        
        Consider a damped harmonic oscillator with natural frequency $\omega_n$ and damping factor $\xi$, driven by a sinusoidal forcing at frequency $\omega$.
        Its motion equation is $\ddot{x} + 2 \xi \omega_n \dot{x} + \omega_n^2 x = \omega_n^2 \sin(\omega t)$, and the steady-state solution has the form $x_{s}(t) = A(\omega, \omega_n, \xi) \sin(\omega t + \varphi)$, with the amplitude given by:
        \begin{equation}
            \label{eq:amplitude}
            A(\omega, \omega_n, \xi) = \frac{\omega_n^2}{\sqrt{(\omega_n^2 - \omega^2)^2 + (2 \xi \omega_n \omega)^2}}.
        \end{equation}
        \begin{definition}           
            \label{def:activation}
            Given the set of learnable parameters $\theta = (\omega, \omega_n, \xi, \varphi)$, with $\omega; \omega_n; \xi > 0$, our proposed activation function, $\sigma \colon \mathbb{R} \to \mathbb{R}$, is defined as:
            \begin{equation}
                \label{eq:activation}
                \sigma(z; \theta) = A(\omega, \omega_n, \xi) \sin(\omega z + \varphi).
            \end{equation}
        \end{definition}
        Each hidden layer uses a parameter set $\theta$ shared across all neurons.
        Hence, the activation is applied element-wise to the input $z = \mathbf{w}^\top \mathbf{x} + b$, yielding $\sigma(\mathbf{W} \mathbf{x} + \mathbf{b}; \theta)$.
        Parameters $\theta$ and weights $(\mathbf{W}, \mathbf{b})$ are jointly optimised via gradient descent.

        It is worth highlighting that for a given set of layer parameters, the amplitude gain $A$ is a scalar constant, not a spatially varying envelope. 
        Hence, the network does not filter different frequency components at inference time; rather, the landscape of $A(\omega, \omega_n, \xi)$ shapes the gradient flow during training, determining how readily the network can amplify or suppress content at different frequencies as the optimisation evolves.
        The inductive bias, therefore, operates through optimisation dynamics, not through the functional form of a single forward pass.

        One might consider achieving similar adaptivity by simply making the amplitude, frequency, and phase of SIREN's sinusoidal activation independently learnable. 
        However, when these quantities are decoupled, nothing prevents the network from assigning a large amplitude to any frequency, including those driven by noise. 
        The oscillator formulation is fundamentally different, as the amplitude $A(\omega, \omega_n, \xi)$ is not a free parameter but is constrained by the relationship between forcing frequency, natural frequency, and damping factor.
        This physical coupling ensures that the optimisation cost to admit higher frequencies is lower for coherent signal content than for random noise, as we formalise in \cref{sec:gradient_analysis}, and that the learned parameters converge to stable equilibria rather than drifting indefinitely, as we verify empirically by comparing our proposal to learnable SIREN activations in \cref{sec:siren_vs_oscillator}.
    
    \subsection{Basic Properties}
        \label{sec:properties}
        
        Before analysing the training dynamics of our proposed activation, dubbed \emph{Forced Damped Harmonic Oscillator} (\algoname{}), we establish three properties: smoothness, Lipschitz continuity, and output boundedness.
        Smoothness in both $z$ and $\theta$ guarantees that all parameters can be jointly optimised via gradient-based methods; Lipschitz continuity bounds the magnitudes of gradients flowing through the network during backpropagation; and the output bound quantifies how the parameters control the activation's dynamic range.
        All proofs are in \cref{sec:properties_full}.

        \begin{proposition}[Smoothness]
            \label{prop:smooth}
            For $\omega_n>0$ and $\xi>0$, the activation function $\sigma(z;\theta)$ is $C^\infty$ with respect to both $z$ and all components of~$\theta$.
        \end{proposition}
        \begin{remark}
            \label{rem:singularity}
            At $\xi = 0$ the denominator of $A(\omega, \omega_n, \xi)$ vanishes when $\omega = \omega_n$ (undamped resonance).
            In practice, we constrain $\xi \in (0, 1)$ via sigmoid reparameterisation, which is physically motivated as all real oscillatory systems exhibit non-zero damping.
        \end{remark}
        \begin{proposition}[Lipschitz Bound]
            \label{prop:lipschitz}
          For fixed $\theta$, the activation is Lipschitz in $z$ with a constant $L_z = |\omega|\,A(\omega,\omega_n,\xi)$.
        \end{proposition}
        \begin{proposition}[Output Bound]
            \label{prop:bound}
            $|\sigma(z;\theta)| \leq A(\omega,\omega_n,\xi)$ uniformly in $z$ for all $ z \in \mathbb{R}$.
            The peak amplitude $A_{\textrm{peak}} = \bigl( 2 \xi \sqrt{1 - \xi^2}\bigr)^{-1}$ is attained at the resonant frequency $\omega_r = \omega_n \sqrt{1 - 2 \xi ^2}$ for $\xi < 1/\sqrt{2}$. For $\xi \ge 1/\sqrt{2}$ the amplitude decreases monotonically from $A(0) = 1$.
        \end{proposition}
        \Cref{prop:lipschitz} reveals a key structural property: representing high-frequency content requires \emph{both} a large forcing frequency $\omega$ \emph{and} a non-negligible amplitude $A$ at that frequency, since the Lipschitz constant is their product.
        This multiplicative coupling is the source of the implicit regularisation that we formalise next.

    \subsection{Spectral Gating}
        \label{sec:gradient_analysis}
        The properties above confirm that our activation is well-behaved for optimisation; however, they do not yet explain why it should favour signal over noise.
        We now turn to the core theoretical question: \emph{how do the gradients of the reconstruction loss with respect to the parameters behave during training?}
        By decomposing the gradient with respect to the damping factor $\xi$, following the principle of analysing implicit biases of gradient descent~\cite{arora2019implicit, li2022lowrank}, we obtain a particularly interpretable form of spectral regularisation, requiring no explicit penalty in the loss.
        The gradient naturally separates into: 
        (i) a term that pushes $\xi$ upward, narrowing the bandwidth unconditionally; 
        (ii) a term that pushes $\xi$ downward, widening the bandwidth, but only when the target contains coherent energy at the carrier frequency.
        We refer to these competing dynamics as \emph{closing} and \emph{opening} the spectral gate, respectively.
        Theorems and corollaries are proven in \cref{sec:gradient_analysis_full}.
        \begin{theorem}[Amplitude is Monotone in $\xi$]
            \label{thm:monotonicity_xi}
            For $\omega; \omega_n; \xi > 0$ then $\frac{\partial A}{\partial \xi} < 0$.
        \end{theorem}
        \Cref{thm:monotonicity_xi} establishes that increasing  $\xi$ always decreases the amplitude.
        The only route to amplifying the activation's output, thereby increasing the network's representational capacity at the carrier frequency, is to decrease $\xi$.
        Since gradient descent updates $\xi \leftarrow \xi-\eta\frac{\partial \mathcal{L}}{\partial \xi}$, the spectral gate opens only when $\frac{\partial \mathcal{L}}{\partial \xi} > 0$.
        To characterise when this condition holds, let us consider the mean squared error loss $\mathcal{L} = \frac{1}{N}\sum_{i=1}^N\bigl(\sigma(z_i;\theta)-y_i\bigr)^2$ for a single neuron with target values $\{y_i\}$.
        \begin{theorem}[Decomposition of the Damping Gradient]
            \label{thm:decomp}
            The gradient of $\mathcal{L}$ with respect to $\xi$ decomposes as:
            \begin{equation}
            \label{eq:decomp}
                \frac{\partial \mathcal{L}}{\partial \xi} = \underbrace{\frac{\partial A}{\partial \xi} \frac{2A}{N} \sum_{i} \sin^2(\omega z_i + \varphi)}_{T_{\mathrm{self}}} + \underbrace{\biggl(-\frac{\partial A}{\partial \xi}\biggr) \frac{2}{N} \sum_{i} y_i \sin(\omega z_i + \varphi)}_{T_{\mathrm{signal}}}.
            \end{equation}
        \end{theorem}
        The self-interaction term satisfies $T_{\mathrm{self}} = \frac{\partial A}{\partial \xi} \cdot \frac{2A}{N} \sum_i \sin^2 (\omega z_i + \varphi)$.
        For dense, uniform sampling\footnote{We analyse the population gradient obtained in the limit $N \to \infty$, where empirical inner products converge to their expectations.} of $z$, the sum converges to $\frac{1}{2}$, yielding $T_{\mathrm{self}} \approx \frac{\partial A}{\partial \xi} A < 0$ by \cref{thm:monotonicity_xi}.
        This term always pushes $\xi$ to increase, acting as a regularising force that tends to close the spectral gate.
        The signal term satisfies $T_{\mathrm{signal}} = -2\frac{\partial A}{\partial \xi} \langle y,\sin(\omega \cdot +\varphi)\rangle_N$, where $\langle \cdot,\cdot \rangle_N$ denotes the empirical inner product.
        Since $-\frac{\partial A}{\partial \xi} > 0$ by \cref{thm:monotonicity_xi}, we have $\sign(T_{\mathrm{signal}}) = \sign\bigl(\langle y, \sin(\omega \cdot +\varphi)\rangle_N\bigr)$.
        That is, $T_{\mathrm{signal}}$ is positive precisely when the target has energy at the carrier frequency.
        \begin{corollary}[Spectral Gating Conditions]
            \label{cor:gating}
            The spectral gate opens ($\frac{\partial \mathcal{L}}{\partial \xi} > 0$) if and only if $T_{\mathrm{signal}} > |T_{\mathrm{self}}|$, which requires:
            \begin{equation}
                \label{eq:gating_condition}
                \langle y, \sin(\omega \cdot +\varphi)\rangle_N > \frac{A}{2}.
            \end{equation}
            Three cases arise:
            \begin{enumerate}
                \item[\emph{(a)}] \textbf{Coherent signal.} 
                    If the target contains a Fourier component $A_y \sin(\omega z + \psi)$ at the carrier frequency $\omega$, then $\langle y, \sin(\omega \cdot +\varphi) \rangle_N \approx \frac{A_y}{2} \cos(\varphi - \psi)$.
                    The gate opens when $A_y \cos(\varphi-\psi) > A$, \ie, when the neuron is underfitting the corresponding frequency component.
                    As $A$ increases toward $A_y$, the inequality saturates, and $\xi$ reaches equilibrium;
                \item[\emph{(b)}] \textbf{Random noise.} 
                    If $y$ is i.i.d.\ with zero mean and variance $\sigma_y^2$, then $\langle y, \sin(\omega \cdot +\varphi)\rangle_N = O_p(N^{-1/2})$ by the central limit theorem, while $A/2 = O(1)$.
                    For any reasonably sized sample, $|T_{\mathrm{self}}|$ dominates, and the gate remains closed;
                \item[\emph{(c)}] \textbf{Mismatched frequency.} 
                    If $y = A' \sin(\omega' z)$ with $\omega' \neq \omega$, the inner product $\langle y, \sin(\omega \cdot +\varphi)\rangle_N \to 0$ as $N \to \infty$ by the orthogonality of sinusoids at distinct frequencies.
                    The self-term dominates, correctly signalling that this neuron should not fit the mismatched component.
            \end{enumerate}
        \end{corollary}
        \cref{cor:gating}, therefore, establishes that the optimisation problem is inherently biased to fit the coherent components of the signal and reject random noise.
        \begin{remark}[Multi-neuron layers]
            \label{rem:multineuron}
            In a full layer with weight vectors $\{\mathbf{w}_j\}$ and biases $\{b_j\}$, each neuron receives $z_j = \mathbf{w}_j^\top \mathbf{x} + b_j$, creating different effective frequencies $\omega \|\mathbf{w}_j\|$ along different directions.
            The oscillator parameters $(\omega, \omega_n, \xi, \varphi)$, shared across the layer, control the envelope of representable frequencies, while the weight distribution determines how individual neurons tile the frequency space within that envelope.
        \end{remark}
    
    \subsection{Multi-Layer Frequency Composition}
        \label{sec:multilayer_short}

        Cascading single-neuron oscillator layers produce a frequency-modulated signal~\cite{carson1922notes} whose harmonic content is governed by the Jacobi-Anger expansion~\cite{arfken2012mathematical,watson1944treatise}.
        The effective modulation index at layer $\ell$ is $\beta_\ell = \omega_\ell w_\ell \prod_{k=1}^{\ell-1} A_k$, which controls the number of non-negligible harmonics, which are approximately $\lfloor \beta \rfloor + 1$, since $|J_n(\beta)| \ll 1$ for $|n| > \beta + O(\beta^{1/3})$~\cite{abramowitz1964handbook,olver2010nist}.
        Since each $A_k$ decreases with $\xi_k$ (\cref{thm:monotonicity_xi}), the damping factors collectively gate the network's harmonic bandwidth: larger $\xi_k$ narrows the spectrum, while smaller $\xi_k$ widens it.
        The product structure ensures geometric decay of $\beta_\ell$ at initialisation ($\beta_\ell \propto A_0^{\ell-1}$ with $A_0 < 1$), so that the deepest layers initially contribute near-zero frequency support and bandwidth expands progressively as the individual $A_k$ grow during training, producing a \emph{coarse-to-fine learning curriculum} without explicit scheduling.
        This extends the observation by~\cite{yuce2022structured} that composing sinusoidal layers in SIREN produces Bessel-governed harmonics. 
        The key difference is that our oscillator amplitude $A_k$ makes the modulation index learnable.
        Full formulation and derivations are provided in \cref{sec:multilayer} of the Supplemental Material.

\section{Experiments}
    \label{sec:experiments}

    We evaluate our method across several tasks organised by observation model: signal representation (\cref{sec:exp_representation}), recovery from corrupted observations (\cref{sec:exp_corrupted}), and recovery from incomplete observations (\cref{sec:exp_incomplete}).
    More tasks are reported in \cref{sec:additonal_experiments} of the Supplemental Materials.
    We compare against eight established INRs: SIREN~\cite{sitzmann2019siren}, Gauss~\cite{ramasinghe2022beyond}, WIRE~\cite{saragadam2023wire}, FINER~\cite{liu2024finer}, BACON~\cite{lindell2021bacon}, MFN~\cite{fathony2021multiplicative}, FR~\cite{shi2024improved}, and Fourier Features~\cite{tancik2020fourfeat}.
    For each competitor, we adopt the task-specific hyperparameters and architecture configuration recommended by the original authors when available. 
    When no configuration is provided for a given task, we manually tune the architecture and hyperparameters to obtain competitive results. 
    In contrast, \ul{our method uses the same architecture and hyperparameters across all tasks}, with no task-specific tuning, directly verifying the claim that the spectral gate adapts to the signal's complexity without manual intervention.
    All these details are in the Supplemental Materials.
    
    To ensure a fair and reproducible comparison, we have re-implemented all methods and tasks within a unified codebase under identical training conditions.
    We will publicly release this library, dubbed \texttt{INRs Playground}, as part of our contributions.
    We run each experiment 10 times and average the results. 
    For image-related tasks, we provide qualitative results of their best run in \cref{fig:qualitatives_transposed}.
    
    \paragraph{Initialisation}
        We initialise $\omega_0$ and $\omega_{n,0}$ such that the forcing frequency lies below the natural frequency (\eg, $\omega_0=40, \omega_{n,0}=45$), placing the system away from resonance with a suppressed initial amplitude. 
        Both frequencies are reparameterised via softplus to ensure positivity.
        We initialise the damping factor to $\xi_0=1/\sqrt{2}$, corresponding to a Butterworth (maximally-flat) magnitude response~\cite{butterworth1930filter}, which provides moderate amplitude without a resonance peak.
        We initialise $\varphi$ to the physical phase lag of the oscillator at the initial parameters: $\varphi_0 = -\arctan \bigl(\frac{2 \xi_0 \omega_{n,0} \omega_0}{\omega_{n,0}^2 - \omega_0^2}\bigr)$.
        This ensures physical consistency at the start of training; the phase is then free to evolve.
        Finally, we follow the SIREN initialisation scheme~\cite{sitzmann2019siren} but scaled by the initial natural frequency: the first layer composed of $m$ neurons draws $w \sim \mathcal{U}(-1/m, 1/m)$ and hidden layers draw $w \sim \mathcal{U} \bigl(-\sqrt{6/m}/\omega_n, \sqrt{6/m}/\omega_n\bigr)$, which maintains approximate unit variance of pre-activations across layers.

    \subsection{Signal Representation}
        \label{sec:exp_representation}
        
        We first evaluate the activation's ability to faithfully encode continuous signals under direct, complete supervision, \ie, fitting a function $y = f(\mathbf{x})$ from dense samples.
        This setting isolates expressivity: every sample point is observed without noise or missing data, so performance differences between activations reflect their capacity to represent diverse spectral content.
        In terms of our theoretical framework, these experiments validate the properties established in \cref{sec:properties}.
        The network must have sufficient spectral bandwidth to match the target and provide baseline conditions under which the spectral gate should open during training.

        \paragraph{1D Signal Fitting.}
            We fit a one-dimensional composite signal containing multiple frequency components.
            This controlled setting allows direct inspection of how the learned \algoname{} parameters align with the target's spectral content, providing the most transparent validation of the spectral gating mechanism described in \cref{sec:gradient_analysis}.
            Results are reported in \cref{tab:signal_fitting}.
            \begin{table}[t]
    \centering
    \caption{
        \textbf{Signal Fitting.} 
        Best results in \textbf{bold}, runner-ups \underline{underlined}.
        }
    \resizebox{\linewidth}{!}{
        \begin{tabular}{l cccccc cccccc}

            \toprule
        
            \multirow{3}{*}{\textbf{{INR}}} & 
            \multicolumn{6}{c}{\textbf{Square Wave}} & 
            \multicolumn{6}{c}{\textbf{Chirp}} \\

            \cmidrule(lr){2-7} \cmidrule(lr){8-13}
            
            & \multicolumn{2}{c}{100 Hz} & \multicolumn{2}{c}{250 Hz} & \multicolumn{2}{c}{500 Hz} 
            & \multicolumn{2}{c}{250 Hz} & \multicolumn{2}{c}{500 Hz} & \multicolumn{2}{c}{1000 Hz} \\

            \cmidrule(lr){2-3} \cmidrule(lr){4-5} \cmidrule(lr){6-7} \cmidrule(lr){8-9} \cmidrule(lr){10-11} \cmidrule(lr){12-13}
            
            & Final PSNR & Peak PSNR & Final PSNR & Peak PSNR & Final PSNR & Peak PSNR
            & Final PSNR & Peak PSNR & Final PSNR & Peak PSNR & Final PSNR & Peak PSNR \\

            \algoname{}                          & \textbf{150.76 $\pm$ 1.51} & \textbf{160.40 $\pm$ 7.32} & \textbf{135.95 $\pm$ 1.39} & \textbf{142.70 $\pm$ 1.22} & \textbf{126.83 $\pm$ 1.17} & \textbf{135.38 $\pm$ 0.69} & \textbf{117.69 $\pm$ 15.63} & 120.97 $\pm$ 17.13 & \textbf{122.38 $\pm$ 16.85} & \textbf{124.98 $\pm$ 18.65} & \textbf{86.39 $\pm$ 4.10} & 94.15 $\pm$ 2.60 \\
            SIREN \cite{sitzmann2019siren}       & 116.68 $\pm$ 1.46 & 148.39 $\pm$ 2.32 & 74.98 $\pm$ 3.06 & \underline{136.40 $\pm$ 0.33} & \underline{64.49 $\pm$ 39.68} & \underline{125.86 $\pm$ 0.80} & 70.43 $\pm$ 9.88 & 126.55 $\pm$ 2.20 & 57.65 $\pm$ 34.99 & 120.78 $\pm$ 3.15 & 61.05 $\pm$ 6.62 & \underline{105.05 $\pm$ 3.27} \\
            Gauss \cite{ramasinghe2022beyond}    & 123.83 $\pm$ 5.02 & 132.48 $\pm$ 0.16 & 91.60 $\pm$ 42.04 & 125.54 $\pm$ 0.07 & 26.63 $\pm$ 8.40 & 42.08 $\pm$ 4.32 & 104.32 $\pm$ 16.93 & 113.31 $\pm$ 15.97 & 49.10 $\pm$ 14.06 & 66.77 $\pm$ 4.61 & 16.70 $\pm$ 1.65 & 17.96 $\pm$ 1.61 \\ 
            WIRE \cite{saragadam2023wire}        & 95.64 $\pm$ 27.03 & 140.20 $\pm$ 1.98 & 54.83 $\pm$ 8.10 & 88.38 $\pm$ 14.79 & 34.41 $\pm$ 6.80 & 45.35 $\pm$ 3.07 & 96.81 $\pm$ 21.76 & 132.48 $\pm$ 3.72 & 51.09 $\pm$ 10.94 & 74.48 $\pm$ 5.82 & 30.92 $\pm$ 7.67 & 34.94 $\pm$ 6.57 \\
            BACON \cite{lindell2021bacon}        & 29.03 $\pm$ 0.00 & 29.03 $\pm$ 0.00 & 33.01 $\pm$ 0.00 & 33.01 $\pm$ 0.00 & 36.02 $\pm$ 0.00 & 36.02 $\pm$ 0.00 & 75.03 $\pm$ 2.71 & 107.59 $\pm$ 1.87 & \underline{79.23 $\pm$ 0.73} & 109.08 $\pm$ 3.24 & 73.28 $\pm$ 0.96 & 78.69 $\pm$ 2.56 \\
            FINER \cite{liu2024finer}            & 107.64 $\pm$ 25.40 & 107.64 $\pm$ 25.40 & \underline{101.85 $\pm$ 10.61} & 132.31 $\pm$ 0.81 & 53.70 $\pm$ 33.03 & 119.07 $\pm$ 0.07 & 82.68 $\pm$ 22.96 & \underline{136.74 $\pm$ 1.10} & 65.76 $\pm$ 6.98 & \underline{123.49 $\pm$ 0.63} & \underline{80.50 $\pm$ 30.38} & \textbf{112.57 $\pm$ 1.36} \\
            MFN \cite{fathony2021multiplicative} & 94.72 $\pm$ 13.13 & 126.92 $\pm$ 8.56 & 52.23 $\pm$ 4.76 & 63.77 $\pm$ 2.22 & 6.03 $\pm$ 0.00 & 6.04 $\pm$ 0.00 & 56.23 $\pm$ 4.40 & 81.54 $\pm$ 2.57 & 14.60 $\pm$ 0.21 & 14.62 $\pm$ 0.22 & 10.90 $\pm$ 0.14 & 10.91 $\pm$ 0.12 \\
            Fourier \cite{tancik2020fourfeat}    & \underline{150.33 $\pm$ 6.87} & \underline{155.47 $\pm$ 2.46} & 54.75 $\pm$ 8.71 & 125.31 $\pm$ 7.14 & 24.75 $\pm$ 13.57 & 34.00 $\pm$ 2.47 & \underline{104.82 $\pm$ 22.65} & \textbf{137.95 $\pm$ 1.28} & 43.50 $\pm$ 9.43 & 60.16 $\pm$ 4.76 & 19.64 $\pm$ 1.53 & 22.31 $\pm$ 1.62 \\
            FR \cite{shi2024improved}            & 6.02 $\pm$ 0.00 & 6.03 $\pm$ 0.00 & 6.02 $\pm$ 0.00 & 6.02 $\pm$ 0.00 & 6.02 $\pm$ 0.00 & 6.02 $\pm$ 0.00 & 9.40 $\pm$ 0.36 & 10.48 $\pm$ 0.20 & 9.47 $\pm$ 0.25 & 9.75 $\pm$ 0.11 & 9.24 $\pm$ 0.16 & 9.36 $\pm$ 0.12 \\
            
            \bottomrule
            
        \end{tabular}
        }
    \label{tab:signal_fitting}
\end{table}

            We evaluate on two signal types: a \textbf{Square Wave} (rich in odd harmonics) and a \textbf{Chirp} (continuously sweeping frequency), at increasing fundamental frequencies, and report both final and peak PSNR over training to distinguish converged reconstruction quality from the best performance achieved at any point during optimisation.
            
            The \algoname{} activation achieves the highest final PSNR across all six settings, demonstrating consistently superior converged reconstruction quality.
            It also attains the highest peak PSNR in four of six conditions, with the two exceptions being the \textbf{Chirp} at $250$ Hz (where Fourier Features reaches $137.95$ dB vs. our $120.97$ dB) and at $1000$ Hz (where FINER reaches $112.57$ dB vs. our $94.15$ dB).
            The fact that \algoname{} leads on final PSNR even in these two settings suggests that the spectral gating mechanism provides more stable long-term convergence: once the gate opens for a coherent frequency component, the equilibrium predicted by \cref{cor:gating}(a) prevents further drift, whereas methods without this mechanism may overshoot during training and settle at a lower endpoint.
            The most revealing pattern emerges at higher frequencies.
            Non-periodic and spatially compact activations, such as Gauss, WIRE, and MFN, degrade severely as frequency increases: on the $500$ Hz \textbf{Square Wave}, Gauss drops to $42$ dB, and MFN collapses to $6$ dB, consistent with their documented low-frequency bias.
            SIREN shows a characteristic pattern: modest final PSNR but competitive peak PSNR, reflecting its all-pass nature: it can match the signal's frequency content but lacks the spectral gate to stabilise convergence, resulting in a gap between peak and final performance.
            \algoname{} combines the strengths of both regimes: strong converged quality (highest final PSNR) with stable training dynamics (smallest gap between peak and final PSNR), consistent with the spectral gating mechanism and the coarse-to-fine curriculum of \cref{sec:multilayer_short}, preventing subsequent drift.
            The variance across runs also deserves attention.
            Several INRs show high instability at elevated frequencies: SIREN at $500$ Hz \textbf{Square Wave} ($\pm39.68$ dB), FINER at $1000$ Hz \textbf{Chirp} ($\pm30.38$ dB), and Gauss at $250$ Hz \textbf{Square Wave} ($\pm42.04$ dB), whereas \algoname{} exhibits a contained variance across all conditions, suggesting that the spectral gating mechanism provides a more robust optimisation landscape.
            
        \paragraph{Audio Fitting.}
            We encode raw audio waveforms as coordinate-based networks mapping time to amplitude, following the protocol of~\cite{sitzmann2019siren}.
            Audio signals are characterised by rich harmonic structure spanning several orders of magnitude in frequency, making them a demanding test of spectral expressivity.
            Results are reported in \cref{tab:audio_fitting}.
            \begin{table}[t]
    \centering
    \caption{
        \textbf{Audio Fitting.} Best results in \textbf{bold}, runner-ups \underline{underlined}.
        }
    \resizebox{\linewidth}{!}{
        \begin{tabular}{cc ccccccccc}

            \toprule
        
            {\textbf{Signal}} &
            {\textbf{Metric}} &
            \multicolumn{1}{c}{\textbf{\algoname{}}} &
            \multicolumn{1}{c}{\textbf{SIREN}} &
            \multicolumn{1}{c}{\textbf{Gauss}} &
            \multicolumn{1}{c}{\textbf{WIRE}} &
            \multicolumn{1}{c}{\textbf{BACON}} &
            \multicolumn{1}{c}{\textbf{FINER}} &
            \multicolumn{1}{c}{\textbf{MFN}} &
            \multicolumn{1}{c}{\textbf{Fourier}} &
            \multicolumn{1}{c}{\textbf{FR}} \\

            \midrule

            \multirow{2}{*}{\textbf{Bach}}
            & Final PSNR & \textbf{56.23 $\pm$ 0.21} & \underline{52.02 $\pm$ 2.14} & 19.72 $\pm$ 0.10 & 19.80 $\pm$ 0.06 & 19.50 $\pm$ 0.00 & 38.69 $\pm$ 1.64 & 7.38 $\pm$ 0.25 & 50.75 $\pm$ 1.14 & 17.98 $\pm$ 1.20 \\
            & Peak PSNR & \underline{56.32 $\pm$ 0.19} & \textbf{58.47 $\pm$ 0.16} & 20.25 $\pm$ 0.15 & 20.80 $\pm$ 0.17 & 19.56 $\pm$ 0.01 & 41.97 $\pm$ 0.19 & 7.38 $\pm$ 0.25 & 52.36 $\pm$ 0.64 & 20.63 $\pm$ 0.05 \\

            \midrule

            \multirow{2}{*}{\textbf{Counting}} 
            & Final PSNR & \textbf{43.17 $\pm$ 2.45} & \underline{41.64 $\pm$ 2.84} & 24.43 $\pm$ 0.01 & 24.83 $\pm$ 0.19 & 24.43 $\pm$ 0.00 & 34.02 $\pm$ 0.60 & 6.84 $\pm$ 0.25 & 36.20 $\pm$ 0.65 & 24.29 $\pm$ 0.18 \\
            & Peak PSNR & \textbf{45.83 $\pm$ 4.20} & \underline{45.22 $\pm$ 0.66} & 24.72 $\pm$ 0.03 & 25.11 $\pm$ 0.23 & 24.43 $\pm$ 0.00 & 36.20 $\pm$ 1.17 & 6.84 $\pm$ 0.25 & 36.55 $\pm$ 0.40 & 24.64 $\pm$ 0.04 \\

            \bottomrule
            
        \end{tabular}
        }
    \label{tab:audio_fitting}
\end{table}

            We evaluate on two clips: \textbf{Bach} (tonal, harmonically structured) and \textbf{Counting} (speech, broadband and transient-rich), again reporting both final and peak PSNR.
            
            \algoname{} achieves the highest final PSNR on both signals ($56.23$ dB on \textbf{Bach}, $43.17$ dB on \textbf{Counting}), confirming that the spectral gating mechanism yields consistently superior converged representations even on spectrally demanding targets.
            On peak PSNR, \algoname{} ranks first on \textbf{Counting} ($45.83$ dB vs. SIREN's $45.22$ dB) and second on \textbf{Bach} (56.32 dB vs. SIREN's $58.47$ dB).
            The peak gap on \textbf{Bach} is informative: SIREN's all-pass, purely periodic activation is well matched to the sustained harmonic structure of instrumental music, allowing it to transiently achieve high fidelity, but it does not maintain this performance at convergence (final PSNR drops to $52.02$ dB, over $4$ dB below our activation).
            This pattern is consistent with SIREN's lack of a stabilising mechanism: without the equilibrium condition of \cref{cor:gating}(a), the network continues to adjust frequencies after they have been adequately captured, leading to a lower endpoint despite a higher peak.
            The broader comparison reveals a stark divide.
            Non-periodic activations, such as Gauss, WIRE, BACON, and FR, plateau around $20$ to $25$ dB on both clips, and MFN fails entirely (below $8$ dB), demonstrating that spatially compact activations lack the spectral bandwidth required for audio-frequency content.
            FINER provides a meaningful improvement ($42$ dB peak on \textbf{Bach}, $36$ dB on \textbf{Counting}) but remains well below \algoname{} and SIREN, indicating that variable-periodic modulation alone is insufficient without a mechanism to coordinate spectral bandwidth across the full frequency range.
            Fourier Features performs competitively on \textbf{Bach} ($52.36$ dB peak) but drops on \textbf{Counting} ($36.55$ dB), suggesting that its fixed random frequency basis struggles with the non-stationary spectral content of speech.
            
        \paragraph{Image Fitting.}
            We fit natural images by learning a mapping from pixel coordinates to RGB values, following the standard protocol established by~\cite{sitzmann2019siren} and adopted across the INR literature~\cite{ramasinghe2022beyond,saragadam2023wire,liu2024finer,shi2024improved,jayasundara2024mire}.
            Natural images exhibit a broad, spatially varying frequency spectrum, ranging from smooth regions to fine textures and sharp edges, requiring the activation to accommodate a wide range of spectral content.
            Results on five diverse images are reported in \cref{tab:image_fitting}.
            \begin{table}[t]
    \centering
    \caption{
        \textbf{Image Fitting.} 
        Best results in \textbf{bold}, runner-ups \underline{underlined}.
        }
    \resizebox{\linewidth}{!}{
        \begin{tabular}{l cccccccccc}

            \toprule
            
            \multirow{3}{*}{\textbf{{INR}}} & 
            \multicolumn{2}{c}{\textbf{Tiger}} & 
            \multicolumn{2}{c}{\textbf{Tiles}} & 
            \multicolumn{2}{c}{\textbf{Bikers}} & 
            \multicolumn{2}{c}{\textbf{Butterfly}} & 
            \multicolumn{2}{c}{\textbf{Knot}} \\ %

            \cmidrule(lr){2-3} \cmidrule(lr){4-5} \cmidrule(lr){6-7} \cmidrule(lr){8-9} \cmidrule(lr){10-11} %
            
            & Final PSNR & Peak PSNR & Final PSNR & Peak PSNR & Final PSNR & Peak PSNR
            & Final PSNR & Peak PSNR & Final PSNR & Peak PSNR \\ %

            \cmidrule(lr){2-3} \cmidrule(lr){4-5} \cmidrule(lr){6-7} \cmidrule(lr){8-9} \cmidrule(lr){10-11} %
            
            \algoname{}                          & \textbf{63.79 $\pm$ 0.23} & \textbf{63.79 $\pm$ 0.23} & \underline{52.24 $\pm$ 5.82} & \underline{57.09 $\pm$ 1.31} & \textbf{57.06 $\pm$ 0.97} & \textbf{57.25 $\pm$ 0.78} & \textbf{57.85 $\pm$ 0.71} & \textbf{57.96 $\pm$ 0.61} & \textbf{59.21 $\pm$ 7.84} & \textbf{64.05 $\pm$ 3.68} \\
            SIREN \cite{sitzmann2019siren}       & \underline{53.69 $\pm$ 1.77} & \underline{58.52 $\pm$ 0.25} & 52.05 $\pm$ 1.76 & \textbf{58.36 $\pm$ 0.21} & 43.11 $\pm$ 11.9 & \underline{55.92 $\pm$ 0.21} & 48.81 $\pm$ 1.51 & \underline{55.83 $\pm$ 0.48} & 51.90 $\pm$ 0.53 & 59.72 $\pm$ 0.78 \\
            Gauss \cite{ramasinghe2022beyond}    & 44.10 $\pm$ 0.62 & 46.46 $\pm$ 0.21 & 42.34 $\pm$ 1.09 & 44.77 $\pm$ 0.19 & 41.45 $\pm$ 0.74 & 42.59 $\pm$ 0.08 & 43.19 $\pm$ 0.62 & 45.01 $\pm$ 0.44 & 43.81 $\pm$ 1.19 & 47.08 $\pm$ 0.24 \\
            WIRE \cite{saragadam2023wire}        & 38.64 $\pm$ 0.19 & 39.20 $\pm$ 0.06 & 40.92 $\pm$ 0.45 & 42.06 $\pm$ 0.21 & 36.71 $\pm$ 0.15 & 36.99 $\pm$ 0.26 & 36.78 $\pm$ 0.33 & 37.10 $\pm$ 0.26 & 40.32 $\pm$ 0.27 & 40.80 $\pm$ 0.11 \\
            BACON \cite{lindell2021bacon}        & 29.88 $\pm$ 0.03 & 29.90 $\pm$ 0.04 & 32.20 $\pm$ 0.04 & 32.30 $\pm$ 0.01 & 27.45 $\pm$ 0.05 & 27.45 $\pm$ 0.05 & 25.01 $\pm$ 0.10 & 25.01 $\pm$ 0.11 & 31.22 $\pm$ 0.03 & 31.25 $\pm$ 0.02 \\
            FINER \cite{liu2024finer}            & 56.58 $\pm$ 0.12 & 57.09 $\pm$ 0.14 & \textbf{54.09 $\pm$ 0.97} & 55.31 $\pm$ 0.27 & \underline{47.46 $\pm$ 3.13} & 53.53 $\pm$ 0.17 & \underline{53.65 $\pm$ 0.86} & 55.10 $\pm$ 0.08 & \underline{53.39 $\pm$ 1.88} & \underline{59.72 $\pm$ 0.20} \\
            MFN \cite{fathony2021multiplicative} & 46.97 $\pm$ 0.99 & 47.01 $\pm$ 0.97 & 41.92 $\pm$ 3.32 & 42.59 $\pm$ 2.60 & 38.11 $\pm$ 0.53 & 38.11 $\pm$ 0.53 & 44.12 $\pm$ 0.74 & 44.25 $\pm$ 0.72 & 48.71 $\pm$ 1.95 & 50.59 $\pm$ 1.02 \\
            Fourier \cite{tancik2020fourfeat}    & 51.79 $\pm$ 0.66 & 52.77 $\pm$ 0.26 & 46.29 $\pm$ 1.45 & 50.79 $\pm$ 0.16 & 45.35 $\pm$ 1.09 & 47.80 $\pm$ 0.12 & 45.30 $\pm$ 3.40 & 51.58 $\pm$ 0.40 & 51.49 $\pm$ 0.52 & 53.02 $\pm$ 0.56 \\
            FR \cite{shi2024improved}            & 23.01 $\pm$ 0.22 & 23.01 $\pm$ 0.22 & 19.79 $\pm$ 0.02 & 19.79 $\pm$ 0.02 & 19.89 $\pm$ 0.05 & 19.89 $\pm$ 0.05 & 18.68 $\pm$ 0.03 & 18.68 $\pm$ 0.03 & 17.35 $\pm$ 0.12 & 17.35 $\pm$ 0.12 \\
            
            \bottomrule
            
        \end{tabular}
        }
    \label{tab:image_fitting}
\end{table}

            \algoname{} achieves the highest final and peak PSNR on four of five images (\textbf{Tiger}, \textbf{Bikers}, \textbf{Butterfly}, \textbf{Knot}), with margins of up to $10$ dB over the closest competitor on final quality (\textbf{Bikers}: $57.06$ dB vs. FINER's $47.46$ dB).
            The sole exception is \textbf{Tiles}, where FINER reaches $54.09$ dB vs. \algoname{}'s $52.24$ dB on final PSNR, and SIREN reaches $58.36$ dB vs. \algoname{}'s $57.09$ dB on peak PSNR.
            \textbf{Tiles} is a highly regular, repetitive texture, precisely the setting where FINER's variable-periodic activations are best suited, as the spectral content is concentrated at a few dominant spatial frequencies with little need for adaptive bandwidth control.
            The most evident pattern is the gap between the final and peak PSNR for SIREN.
            On \textbf{Bikers}, SIREN peaks at $55.92$ dB but converges to only $43.11$ dB, a degradation of nearly $13$ dB; similar drops appear on \textbf{Tiger} ($-4.83$ dB), \textbf{Butterfly} ($-7.02$ dB), and \textbf{Tiles} ($-6.31$ dB).
            This is the overshoot-then-decay behaviour expected from an all-pass activation: SIREN reaches the target frequencies but, lacking a stabilising mechanism, continues to adjust its parameters and drifts away from optimal solutions.
            \algoname{} shows no such degradation: on \textbf{Tiger}, final and peak PSNR are identical ($63.79$ dB), and across all images the gap never exceeds $5$ dB (with the highest being on \textbf{Knot}), consistent with the equilibrium condition of \cref{cor:gating}(a) that stabilises $\xi$ once the neuron adequately represents its frequency component.
            Among the remaining INRs, FINER and SIREN are the strongest overall competitors.
            Gauss, WIRE, and MFN cluster between $37$ and $50$ dB, reflecting the low-frequency ceiling of non-periodic and spatially compact activations on images with significant high-frequency content.
            FR and BACON remain below $32$ dB across all images.
            The variance of \algoname{} is moderate across conditions, while several INRs exhibit notable instability (\ie, on \textbf{Bikers}: SIREN $\pm11.9$ dB vs. \algoname{} $\pm0.97$ dB).
            \begin{figure}[t]
    \centering
    \resizebox{\linewidth}{!}{
    \begin{tabular}{lcc @{\hspace{0.75em}} ccccccccc}        

        & {\tiny Supervision} 
        & {\tiny GT} 
        & {\tiny \algoname{}} 
        & {\tiny SIREN} 
        & {\tiny Gauss} 
        & {\tiny WIRE} 
        & {\tiny BACON} 
        & {\tiny FINER} 
        & {\tiny MFN} 
        & {\tiny Fourier} 
        & {\tiny FR} \\
        
        \rotatebox{90}{\tiny \hspace{0.025em} \emph{Fitting}} &
        \includegraphics[width=0.09\linewidth]{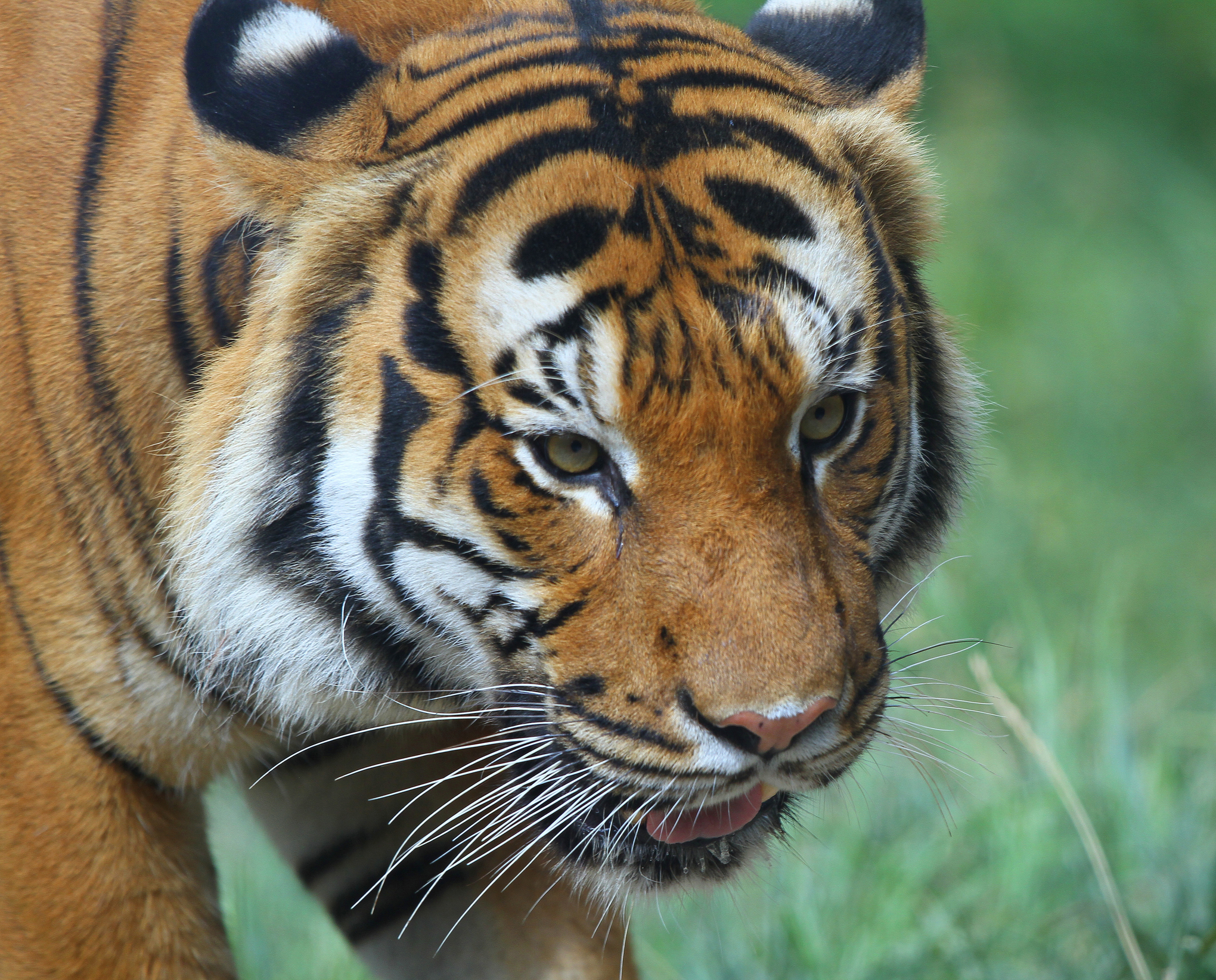} & 
        \includegraphics[width=0.09\linewidth]{figures/qualitatives/supervision/tiger_fitting.png} & 
        \includegraphics[width=0.09\linewidth]{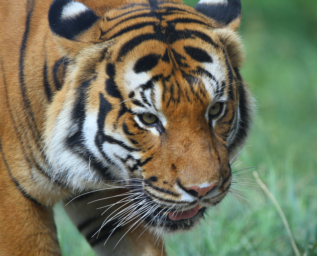} & 
        \includegraphics[width=0.09\linewidth]{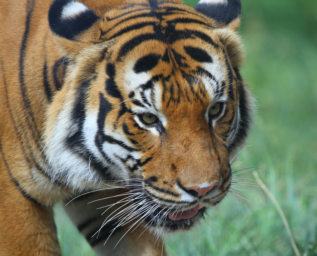} & 
        \includegraphics[width=0.09\linewidth]{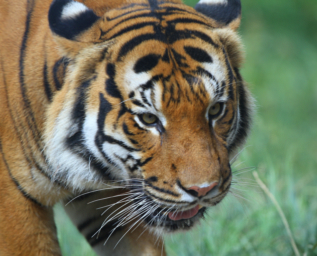} & 
        \includegraphics[width=0.09\linewidth]{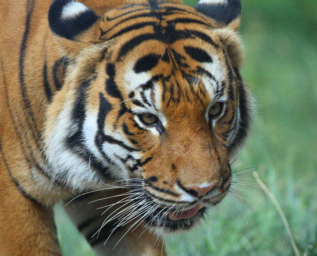} & 
        \includegraphics[width=0.09\linewidth]{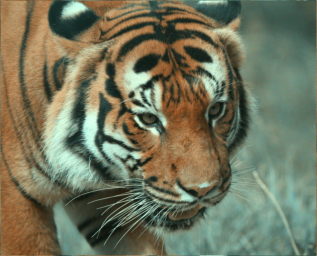} & 
        \includegraphics[width=0.09\linewidth]{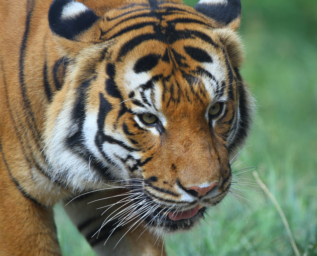} & 
        \includegraphics[width=0.09\linewidth]{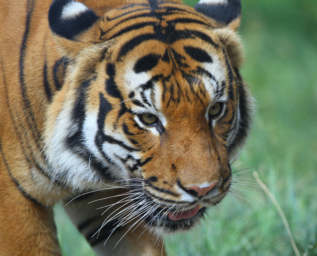} & 
        \includegraphics[width=0.09\linewidth]{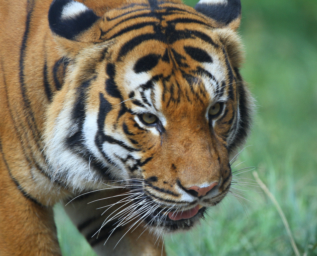} &
        \includegraphics[width=0.09\linewidth]{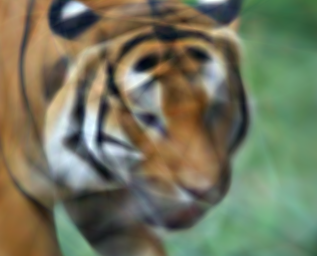} \\

        \rotatebox{90}{\tiny \hspace{0.25em} \emph{Denoising}} &
        \includegraphics[width=0.09\linewidth]{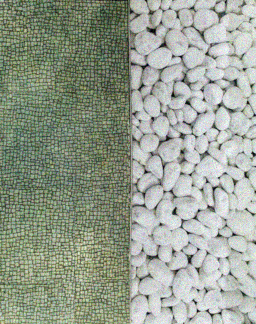} & 
        \includegraphics[width=0.09\linewidth]{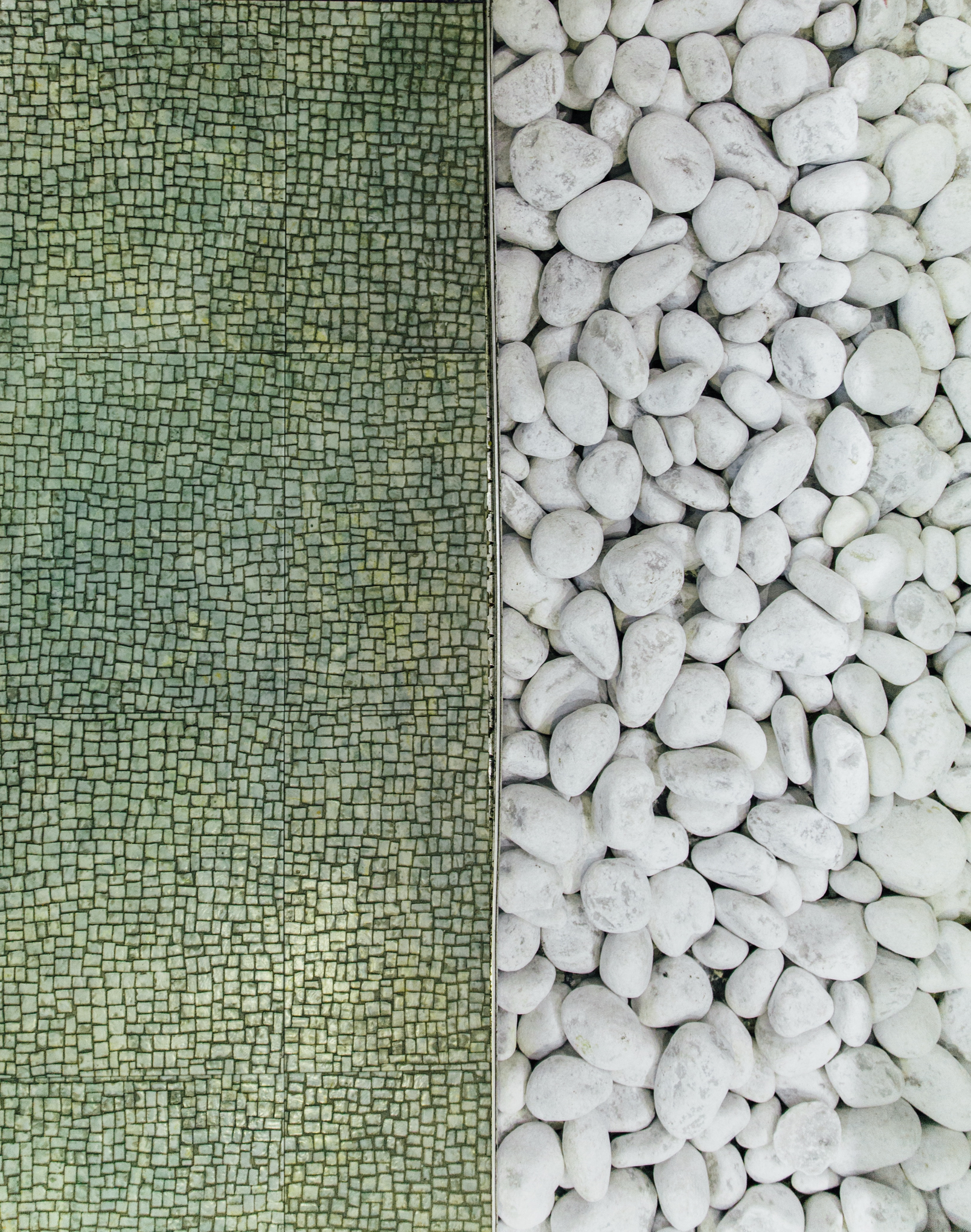} & 
        \includegraphics[width=0.09\linewidth]{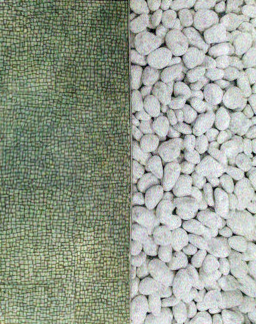} & 
        \includegraphics[width=0.09\linewidth]{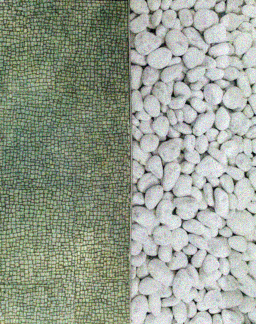} & 
        \includegraphics[width=0.09\linewidth]{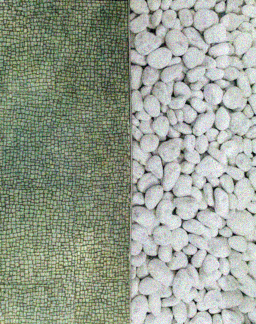} & 
        \includegraphics[width=0.09\linewidth]{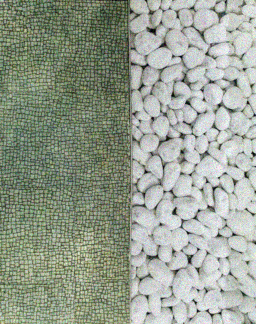} & 
        \includegraphics[width=0.09\linewidth]{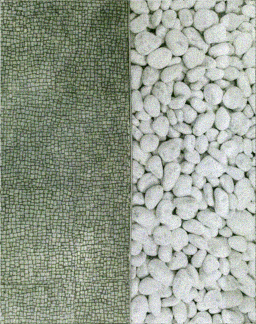} & 
        \includegraphics[width=0.09\linewidth]{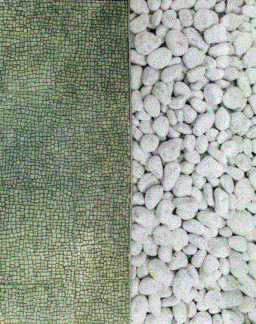} & 
        \includegraphics[width=0.09\linewidth]{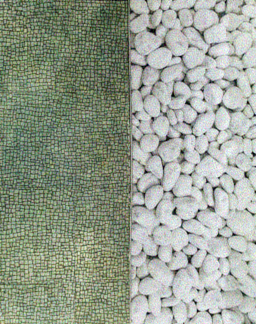} & 
        \includegraphics[width=0.09\linewidth]{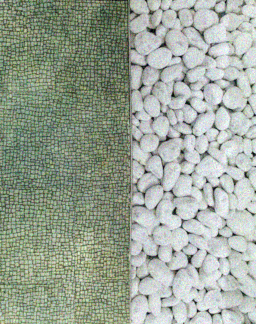} &
        \includegraphics[width=0.09\linewidth]{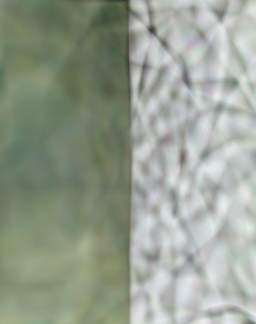} \\

        \rotatebox{90}{\tiny \hspace{0.5em} \emph{CT}} &
        \includegraphics[width=0.09\linewidth]{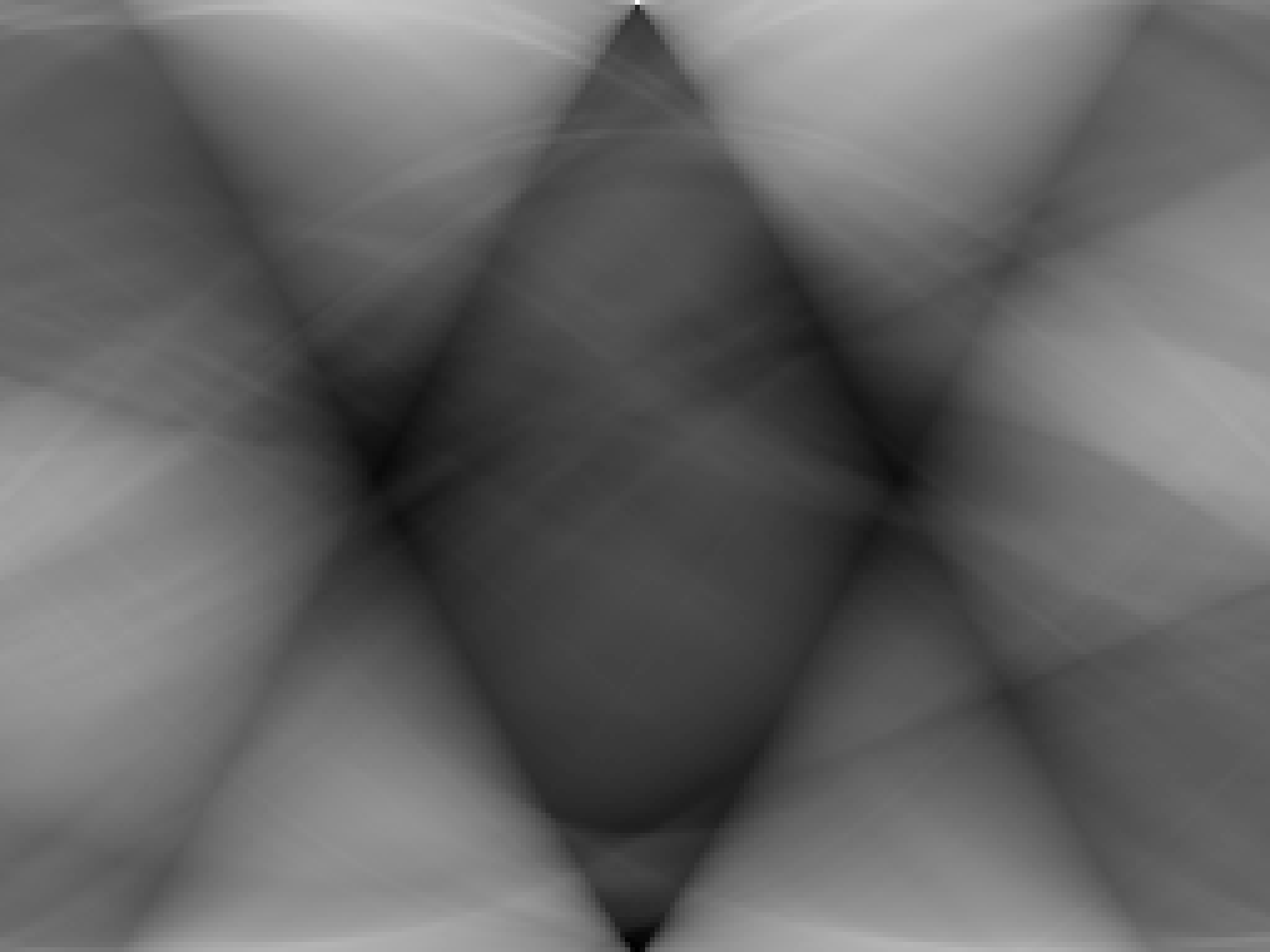} & 
        \includegraphics[width=0.09\linewidth]{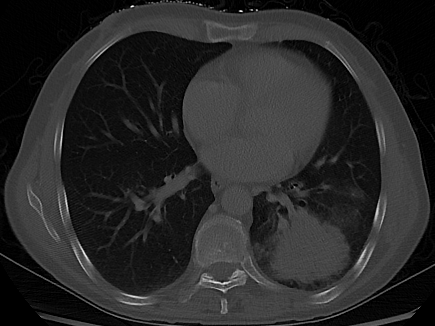} & 
        \includegraphics[width=0.09\linewidth]{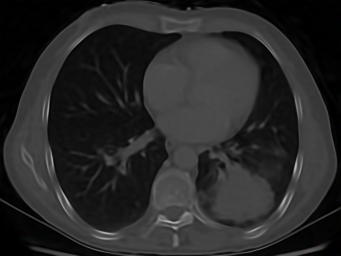} & 
        \includegraphics[width=0.09\linewidth]{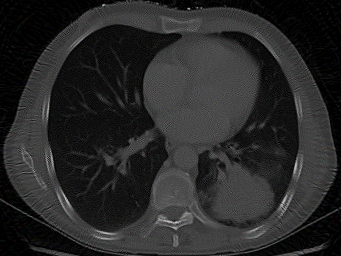} & 
        \includegraphics[width=0.09\linewidth]{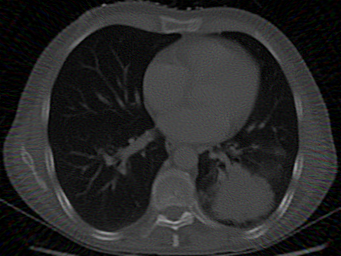} & 
        \includegraphics[width=0.09\linewidth]{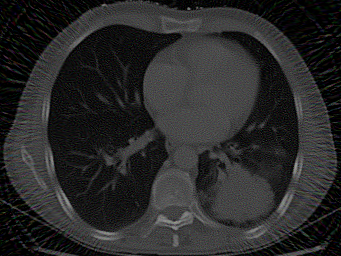} & 
        \includegraphics[width=0.09\linewidth]{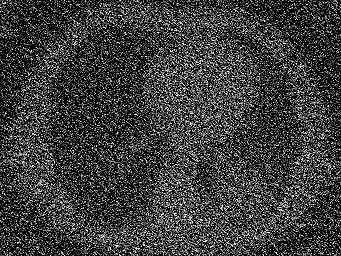} & 
        \includegraphics[width=0.09\linewidth]{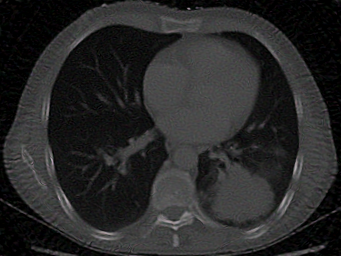} & 
        \includegraphics[width=0.09\linewidth]{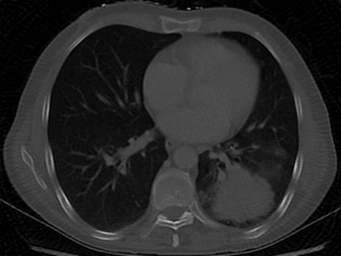} & 
        \includegraphics[width=0.09\linewidth]{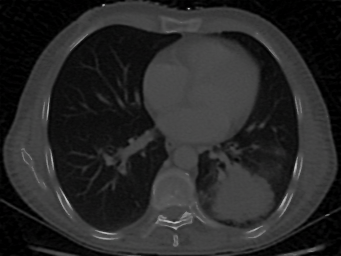} &
        \includegraphics[width=0.09\linewidth]{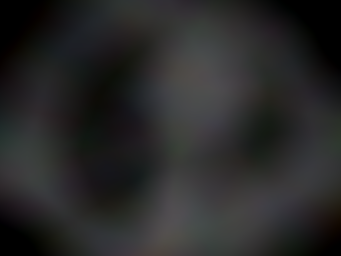} \\

        \rotatebox{90}{\tiny \hspace{0.4em} \emph{Inp.}} &
        \includegraphics[width=0.09\linewidth]{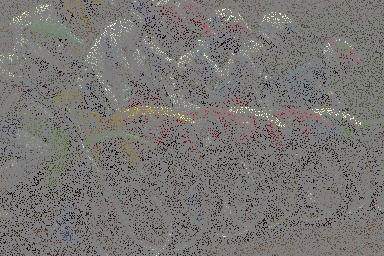} & 
        \includegraphics[width=0.09\linewidth]{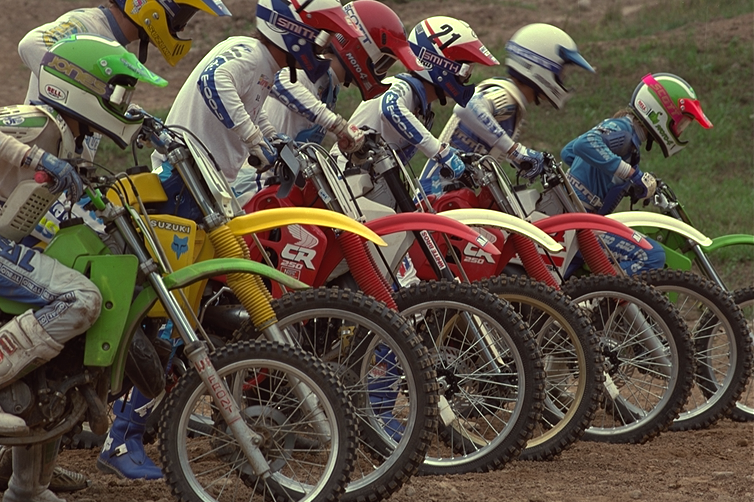} & 
        \includegraphics[width=0.09\linewidth]{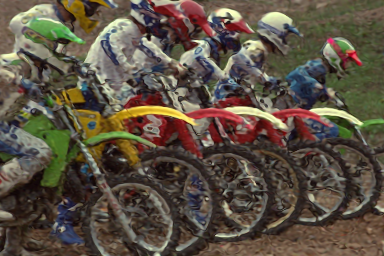} & 
        \includegraphics[width=0.09\linewidth]{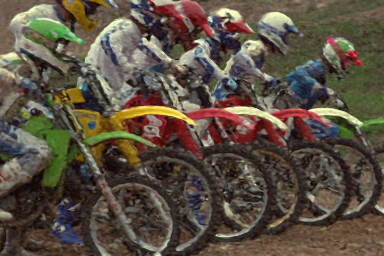} & 
        \includegraphics[width=0.09\linewidth]{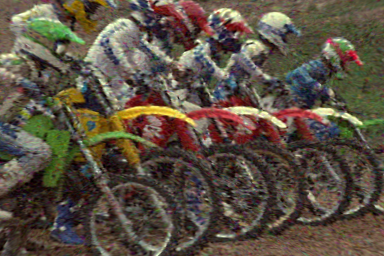} & 
        \includegraphics[width=0.09\linewidth]{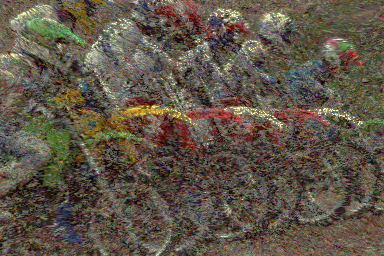} & 
        \includegraphics[width=0.09\linewidth]{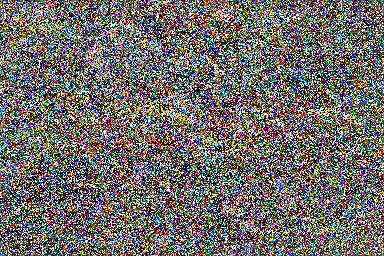} & 
        \includegraphics[width=0.09\linewidth]{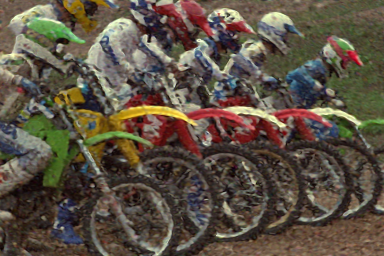} & 
        \includegraphics[width=0.09\linewidth]{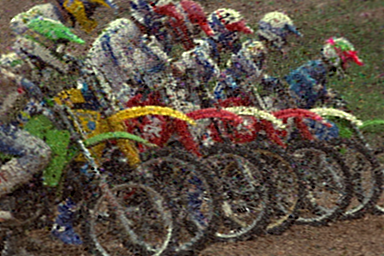} & 
        \includegraphics[width=0.09\linewidth]{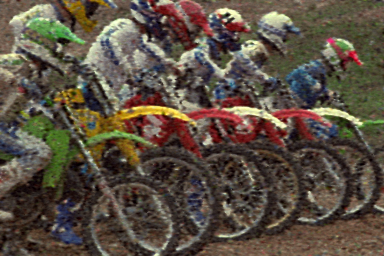} &
        \includegraphics[width=0.09\linewidth]{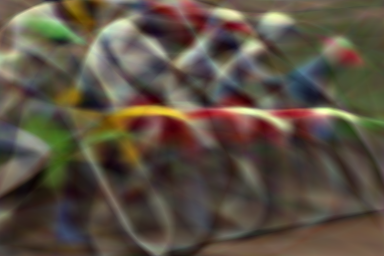} \\

        \rotatebox{90}{\tiny \emph{SR}} &
        \makebox[0.09\linewidth][c]{\includegraphics[width=0.05\linewidth, valign=m]{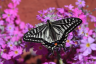}} &
        \includegraphics[width=0.09\linewidth, valign=m]{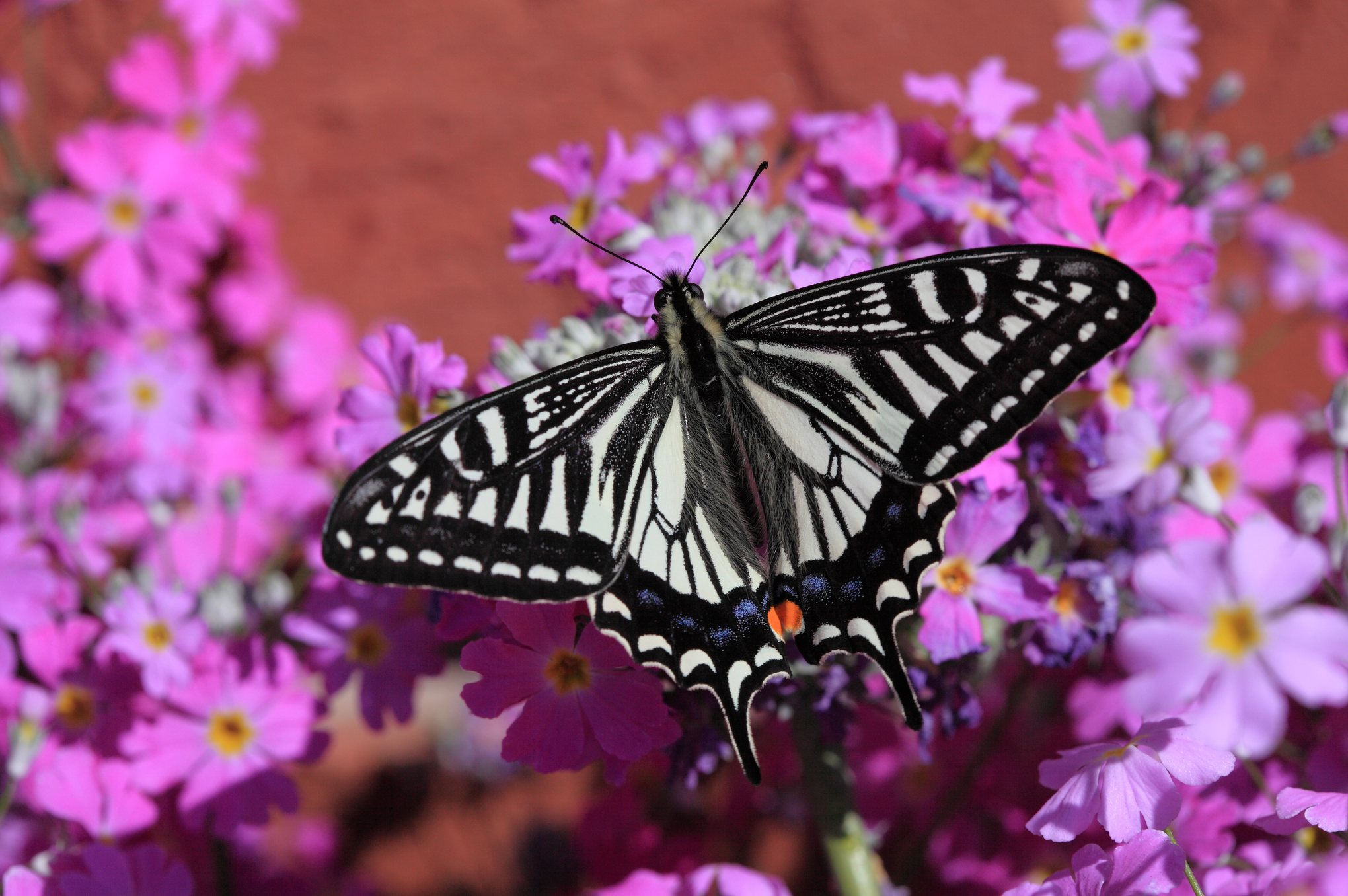} & 
        \includegraphics[width=0.09\linewidth, valign=m]{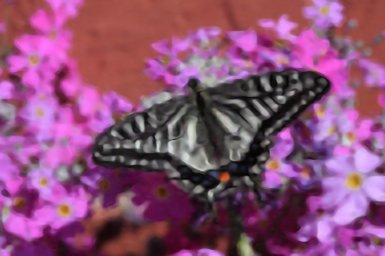} & 
        \includegraphics[width=0.09\linewidth, valign=m]{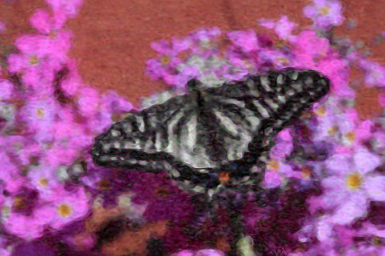} & 
        \includegraphics[width=0.09\linewidth, valign=m]{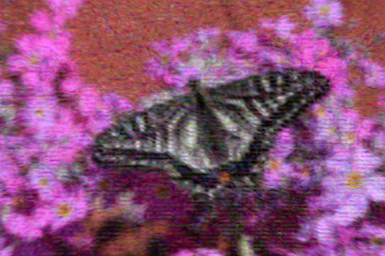} & 
        \includegraphics[width=0.09\linewidth, valign=m]{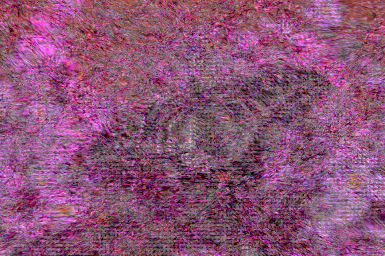} & 
        \includegraphics[width=0.09\linewidth, valign=m]{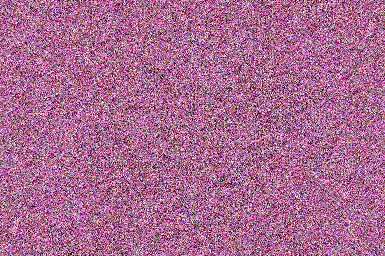} & 
        \includegraphics[width=0.09\linewidth, valign=m]{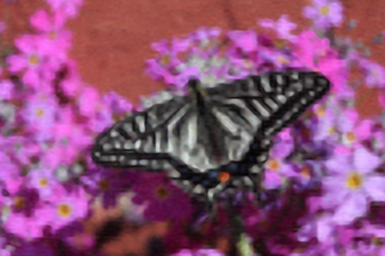} & 
        \includegraphics[width=0.09\linewidth, valign=m]{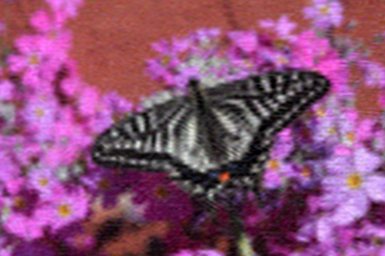} & 
        \includegraphics[width=0.09\linewidth, valign=m]{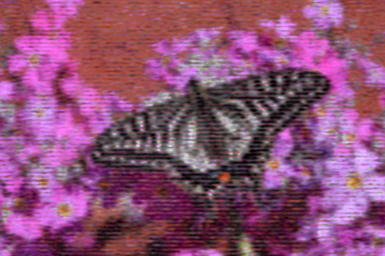} &
        \includegraphics[width=0.09\linewidth, valign=m]{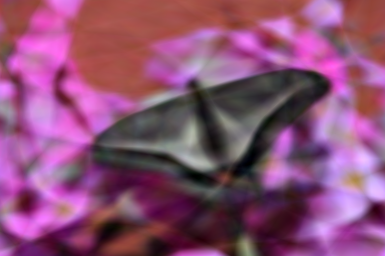} \\[1.75ex]
        
        \rotatebox{90}{\tiny \hspace{0.1em} \emph{Poisson}} &
        \includegraphics[width=0.09\linewidth]{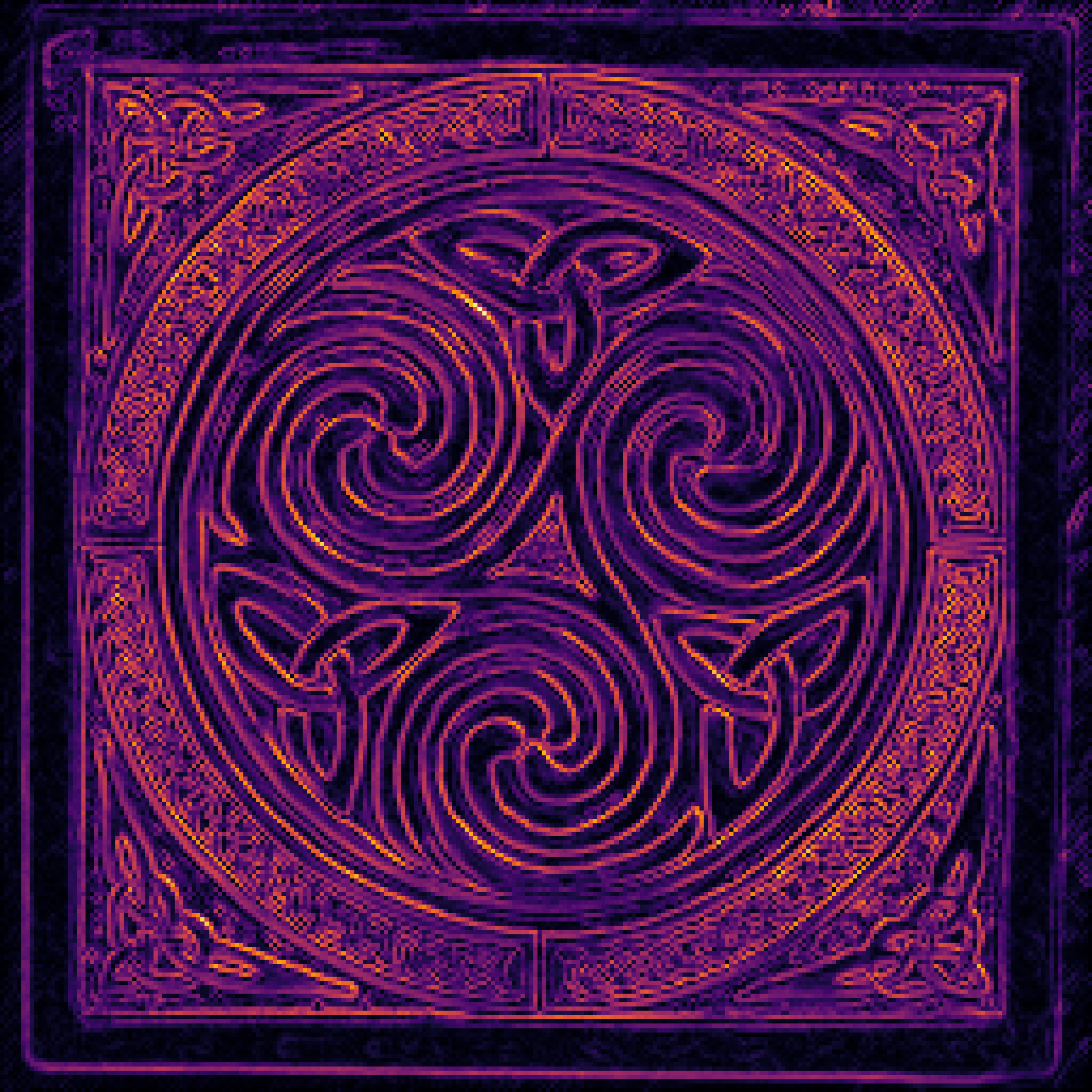} & 
        \includegraphics[width=0.09\linewidth]{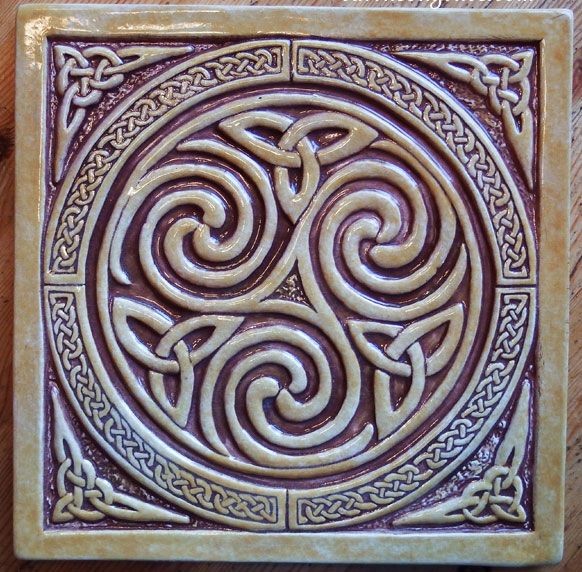} & 
        \includegraphics[width=0.09\linewidth]{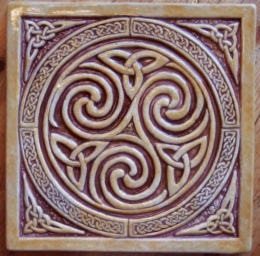} & 
        \includegraphics[width=0.09\linewidth]{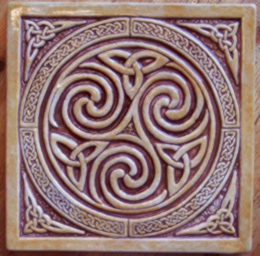} & 
        \includegraphics[width=0.09\linewidth]{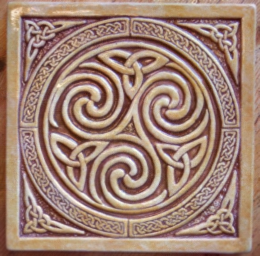} & 
        \includegraphics[width=0.09\linewidth]{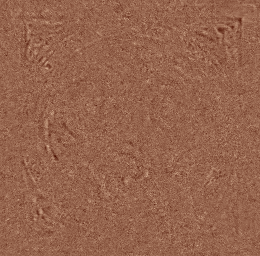} & 
        \includegraphics[width=0.09\linewidth]{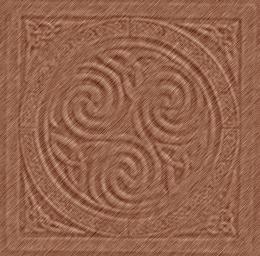} & 
        \includegraphics[width=0.09\linewidth]{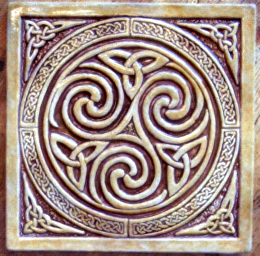} & 
        \includegraphics[width=0.09\linewidth]{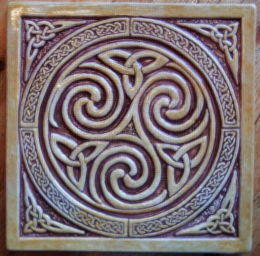} & 
        \includegraphics[width=0.09\linewidth]{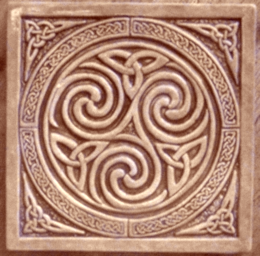} &
        \includegraphics[width=0.09\linewidth]{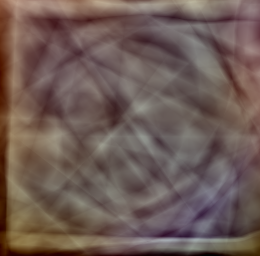} \\

    \end{tabular}
    }
    \caption{
    \textbf{Qualitative results.}
    Please zoom to appreciate the differences.
    }
    \label{fig:qualitatives_transposed}
\end{figure}

    \subsection{Signal Recovery from Corrupted Observations}
        \label{sec:exp_corrupted}
        
        We next evaluate settings where the observed data is degraded, following the model $y = f(\mathbf{x}) + \varepsilon$, where $\varepsilon$ represents observation noise.
        These experiments directly test the noise rejection mechanism formalised in \cref{sec:gradient_analysis}. 
        According to \cref{cor:gating}, the spectral gate should remain closed for frequency components driven by incoherent noise while opening for those supported by genuine signal content, validating that the implicit regularisation arising from the $T_{\mathrm{self}}$ / $T_{\mathrm{signal}}$ decomposition (\cref{thm:decomp}) translates into practical robustness.

        \paragraph{Image Denoising.}
            We fit a noisy image  (with $\sigma=0.1$) directly, without any explicit denoising loss or regularisation, relying entirely on the INR's inductive bias to separate signal from noise.
            This is the purest test of \cref{cor:gating}(b): the noise component $\varepsilon$ is i.i.d.\ and produces $O_p(N^{-1/2})$ correlation with the carrier, so the gate should remain closed for noise-driven frequencies while opening for the underlying clean signal.
            We follow the evaluation protocol described in~\cite{saragadam2023wire,liu2024finer} and report only the final PSNR against the clean ground truth in \cref{tab:image_denoising}, since peak PSNR is not meaningful when the training target itself is corrupted.
            
            \algoname{} achieves the highest PSNR on all five images.
            This is a direct empirical confirmation of the noise rejection mechanism: without any architectural modification, early stopping, or regularisation term, the spectral gate autonomously prevents the network from fitting the noise, as it adapts the bandwidth to the signal content rather than imposing a fixed spectral ceiling.
            The runner-up across all images is MFN, which is notable, as in the clean image fitting experiments (\cref{tab:image_fitting}), MFN ranked among the weaker methods.
            Its relative strength here stems from its multiplicative structure imposing a strong low-frequency bias that acts as implicit regularisation in the presence of noise, though at the cost of severely limited representational capacity.
            SIREN, which led or rivalled \algoname{} on peak PSNR in clean fitting, drops to $26.2$ dB here, nearly identical to Gauss and WIRE, confirming that its all-pass nature offers no protection against noise memorisation.
            FINER similarly falls to the level of the non-periodic activations, indicating that its variable-periodic mechanism does not provide noise rejection.
            FR performs worst overall (below $23$ dB on four of five images), suggesting that its Fourier reparametrisation amplifies rather than suppresses noise-driven frequency content.
            The variance of our activation is slightly higher than that of the other baselines (e.g., $\pm1.25$ dB on \textbf{Tiger} vs. $\pm 0.02$ dB for SIREN), which reflects the fact that \algoname{} is actively making non-trivial spectral decisions, \ie, determining which frequencies to fit and which to suppress, whereas the baselines converge to a similar worse solution regardless of the run.

        \paragraph{Computed Tomography (CT) Reconstruction.}
            We reconstruct a 2D image from a sparse set of Radon projections, and observations are corrupted by measurement noise.
            This task combines aspects of both indirect and incomplete observation models: the signal is accessed only indirectly through line integrals, and the limited number of projection angles makes the inverse problem severely ill-posed.
            The INR must leverage its inductive bias to regularise the reconstruction, making this a joint test of noise robustness and the coarse-to-fine learning curriculum described in \cref{sec:multilayer}.
            We follow the setup of~\cite{saragadam2023wire,sun2021coil} and report results in the rightmost column of \cref{tab:image_denoising}.
            \begin{table}[t]
    \centering
    \caption{
        \textbf{Image Denoising and CT Scan Reconstruction.}
        Best results in \textbf{bold}, runner-ups \underline{underlined}.
        }
    \resizebox{\linewidth}{!}{
        \begin{tabular}{l ccccc c}

            \toprule
            
            \textbf{INR} & 
            \textbf{Tiger} & 
            \textbf{Tiles} & 
            \textbf{Bikers} & 
            \textbf{Butterfly} & 
            \textbf{Knot} &
            \textbf{CT Scan} \\ %

            \cmidrule(lr){2-6} \cmidrule(lr){7-7}
            
            \algoname{}                          & \textbf{30.81 $\pm$ 1.25} & \textbf{28.57 $\pm$ 0.29} & \textbf{29.35 $\pm$ 0.32} & \textbf{31.86 $\pm$ 0.69} & \textbf{28.67 $\pm$ 0.83} & \textbf{39.97 $\pm$ 0.25} \\
            SIREN \cite{sitzmann2019siren}       & 26.22 $\pm$ 0.02 & 26.21 $\pm$ 0.02 & 26.26 $\pm$ 0.08 & 26.40 $\pm$ 0.03 & 26.21 $\pm$ 0.01 & 29.81 $\pm$ 1.81 \\
            Gauss \cite{ramasinghe2022beyond}    & 26.44 $\pm$ 0.03 & 26.46 $\pm$ 0.03 & 26.77 $\pm$ 0.06 & 27.05 $\pm$ 0.06 & 26.25 $\pm$ 0.03 & 31.22 $\pm$ 0.06 \\
            WIRE \cite{saragadam2023wire}        & 26.36 $\pm$ 0.02 & 26.69 $\pm$ 0.02 & 26.72 $\pm$ 0.02 & 26.69 $\pm$ 0.05 & 26.29 $\pm$ 0.03 & 27.15 $\pm$ 0.19 \\
            BACON \cite{lindell2021bacon}        & 25.87 $\pm$ 0.02 & 27.30 $\pm$ 0.06 & 25.48 $\pm$ 0.04 & 23.60 $\pm$ 0.04 & 26.01 $\pm$ 0.01 & 11.11 $\pm$ 0.12 \\
            FINER \cite{liu2024finer}            & 26.23 $\pm$ 0.01 & 26.20 $\pm$ 0.03 & 26.39 $\pm$ 0.01 & 26.54 $\pm$ 0.02 & 26.21 $\pm$ 0.01 & 28.42 $\pm$ 6.06 \\
            MFN \cite{fathony2021multiplicative} & \underline{27.65 $\pm$ 0.19} & \underline{27.62 $\pm$ 0.04} & \underline{28.21 $\pm$ 0.10} & \underline{28.18 $\pm$ 0.69} & \underline{26.89 $\pm$ 0.10} & \underline{37.19 $\pm$ 0.07} \\
            Fourier \cite{tancik2020fourfeat}    & 26.59 $\pm$ 0.01 & 26.13 $\pm$ 0.44 & 26.85 $\pm$ 0.10 & 27.36 $\pm$ 0.06 & 26.08 $\pm$ 0.25 & 34.92 $\pm$ 0.19 \\
            FR \cite{shi2024improved}            & 22.87 $\pm$ 0.15 & 19.77 $\pm$ 0.02 & 19.85 $\pm$ 0.07 & 18.62 $\pm$ 0.04 & 17.29 $\pm$ 0.16 & 22.84 $\pm$ 0.62 \\
            
            \bottomrule
            
        \end{tabular}
        }
    \label{tab:image_denoising}
\end{table}

            \algoname{} achieves $39.97$ dB, leading the second-best method (MFN, $37.19$ dB) by nearly $2.8$ dB, the largest absolute margin among all experiments in this section.
            This result is particularly significant because CT reconstruction compounds two challenges that the spectral gate is designed to address: noise rejection (the projections are noisy) and progressive frequency recovery (sparse angular sampling provides incomplete spectral coverage, the network must reconstruct missing frequency content in a principled order).
            The ranking mirrors the denoising results: MFN is the runner-up, benefiting from its conservative spectral bias, followed by Fourier Features ($34.92$ dB) and Gauss ($31.22$ dB).
            SIREN drops to $29.81$ dB, confirming that its all-pass nature is a liability when the observations are both indirect and noisy: without spectral gating, SIREN fits noise in the sinogram, which propagates into streak artefacts in the reconstructed image (\cref{fig:qualitatives_transposed}).
            FINER shows instability in this setting ($\pm6.06$ dB), suggesting that its variable-periodic activations are sensitive to the conditioning of the Radon inverse problem.
            BACON ($11.11$ dB) and FR ($22.84$ dB) fail to produce useful reconstructions, indicating that their spectral structures are incompatible with the frequency-domain characteristics of tomographic data.

    \subsection{Signal Recovery from Incomplete Observations}
        \label{sec:exp_incomplete}
        
        Finally, we evaluate settings where the signal is observed only partially or indirectly, following the model $y = \Phi(f(\mathbf{x}))$ with $\Phi$ a non-invertible forward operator that discards information.
        Unlike the corrupted setting, the challenge here is not noise but ill-posedness: infinitely many signals are consistent with the observations, and the network must rely on its inductive bias to select a plausible reconstruction.
        These experiments test whether the coarse-to-fine learning curriculum (\cref{sec:multilayer_short} and \cref{sec:multilayer}) and the progressive bandwidth expansion (\cref{sec:ntk}) provide useful structural priors for underdetermined inverse problems.
    
        \paragraph{Image Inpainting.}
            We reconstruct an image from a subset of observed pixels, with the forward operator $\Phi$ being a binary mask that retains only a fraction ($20$\%) of the pixel locations. 
            Thus, the network must interpolate coherently across missing regions.
            We follow the protocol described in~\cite{saragadam2023wire,liu2024finer} and report results in \cref{tab:image_inpainting_and_super_resolution}.
            \begin{table}[t]
    \centering
    \caption{
        \textbf{Image Inpainting and Super-resolution.} 
        Best results in \textbf{bold}, runner-ups \underline{underlined}. 
        }
    \resizebox{\linewidth}{!}{
        \begin{tabular}{l ccccc ccccc}

            \toprule

            \multirow{3}{*}{\textbf{{INR}}} & \multicolumn{5}{c}{\textbf{Image Inpainting}} & \multicolumn{5}{c}{\textbf{Image Super-resolution}} \\

            \cmidrule(lr){2-6} \cmidrule(lr){7-11}
            
            & 
            \textbf{Tiger} & 
            \textbf{Tiles} & 
            \textbf{Bikers} & 
            \textbf{Butterfly} & 
            \textbf{Knot} & 
            \textbf{Tiger} & 
            \textbf{Tiles} & 
            \textbf{Bikers} & 
            \textbf{Butterfly} & 
            \textbf{Knot} \\

            \cmidrule(lr){2-6} \cmidrule(lr){7-11}
            
            \algoname{}                          & \underline{26.86 $\pm$ 0.39} & 20.47 $\pm$ 0.52 & \underline{23.05 $\pm$ 0.10} & \textbf{24.65 $\pm$ 0.10} & \textbf{21.11 $\pm$ 0.09} & \textbf{23.92 $\pm$ 0.06} & 19.60 $\pm$ 0.81 & \textbf{21.45 $\pm$ 0.02} & \textbf{20.81 $\pm$ 0.04} & \textbf{17.84 $\pm$ 0.03} \\
            SIREN \cite{sitzmann2019siren}       & \textbf{27.35 $\pm$ 0.09} & \underline{20.70 $\pm$ 0.06} & \textbf{23.33 $\pm$ 0.08} & 24.15 $\pm$ 0.17 & \underline{20.91 $\pm$ 0.04} & \underline{23.59 $\pm$ 0.02} & 20.27 $\pm$ 0.28 & 20.69 $\pm$ 0.04 & 20.24 $\pm$ 0.26 & 17.40 $\pm$ 0.17 \\
            Gauss \cite{ramasinghe2022beyond}    & 24.69 $\pm$ 0.04 & 20.66 $\pm$ 0.06 & 22.38 $\pm$ 0.03 & 22.97 $\pm$ 0.07 & 19.74 $\pm$ 0.01 & 20.93 $\pm$ 0.04 & 19.41 $\pm$ 0.03 & 19.79 $\pm$ 0.03 & 18.97 $\pm$ 0.02 & 16.62 $\pm$ 0.03 \\
            WIRE \cite{saragadam2023wire}        & 17.32 $\pm$ 0.14 & 16.86 $\pm$ 0.21 & 17.51 $\pm$ 0.10 & 15.66 $\pm$ 0.21 & 15.87 $\pm$ 0.10 & 15.47 $\pm$ 0.11 & 15.39 $\pm$ 0.11 & 15.72 $\pm$ 0.10 & 13.37 $\pm$ 0.12 & 14.47 $\pm$ 0.07 \\
            BACON \cite{lindell2021bacon}        & 8.21  $\pm$ 0.05 & 9.42  $\pm$ 0.07 & 7.65  $\pm$ 0.03 & 6.92  $\pm$ 0.02 & 9.69  $\pm$ 0.06 & 5.50  $\pm$ 0.12 & 8.01  $\pm$ 3.68 & 11.90 $\pm$ 0.09 & 9.00  $\pm$ 0.06 & 5.72  $\pm$ 0.19 \\
            FINER \cite{liu2024finer}            & 26.78 $\pm$ 0.10 & 20.90 $\pm$ 0.03 & 23.00 $\pm$ 0.07 & \underline{24.61 $\pm$ 0.08} & 20.90 $\pm$ 0.01 & 23.57 $\pm$ 0.05 & \underline{20.52 $\pm$ 0.14} & \underline{21.34 $\pm$ 0.07} & \underline{20.67 $\pm$ 0.00} & \underline{17.76 $\pm$ 0.02} \\
            MFN \cite{fathony2021multiplicative} & 24.65 $\pm$ 0.04 & 19.42 $\pm$ 0.14 & 20.74 $\pm$ 0.06 & 22.83 $\pm$ 0.13 & 19.33 $\pm$ 0.12 & 23.34 $\pm$ 0.10 & 20.56 $\pm$ 0.07 & 21.00 $\pm$ 0.07 & 20.42 $\pm$ 0.04 & 17.31 $\pm$ 0.04 \\
            Fourier \cite{tancik2020fourfeat}    & 25.25 $\pm$ 0.10 & \textbf{20.87 $\pm$ 0.03} & 22.34 $\pm$ 0.17 & 23.33 $\pm$ 0.11 & 20.12 $\pm$ 0.18 & 21.28 $\pm$ 0.25 & 19.52 $\pm$ 0.23 & 20.27 $\pm$ 0.14 & 18.87 $\pm$ 0.13 & 16.76 $\pm$ 0.09 \\
            FR \cite{shi2024improved}            & 22.47 $\pm$ 0.18 & 19.46 $\pm$ 0.02 & 19.54 $\pm$ 0.04 & 18.40 $\pm$ 0.03 & 16.73 $\pm$ 0.08 & 21.76 $\pm$ 0.10 & 19.48 $\pm$ 0.02 & 19.33 $\pm$ 0.04 & 18.13 $\pm$ 0.02 & 16.18 $\pm$ 0.07 \\
            
            \bottomrule
            
        \end{tabular}
        }   
    \label{tab:image_inpainting_and_super_resolution}
\end{table}

            Inpainting is the most competitive setting across baselines, with the top three methods, \algoname{}, SIREN, and FINER, separated by fractions of a decibel on most images.
            \algoname{} ranks first on \textbf{Butterfly} ($24.65$ dB) and \textbf{Knot} ($21.11$ dB), and second on \textbf{Tiger} ($26.86$ dB, behind SIREN's $27.35$ dB) and \textbf{Bikers} ($23.05$ dB, behind SIREN's $23.33$ dB).
            On \textbf{Tiles}, Fourier Features leads ($20.87$ dB), with SIREN second ($20.70$ dB) and \algoname{} at $20.47$ dB.
            Unlike denoising and super-resolution, where the network must generalise beyond corrupted or downsampled observations, inpainting provides clean pixel values at the observed locations.
            The missing data is spatially, not spectrally, incomplete.
            The task, therefore, rewards full spectral bandwidth at the observed coordinates, which is exactly SIREN's strength.
            The noise rejection is less relevant here as the observations are uncorrupted, and the coarse-to-fine curriculum offers a smaller benefit when the low-resolution structure is already fully determined by the observed pixels scattered across the image.
            The fact that \algoname{} nonetheless matches or exceeds SIREN on two of five images, while remaining within $0.5$ dB on two others, indicates that the spectral gate does not significantly penalise the network when noise rejection is unnecessary.
            FINER again emerges as a strong competitor, trailing our activation on \textbf{Butterfly} and matching it closely elsewhere, consistent with its strong performance in the clean fitting regime.
            WIRE, BACON, and FR underperform, confirming that spatial localisation without sufficient spectral bandwidth is inadequate for coherent interpolation across missing regions.

        \paragraph{Image Super-resolution.}
            We recover a higher-resolution ($512$ px) image from a lower-resolution input ($128$ px), where $\Phi$ represents spatial downsampling.
            The network is supervised only on the low-resolution grid and must hallucinate plausible high-frequency content at the fine scale.
            This tests whether the coarse-to-fine curriculum produces a natural frequency hierarchy: the network should first lock onto the observed low-frequency content and then extend to higher frequencies guided by the signal's spectral structure.
            We follow the protocol described in~\cite{saragadam2023wire,liu2024finer,shi2024improved} and report results in \cref{tab:image_inpainting_and_super_resolution}.
            
            \algoname{} achieves the highest PSNR on four of five images (\textbf{Tiger}, \textbf{Bikers}, \textbf{Butterfly}, \textbf{Knot}), with FINER as the consistent runner-up on those four.
            The margins are tighter than in clean fitting or denoising, as expected, since super-resolution is fundamentally constrained by the information available in the low-resolution input.
            On \textbf{Tiles}, \algoname{} ($19.60$ dB) falls behind MFN ($20.56$ dB) and FINER ($20.52$ dB), echoing the results in clean fitting, where FINER also led: the highly periodic texture of \textbf{Tiles} favours activations whose spectral support aligns with its regular structure, reducing the advantage of adaptive bandwidth control.
            An interesting shift from the fitting and denoising experiments is MFN's performance.
            MFN was among the weakest methods on clean fitting (\cref{tab:image_fitting}) but the second-best on denoising (\cref{tab:image_denoising}), and here it is competitive with FINER and SIREN across all images.
            Super-resolution shares a key property with denoising: the network must generalise beyond the observed data, and methods with conservative spectral behaviour avoid hallucinating high-frequency artefacts at unsupervised locations.
            \algoname{} achieves this conservatism through the spectral gating. 
            The coarse-to-fine curriculum ensures that low-frequency structure is captured first and higher frequencies are admitted only when justified, while MFN achieves it through a fixed low-frequency ceiling that, as the clean fitting results show, limits its capacity when the full spectrum must be represented.
            WIRE and BACON further degrade in the super-resolution setting, suggesting that their spatial localisation, while beneficial for fitting observed data, does not extrapolate well to unobserved coordinates.
            SIREN remains competitive on \textbf{Tiger} ($23.59$ dB, second-best) but falls behind on the remaining images, consistent with its lack of a mechanism to prioritise the observed low-frequency content before extending to higher frequencies.

\section{Discussion}
\label{sec:discussion}

    We have presented \algoname{}, a physics-grounded activation function for INRs derived from the steady-state response of a damped harmonic oscillator. 
    The core insight is that coupling amplitude to frequency through a physical transfer function creates an implicit spectral regulariser: the gradient of the reconstruction loss with respect to the damping factor decomposes into a self-closing force and a signal-driven opening force, yielding a precise gating condition that distinguishes coherent signal from noise without any explicit penalty in the loss.
    
    This mechanism yields state-of-the-art or competitive results across all the considered tasks without the need to tune either the architecture or any other hyperparameters, whereas each competitor INR needs specific tuning, and performs well in some settings but degrades significantly in others.
    Notably, the cases where \algoname{} turns out less beneficial -- clean, spatially incomplete observations and highly regular textures -- are those where \cref{cor:gating} predicts the least advantage, confirming that the theory correctly characterises the method's behaviour.
    
    Extending the formulation to per-neuron parameters and formalising the multi-layer gradient dynamics are the main directions for future work.

\bibliographystyle{splncs04}
\bibliography{main}

\clearpage

\setcounter{section}{0}
\renewcommand{\thesection}{\Alph{section}}

\begin{center}
    {
    \Large \bfseries
    \begin{minipage}{0.9\linewidth}
        \centering
        \setlength{\baselineskip}{1.2\baselineskip}
        Spectral Gating via Damped Oscillations for Adaptive Implicit Neural Representations \\
        Supplemental Materials
        \end{minipage}
    }
\end{center}

    In this document, we propose complementary results to the main paper. 
    First, we provide additional theoretical results, along with proofs and derivations, in ~\cref{sec:additonal_formulation}.
    After that, we illustrate implementation details in \cref{sec:implementation_details}.
    Then, we provide a complementary analysis on the properties of the \algoname{} in~\cref{sec:complementary_analysis}.
    Finally, we propose additional tasks in~\cref{sec:additonal_experiments}, along with more qualitative results from the tasks proposed in the main paper. 
    
    \section{Additional Formulation}
\label{sec:additonal_formulation}
    Throughout this section, we write D$(\omega, \omega_n, \xi) = (\omega_n^2 - \omega^2)^2 + (2 \xi \omega_n \omega)^2$ for the radicand in~\cref{eq:amplitude}, so that $A = \omega_n^2 D^{-1/2}$.

    \subsection{Derivations of Basic Properties}
        \label{sec:properties_full}

        \Cref{prop:smooth} states that the activation function $\sigma(z;\theta)$ is smooth with respect to both $z$ and all components of~$\theta$.
        \begin{proof}
            The sine term is $C^\infty$ with respect to $z$, $\omega$, and $\varphi$.
            The function $D$ is a polynomial in $(\omega, \xi, \omega_n)$ and satisfies $D > 0$ for all $\omega \in \mathbb{R}$ whenever $\omega_n > 0$ and $\xi > 0$, since the two non-negative terms $(\omega_n^2 - \omega^2)^2$ and $(2 \xi \omega_n \omega)^2$ cannot vanish simultaneously.  
            Therefore $D^{-1/2}$ is $C^\infty$ in $(\omega, \xi, \omega_n)$.  
            Since products and compositions of $C^\infty$ functions are $C^\infty$, it follows that $\sigma(z; \theta) = \omega_n^2 D^{-1/2} \sin(\omega z + \varphi)$ is $C^\infty$ with respect to $z$ and all components of $\theta$.
        \end{proof}
        
        \Cref{prop:lipschitz} states that the activation function is Lipschitz in $z$ with a constant $L_z = |\omega|A(\omega,\omega_n,\xi)$.
        \begin{proof}
          $\frac{\partial \sigma}{\partial z} = \omega A \cos(\omega z + \varphi)$, so $L_z=\sup_{z \in \mathbb{R}}$ $\bigl|\frac{\partial \sigma}{\partial z}\bigr| = |\omega|\ A \geq \frac{\partial \sigma}{\partial z}$.
        \end{proof}
        
        \Cref{prop:bound} states that:
        \begin{itemize}
            \item[(i)] $|\sigma(z;\theta)| \leq A(\omega,\omega_n,\xi)$ uniformly in $z$ $\forall z \in \mathbb{R}$;
            \item[(ii)] The peak amplitude is $A_{\max} = \bigl( 2 \xi \sqrt{1 - \xi^2}\bigr)^{-1}$, which is attained at the resonant frequency $\omega_r = \omega_n \sqrt{1 - 2 \xi ^2}$ for $\xi < 1/\sqrt{2}$;
            \item[(iii)] For $\xi \ge 1/\sqrt{2}$ the amplitude decreases monotonically from $A(0) = 1$.
        \end{itemize}
        \begin{proof}
            Part (i): as $|\sin(\cdot)| \leq 1$, the output bound depends solely on $A(\omega,\omega_n,\xi)$.
            
            Part (ii): We can find peak amplitude by maximising $A$ or, as the numerator is positive by construction ($\omega_n^2>0$), by minimising $D$.
            Remembering that $D = (\omega_n^2 - \omega)^2 + 4 \xi^2 \omega_n^2 \omega^2$, differentiating with respect to $\omega$ we obtain:
            \begin{align}
                \frac{\partial D}{\partial \omega} = 0 \implies \frac{\partial D}{\partial \omega} &= 2(\omega_n^2 - \omega^2) \cdot - 2 \omega + 2(2 \xi \omega_n \omega) \cdot 2 \xi \omega_n
                \\
                &= 8 \xi^2 \omega_n^2 \omega - 4 \omega_n^2 \omega + 4 \omega^3
                \\
                &= 4 \omega (2 \xi^2 \omega_n^2 - \omega_n^2 + \omega^2) = 0
            \end{align}
            which yields the trivial solution $\omega_{\textrm{peak}}=0 \implies A(0)=1$ (DC gain).
            The remaining expression can be rearranged as:
            \begin{align}
                2 \xi^2 \omega_n^2 - \omega_n^2 + \omega^2 &= 0
                \\
                \omega^2 &= \omega_n^2(1 - 2 \xi^2) \implies \omega_{\textrm{peak}} = \omega_n \sqrt{1 - 2 \xi^2}.
            \end{align}
            Note that we discard the negative solution as $\omega>0$.
            The positive solution, more commonly referred to as resonant frequency $\omega_r$, exists when the radicand is strictly positive, hence for $\xi < 1 / \sqrt{2}$.
            By substituting back into~\cref{eq:amplitude}~\cite{ogata2010modern}, we obtain:
            \begin{align}
                A (\omega_r) &= \frac{\omega_n^2} {\sqrt{\bigl(\omega_n^2 - \omega_n^2(1 - 2 \xi^2)\bigr)^2 + 4 \xi^2 \omega_n^2 \omega_n^2 (1 - 2 \xi^2)}} \\
                &= \frac{\omega_n^2} {\sqrt{\omega_n^4 + \omega_n^4(1 - 2 \xi^2)^2 - 2 \omega_n^4 (1 - 2 \xi^2) + 4 \xi^2 \omega_n^4 (1 - 2 \xi^2)}} \\
                &= \frac{\omega_n^2} {\sqrt{ \omega_n^4 \bigl(1 + (1 - 2 \xi^2)^2 - 2(1 - 2 \xi^2)+ 4 \xi^2(1 - 2 \xi^2) \bigr)}} \\
                &= \frac{1}{\sqrt{1 + (1 - 2 \xi^2) \bigl((1 - 2 \xi^2) - 2 + 4 \xi^2 \bigr)}} \\
                &= \frac{1}{\sqrt{1 + (1 - 2 \xi^2) (2 \xi^2 - 1)}} \\
                &= \frac{1}{\sqrt{1 + 2 \xi^2 -1 -4 \xi^4 + 2 \xi^2}} \\
                &= \frac{1}{\sqrt{4 \xi^2 -4 \xi^4}} \\
                &= \frac{1}{\sqrt{4 \xi^2 (1 - 2 \xi^2)}} \\
                &= \frac{1}{2 \xi\sqrt{1 - 2 \xi^2}} = A_{\textrm{peak}}.
            \end{align}

            Part (iii): By differentiating $D$ with respect to $\omega$ we already computed that:
            \begin{equation}
                \frac{\partial D}{\partial \omega} = 4 \omega (2 \xi^2 \omega_n^2 - \omega_n^2 + \omega^2),
            \end{equation}
            and since $\omega; \omega_n ; \xi >0$ and $2 \xi^2 \omega_n^2 - \omega_n^2 > 0$ for $\xi \ge 1/\sqrt{2}$, then $\frac{\partial D}{\partial \omega} > 0$, meaning that $D$ is strictly increasing $\forall \omega \in (0,\infty)$, hence $A$ is strictly decreasing $\forall \omega \in (0,\infty)$.
        \end{proof}

    \subsection{Derivations of Spectral Gating}
        \label{sec:gradient_analysis_full}
        \Cref{thm:monotonicity_xi} establishes that the amplitude $A$ is monotone in $\xi$.
        \begin{proof}
            Since we can rewrite the amplitude as $A = \omega_n^2 D^{-1/2}$, we have:
            \begin{equation}
                \label{eq:dAdxi_full}
                \frac{\partial A}{\partial \xi} = \frac{\partial}{\partial \xi} \Bigl[\omega_n^2 D^{-1/2}\Bigr] = -\frac{1}{2} \omega_n^2 D^{-3/2}\frac{\partial D}{\partial \xi}.
            \end{equation}
            Remembering that $D = (\omega_n^2 - \omega)^2 + 4 \xi^2 \omega_n^2 \omega^2$, differentiating with respect to $\xi$ we obtain $\frac{\partial D}{\partial \xi} = 8 \xi \omega_n^2 \omega^2$, and substituting in ~\cref{eq:dAdxi_full} yields:
            \begin{equation}
                \frac{\partial A}{\partial \xi} = -\frac{1}{2}\omega_n^2 D^{-3/2} \cdot 8 \xi \omega_n^2 \omega^2 = - \omega_n^4 \omega^2 \xi D^{-3/2} \implies  \frac{\partial A}{\partial \xi} < 0,
            \end{equation}
            as $\omega; \omega_n; \xi > 0$ by definition, while \cref{rem:singularity} establishes the vanishing condition for the radicand $D$.
        \end{proof}

        \Cref{thm:decomp} establishes how the gradient of $\mathcal{L}$ with respect to $\xi$ decomposes, while \cref{cor:gating} how its components interacts with different signals.
        \begin{proof}
            As the pointwise mean squared error loss for target values $\{y_i\}$ is $\mathcal{L} = \frac{1}{N}\sum_{i=1}^N\bigl(\sigma(z_i;\theta)-y_i\bigr)^2$, its derivative with respect to $\xi$ is\footnote{$\sigma(z_i;\theta) = \sigma_i$ for sake of compactness.}:
            \begin{align}
                \label{eq:dLdxi}
                \frac{\partial \mathcal{L}}{\partial \xi} = \frac{2}{N} \sum_{i=1}^N (\sigma_i - y_i) \frac{\partial \sigma_i}{\partial \xi}.
            \end{align}
            Since: 
            \begin{equation}
                \frac{\partial \sigma_i}{\partial \xi} = \frac{\partial A}{\partial \xi} \sin{(\omega z + \varphi)},
            \end{equation}
            we can plug it back into \cref{eq:dLdxi} obtaining:
            \begin{equation}
                \frac{\partial \mathcal{L}}{\partial \xi} = \frac{2}{N} \sum_{i=1}^N (\sigma_i - y_i) \frac{\partial A}{\partial \xi} \sin{(\omega z + \varphi)}.
             \end{equation}
            Plugging the definition of $\sigma_i$ we obtain:
            \begin{align}
                \frac{\partial \mathcal{L}}{\partial \xi} &= \frac{2}{N} \sum_{i=1}^N (A \sin{(\omega z + \varphi)} - y_i) \frac{\partial A}{\partial \xi} \sin{(\omega z + \varphi)} 
                \\
                &= \frac{2}{N} \frac{\partial A}{\partial \xi} \sum_{i=1}^N (A \sin^2{(\omega z + \varphi)} - y_i \sin{(\omega z + \varphi)})
                \\
                &= \underbrace{\frac{2A}{N} \cdot \frac{\partial A}{\partial \xi}  \sum_{i=1}^N \sin^2(\omega z_i + \varphi)}_{T_{\mathrm{self}}} + \underbrace{\biggl(- \frac{2}{N} \cdot \frac{\partial A}{\partial \xi}\biggr) \sum_{i}^N y_i \sin(\omega z_i + \varphi)}_{T_{\mathrm{signal}}}.         
             \end{align}

            We can rewrite the self-term by using the identity $\sin^2x = \frac{1}{2}(1-\cos(2x))$:
            \begin{align}
                T_{\mathrm{self}} &= \frac{2A}{N} \cdot \frac{\partial A}{\partial \xi} \sum_{i=1}^{N} \sin^2{(\omega z_i + \varphi)} 
                \\
                &= \frac{2A}{N} \cdot \frac{\partial A}{\partial \xi} \sum_{i=1}^{N} \biggl( \frac{1}{2} - \frac{1}{2} \cos{\bigl(2(\omega z_i + \varphi )\bigr)} \biggr) 
                \\
                &= \frac{A}{N} \cdot \frac{\partial A}{\partial \xi} + \frac{A}{N} \cdot \frac{\partial A}{\partial \xi} \sum_{i=1}^{N} \cos{\bigl(2(\omega z_i + \varphi )\bigr)}
                \\
                &= \frac{A}{N} \cdot \frac{\partial A}{\partial \xi} \sum_{i=1}^{N}\bigl[ 1 + \cos{(2\omega z_i + 2\varphi)} \bigr].
                \\
            \end{align}
            For $N \to +\infty$, the first term evaluates to $N$ while the second term oscillates symmetrically, evaluating to 0 due to phase cancellation:
            \begin{equation}
                T_{\mathrm{self}} \approx \frac{A}{N} \cdot \frac{\partial A}{\partial \xi} [ N + 0\bigr] = A \cdot \frac{\partial A}{\partial \xi}.
            \end{equation}
            Since $A>0$ by construction and $\frac{\partial A}{\partial \xi}<0$ by \cref{thm:monotonicity_xi}, then $T_{\mathrm{self}}<0$.
            
            We can rewrite the signal-term as an empirical inner product:
            \begin{equation}
                T_{\mathrm{signal}} = \biggl(- \frac{2}{N} \cdot \frac{\partial A}{\partial \xi}\biggr) \sum_{i}^N y_i \sin(\omega z_i + \varphi) = -2 \cdot \frac{\partial A}{\partial \xi}\langle y,\sin(\omega \cdot +\varphi)\rangle_N.
            \end{equation}
            Since $\frac{\partial A}{\partial \xi}<0 \implies -\frac{\partial A}{\partial \xi} > 0$ by \cref{thm:monotonicity_xi}, we have:
            \begin{equation}
                \sign(T_{\mathrm{signal}}) = \sign\bigl(\langle y, \sin(\omega \cdot +\varphi)\rangle_N\bigr),
            \end{equation}
            since we cannot remove $y$, the signal term depend on the joint structure of $y$ and 
            $\sin(\omega z_i + \varphi)$.
            We distinguish three cases:
            \begin{itemize}
                \item[(a)] With a coherent signal $y_i \approx A_y \sin{(\omega z_i + \phi)}$, hence with a frequency close to $\omega$, the signal-term can be written as\footnote{For the sake of clarity, we dropped the derivative term as we already know its sign.}:
                \begin{equation}
                    T_{\mathrm{signal}} = \frac{2}{N} \sum_{i}^N A_y \sin{(\omega z_i + \phi)} \sin(\omega z_i + \varphi),
                \end{equation}
                employing the Werner's formula for the product of sines:
                \begin{equation}
                    \sin a \sin b = \cos{(a-b)} - \cos{(a+b)},
                \end{equation}
                we can rewrite:
                \begin{equation}
                    T_{\mathrm{signal}} = \frac{2A_y}{N} \sum_{i=i}^N \cos{(\varphi - \phi)} - \frac{2A_y}{N} \sum_{i=i}^N \cos{(2 \omega z_i + \varphi + \phi)}.
                \end{equation}
                For $N \to +\infty$, the second term oscillates symmetrically, evaluating to 0 due to phase cancellation, while the second term is non-zero, hence the gate can open;
                \item[(b)] With i.i.d noise with zero mean and finite variance we have that $T_{\mathrm{signal}} = O_p(N^{-\frac{1}{2}})$, hence for $N \to +\infty$ the signal-term vanishes asymptotically and the gate stays closed\footnote{$O_p$ denotes boundedness in probability: $X_N = O_p(a_N)$ means that for any $\epsilon > 0$ there exists $M$ such that $P(|X_N| > M a_N) < \epsilon$ for all large $N$. We use $O_p$ because the signal term is a random variable (a function of the stochastic target $y$), whereas $A/2 = O(1)$ is deterministic.};
                \item[(c)] With a signal $y_i \approx A' \sin{(\omega' z_i)}$, hence with a frequency $\omega' \neq \omega$ far from $\omega$, the signal-term can be written as:
                \begin{equation}
                    T_{\mathrm{signal}} = \frac{2}{N} \sum_{i}^N A' \sin{(\omega' z_i)} \sin(\omega z_i + \varphi).
                \end{equation}
                We can employ once again the Werner's formula for the product of sines, but since we have different frequencies, for $N \to +\infty$ both terms evaluate to 0, and the gate stays closed.
            \end{itemize}
        \end{proof}

        \subsubsection{Gradients with respect to $\omega$ and $\omega_n$.}
            The remaining parameter gradients complete the picture of how the spectral gate is positioned and shaped during training.
            The partial derivatives of $A$ are:
            \begin{gather}
              \frac{\partial A}{\partial\omega} = \frac{\omega_n^2\bigl[2\omega(\omega_n^2-\omega^2)-4\xi^2\omega_n^2\omega\bigr]}{D^{3/2}},
              \\
              \frac{\partial A}{\partial\omega_n} = \frac{2\omega_n}{D^{1/2}} - \frac{\omega_n^2\bigl[4\omega_n(\omega_n^2-\omega^2)+8\xi^2\omega_n\omega^2\bigr]}{2D^{3/2}}.
            \end{gather}
            Crucially, $\frac{\partial A}{\partial \omega}$ changes sign at $\omega^2 = \omega_n^2(1-2\xi^2)$, the resonance frequency, meaning gradient descent on $\omega$ drives the forcing frequency toward resonance when the current $\omega$ is below the peak, and away when above.
            This creates a gradient-driven frequency-locking mechanism that aligns the carrier with the target's dominant spectral content, complementing the bandwidth control exerted by $\xi$.

        \subsection{Multi-Layer Frequency Composition}
            \label{sec:multilayer}
            
            The analysis in \cref{sec:gradient_analysis} considers a single layer in isolation. 
            In practice, the interaction between stacked layers determines an INR's representational power. 
            We now show that cascading oscillator layers produces a frequency-modulated signal~\cite{carson1922notes} whose harmonic bandwidth is controlled by the per-layer amplitudes $A_k$ and, therefore, by the damping factors $\xi_k$. 
            This connects single-layer spectral gating to the network's global frequency support and extends the observation by~\cite{yuce2022structured} that composing sinusoidal layers produces Bessel-governed harmonics: our \algoname{} amplitude $A_k$ makes the modulation index learnable, giving the network explicit control over its harmonic bandwidth at each depth.
            
            \begin{proposition}[Frequency Modulation Structure]
                \label{prop:fm}
                Consider two consecutive oscillator layers:
                \begin{align}
                    h_1(x) &= A_1\sin \bigl(\omega_1(w_1 x+b_1)+\varphi_1\bigr),
                    \\
                    h_2(x) &= A_2\sin \bigl(\omega_2(w_2 h_1(x)+b_2)+\varphi_2\bigr).
                \end{align}
                The two-layer composition is a frequency-modulated signal:
                \begin{equation}\label{eq:fm}
                    h_2(x) = A_2\sin \bigl(\beta\sin(\alpha(x))+c\bigr),
                \end{equation}
                with modulation index $\beta=\omega_2 w_2 A_1$, carrier $\alpha(x) = \omega_1(w_1 x + b_1) + \varphi_1$, and offset $c = \omega_2 b_2 + \varphi_2$.
            \end{proposition}
            \begin{proof}
              Direct substitution of $h_1$ into the argument of $h_2$:
              \begin{align}
                  h_2(x) &= A_2 \sin{(\omega_2 (w_2 A_1 \sin(\omega_1(w_1 x + b_1) + \varphi_1)+b_2) + \varphi_2)} \\
                  &= A_2 \sin{(\underbrace{\omega_2 w_2 A_1}_{\text{Amplitude } \beta}} \sin{(\underbrace{\omega_1 (w_1 x + b_1) + \varphi_1}_{\text{Carrier} \alpha(x)} ) + \underbrace{\omega_2 b_2 + \varphi_2}_{\text{Offset } c})} \\
                  &= A_2 \sin{(\beta \sin{(\alpha (x)) + c})}.
              \end{align}
            \end{proof}
            
            Expanding~\cref{eq:fm} via the Jacobi-Anger identity~\cite{watson1944treatise, arfken2012mathematical}:
            \begin{equation}    
                \label{eq:jacobi}
                \sin(\beta\sin\alpha+c) = \sum_{n=-\infty}^{\infty}J_n(\beta)\sin(n\alpha+c),
            \end{equation}
            where $J_n$ is the Bessel function of the first kind~\cite{abramowitz1964handbook}.
            Thus, the output contains harmonics at integer multiples of the first-layer frequency, with amplitudes $A_2|J_n(\beta)|$.
            
            \begin{proof}
                Starting from the Jacobi-Anger identity and rearranging:
                \begin{equation}
                    e^{i \beta \sin \alpha} = \sum_{n=-\infty}^{+\infty} J_n (\beta) e^{i n \alpha},
                \end{equation}
                we can multiply both sides by a constant $e^{ic}$ to introduce the offset $c$:
                \begin{align}
                    e^{i \beta \sin \alpha} \cdot e^{ic} &= \sum_{n=-\infty}^{+\infty} J_n (\beta) e^{i n \alpha} \cdot e^{ic} \\
                    e^{i (\beta \sin \alpha + c)} &= \sum_{n=-\infty}^{+\infty} J_n (\beta) e^{i(n \alpha + c)}.
                \end{align}
                Applying Euler's Formula, $e^{ix} = \cos x + i \sin x$, we can rewrite:
                \begin{align}
                    \cos{(\beta \sin \alpha + c)} & + i \sin{(\beta \sin \alpha + c)}
                    \\
                    &= \sum_{n=-\infty}^{+\infty} J_n (\beta) [\cos{(n \alpha + c)} + i \sin{(n \alpha + c)}].
                \end{align}
                We can then evaluate the real and imaginary parts separately:
                \begin{align}
                    \Re &: \cos{(\beta \sin \alpha + c)} = \sum_{n=-\infty}^{+\infty} J_n (\beta) \cos{(n \alpha + c)},
                    \\
                    \Im &: \sin{(\beta \sin \alpha + c)} = \sum_{n=-\infty}^{+\infty} J_n (\beta) \sin{(n \alpha + c)}.
                \end{align}
                Finally:
                \begin{equation}
                    h_2(x) = A_2 \sum_{n=-\infty}^{+\infty} J_n (\beta) \sin{(n \alpha + c)}
                \end{equation}
            \end{proof}
            
            \begin{corollary}[Damping-controlled Harmonic Bandwidth]
                \label{cor:bandwidth}
                The number of harmonics with non-negligible amplitude is approximately $\lfloor\beta\rfloor+1$, since $|J_n(\beta)|\ll 1$ for $|n|>\beta+O(\beta^{1/3})$~\cite{abramowitz1964handbook,olver2010nist}.
                Since $\beta = \omega_2 w_2 A_1$ and $A_1$ is a decreasing function of $\xi_1$ (\cref{thm:monotonicity_xi}):
                \begin{itemize}
                    \item Larger $\xi_1$ $\Rightarrow$ smaller $A_1$ $\Rightarrow$ smaller $\beta$ $\Rightarrow$ fewer harmonics (narrowband);
                    \item Smaller $\xi_1$ $\Rightarrow$ larger $A_1$ $\Rightarrow$ larger $\beta$ $\Rightarrow$ more harmonics (wideband).
                \end{itemize}
            \end{corollary}
            
            This analysis extends to deeper networks by induction.
            
            \begin{proposition}[Cascaded Spectral Gating]
            \label{prop:cascade}
                In an $L$-layer oscillator network, the effective modulation index at layer $\ell$ is
                \begin{equation}
                    \label{eq:cascade}
                    \beta_\ell = \omega_\ell w_\ell \prod_{k = 1}^{\ell - 1} A_k.
                    \end{equation}
                The total harmonic bandwidth is therefore exponentially sensitive to the product of per-layer amplitudes $\prod_k A_k$.
            \end{proposition}
            
            At initialisation with per-layer amplitudes $A_k\approx A_0<1$, the modulation index at layer $\ell$ is $\beta_\ell \propto A_0^{\ell-1}$, which decays geometrically with depth.
            The deepest layers initially contribute near-zero frequency support, and bandwidth expands from the input layer outward as the individual $A_k$ grow during training.
            This provides a rigorous mechanism for the empirically observed coarse-to-fine learning: the multiplicative structure ensures that even modest decreases in individual $\xi_k$ values have compounding effects on spectral bandwidth.
            
             This analysis mirrors the classical observation in \cite{yuce2022structured} that composing sinusoidal layers produces harmonics whose amplitudes are governed by Bessel coefficients; the key difference here is that the oscillator amplitude $A_k$ provides a learnable per-layer control over the modulation index, whereas in SIREN this quantity is determined solely by the fixed product of frequency and weight magnitude.
    
    \subsection{Neural Tangent Kernel}
        \label{sec:ntk}
    
        The gradient decomposition and multi-layer analysis describe \algoname{}'s inductive bias in terms of individual parameter updates and harmonic structure. 
        We complement these with a Neural Tangent Kernel (NTK) perspective, which characterises how the oscillator parameters shape the global learning dynamics of the network. 
        The NTK framework~\cite{jacot2018neural} relates the convergence rate of gradient descent to the eigenspectrum of the kernel. 
        By showing that the kernel magnitude scales with $A^2$, we establish that the oscillator parameters modulate the network's effective learning rate. 
        This provides an additional, kernel-level explanation for the progressive learning behaviour: the network bootstraps its own convergence speed by opening the spectral gate. 
        NTK analysis has been applied to INRs previously: \cite{tancik2020fourfeat} uses it to explain why Fourier feature mappings accelerate learning of high-frequency functions, and \cite{ramasinghe2022beyond} analyses the NTK's eigenspectrum of Gaussian activations to explain their low-frequency bias.
        
        \begin{theorem}[Decomposition]
            \label{thm:ntk}
            Consider a single hidden-layer oscillator network of $m$ neurons, $f(x)=\sum_{j=1}^m a_j \sigma(w_j x+b_j;\theta)$, with network parameters $\mathrm{w}=\{a_j,w_j,b_j\}_{j=1}^m$ and $\theta=\{\omega,\omega_n,\xi,\varphi\}$ for our activation function, such that $\Theta = \mathrm{w}\cup\theta$.
            Its NTK decomposes as $K_{\Theta}(x,x') = K_{\mathrm{w}}(x,x') + K_\theta(x,x')$, where:
            \begin{enumerate}
                \item[\emph{(i)}] The weight contribution satisfies, in the infinite-width limit:
                    \begin{equation}
                        \label{eq:ntk_weights}
                        K_w(x,x') = \frac{A^2}{2} \mathbb{E}_w \bigl[\cos \bigl(\omega w (x - x')\bigr)\bigr],
                    \end{equation}
                where the expectation is over the weight initialisation distribution.
            
                For $w \sim \mathcal{U}[-c,c]$, this evaluates to $K_w(x,x') = \frac{A^2}{2} \sinc \bigl(\frac{\omega c (x - x')}{\pi}\bigr)$.
                \item[\emph{(ii)}] The oscillator parameter contribution is a rank-$4$ correction:
                    \begin{equation}
                        \label{eq:ntk_osc}
                        K_\theta(x,x') = \sum_{\alpha\in\theta}\biggl(\sum_{j=1}^m a_j\frac{\partial\sigma_j}{\partial\alpha}(x)\biggr) \biggl(\sum_{j=1}^m a_j\frac{\partial\sigma_j}{\partial\alpha}(x')\biggr).
                    \end{equation}
            \end{enumerate}
        \end{theorem}
        \begin{proof}
        
            Part (i): since the Neural Tangent Kernel is defined as:
            \begin{equation}
                K(x,x') = \langle \nabla_\Theta f(x), \nabla_\Theta f(x') \rangle,
            \end{equation}
            and $\frac{\partial f(x)}{\partial a_j} = \sigma_j(x)$, the contribution from $\{a_j\}$ is: 
            \begin{align}
                K_{\mathrm{w}}(x,x') &= \sum_{j=1}^{m} \frac{\partial f(x)}{\partial a_j} \cdot \frac{\partial f(x')}{\partial a_j} \\
                &= \sum_{j=1}^{m} \sigma_j(x) \cdot \sigma_j(x') \\
                &= \sum_{j=1}^{m} A \sin(\omega(w_j x + b_j) + \varphi) \cdot A\sin(\omega(w_j x' + b_j) + \varphi).
            \end{align}
            Applying the product-to-sum identity: 
            \begin{equation}
                \sin{u} \cdot \sin{v} = \frac{1}{2} [\cos{(u-v)} - \cos{(u+v)}],
            \end{equation}
            gives:
            \begin{align}
                u - v &= \omega w_j (x - x'), \\
                u + v &= \omega w_j (x + x') + 2(\omega b_j + \varphi).
            \end{align}
            Hence, we can then rewrite:
            \begin{align}
                K_{\mathrm{w}}(x,x') &= \sum_{j=1}^{m} A \sin(\omega(w_j x + b_j) + \varphi) \cdot A\sin(\omega(w_j x' + b_j) + \varphi) \\
                &= \frac{A^2}{2} \sum_{j=1}^{m} \cos(\omega w_j (x - x')) - \cos(\omega w_j (x - x') + 2(\omega b_j + \varphi)).
            \end{align}
            
            and taking expectations over the symmetric initialisation of $b$ eliminates the cross-term $\cos(\omega w(x+x') + 2(\omega b +\varphi))$.
            
            The infinite-width limit of the remaining term follows from the law of large numbers.
            Since for a uniform distribution $w \sim \mathcal{U}[-c,c]$ its probability distribution function is $p(w) = \frac{1}{2c}$, the expectation of function $f(w)$ over such distribution is $\mathbb{E}[f(w)] = \int_{-c}^{+c} p(w) f(w) dw$.
            
            Hence, following the fundamental theorem of calculus and for the parity of the sine, the remaining term evaluates as:
            \begin{align}
                \mathbb{E}_w [\cos(\omega w (x - x'))] &= \int_{-c}^{+c} \frac{1}{2c} \cos(\omega w (x - x')) 
                \\
                &= \frac{1}{2c} \Biggl[\frac{\sin{(\omega w (x - x'))}}{\omega (x - x')}\Biggr]_{-c}^{+c} 
                \\
                &= \frac{1}{2c} \Biggl[\frac{\sin{(\omega c (x - x'))}}{\omega (x - x')} + \frac{\sin{(-\omega c (x - x'))}}{\omega (x - x')}\Biggr],
                \\
                &= \frac{1}{2c} \cdot 2 \cdot \frac{\sin{(\omega c (x - x'))}}{\omega (x - x')} 
                \\
                &= \frac{\sin{(\omega c (x - x'))}}{\omega c (x - x')} 
                \\
                &= \sinc \Biggl( \frac{\omega c (x - x')}{\pi} \Biggr).
            \end{align}

            Finally:
            \begin{equation}
                K_w(x,x') = \frac{A^2}{2} \sinc \Biggl( \frac{\omega c (x - x')}{\pi} \Biggr).
            \end{equation}
                        
            Part (ii): each oscillator parameter 
            $\alpha\in\{\omega,\omega_n,\xi,\varphi\}$
            is shared across all neurons, so $\frac{\partial f}{\partial \alpha} = \sum_j a_j \frac{\partial \sigma_j}{\partial \alpha}$.
            The NTK contribution is $\frac{\partial f}{\partial\alpha}(x)\frac{\partial f}{\partial\alpha}(x')$, summed over the four parameters.
        \end{proof}
        
        \begin{corollary}[Amplitude-modulated Learning Dynamics]
            \label{cor:ntk_amplitude}
            The kernel magnitude in~\cref{eq:ntk_weights} scales as $A^2$.
            As the NTK eigenvalues govern convergence rates under gradient flow~\cite{jacot2018neural,arora2019fine}, the oscillator parameters modulate the effective learning rate:
            \begin{itemize}
                \item At initialisation with $A \approx A_0$, all eigenvalues are scaled by $A_0^2$, yielding slow initial learning when $A_0 < 1$;
                \item As $\xi$ decreases during training and $A$ grows, the kernel magnitude increases, accelerating convergence.
            \end{itemize}
            The network thus bootstraps its convergence rate by progressively opening the spectral gate.
        \end{corollary}

\section{Implementation Details}
\label{sec:implementation_details}

    All experiments are implemented within a unified codebase under identical training conditions (optimiser, loss function, evaluation protocol) to ensure a fair comparison.
    For each INR, we adopt the architecture and hyperparameters recommended by the original authors when available for a given task; when no configuration is provided, we manually tune to obtain competitive results.
    Our proposed activation, \algoname{}, uses the same oscillator parameters across all tasks, as described in \cref{sec:experiments} of the main paper.
    Below, we report the per-task configuration for each baseline.
    In all tables, $m$ denotes the hidden feature dimension, $L$ the number of hidden layers, and $\eta$ the learning rate.

    \paragraph{\algoname{} Configuration.}
        Our activation uses a single configuration across all tasks: $m = 256$, $L = 5$, $\omega_0 = 45$, $\omega_{n,0} = 50$, $\xi_0 = 1/\sqrt{2}$, and $\varphi_0 = - 1.423$.
        The network weights are optimised with a learning rate of $\eta_{\text{w}} = 10^{-4}$, while the oscillator parameters $(\omega, \omega_n, \xi, \varphi)$ use a separate learning rate of $\eta_{\theta} = 10^{-2}$.
        No task-specific tuning of any parameter is performed.
    
    \paragraph{Optimisation.}
        All INRs use Adam for all tasks.
        Three methods employ learning rate schedulers.
        FDHO uses \texttt{ReduceLROnPlateau} (factor $0.1$, patience $500$, minimum learning rate $10^{-6}$), which monitors the training loss and reduces the learning rate only when progress stalls.
        WIRE uses a linear warmup schedule that scales the learning rate by $0.1 \cdot \min(t / T, 1)$, where $T$ is the total number of training steps.
        FR uses a step decay schedule (\texttt{StepLR}) that multiplies the learning rate by $0.1$ every $3000$ steps.
        All remaining INRs use a constant learning rate throughout training.
        
    \paragraph{1D Signal Fitting and Audio Fitting.}
        All methods use a 1D input ($d_{\text{in}} = 1$) and are trained for 10000 steps.
        The square waves at 100 Hz, 250 Hz and 500 Hz are sampled at 400, 1000, and 2000 points, respectively.
        The chirp signals at 250 Hz, 500 Hz and 1000 Hz are sampled at 550, 1100, and 2200 points, respectively.
        Audio waveforms are sampled with their respective sampling rate as done in \cite{sitzmann2019siren}.
        Configurations are reported in \cref{tab:config_signal}.
        \begin{table}[ht]
            \centering
            \caption{\textbf{Configurations for 1D Signal Fitting and Audio Fitting.}}
            \label{tab:config_signal}
            \resizebox{0.9\linewidth}{!}{
            \begin{tabular}{lcccl}
                \toprule
                \textbf{INR} & $m$ & $L$ & $\eta$ & Task-specific parameters \\
                \midrule
                SIREN~\cite{sitzmann2019siren}       & 256 & 5 & $10^{-4}$ & $\omega_0 = 30$ \\
                Gauss~\cite{ramasinghe2022beyond}    & 256 & 5 & $10^{-4}$ & $\sigma = 0.1$ \\
                WIRE~\cite{saragadam2023wire}        & 256 & 2 & $5 \times 10^{-3}$ & $\omega_0 = 20$, $s = 30$ \\
                BACON~\cite{lindell2021bacon}        & 128 & 4 & $10^{-4}$ & freq.\ limit $= 2048$, quantised \\
                FINER~\cite{liu2024finer}            & 256 & 3 & $10^{-4}$ & $\omega_0 = 30$ (first and hidden) \\
                MFN~\cite{fathony2021multiplicative} & 256 & 3 & $10^{-2}$ & weight scale $= 1.0$, $\alpha = 6.0$ \\
                Fourier~\cite{tancik2020fourfeat}    & 256 & 4 & $10^{-4}$ & scale $= 10$ \\
                FR~\cite{shi2024improved}            & 256 & 5 & $10^{-4}$ & high-freq.\ components $= 256$ \\
                \bottomrule
            \end{tabular}}
        \end{table}
    
    \paragraph{Image Fitting, Image Denoising, Image Inpainting, and Image Super-resolution.}
        All methods use a 2D input ($d_{\text{in}} = 2$) mapping pixel coordinates to RGB values, and are trained for 10000 steps on $512 \times 512$ images.
        For denoising, Gaussian noise with $\sigma = 0.1$ is added to the target image.
        For inpainting, a binary mask retains 20\% of the pixel locations.
        For super-resolution, the network is supervised on a $128 \times 128$ low-resolution grid and evaluated at $512 \times 512$.
        Configurations are reported in \cref{tab:config_image}.
        \begin{table}[ht]
            \centering
            \caption{\textbf{Configurations for Image Fitting, Image Denoising, Image Inpainting, and Image Super-resolution.}}
            \label{tab:config_image}
            \resizebox{0.9\linewidth}{!}{
            \begin{tabular}{lcccl}
                \toprule
                \textbf{INR} & $m$ & $L$ & $\eta$ & Task-specific parameters \\
                \midrule
                SIREN~\cite{sitzmann2019siren}       & 256 & 5 & $10^{-4}$ & $\omega_0 = 30$ \\
                Gauss~\cite{ramasinghe2022beyond}    & 256 & 5 & $10^{-4}$ & $\sigma = 0.1$ \\
                WIRE~\cite{saragadam2023wire}        & 256 & 2 & $5 \times 10^{-3}$ & $\omega_0 = 20$, $s = 30$ ($\omega_0 = 4$, $s = 4$ for denoising) \\
                BACON~\cite{lindell2021bacon}        & 256 & 8 & $5 \times 10^{-4}$ & freq.\ limit $= 2048$, quantised \\
                FINER~\cite{liu2024finer}            & 256 & 3 & $10^{-4}$ & $\omega_0 = 30$ (first and hidden) \\
                MFN~\cite{fathony2021multiplicative} & 256 & 3 & $10^{-2}$ & weight scale $= 16$, $\alpha = 2.5$ \\
                Fourier~\cite{tancik2020fourfeat}    & 256 & 4 & $10^{-3}$ & scale $= 10$ \\
                FR~\cite{shi2024improved}            & 256 & 3 & $10^{-4}$ & high-freq.\ components $= 128$ \\
                \bottomrule
            \end{tabular}}
        \end{table}
    
    \paragraph{CT Reconstruction.}
        All methods use a 2D input ($d_{\text{in}} = 2$) and are trained for 50000 steps.
        The reconstruction target is a $512 \times 512$ image observed through 100 projection angles.
        Configurations are reported in \cref{tab:config_ct}.
        \begin{table}[ht]
            \centering
            \caption{\textbf{Configurations for CT Reconstruction.}}
            \label{tab:config_ct}
            \resizebox{0.9\linewidth}{!}{
            \begin{tabular}{lcccl}
                \toprule
                \textbf{INR} & $m$ & $L$ & $\eta$ & Task-specific parameters \\
                \midrule
                SIREN~\cite{sitzmann2019siren}       & 256 & 5 & $10^{-4}$ & $\omega_0 = 30$ \\
                Gauss~\cite{ramasinghe2022beyond}    & 256 & 5 & $10^{-4}$ & $\sigma = 0.1$ \\
                WIRE~\cite{saragadam2023wire}        & 256 & 2 & $5 \times 10^{-3}$ & $\omega_0 = 20$, $s = 30$ \\
                BACON~\cite{lindell2021bacon}        & 32  & 4 & $5 \times 10^{-4}$ & freq.\ limit $= 2048$, quantised \\
                FINER~\cite{liu2024finer}            & 256 & 3 & $10^{-4}$ & $\omega_0 = 30$ (first and hidden) \\
                MFN~\cite{fathony2021multiplicative} & 256 & 3 & $10^{-2}$ & weight scale $= 16$, $\alpha = 2.5$ \\
                Fourier~\cite{tancik2020fourfeat}    & 256 & 4 & $10^{-3}$ & scale $= 10$ \\
                FR~\cite{shi2024improved}            & 128 & 3 & $10^{-6}$ & high-freq.\ components $= 128$ \\
                \bottomrule
            \end{tabular}}
        \end{table}
    
    \paragraph{SDF Fitting.}
        All methods use a 3D input ($d_{\text{in}} = 3$) and are trained for 50000 steps with gradient accumulation over 2 steps.
        We sample 125000 on-surface points per batch and evaluate on a $512^3$ grid.
        Configurations are reported in \cref{tab:config_sdf}.
        \begin{table}[ht]
            \centering
            \caption{\textbf{Configurations for SDF Fitting.}}
            \label{tab:config_sdf}
            \resizebox{0.9\linewidth}{!}{
            \begin{tabular}{lcccl}
                \toprule
                \textbf{INR} & $m$ & $L$ & $\eta$ & Task-specific parameters \\
                \midrule
                SIREN~\cite{sitzmann2019siren}       & 256 & 5 & $10^{-4}$ & $\omega_0 = 30$ \\
                Gauss~\cite{ramasinghe2022beyond}    & 256 & 5 & $10^{-4}$ & $\sigma = 0.1$ \\
                WIRE~\cite{saragadam2023wire}        & 256 & 2 & $5 \times 10^{-3}$ & $\omega_0 = 10$, $s = 40$ \\
                BACON~\cite{lindell2021bacon}        & 128 & 6 & $10^{-4}$ & freq.\ limit $= 2048$, quantised \\
                FINER~\cite{liu2024finer}            & 256 & 3 & $10^{-4}$ & $\omega_0 = 30$ (first and hidden) \\
                MFN~\cite{fathony2021multiplicative} & 256 & 3 & $10^{-2}$ & weight scale $= 16$, $\alpha = 2.5$ \\
                Fourier~\cite{tancik2020fourfeat}    & 256 & 4 & $10^{-3}$ & scale $= 10$ \\
                FR~\cite{shi2024improved}            & 256 & 2 & $5 \times 10^{-3}$ & high-freq.\ components $= 256$ \\
                \bottomrule
            \end{tabular}}
        \end{table}
    
    \paragraph{Poisson Image Reconstruction.}
        All methods use a 2D input ($d_{\text{in}} = 2$) and are trained for 50000 steps on $512 \times 512$ images.
        Supervision is applied to the network's Laplacian rather than to pixel values directly, using the top-left ground-truth pixel as an anchor.
        Configurations are reported in \cref{tab:config_poisson}.
        \begin{table}[ht]
            \centering
            \caption{\textbf{Configurations for Poisson Image Reconstruction.}}
            \label{tab:config_poisson}
            \resizebox{0.9\linewidth}{!}{
            \begin{tabular}{lcccl}
                \toprule
                \textbf{INR} & $m$ & $L$ & $\eta$ & Task-specific parameters \\
                \midrule
                SIREN~\cite{sitzmann2019siren}       & 256 & 5 & $10^{-4}$ & $\omega_0 = 30$ \\
                Gauss~\cite{ramasinghe2022beyond}    & 256 & 5 & $10^{-4}$ & $\sigma = 0.1$ \\
                WIRE~\cite{saragadam2023wire}        & 256 & 2 & $5 \times 10^{-3}$ & $\omega_0 = 20$, $s = 30$ \\
                BACON~\cite{lindell2021bacon}        & 128 & 8 & $5 \times 10^{-4}$ & freq.\ limit $= 2048$, quantised \\
                FINER~\cite{liu2024finer}            & 256 & 3 & $10^{-4}$ & $\omega_0 = 30$ (first and hidden) \\
                MFN~\cite{fathony2021multiplicative} & 256 & 3 & $10^{-2}$ & weight scale $= 16$, $\alpha = 2.5$ \\
                Fourier~\cite{tancik2020fourfeat}    & 256 & 4 & $10^{-3}$ & scale $= 10$ \\
                FR~\cite{shi2024improved}            & 256 & 3 & $10^{-4}$ & high-freq.\ components $= 128$ \\
                \bottomrule
            \end{tabular}}
        \end{table}

    \section{Complementary Analysis}
\label{sec:complementary_analysis}

    \subsection{Adaptive SIREN vs. \algoname{}}
        \label{sec:siren_vs_oscillator}

        As anticipated in \cref{sec:introduction}, one might consider achieving the \algoname{}'s adaptivity by simply making the amplitude $a$, frequency $\omega$, and phase $\varphi$ of a sinusoidal activation $\sigma(z; \theta) = a \sin(\omega z + \varphi)$ independently learnable, effectively extending SIREN with a per-layer gain parameter.
        We test this directly by initialising $a_0 = 1.0$, $\omega_0 = 30.0$, and $\varphi_0 = 0.0$, making the initial activation identical to standard SIREN, and comparing against our \algoname{} under fixed architecture and training protocol on a controlled toy signal: the sum of four pure tones at $\omega_1 = 1.0$, $\omega_2 = 10.0$, $\omega_3 = 50.0$, and $\omega_4 = 100.0$.
        \algoname{} achieves a final PSNR of $148.98 \pm 0.59$ dB and a peak of $155.20 \pm 2.08$ dB, while the adaptive SIREN reaches a comparable peak $154.74 \pm 1.34$ dB, but a substantially lower final PSNR $72.26 \pm 11.79$ dB.
        The nearly identical peaks confirm that both formulations have sufficient capacity to represent the target, as they both can adapt their frequency. 
        The 77 dB gap at convergence and the tenfold increase in variance reveal that decoupled parameters lack the stabilising mechanism provided by the \algoname{}'s physical coupling.
        The training dynamics in \cref{fig:training_dynamics} illustrate this concretely: the adaptive SIREN's unconstrained amplitude and frequency drift after reaching good solutions, whereas the \algoname{}'s parameters settle into the equilibrium predicted by \cref{cor:gating}(a). 
        \begin{figure}
    \centering
        \includegraphics[width=\linewidth]{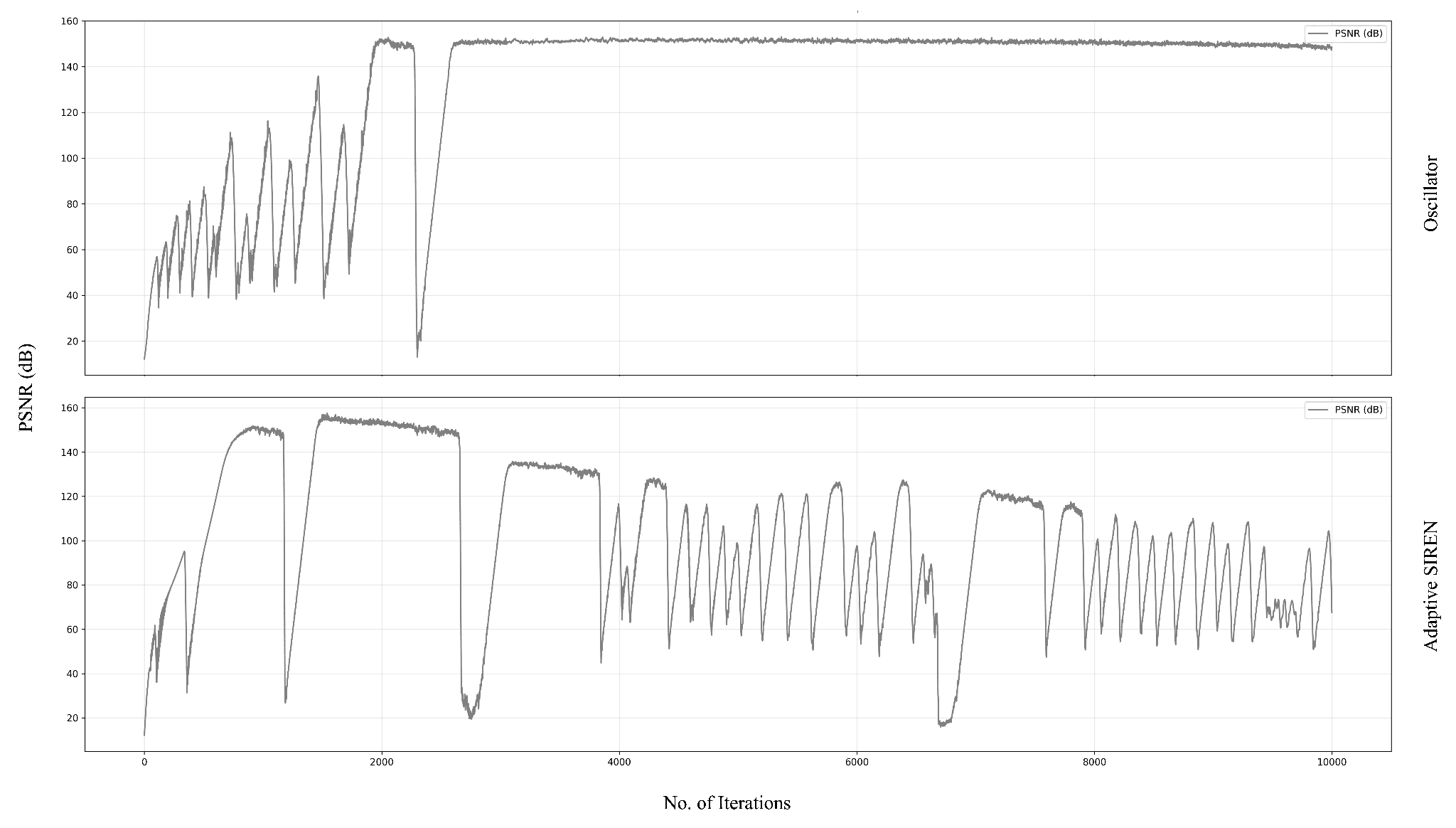}
    \caption{\textbf{Training Dynamics.}
    }
    \label{fig:training_dynamics}
\end{figure}

        \Cref{fig:oscillator_evolution} and \cref{fig:sirenator_evolution} compare the parameter evolution across layers.
        The oscillator's parameters converge smoothly to stable equilibria within the first few thousand steps: frequencies settle, $\xi$ decreases monotonically and plateaus, and phases lock, consistent with the equilibrium condition of \cref{cor:gating}(a).
        The adaptive SIREN exhibits qualitatively different behaviour: all three parameters oscillate erratically throughout training with no sign of convergence, even in layers that have already contributed to a high peak PSNR.
        This instability arises because the decoupled parametrisation admits degenerate solutions: many combinations of $(a, \omega, \varphi)$ yield similar outputs, creating flat directions in the loss landscape along which gradient descent drifts indefinitely.
        The oscillator's transfer function eliminates this degeneracy by coupling amplitude to frequency through $\xi$, yielding a unique equilibrium for each frequency component and the stable convergence observed in practice.
        \begin{figure}[t]
    \centering
    \begin{minipage}[t]{0.48\linewidth}
        \centering
        \includegraphics[width=\linewidth]{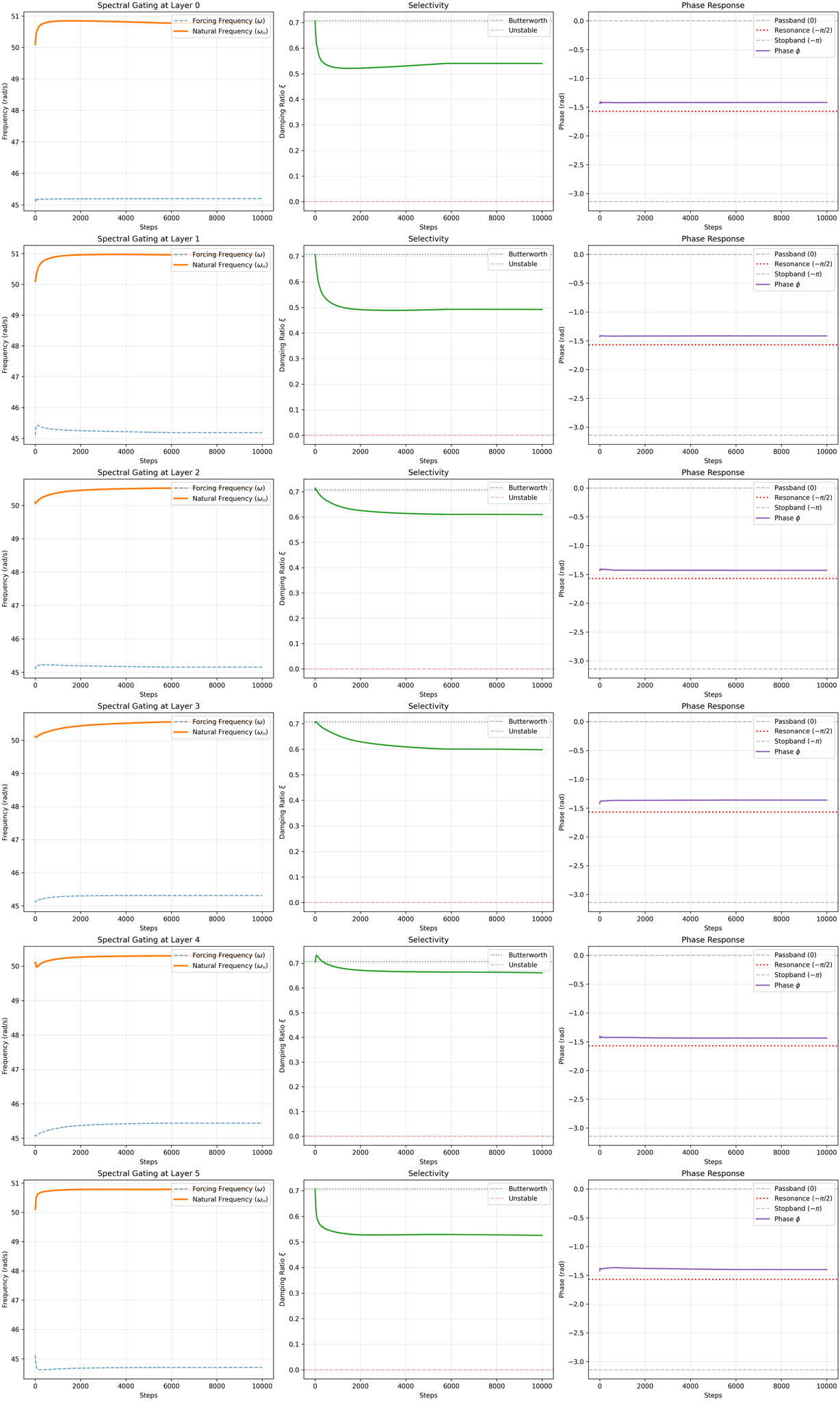}
        \caption{\textbf{\algoname{}'s parameters evolution}}
        \label{fig:oscillator_evolution}
    \end{minipage}
    \hfill
    \begin{minipage}[t]{0.48\linewidth}
        \centering
        \includegraphics[width=\linewidth]{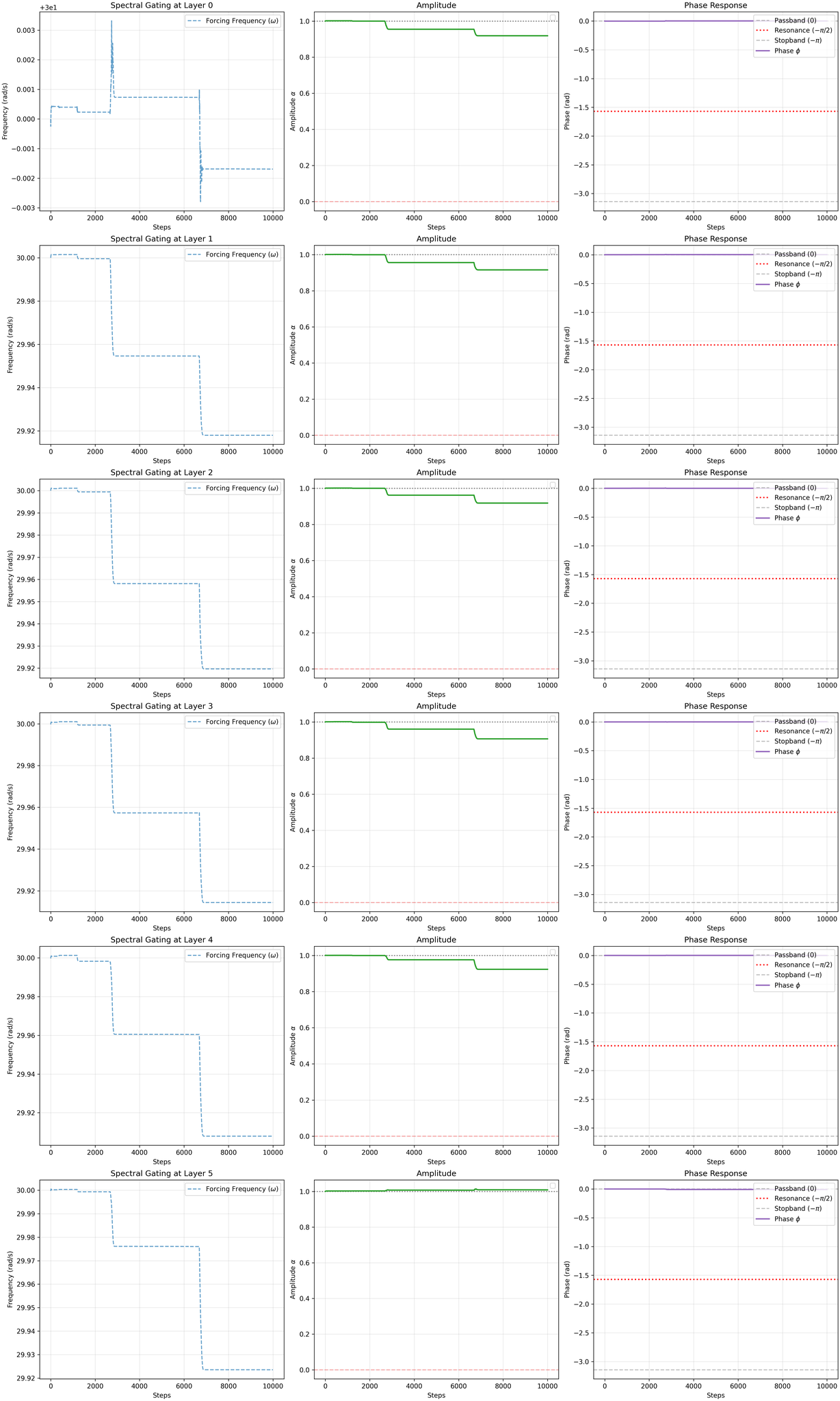}
        \caption{\textbf{Adaptive SIREN's parameters evolution}}
        \label{fig:sirenator_evolution}
    \end{minipage}
\end{figure}

    \subsection{Control Group Study}
        We run \algoname{}, Adaptive SIREN and SIREN using the same architecture ($m = 256$, $L = 5$) and optimiser (Adam with $\eta = 10^{-4}$ for both weights and activation parameters, no scheduler) for the Image Fitting on \textbf{Tiger} and for CT Reconstruction, reporting the results in \cref{tab:ctrl_study}. 
        \algoname{} yields the best results, proving that the improvements come from the formulation of the activation function itself rather than the training dynamics
        \begin{table}[ht]
            \centering
            \caption{
            \textbf{Control Group Study.}
            }
            \label{tab:ctrl_study}
            \resizebox{0.9\linewidth}{!}{
            \begin{tabular}{l cccc}
    
                \toprule
                
                \multirow{3}{*}{\textbf{{INR}}} & 
                \multicolumn{2}{c}{\textbf{Image Fitting (Tiger)}} & 
                \textbf{CT Reconstruction} \\
    
                \cmidrule(lr){2-3} \cmidrule(lr){4-4}
                
                & Final PSNR & Peak PSNR & PSNR \\
    
                \cmidrule(lr){2-3} \cmidrule(lr){4-4}
                
                FDHO & \textbf{63.27 $\pm$ 1.30} & \textbf{63.27 $\pm$ 1.30} & \textbf{40.16 $\pm$ 0.46} \\
                Adaptive SIREN & 54.65 $\pm$ 3.97 & 58.21 $\pm$ 0.17 & 32.44 $\pm$ 1.00 \\
                SIREN & 53.37 $\pm$ 1.52 & 58.19 $\pm$ 0.51 & 33.28 $\pm$ 1.06 \\
                
                \bottomrule
                
            \end{tabular}}
        \end{table}

    \section{Additional Experiments}
\label{sec:additonal_experiments}

    \subsection{Computational Cost Comparison}
    \label{sec:computational_cost}
        We report the computational overhead of \algoname{} relative to the INRs in \cref{tab:cost}, averaged across the five images used in the image fitting task ($512 \times 512$, 10000 steps, on a single NVIDIA GeForce RTX 4090).
        
        \begin{table}[ht]
            \centering
            \caption{
            \textbf{Computational cost comparison on Image Fitting.}
            Averaged over five images and 10 runs. 
            Training time and memory are measured end-to-end, including optimiser state.}
            \label{tab:cost}
            \resizebox{\linewidth}{!}{
            \begin{tabular}{lcccc}
                \toprule
                \textbf{INR} & No. Params & Model Size (MB) & Train Time (s) & Train Memory (MB) \\
                \cmidrule(lr){2-5}
                \algoname{}                          & 330523 & 1.26 & 191 & 2908 \\
                SIREN~\cite{sitzmann2019siren}       & 330499 & 1.26 & 134 & 1205 \\
                Gauss~\cite{ramasinghe2022beyond}    & 330499 & 1.26 & 196 & 1285 \\
                WIRE~\cite{saragadam2023wire}        & 66979  & 0.51 & 212 & 1265 \\
                BACON~\cite{lindell2021bacon}        & 538904 & 2.06 & 287 & 2816 \\
                FINER~\cite{liu2024finer}            & 198915 & 0.76 & 122 & 1202 \\
                MFN~\cite{fathony2021multiplicative} & 204291 & 0.78 & 257 & 2379 \\
                Fourier~\cite{tancik2020fourfeat}    & 395779 & 1.51 & 122 & 877 \\
                FR~\cite{shi2024improved}            & 12585219 & 48.01 & 126 & 790 \\
                \bottomrule
            \end{tabular}}
        \end{table}
        
        \algoname{} adds only 24 scalar parameters over the SIREN backbone (four oscillator parameters per layer across five hidden layers), bringing the total from 330499 to 330523, a negligible increase of 0.007\%.
        Model size and inference cost are therefore effectively identical to SIREN.
        
        The main overhead is in training.
        \algoname{} requires approximately 1.4$\times$ the wall-clock time of SIREN (191 s vs. 134 s), attributable to the two-group optimiser and the \texttt{ReduceLROnPlateau} scheduler rather than the activation itself.
        This overhead is comparable to Gaussian and WIRE, and substantially lower than BACON (287 s) and MFN (257 s).
        Training memory is approximately 2.4$\times$ that of SIREN (2908 MB vs. 1205 MB), which reflects the additional gradient computation for the shared oscillator parameters through all neurons in each layer.
        This cost is comparable to BACON (2816 MB) and remains well below FR, which, despite low training memory per step, requires 48 MB for the model alone due to its 12.6 M parameters.
        
        In summary, the computational cost of \algoname{} is moderate and falls within the range of existing methods.
        The activation itself introduces negligible parameter and inference overhead. 
        The training cost, while higher than the simplest baselines, is the price of the adaptive spectral gating mechanism, which, as the experimental results demonstrate, yields substantial quality gains across all tasks.

    \subsection{Additional Tasks}
        
        \paragraph{Signed Distance Function (SDF) Fitting.}
            We represent 3D shapes by learning a continuous signed distance function from spatial coordinates, following~\cite{park2019deepsdf,sitzmann2019siren}.
            SDFs require the network to simultaneously capture smooth surfaces and sharp geometric features such as edges and corners, testing the activation's ability to represent signals with localised high-frequency content in three dimensions. We follow the protocol described
            in~\cite{sitzmann2019siren, liu2024finer} and evaluate on four meshes of increasing geometric complexity. 
            We report Chamfer Distance (CD, lower is better), which measures the average nearest-neighbour distance between the predicted and ground-truth point clouds, and Normal Consistency (NC, higher is better), which quantifies the alignment of surface normals.
            Results are in \cref{tab:sdf_fitting} and qualitative comparisons in \cref{fig:qual_sdf_fitting}.
            Unlike other proposed tasks, SDF Fitting experiments were run only once due to the extended computation time required by a single run.
            
            \begin{table}[t]
    \centering
    \caption{
        \textbf{Signed Distance Function (SDF) Fitting.} 
        }
    \resizebox{\linewidth}{!}{
        \begin{tabular}{l cc cc cc cc}

            \toprule
            
            \multirow{3}{*}{\textbf{{INR}}} & 
            \multicolumn{2}{c}{\textbf{Armadillo}} & 
            \multicolumn{2}{c}{\textbf{Dragon}} & 
            \multicolumn{2}{c}{\textbf{Thai Statue}} &
            \multicolumn{2}{c}{\textbf{Lucy Statue}} \\ 

            \cmidrule(lr){2-3} \cmidrule(lr){4-5} \cmidrule(lr){6-7} \cmidrule(lr){8-9}

            &
            CD & NC &
            CD & NC &
            CD & NC &
            CD & NC \\
            
            \cmidrule(lr){2-3} \cmidrule(lr){4-5} \cmidrule(lr){6-7} \cmidrule(lr){8-9}
            
            \algoname{}                          & $0.038 \times 10^{-3}$ & 0.9915 & $0.024 \times 10^{-3}$ & 0.9688 & $0.049 \times 10^{-3}$ & 0.9877 & $0.025 \times 10^{-3}$ & 0.9637 \\
            SIREN \cite{sitzmann2019siren}       & $0.041 \times 10^{-3}$ & 0.9915 & $0.024 \times 10^{-3}$ & 0.9727 & $0.061 \times 10^{-3}$ & 0.9882 & $0.113 \times 10^{-3}$ & 0.9439 \\
            Gauss \cite{ramasinghe2022beyond}    & $0.039 \times 10^{-3}$ & 0.9922 & $0.025 \times 10^{-3}$ & 0.9729 & $0.181 \times 10^{-3}$ & 0.8468 & $0.098 \times 10^{-3}$ & 0.8490 \\
            WIRE \cite{saragadam2023wire}        & $0.042 \times 10^{-3}$ & 0.9911 & $0.025 \times 10^{-3}$ & 0.9724 & $0.169 \times 10^{-3}$ & 0.8694 & $0.079 \times 10^{-3}$ & 0.9021 \\
            BACON \cite{lindell2021bacon}        & $0.038 \times 10^{-3}$ & 0.9081 & $0.066 \times 10^{-3}$ & 0.4865 & $0.121 \times 10^{-3}$ & 0.9339 & $0.032 \times 10^{-3}$ & 0.9652 \\
            FINER \cite{liu2024finer}            & $0.038 \times 10^{-3}$ & 0.9920 & $0.024 \times 10^{-3}$ & 0.9767 & $0.046 \times 10^{-3}$ & 1.0000 & $0.027 \times 10^{-3}$ & 0.9820 \\
            MFN \cite{fathony2021multiplicative} & $0.038 \times 10^{-3}$ & 0.9910 & $0.025 \times 10^{-3}$ & 0.9673 & $0.298 \times 10^{-3}$ & 0.8936 & $0.182 \times 10^{-3}$ & 0.8564 \\
            Fourier \cite{tancik2020fourfeat}    & $0.049 \times 10^{-3}$ & 0.8503 & $0.039 \times 10^{-3}$ & 0.7891 & $0.286 \times 10^{-3}$ & 0.8035 & $0.131 \times 10^{-3}$ & 0.7650 \\
            FR \cite{shi2024improved}            & $0.087 \times 10^{-3}$ & 0.9395 & $0.104 \times 10^{-3}$ & 0.8770 & $0.316 \times 10^{-3}$ & 0.7661 & $0.176 \times 10^{-3}$ & 0.8304 \\
            
            \bottomrule
            
        \end{tabular}
        }
    \label{tab:sdf_fitting}
\end{table}

            \begin{figure}[t]
    \centering
    \resizebox{\linewidth}{!}{
    \begin{tabular}{c cccc c cccc}

        & \textbf{Armadillo} & \textbf{Dragon} & \textbf{Thai Statue} & \textbf{Lucy Statue} & 
        & \textbf{Armadillo} & \textbf{Dragon} & \textbf{Thai Statue} & \textbf{Lucy Statue} \\

        \rotatebox{90}{\tiny \hspace{0.5em} Ground-truth} &
        \includegraphics[width=0.16\linewidth]{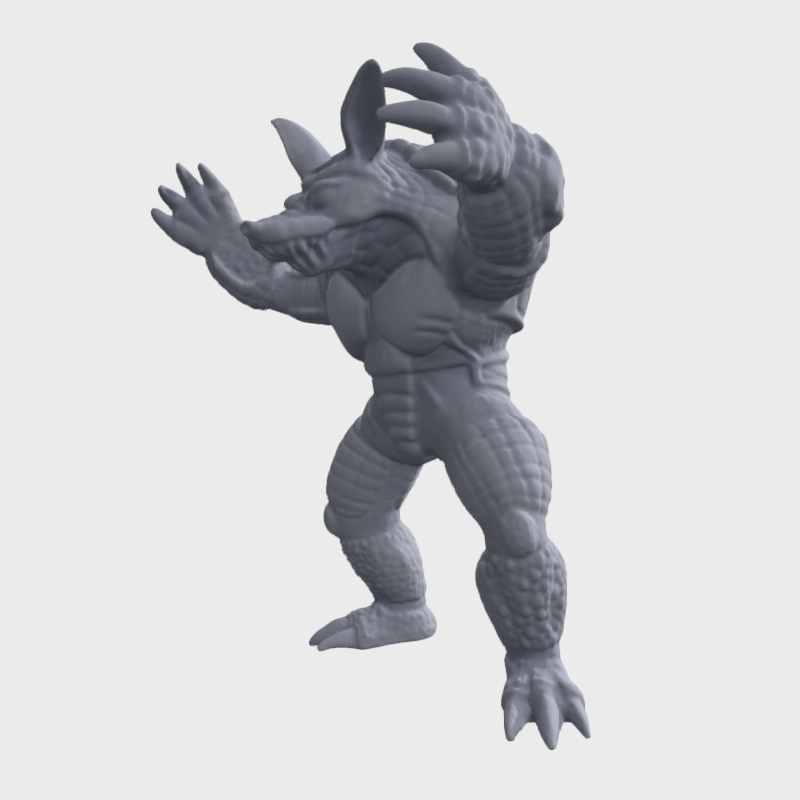} &
        \includegraphics[width=0.16\linewidth]{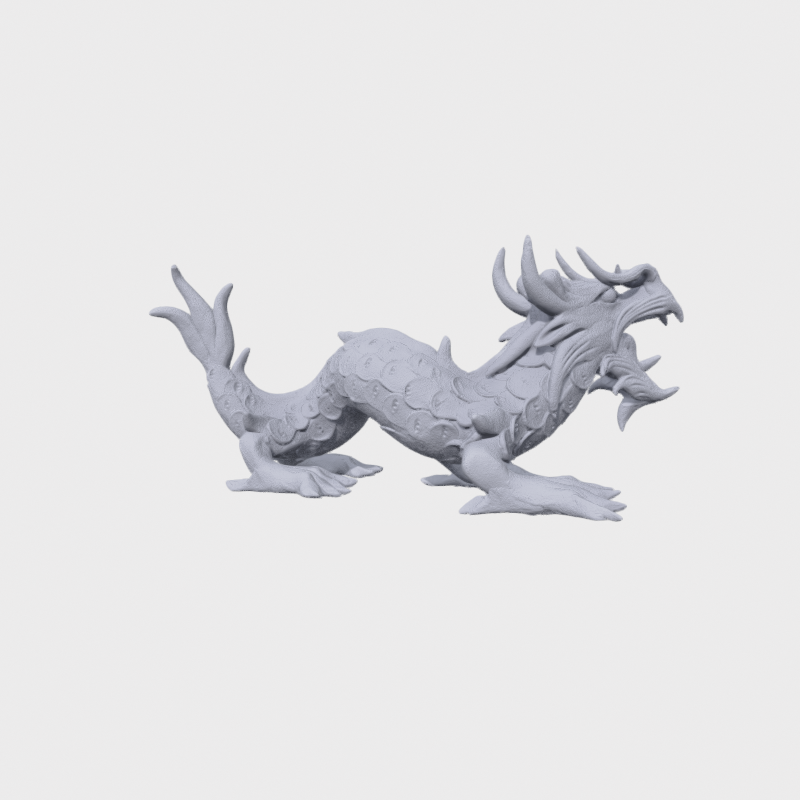} & 
        \includegraphics[width=0.16\linewidth]{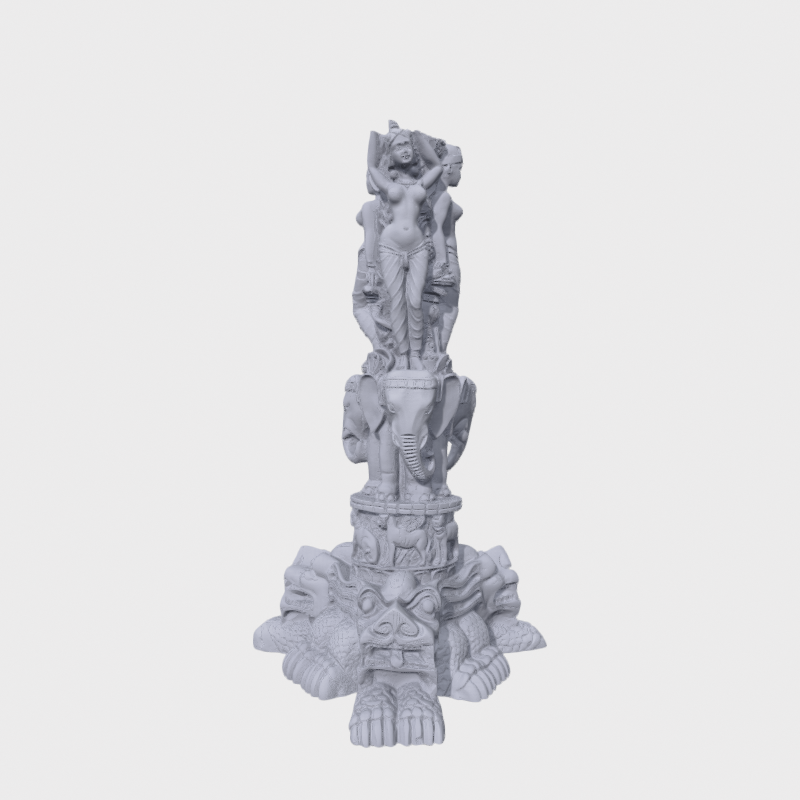} &
        \includegraphics[width=0.16\linewidth]{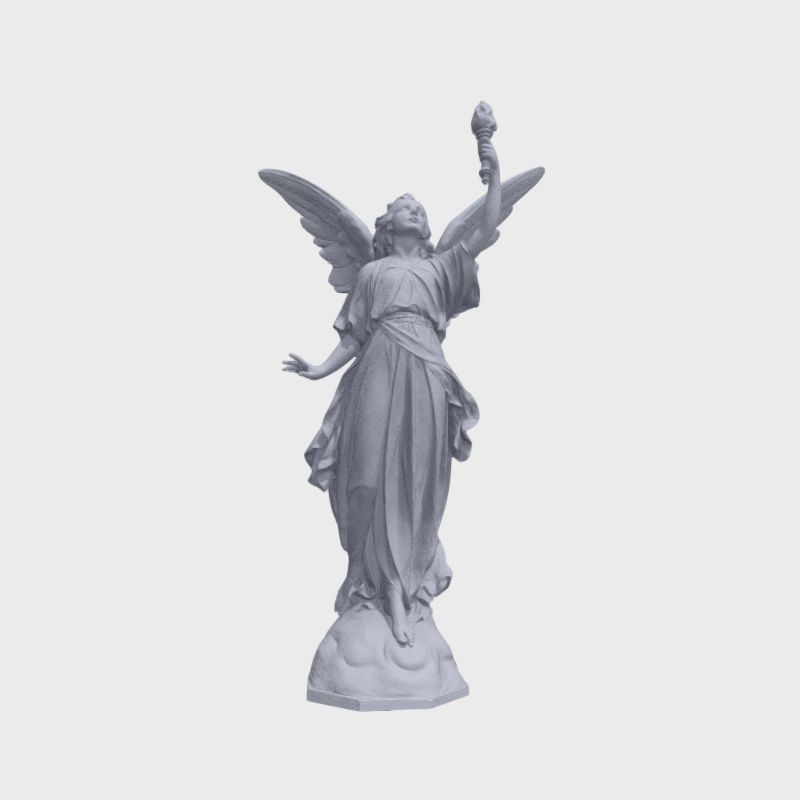} &

        \rotatebox{90}{\tiny \hspace{2em} \algoname{}} &
        \includegraphics[width=0.16\linewidth]{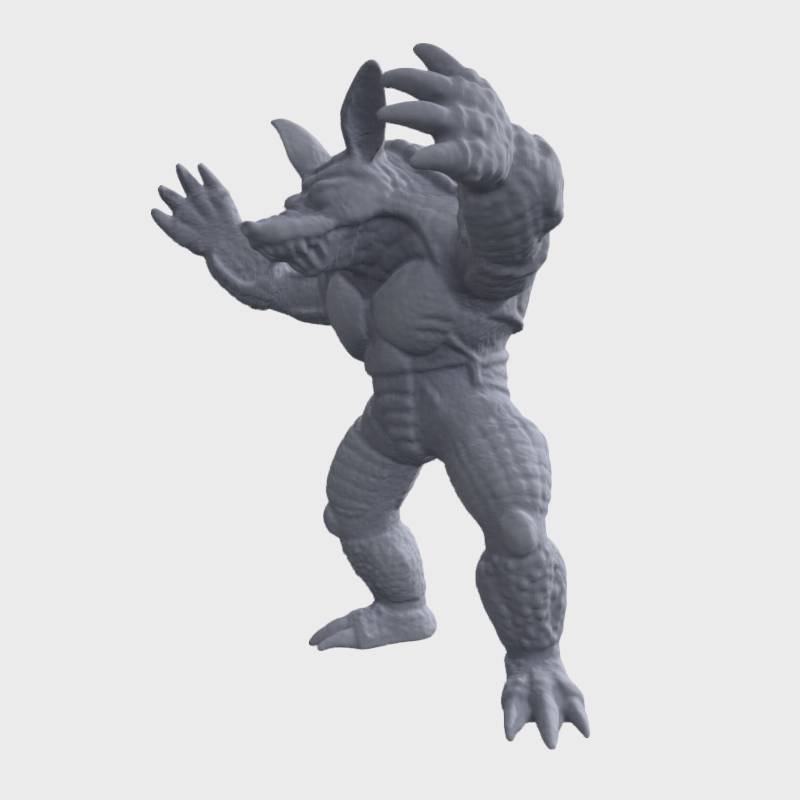} &
        \includegraphics[width=0.16\linewidth]{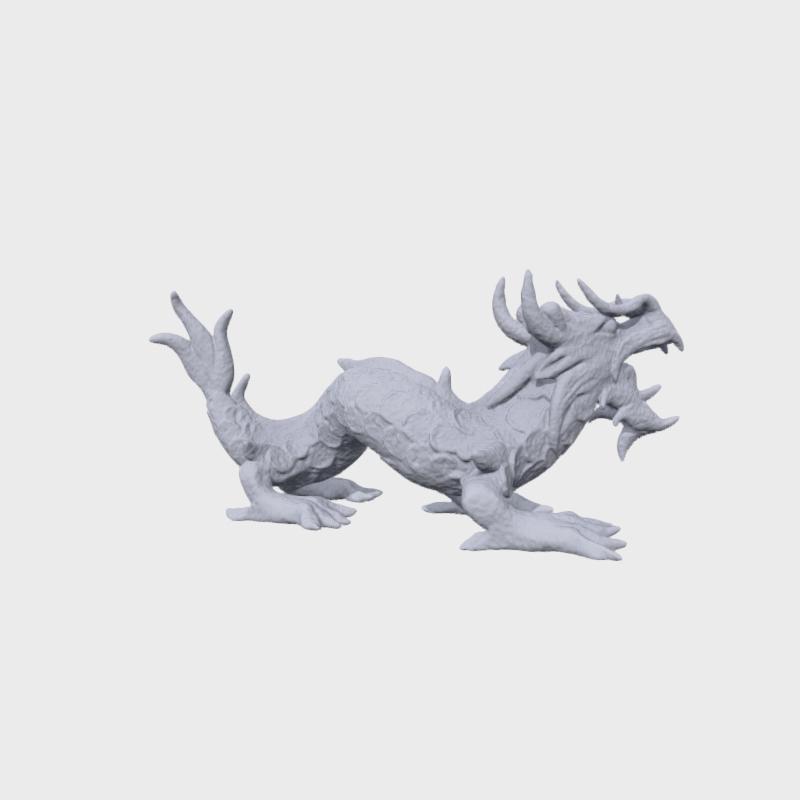} & 
        \includegraphics[width=0.16\linewidth]{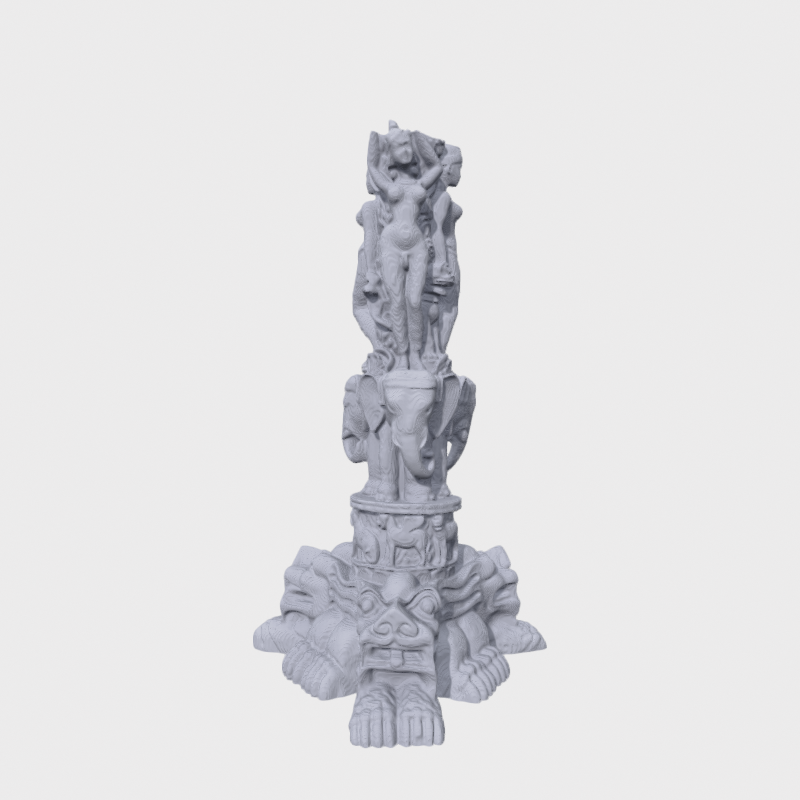} &
        \includegraphics[width=0.16\linewidth]{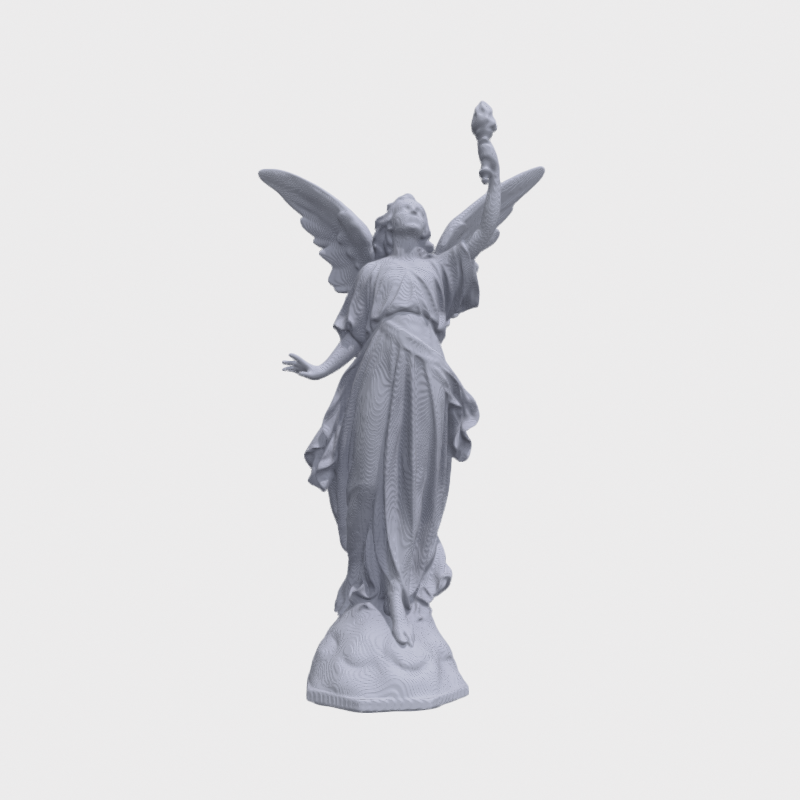} \\

        \rotatebox{90}{\tiny \hspace{2em} SIREN} &
        \includegraphics[width=0.16\linewidth]{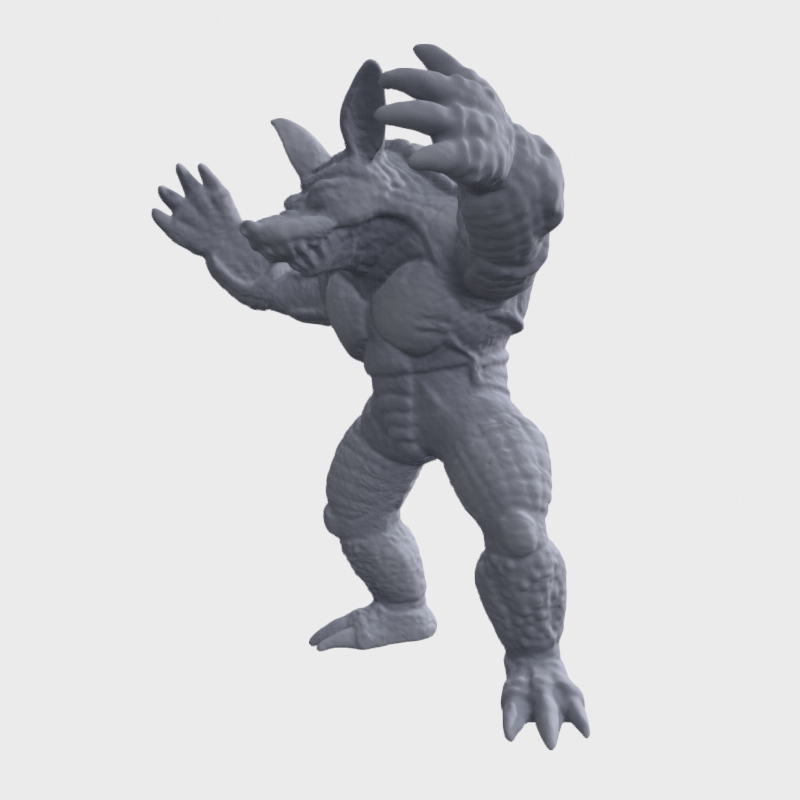} &
        \includegraphics[width=0.16\linewidth]{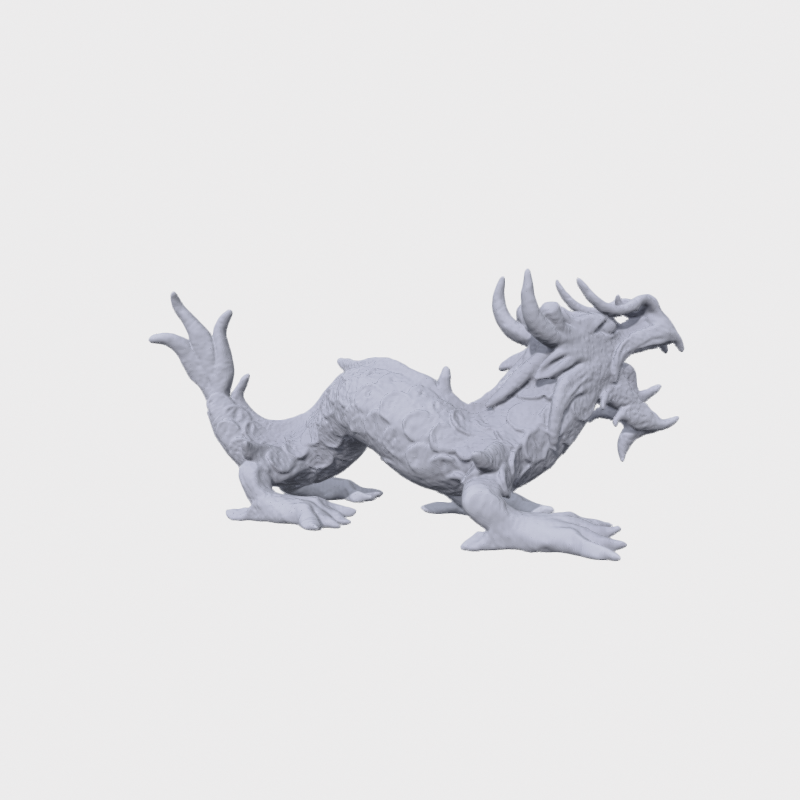} & 
        \includegraphics[width=0.16\linewidth]{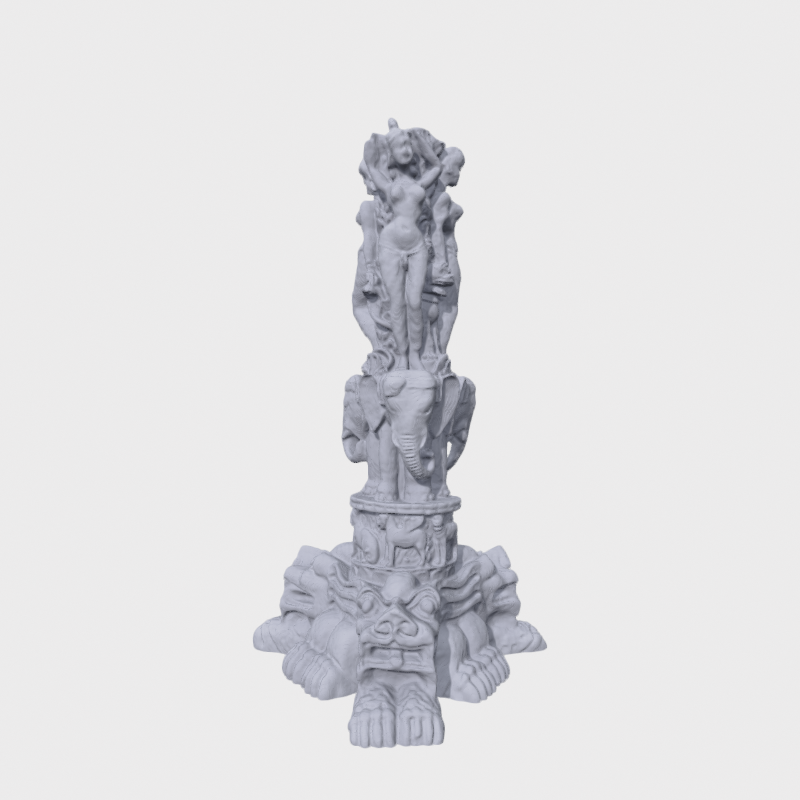} &
        \includegraphics[width=0.16\linewidth]{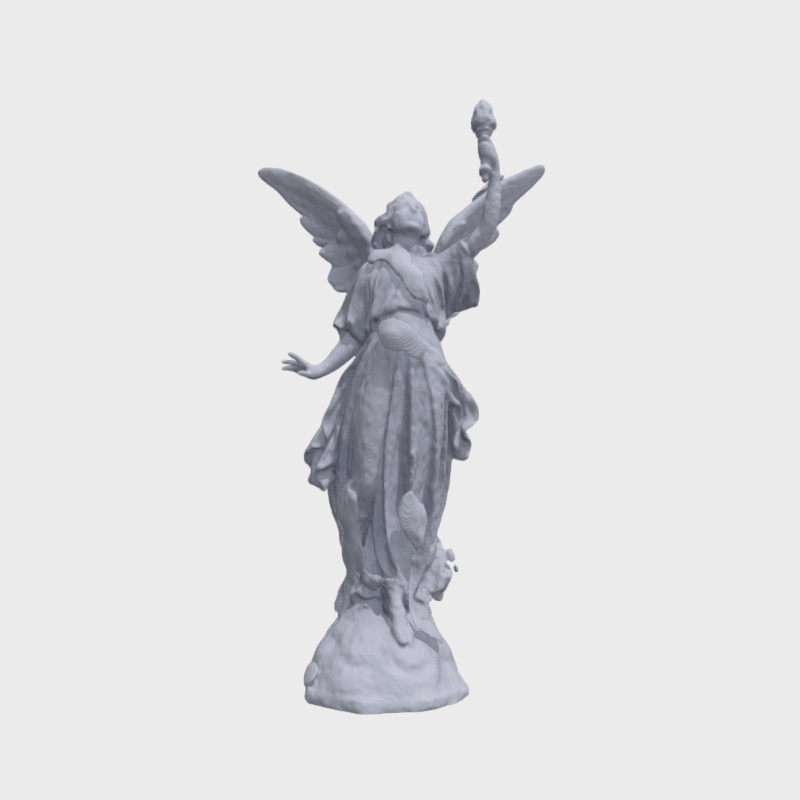} &

        \rotatebox{90}{\tiny \hspace{2em} Gauss} &
        \includegraphics[width=0.16\linewidth]{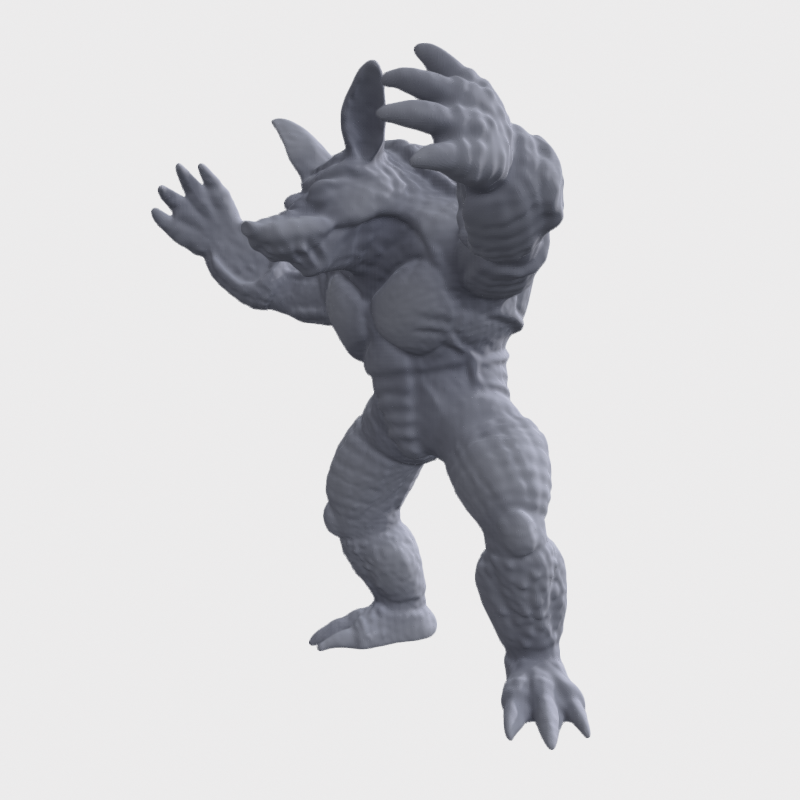} &
        \includegraphics[width=0.16\linewidth]{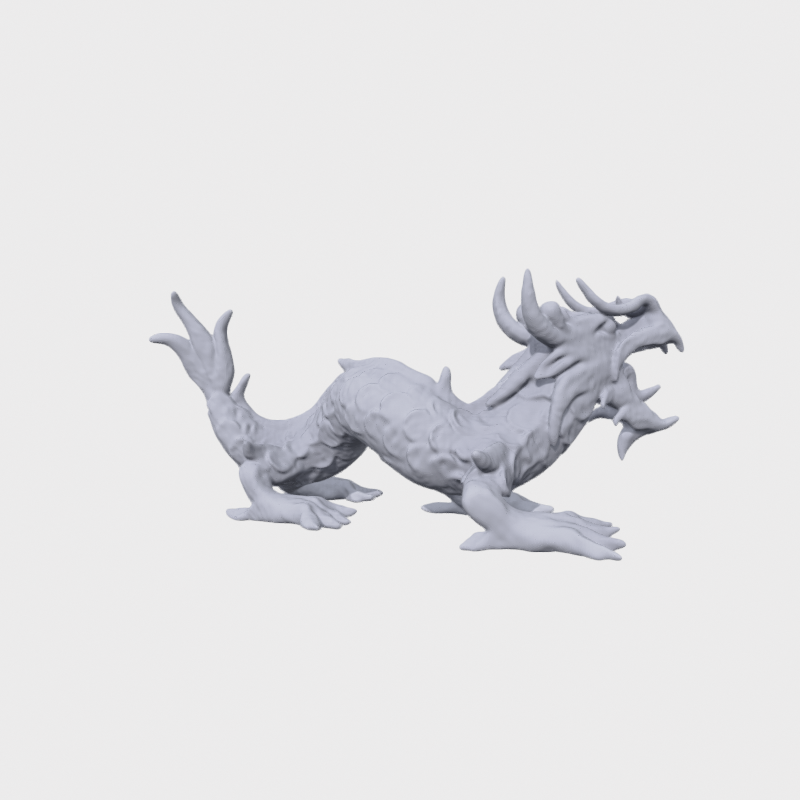} &
        \includegraphics[width=0.16\linewidth]{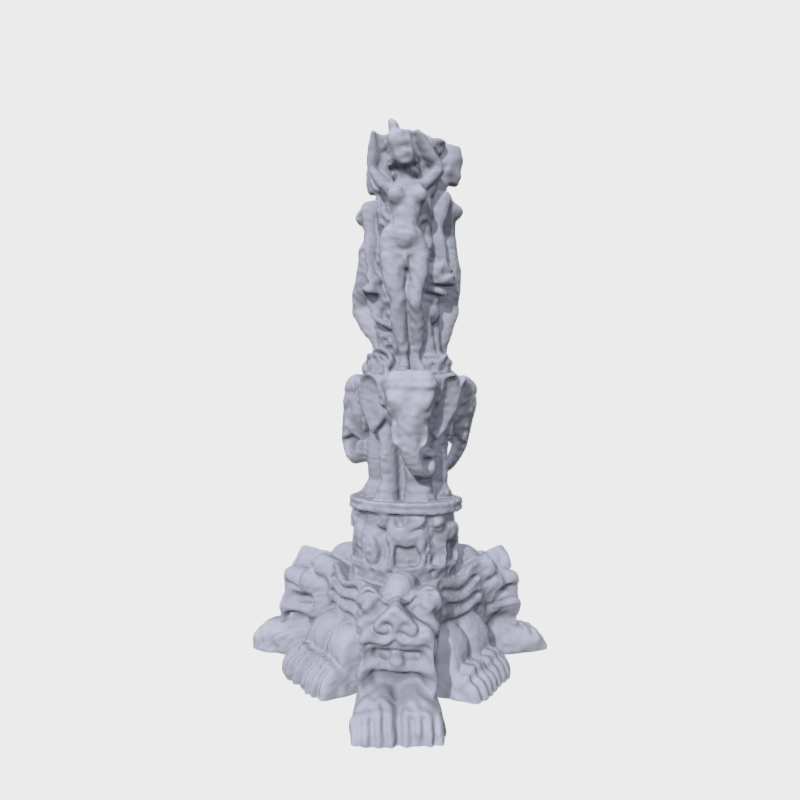} &
        \includegraphics[width=0.16\linewidth]{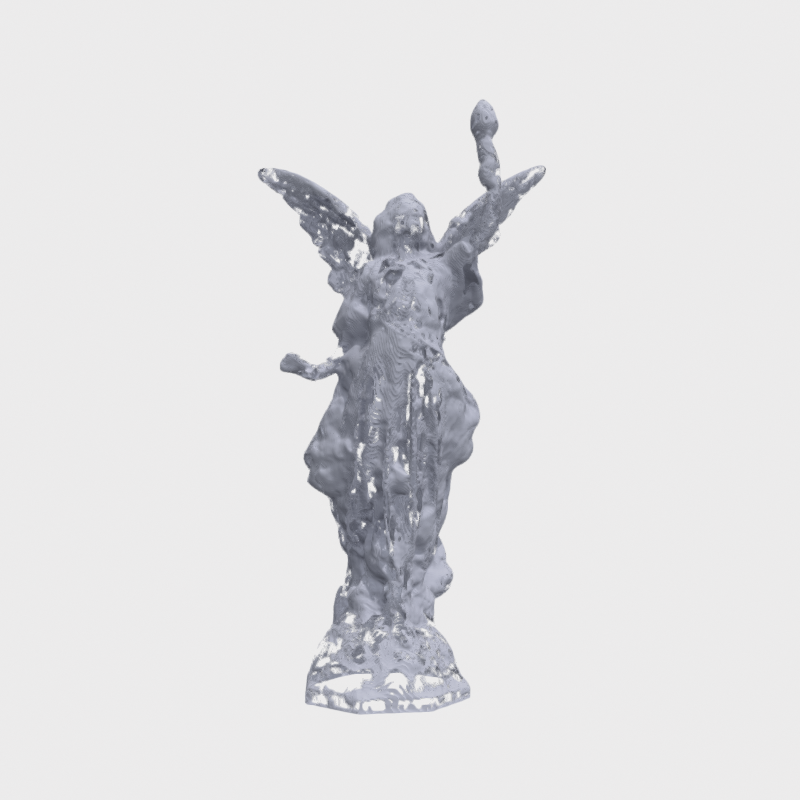} \\

        \rotatebox{90}{\tiny \hspace{2em} WIRE} &
        \includegraphics[width=0.16\linewidth]{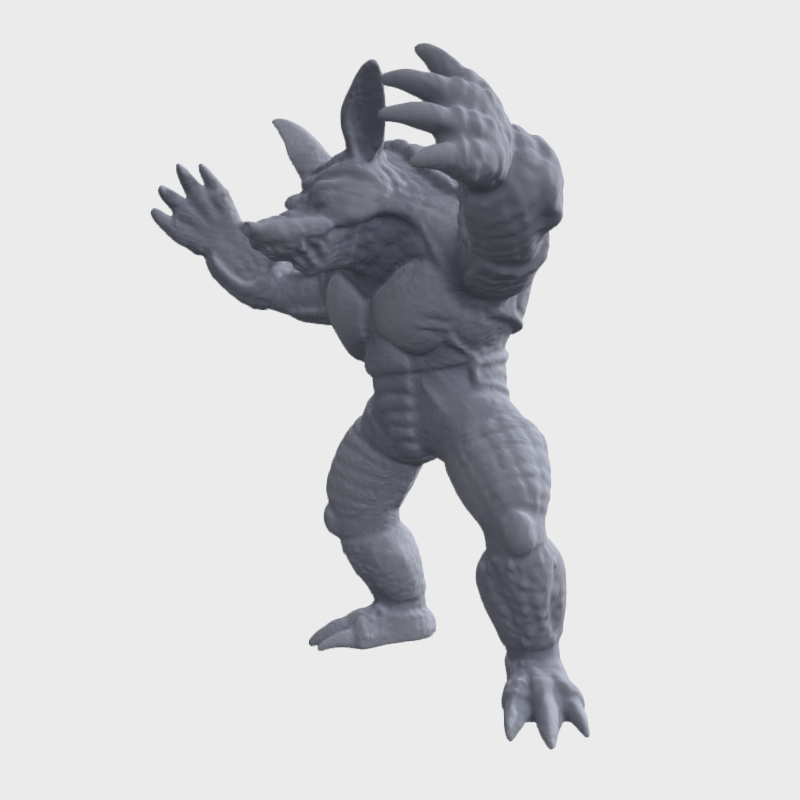} &
        \includegraphics[width=0.16\linewidth]{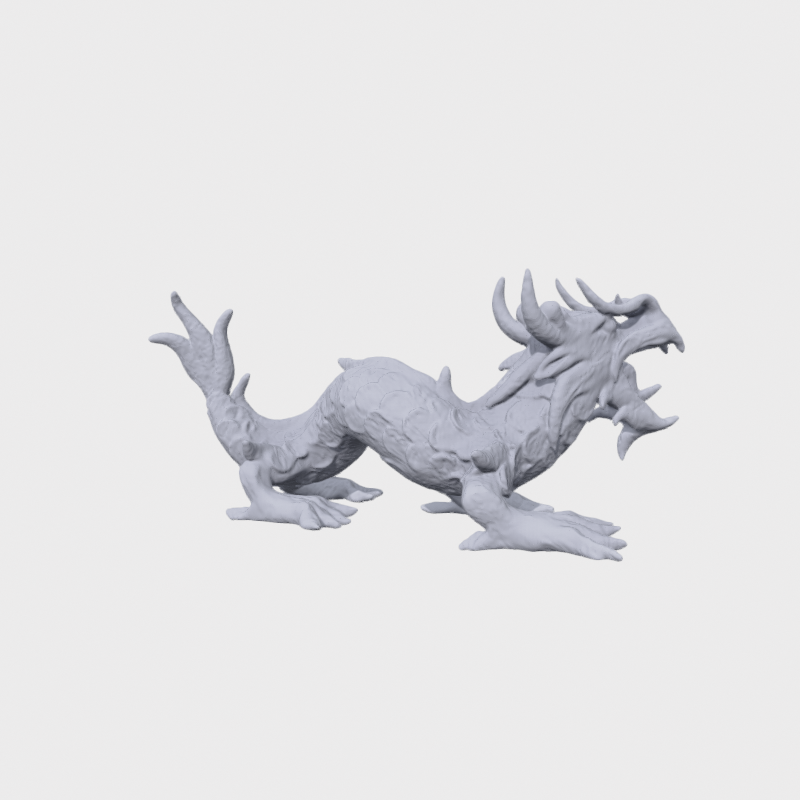} &
        \includegraphics[width=0.16\linewidth]{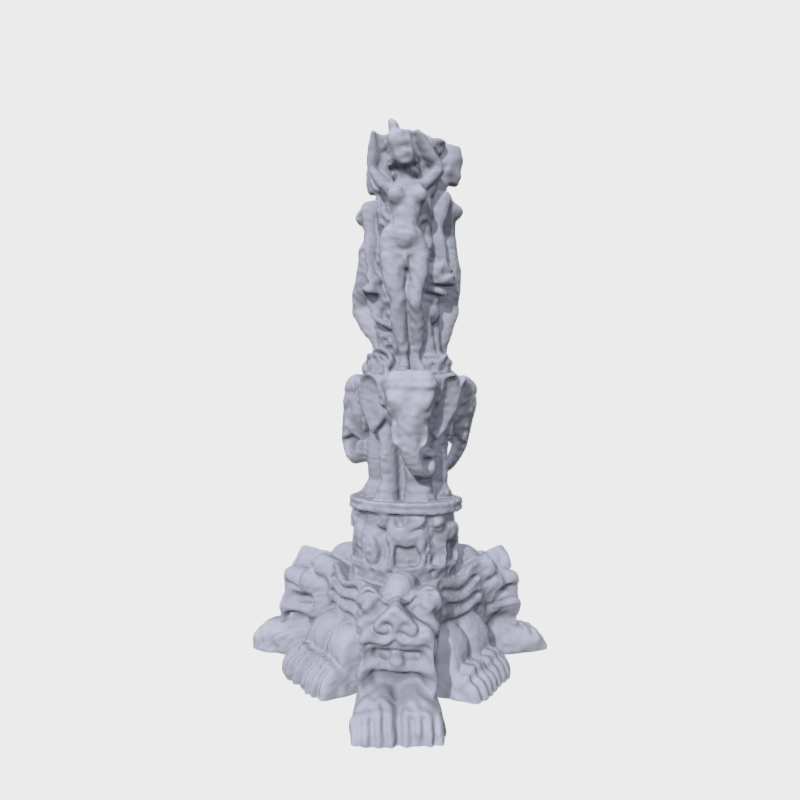} &
        \includegraphics[width=0.16\linewidth]{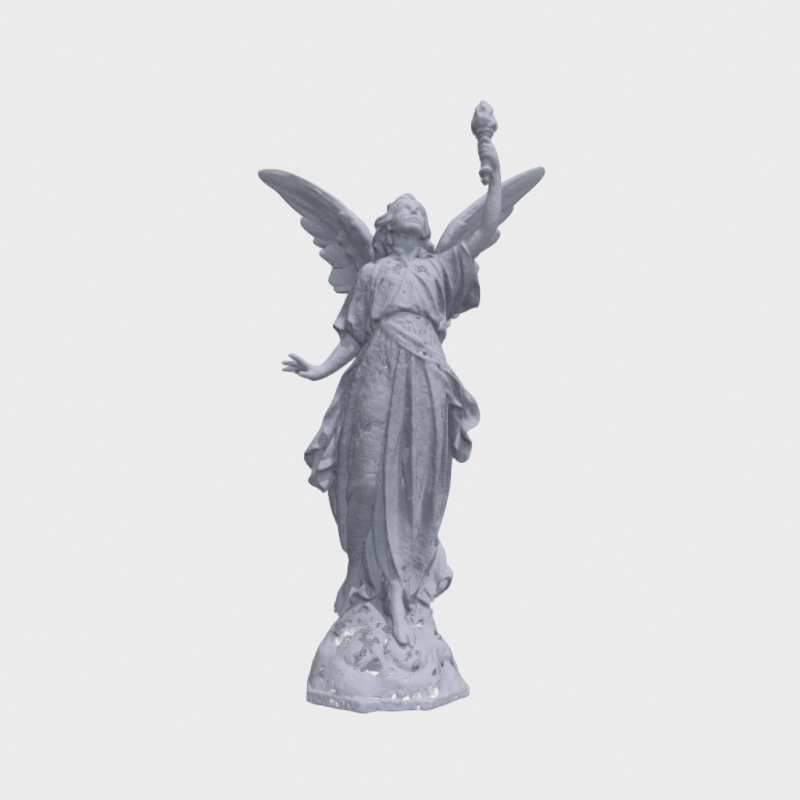} &

        \rotatebox{90}{\tiny \hspace{1.5em} BACON} &
        \includegraphics[width=0.16\linewidth]{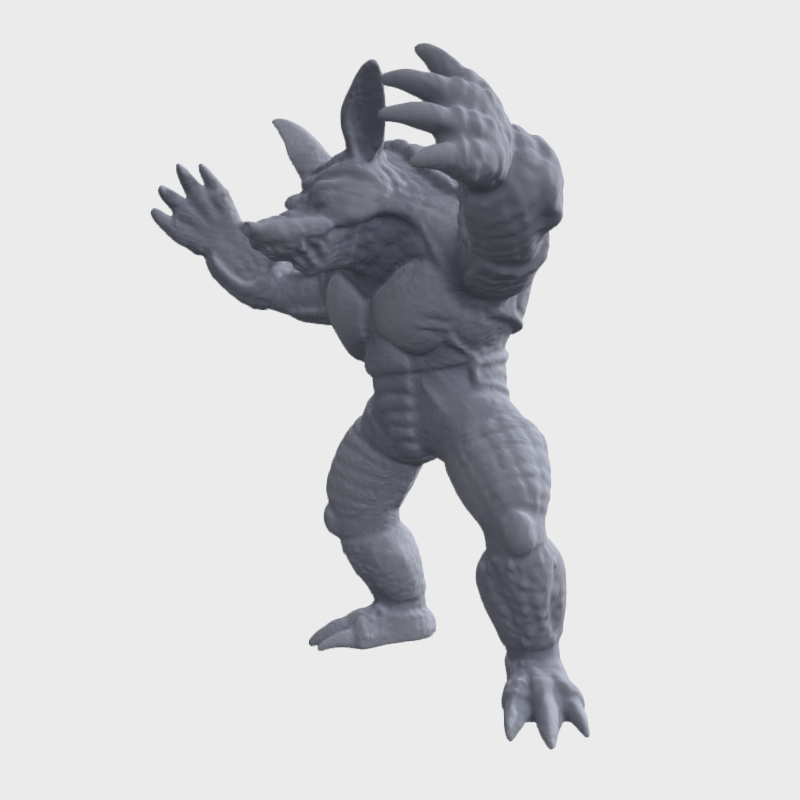} &
        \includegraphics[width=0.16\linewidth]{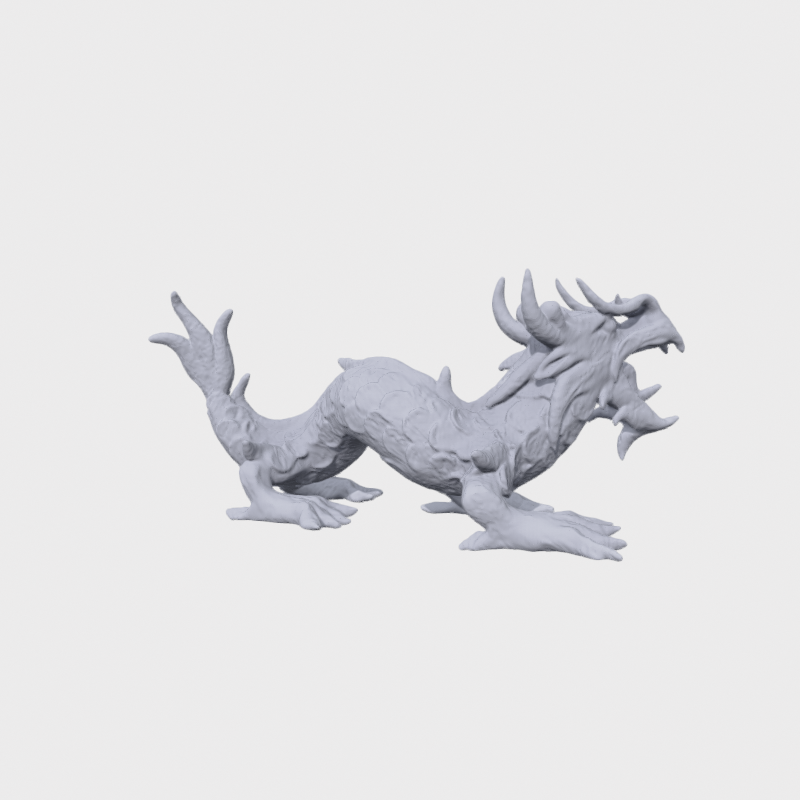} &
        \includegraphics[width=0.16\linewidth]{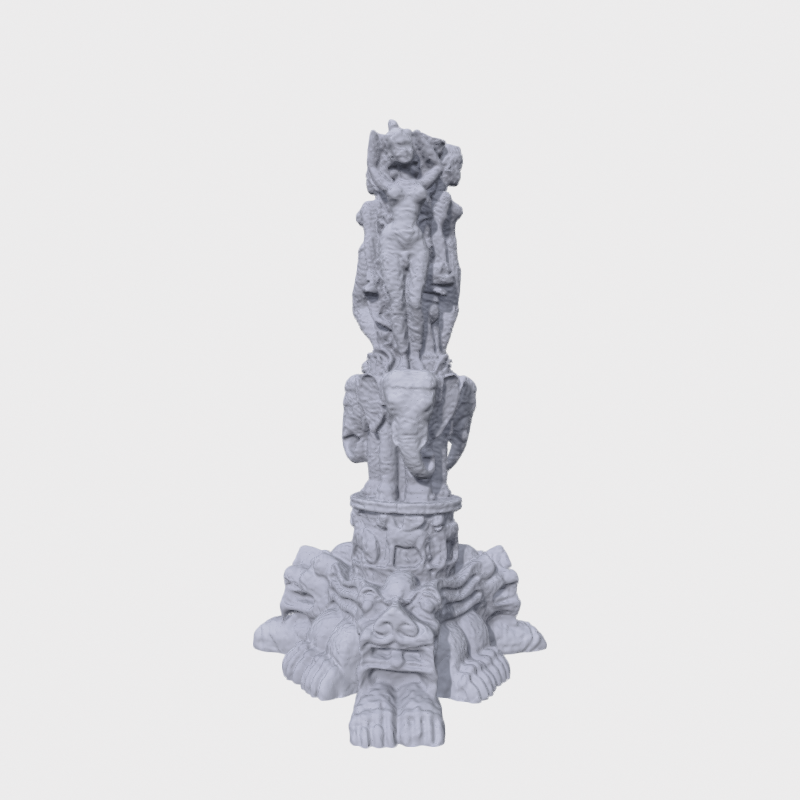}&
        \includegraphics[width=0.16\linewidth]{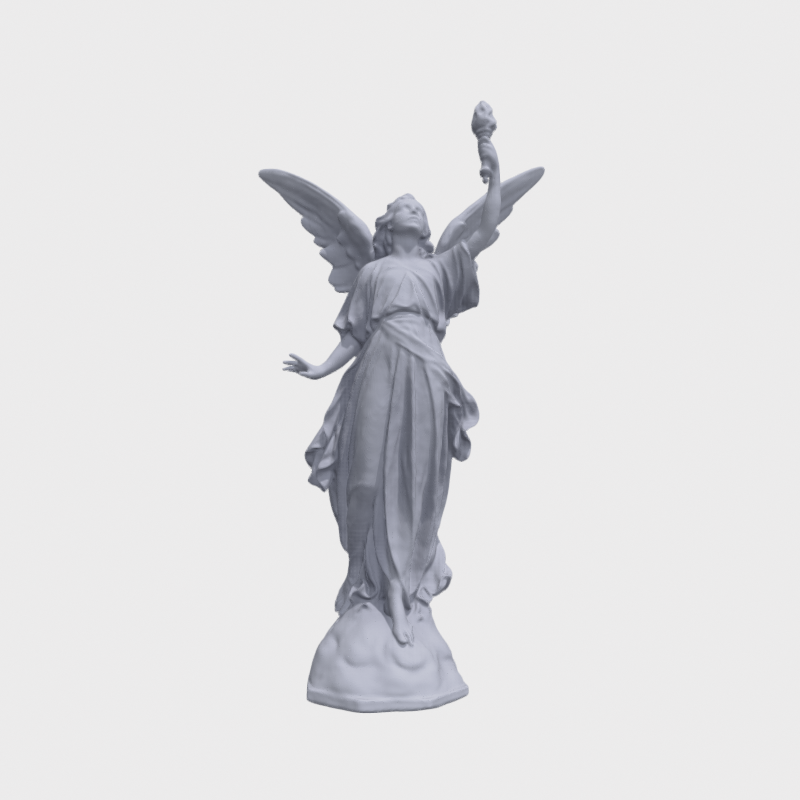} \\

        \rotatebox{90}{\tiny \hspace{2em} FINER} &
        \includegraphics[width=0.16\linewidth]{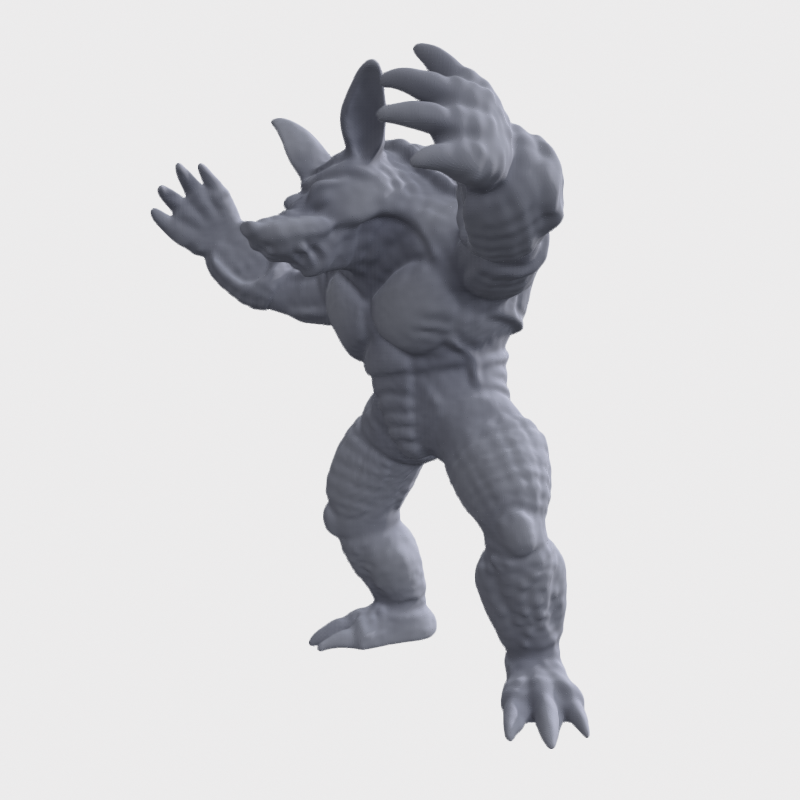} &
        \includegraphics[width=0.16\linewidth]{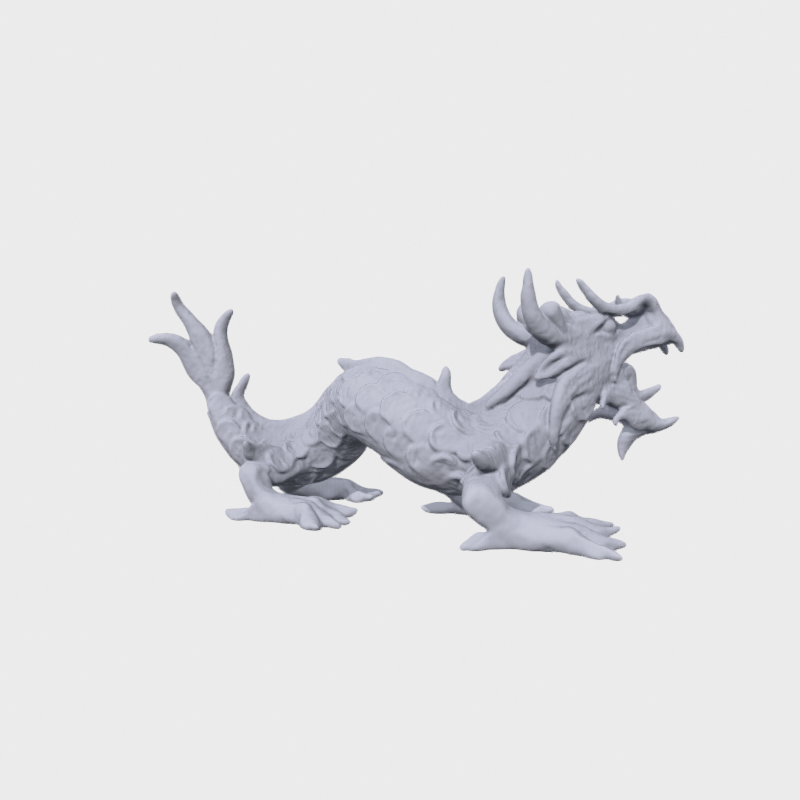} &
        \includegraphics[width=0.16\linewidth]{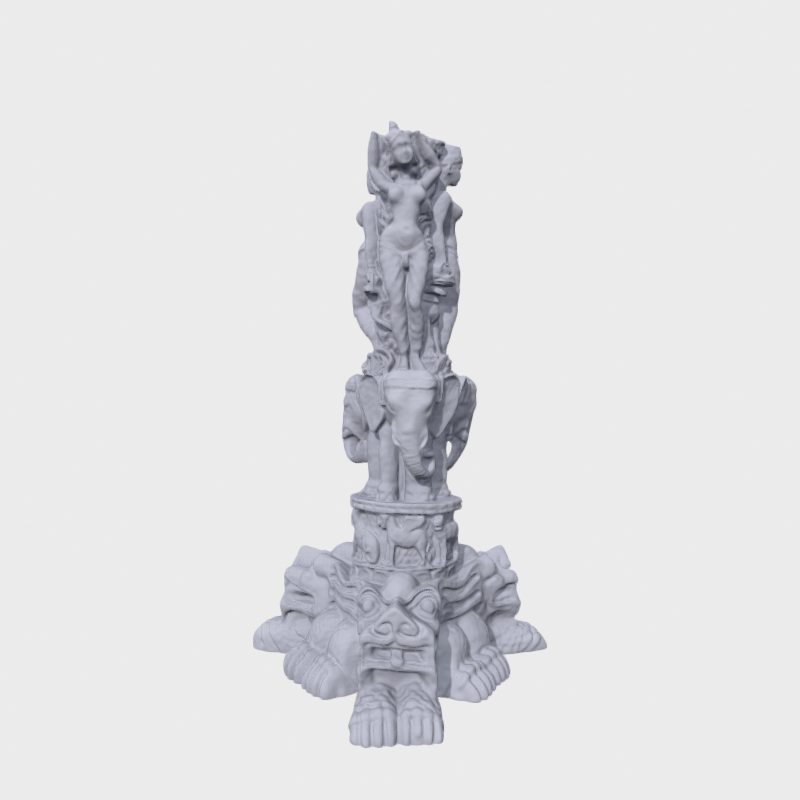} &
        \includegraphics[width=0.16\linewidth]{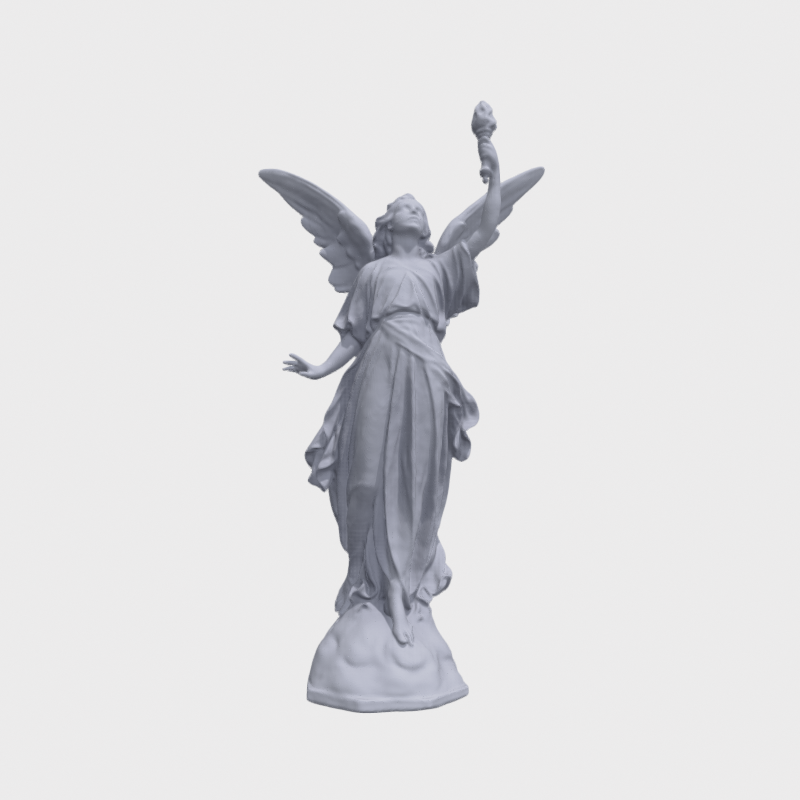} &

        \rotatebox{90}{\tiny \hspace{2.5em} MFN} &
        \includegraphics[width=0.16\linewidth]{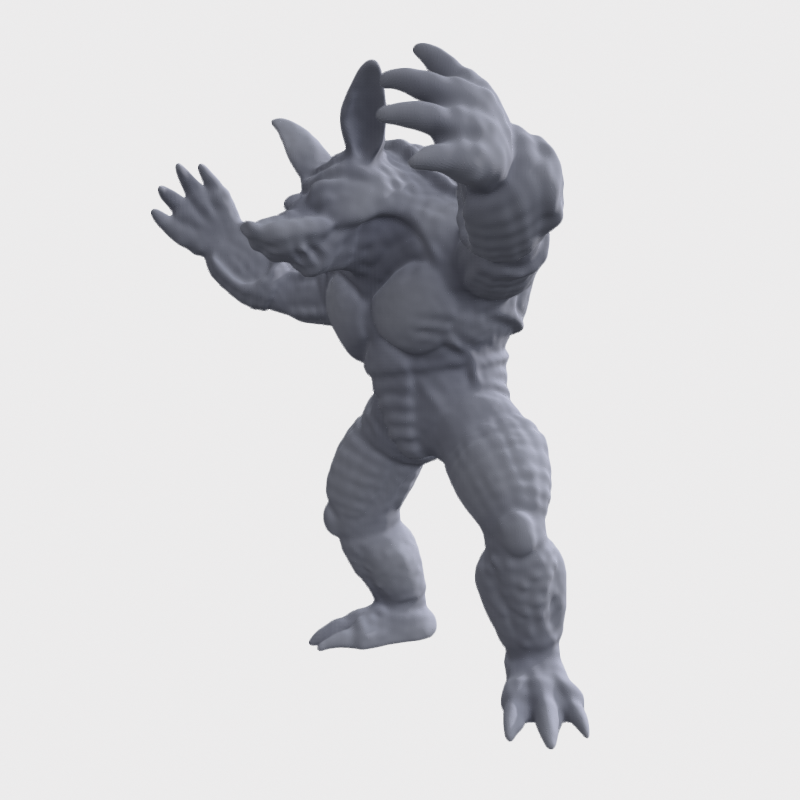} &
        \includegraphics[width=0.16\linewidth]{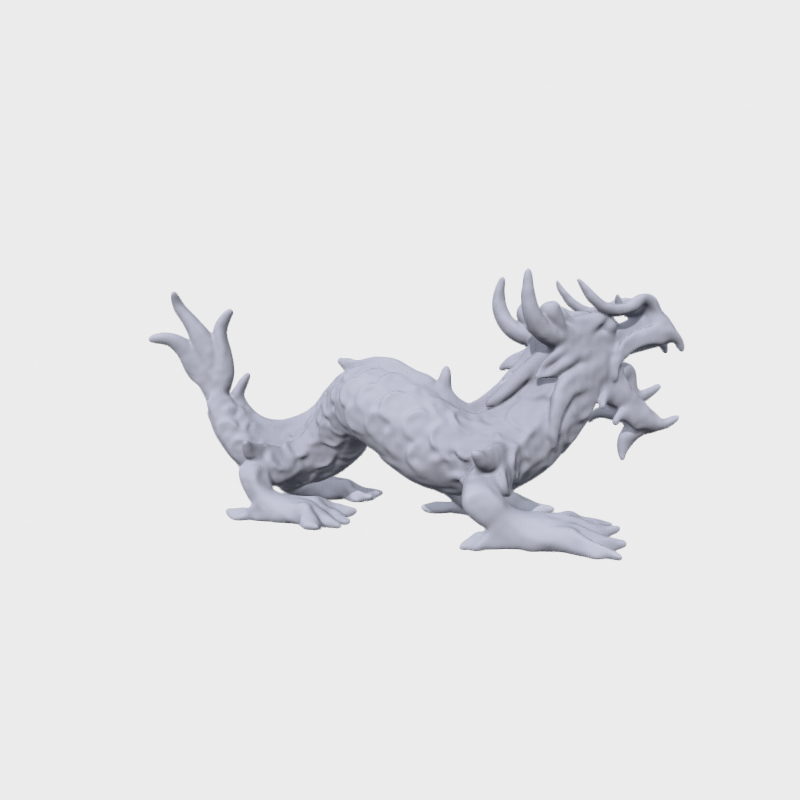} &
        \includegraphics[width=0.16\linewidth]{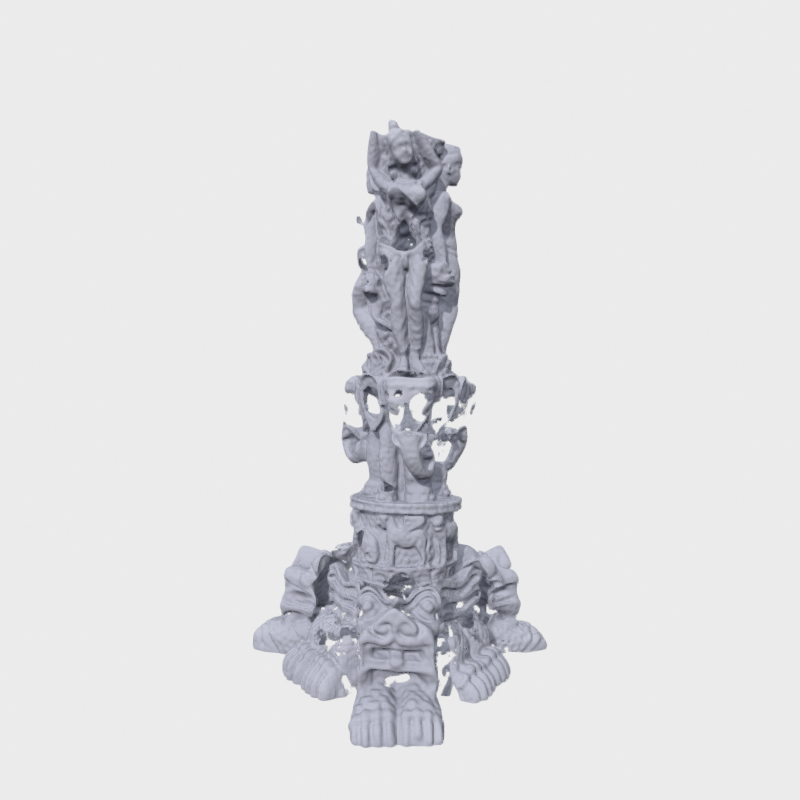} &
        \includegraphics[width=0.16\linewidth]{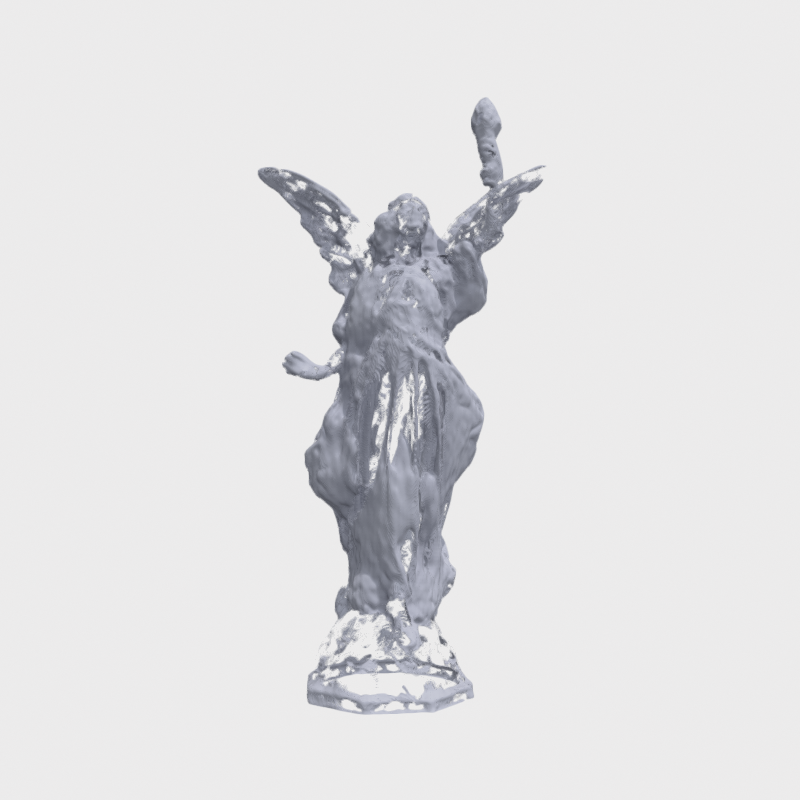} \\

        \rotatebox{90}{\tiny \hspace{2em} Fourier} &
        \includegraphics[width=0.16\linewidth]{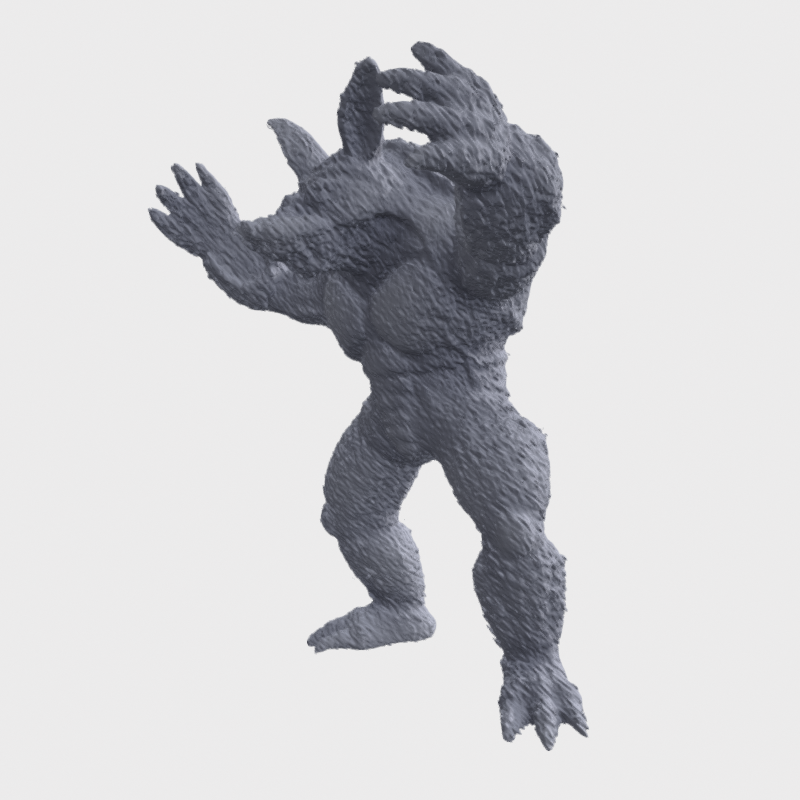} &
        \includegraphics[width=0.16\linewidth]{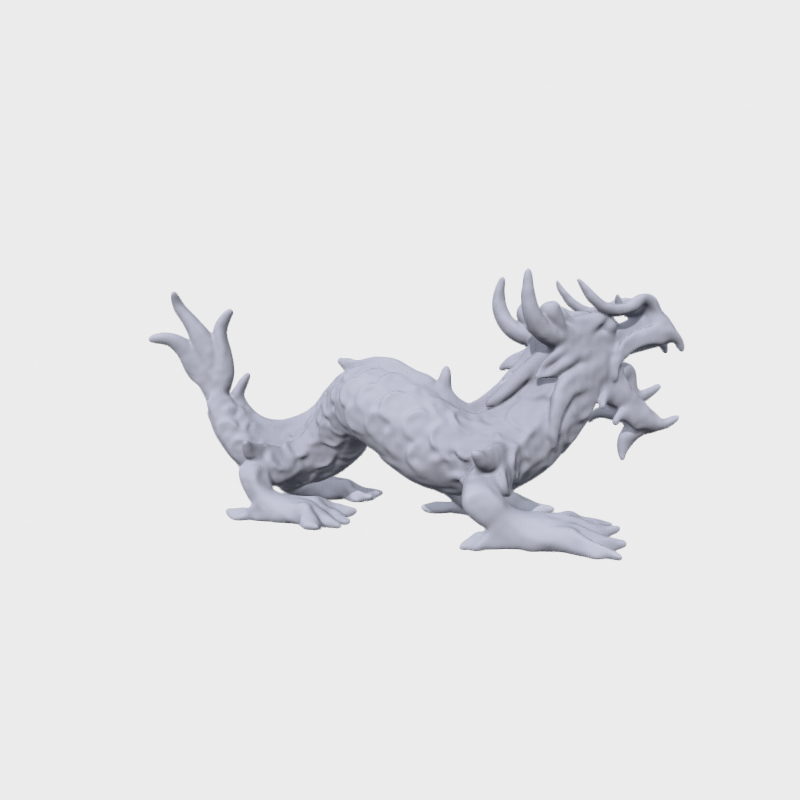} &
        \includegraphics[width=0.16\linewidth]{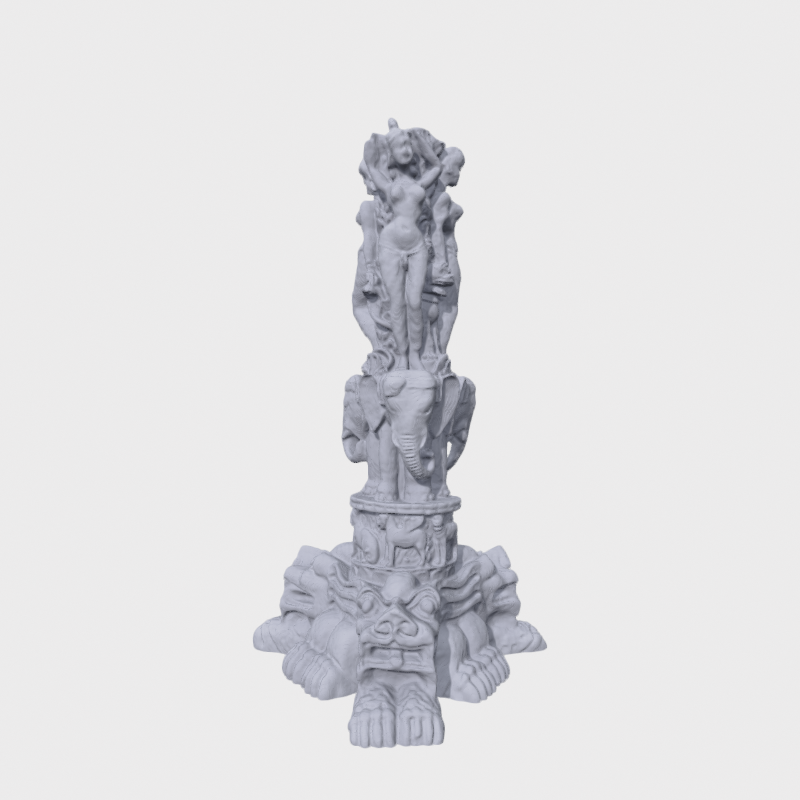} &
        \includegraphics[width=0.16\linewidth]{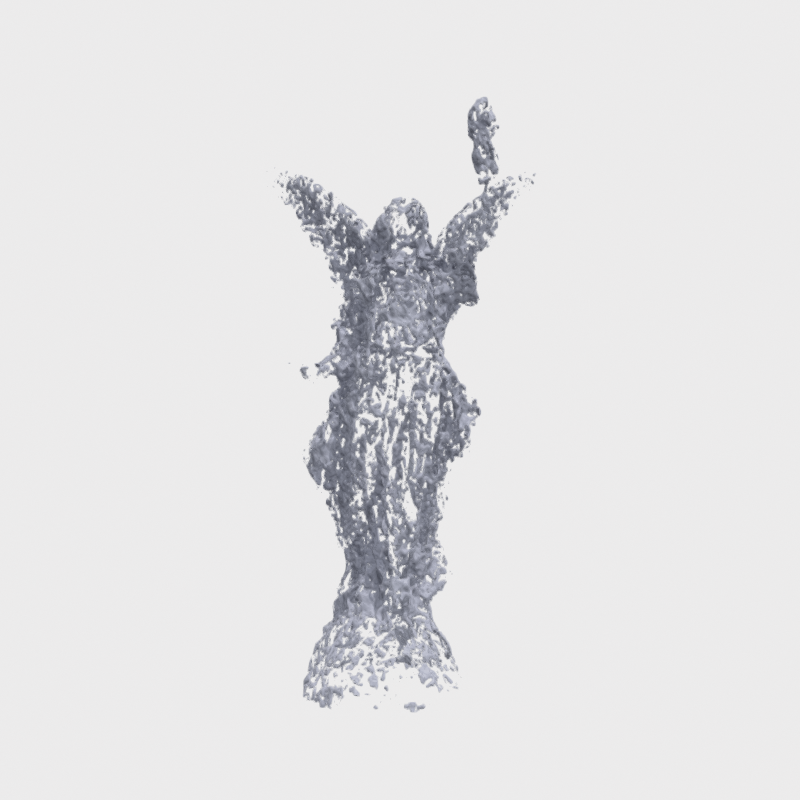} &

        \rotatebox{90}{\tiny \hspace{3em} FR} &
        \includegraphics[width=0.16\linewidth]{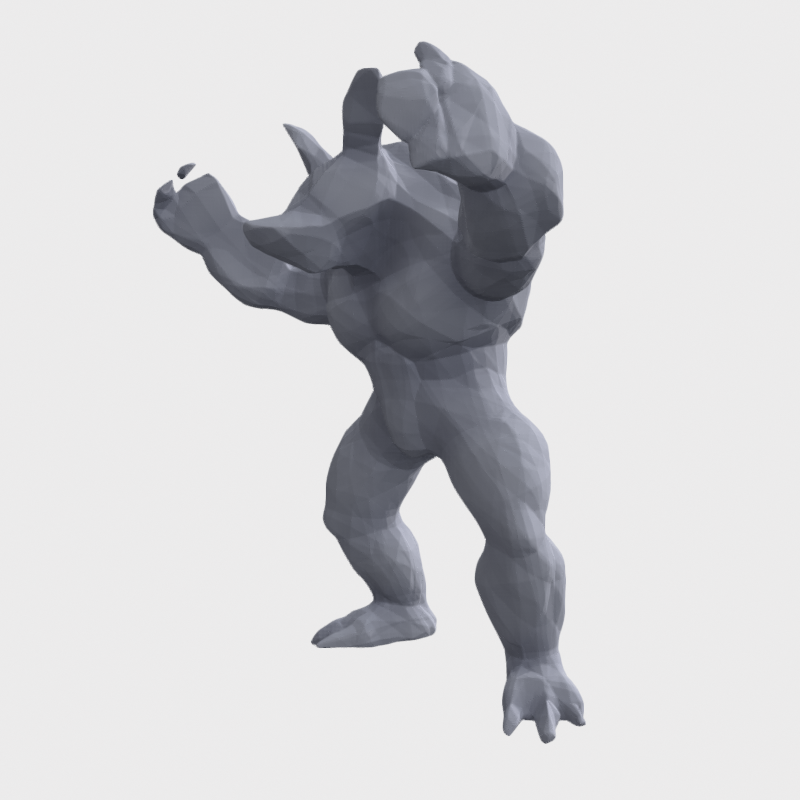} &
        \includegraphics[width=0.16\linewidth]{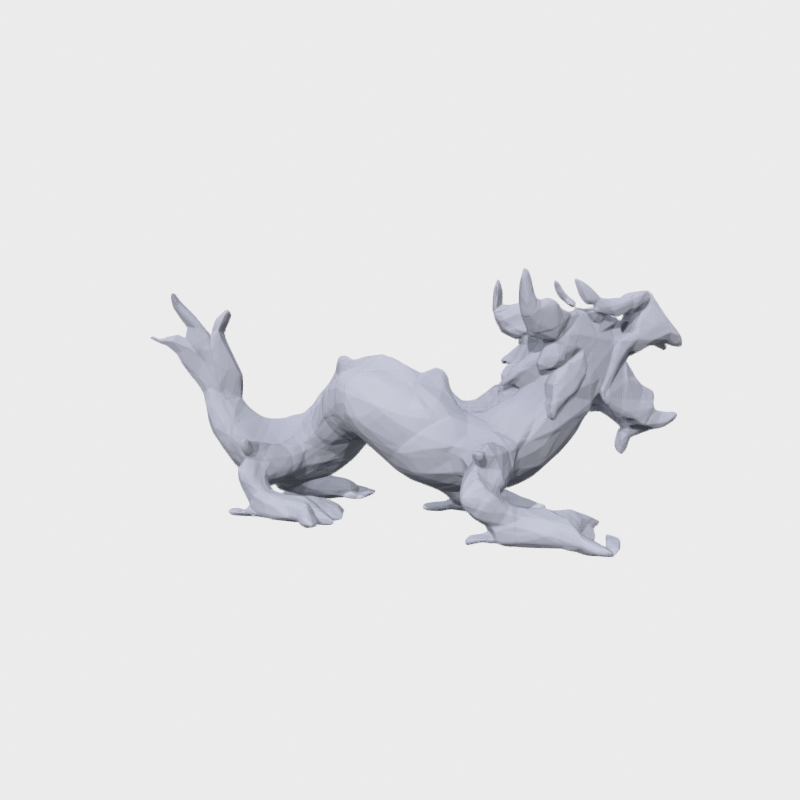} & 
        \includegraphics[width=0.16\linewidth]{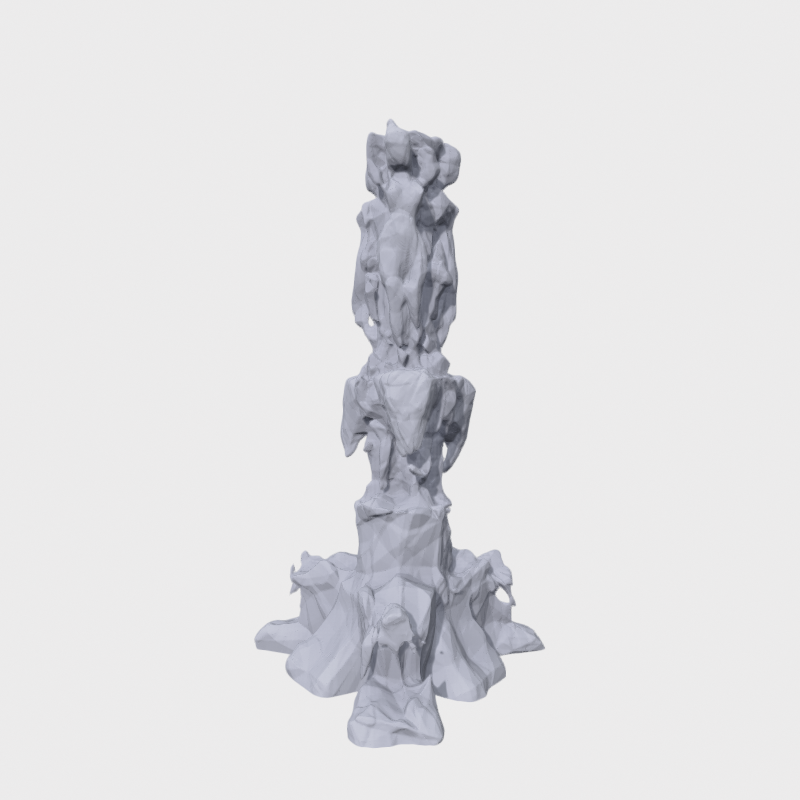} &
        \includegraphics[width=0.16\linewidth]{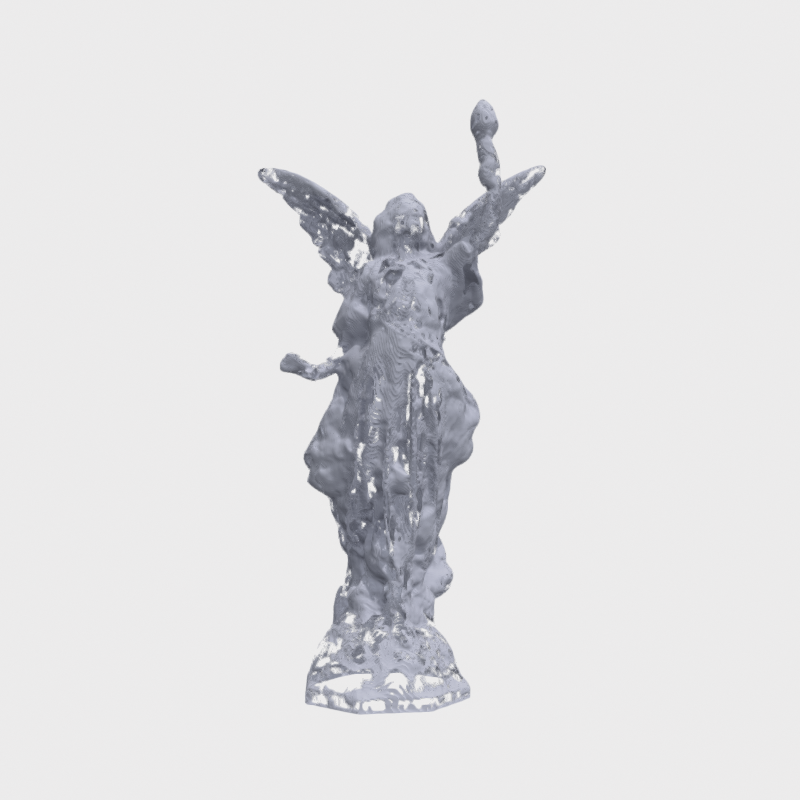} \\

    \end{tabular}
    }
    \caption{
    \textbf{Qualitative results on Signed Distance Function (SDF) Fitting.}
    }
    \label{fig:qual_sdf_fitting}
\end{figure}

            On the simpler meshes (\textbf{Armadillo}, \textbf{Dragon}), most periodic activations perform comparably: CD values fall within a narrow range and NC exceeds 0.96 for the top methods. 
            The differences become more pronounced on the geometrically complex shapes. 
            On \textbf{Thai Statue}, FINER achieves the lowest CD ($0.046 \times 10^{-3}$) with \algoname{} close behind ($0.049 \times 10^{-3}$), while SIREN ($0.061 \times 10^{-3}$) and the remaining baselines degrade more substantially. 
            On \textbf{Lucy}, \algoname{} leads ($0.025 \times 10^{-3}$) followed by FINER ($0.027 \times 10^{-3}$), with SIREN dropping to $0.113 \times 10^{-3}$.
            
            Overall, \algoname{} and FINER emerge as the two strongest methods, with \algoname{} performing well across all four meshes without task-specific tuning. 
            MFN, which dominated the Poisson task, struggles here (CD $0.298 \times 10^{-3}$ on \textbf{Thai Statue}), consistent with its limited representational capacity when fine geometric details must be resolved.
            
        \paragraph{Poisson Image Reconstruction.}
            We solve the Poisson equation $\nabla^2 f = g$ by supervising the network's Laplacian against the given source term $g$, recovering the image $f$ up to boundary conditions.
            This task is distinctive because the network never observes pixel values directly. 
            Indeed, the supervision is applied to second-order derivatives, requiring the activation to propagate gradient information through two levels of differentiation faithfully.
            The smoothness and Lipschitz properties established in \cref{sec:properties} are essential here, as they guarantee well-behaved higher-order derivatives.
            We follow the protocol described in~\cite{sitzmann2019siren} and report results in \cref{tab:poisson_solver}.
            
            \begin{table}[t]
    \centering
    \caption{
        \textbf{Poisson Image Reconstruction.} 
        Best results in \textbf{bold}, runner-ups \underline{underlined}. 
        }
    \resizebox{\linewidth}{!}{
        \begin{tabular}{l ccccc}

            \toprule
            
            \textbf{INR} & 
            \textbf{Tiger} & 
            \textbf{Tiles} & 
            \textbf{Bikers} & 
            \textbf{Butterfly} & 
            \textbf{Knot} \\ %

            \cmidrule(lr){2-6}
            
            \algoname{}                          & \underline{26.13 $\pm$ 0.31} & 14.26 $\pm$ 1.91 & \underline{26.93 $\pm$ 0.25} & \textbf{33.76 $\pm$ 0.02} & \underline{25.83 $\pm$ 0.91} \\
            SIREN \cite{sitzmann2019siren}       & 25.89 $\pm$ 1.45 & \underline{19.13 $\pm$ 1.27} & 23.39 $\pm$ 1.22 & 30.69 $\pm$ 0.11 & 24.74 $\pm$ 0.26 \\
            Gauss \cite{ramasinghe2022beyond}    & 22.21 $\pm$ 0.28 & 13.47 $\pm$ 0.13 & 17.37 $\pm$ 0.46 & 20.83 $\pm$ 0.49 & 22.03 $\pm$ 0.43 \\
            WIRE \cite{saragadam2023wire}        & 12.94 $\pm$ 0.18 & 14.23 $\pm$ 0.19 & 14.20 $\pm$ 0.28 & 10.95 $\pm$ 0.10 & 13.22 $\pm$ 0.06 \\
            BACON \cite{lindell2021bacon}        & 12.80 $\pm$ 0.00 & 14.29 $\pm$ 0.01 & 14.14 $\pm$ 0.01 & 10.83 $\pm$ 0.01 & 14.07 $\pm$ 0.05 \\
            FINER \cite{liu2024finer}            & 19.85 $\pm$ 0.42 & 16.85 $\pm$ 0.97 & 15.78 $\pm$ 0.82 & 18.73 $\pm$ 0.20 & 17.68 $\pm$ 0.78 \\
            MFN \cite{fathony2021multiplicative} & \textbf{36.76 $\pm$ 0.13} & \textbf{27.21 $\pm$ 0.78} & \textbf{31.52 $\pm$ 0.22} & \underline{33.11 $\pm$ 0.57} & \textbf{29.85 $\pm$ 0.32} \\
            Fourier \cite{tancik2020fourfeat}    & 17.98 $\pm$ 6.43 & 21.03 $\pm$ 2.51 & 24.61 $\pm$ 0.70 & 19.13 $\pm$ 1.52 & 17.32 $\pm$ 0.87 \\
            FR \cite{shi2024improved}            & 15.73 $\pm$ 1.57 & 15.38 $\pm$ 1.21 & 15.50 $\pm$ 1.35 & 12.46 $\pm$ 0.77 & 14.13 $\pm$ 0.29 \\
            
            \bottomrule
            
        \end{tabular}
        }
    \label{tab:poisson_solver}
\end{table}

            In this task, MFN leads on four of five images by substantial margins (up to 10 dB on \textbf{Tiger}), with \algoname{} leading on \textbf{Butterfly} (33.76 dB vs. 33.11 dB) and ranking second on \textbf{Tiger}, \textbf{Bikers}, and \textbf{Knot}.
            On Tiles, \algoname{} drops to 14.26 dB, well below both MFN (27.21 dB) and SIREN (19.13 dB).
            Nonetheless, \algoname{} remains the only method other than MFN to rank consistently in the top two: it places first or second on four of five images, being the runner-up method on average.
            This is notable because, across all previous tasks, the identity of the runner-up changes frequently: MFN on denoising and CT, SIREN on inpainting, FINER on super-resolution. 
            \algoname{} is the only activation that remains competitive throughout.

            The explanation lies in how Poisson supervision alters the spectral gating dynamics.
            \cref{cor:gating} was derived under a pixel-level MSE loss, where the signal term $T_{\mathrm{signal}}$ depends on the inner product $\langle y, \sin(\omega \cdot + \varphi)\rangle_N$.
            Under Poisson supervision, the loss is computed on the Laplacian $\nabla^2 f$ rather than on $f$ itself.
            Differentiating a sinusoidal component at frequency $\omega$ twice introduces a factor of $\omega^2$, so the effective target seen by the optimiser amplifies high-frequency content by $\omega^2$ relative to pixel-level supervision.
            Within the MSE loss, this amplification is squared again, meaning the effective signal energy at frequency $\omega$ scales as $\omega^4$ compared to the pixel-level case.
            This fundamentally shifts the balance in the gating condition of \cref{eq:gating_condition}: $T_{\mathrm{signal}}$ grows with $\omega^4$ while $T_{\mathrm{self}}$ remains governed by $A$, which starts small due to the stopband initialisation.
            The coarse-to-fine curriculum that benefits noise rejection and progressive frequency recovery under pixel-level losses now becomes a liability, as low-amplitude but high-frequency Laplacian components carry disproportionate energy in the loss landscape, yet the conservative initial bandwidth delays their recovery.
            MFN's multiplicative Gabor structure does not face this issue: its filters compose through element-wise products whose derivatives remain well-conditioned through multiple levels of differentiation, making it inherently suited to derivative-based supervision.
            Adapting \algoname{}'s initialisation to account for the $\omega^2$ scaling, for instance, by starting with a wider initial bandwidth when the supervision is applied to differential operators, is a natural direction for future work.
            As an example, running \algoname{} with a wider initial bandwidth (from $\omega_0 = 45, \omega_{n,0} = 50$ to $\omega_0 = 15, \omega_{n,0} = 50$) improves the performance on \textbf{Tiger} from $26.13 \pm 0.31$ dB to $30.11 \pm 0.64$ dB.
            
        \paragraph{Compatibility with Other Works.}
            \algoname{} can be seamlessly integrated with other established techniques, such as faster initialisation methods like FreSh \cite{kania2024freshfrequencyshiftingaccelerated}, as well as be extended to temporal signals via ResField \cite{mihajlovic2024ResFields}.
            Indeed, we run both FreSh and ResField by replacing SIREN's activation function with \algoname{}'s activation function.
            With FreSh, the test PSNR for Image Fitting on \emph{Kodak Image 07} goes from $30.32$ dB (SIREN) to $34.39$ dB (\algoname{}).
            Likewise for ResFields, the test PSNR for Video Fitting on \emph{carphone} goes from $28.4$ dB (SIREN) to $35.47$ dB (\algoname{}).
            
        \paragraph{Additional Baseline Comparisons.}
            We run TUNER \cite{novello2025tuner}, SASNet \cite{feng2026sasnet}, and SAPE \cite{hertz2021sape} for the Image Fitting task on Tiger, following our protocol and show the results in \cref{tab:additional_baselines}. 
            \algoname{} provides the best results.
            \begin{table}[ht]
                \centering
                \caption{\textbf{Image Fitting on Tiger across Additional Baselines.}}
                \label{tab:additional_baselines}
                \resizebox{\linewidth}{!}{
                \begin{tabular}{l cccc}
                    \toprule
                    
                    \textbf{INR} & \textbf{FDHO} & TUNER & SASNet & SAPE \\
                    \midrule
                    
                    Final PSNR & \textbf{63.79 $\pm$ 0.23} & 59.03 $\pm$ 2.15 & 46.60 $\pm$ 16.79 & 38.71 $\pm$ 0.65 \\
                    Peak PSNR  & \textbf{63.79 $\pm$ 0.23} & 61.51 $\pm$ 0.17 & 48.93 $\pm$ 0.31 & 40.22 $\pm$ 0.03 \\
                    
                    \bottomrule
                    
                \end{tabular}}
            \end{table}
            
    \subsection{Additional Qualitative Results}
        We report in \cref{fig:qual_signal_fitting} and \cref{fig:qual_audio_fitting} all the qualitative results for the signal and audio fitting tasks.
        \begin{figure}[t]
    \centering
    \resizebox{\linewidth}{!}{
    \begin{tabular}{c cccccc}        

        & \multicolumn{3}{c}{\textbf{Square Wave}} & \multicolumn{3}{c}{\textbf{Chirp}} \\
        
        & {\tiny 100 Hz} & {\tiny 250 Hz} & {\tiny 500 Hz} & {\tiny 250 Hz} & {\tiny 500 Hz} & {\tiny 1000 Hz} \\
                
        {\tiny \algoname{}} &
        \includegraphics[width=0.16\linewidth, valign=c]{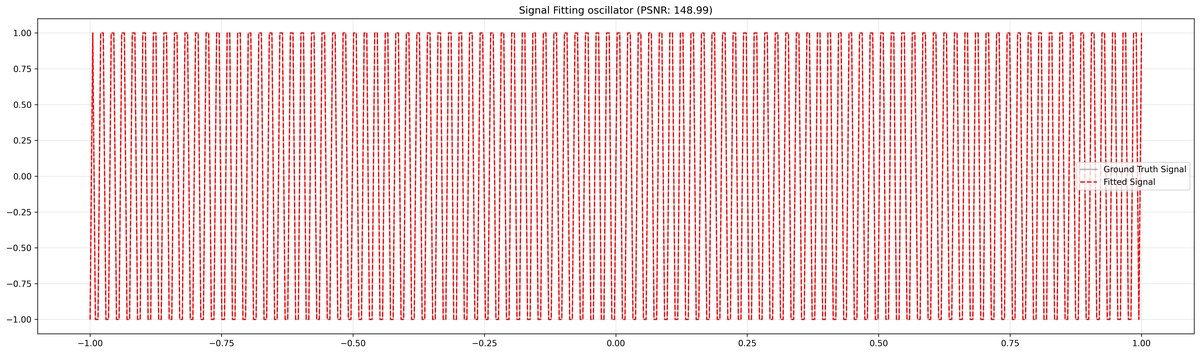} & 
        \includegraphics[width=0.16\linewidth, valign=c]{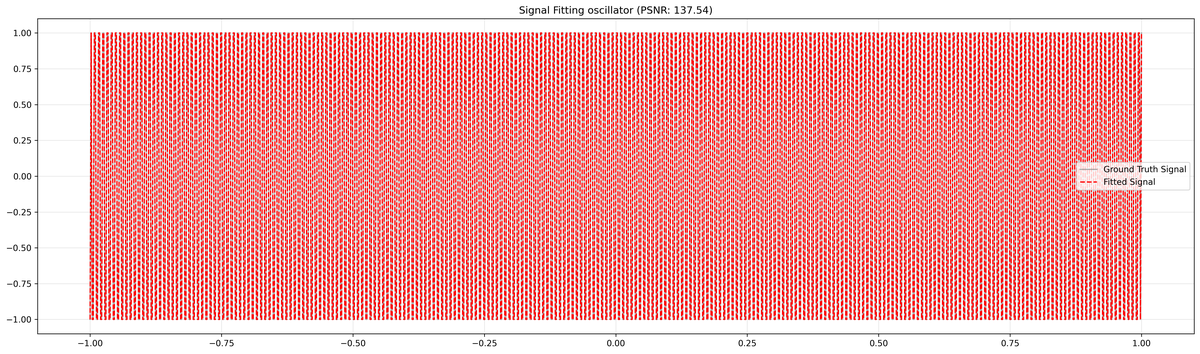} & 
        \includegraphics[width=0.16\linewidth, valign=c]{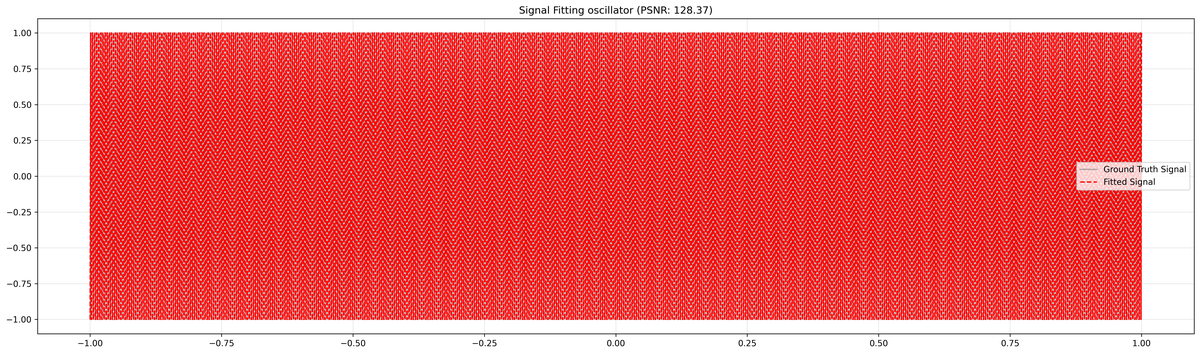} & 
        \includegraphics[width=0.16\linewidth, valign=c]{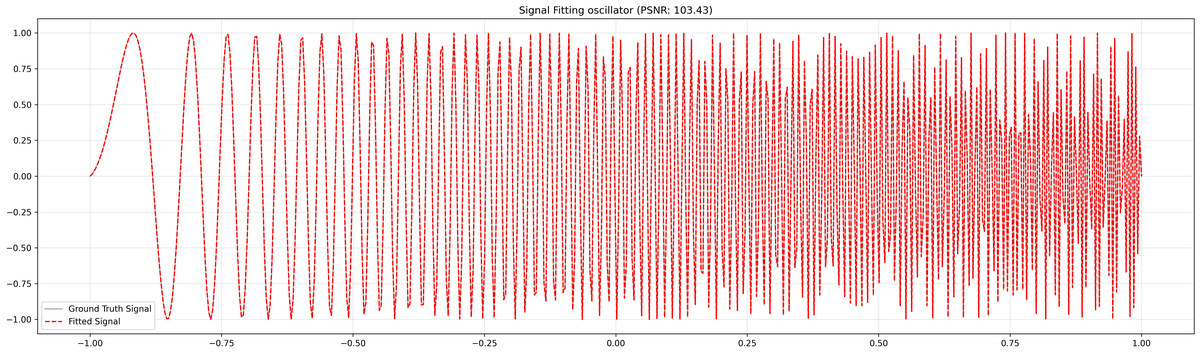} & 
        \includegraphics[width=0.16\linewidth, valign=c]{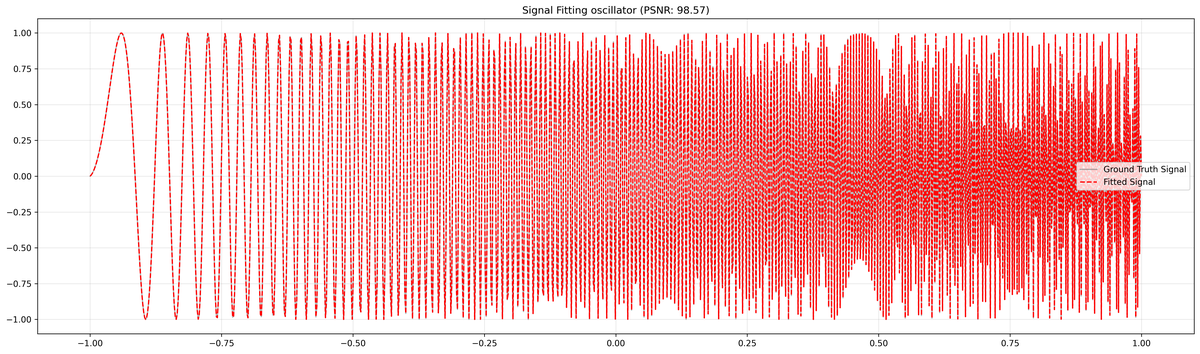} & 
        \includegraphics[width=0.16\linewidth, valign=c]{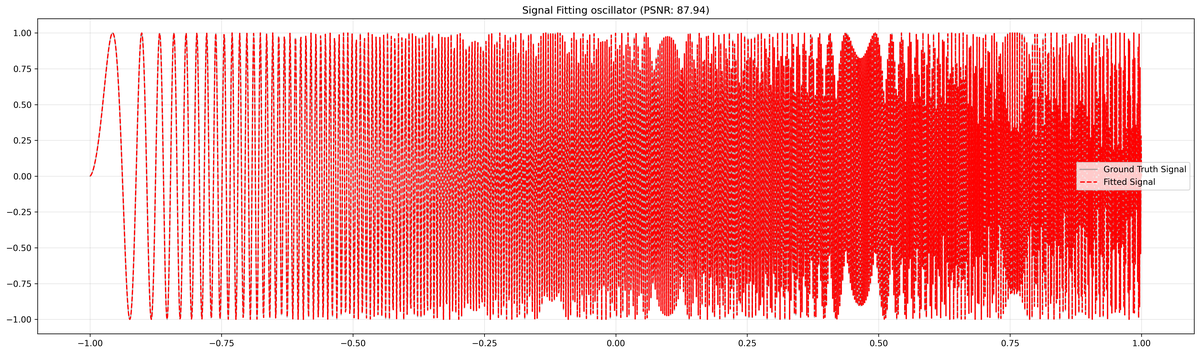} \\

        {\tiny SIREN} &
        \includegraphics[width=0.16\linewidth, valign=c]{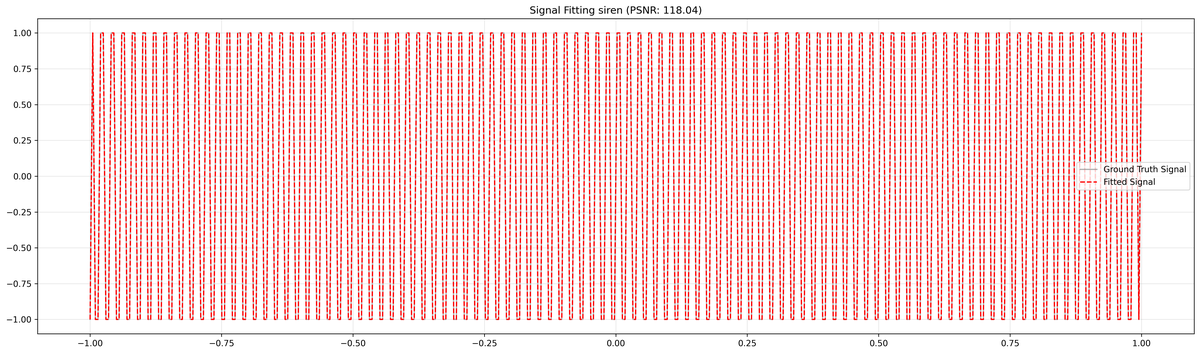} & 
        \includegraphics[width=0.16\linewidth, valign=c]{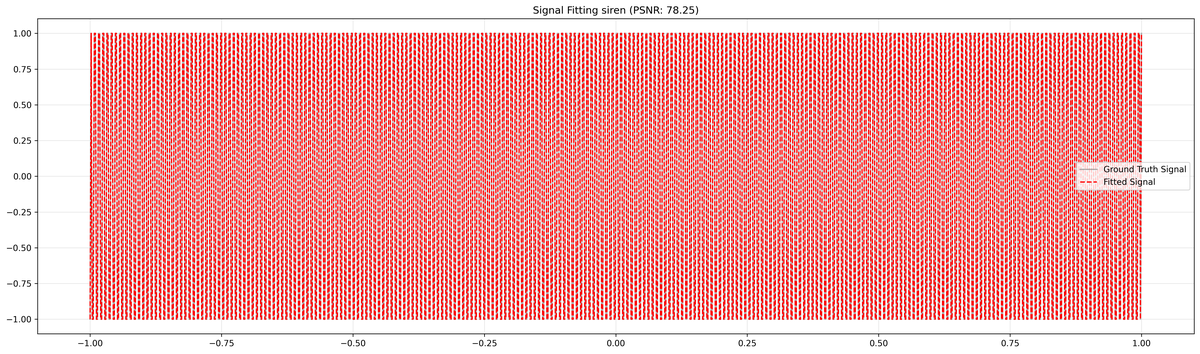} & 
        \includegraphics[width=0.16\linewidth, valign=c]{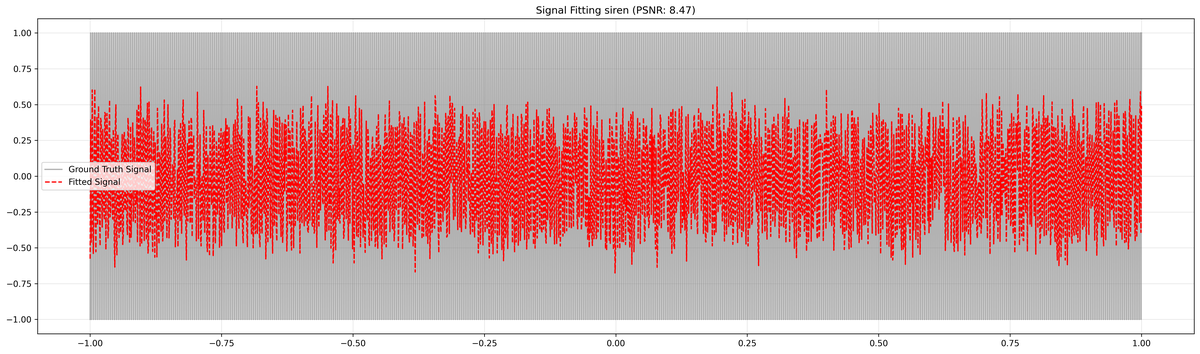} & 
        \includegraphics[width=0.16\linewidth, valign=c]{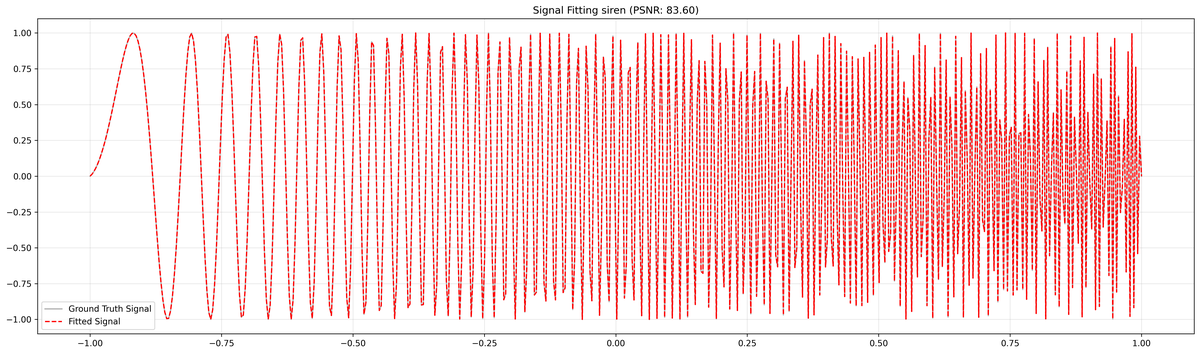} & 
        \includegraphics[width=0.16\linewidth, valign=c]{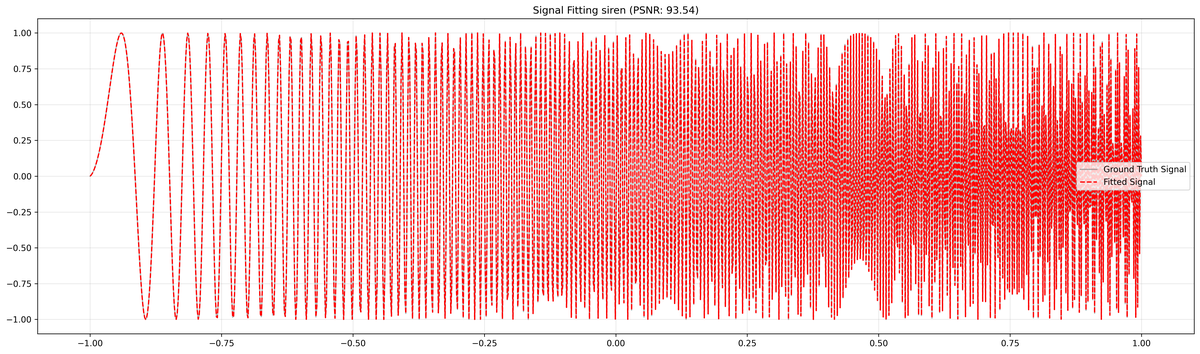} & 
        \includegraphics[width=0.16\linewidth, valign=c]{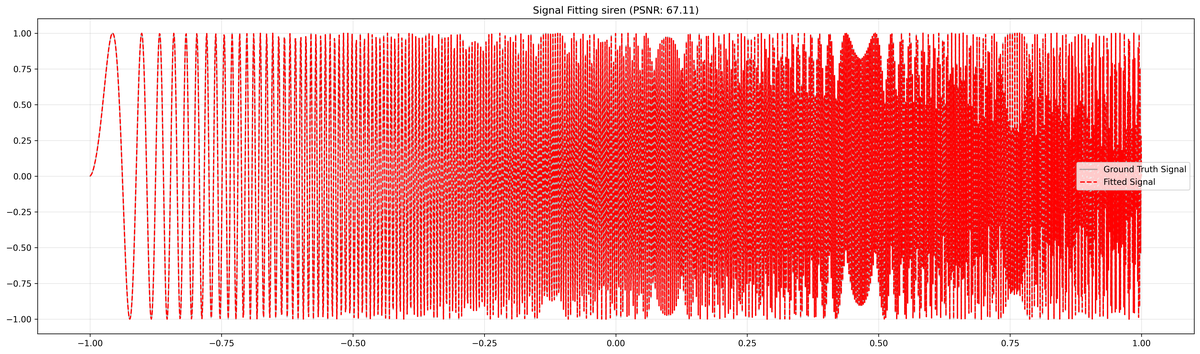} \\

        {\tiny Gauss} &
        \includegraphics[width=0.16\linewidth, valign=c]{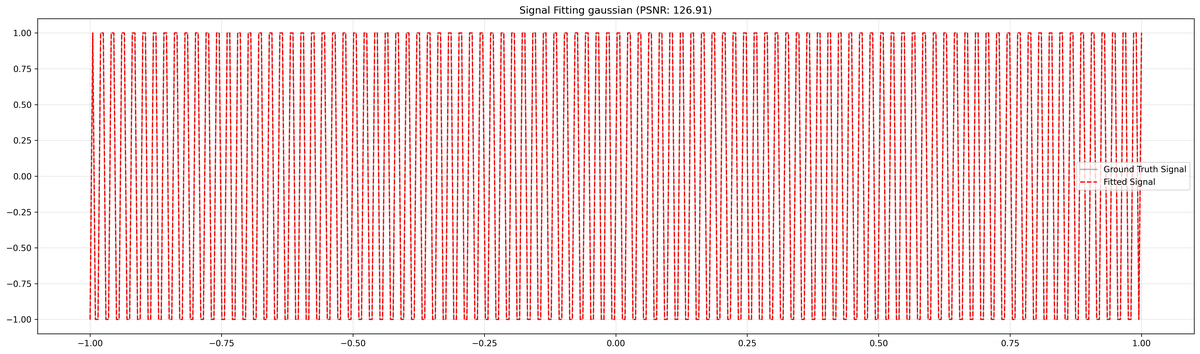} & 
        \includegraphics[width=0.16\linewidth, valign=c]{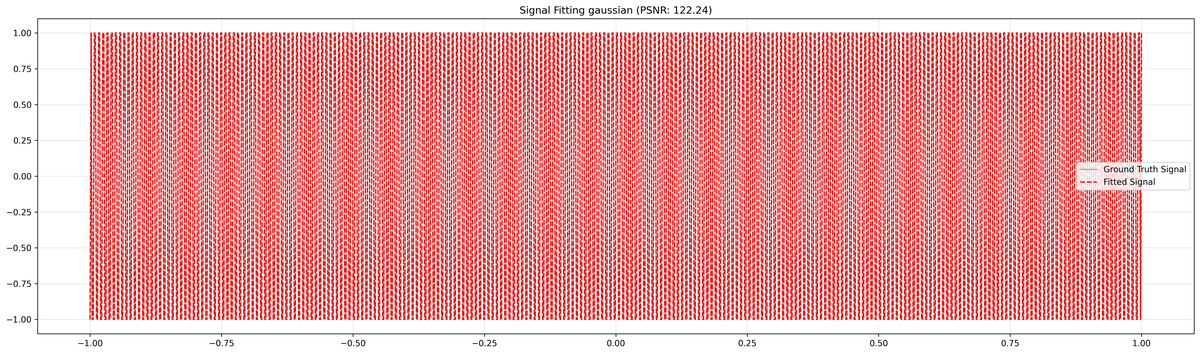} & 
        \includegraphics[width=0.16\linewidth, valign=c]{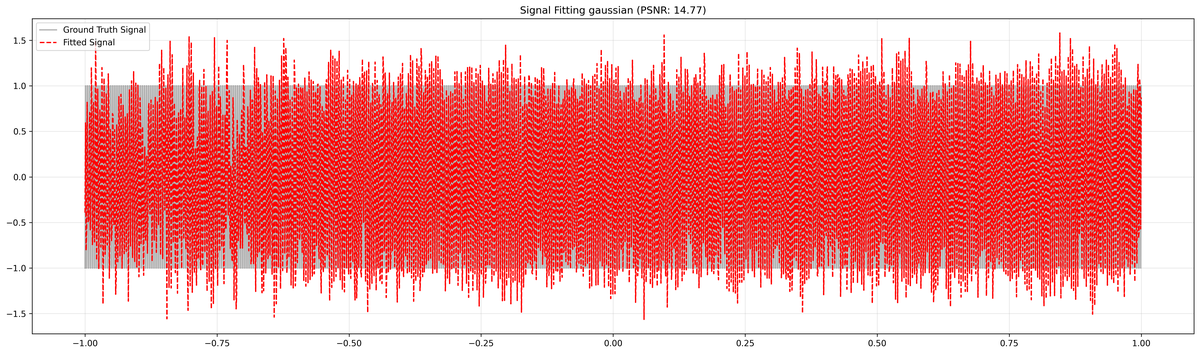} & 
        \includegraphics[width=0.16\linewidth, valign=c]{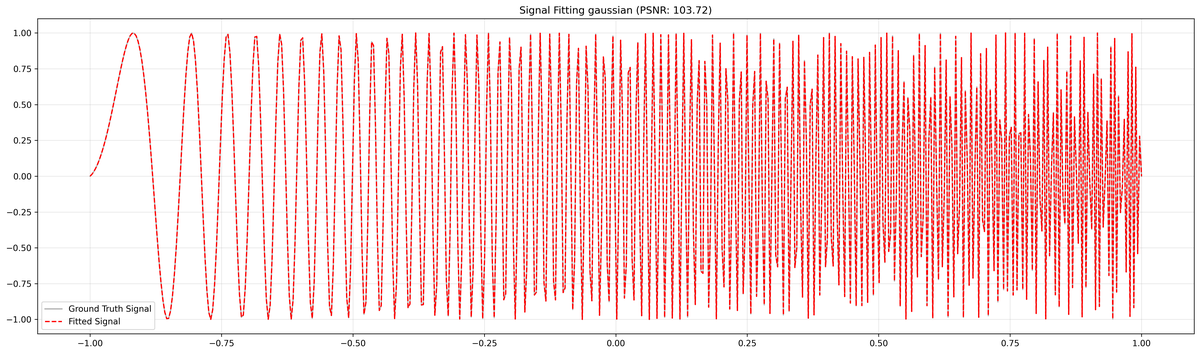} & 
        \includegraphics[width=0.16\linewidth, valign=c]{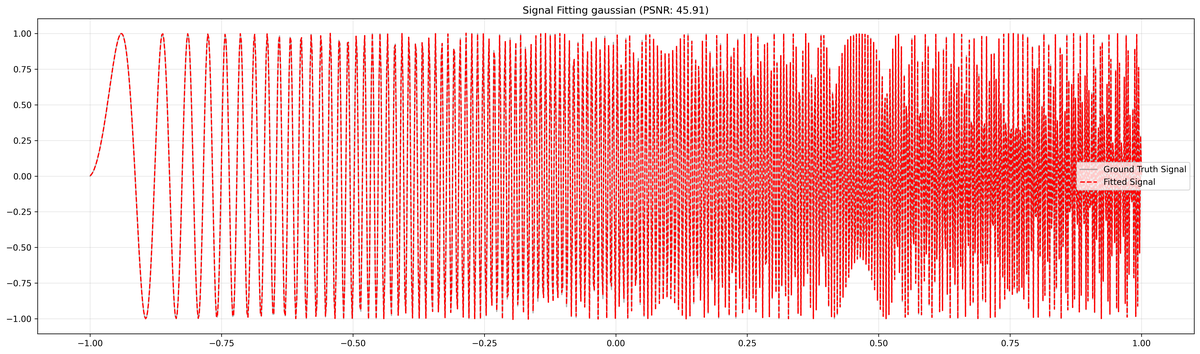} & 
        \includegraphics[width=0.16\linewidth, valign=c]{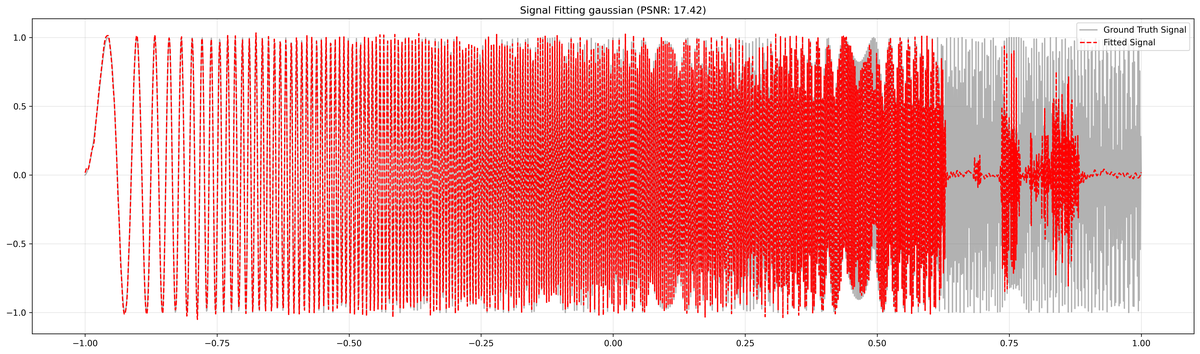} \\

        {\tiny WIRE} &
        \includegraphics[width=0.16\linewidth, valign=c]{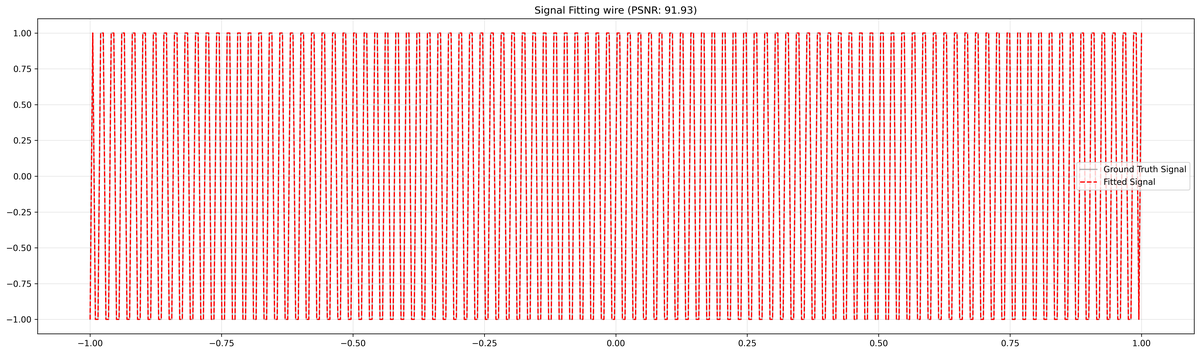} & 
        \includegraphics[width=0.16\linewidth, valign=c]{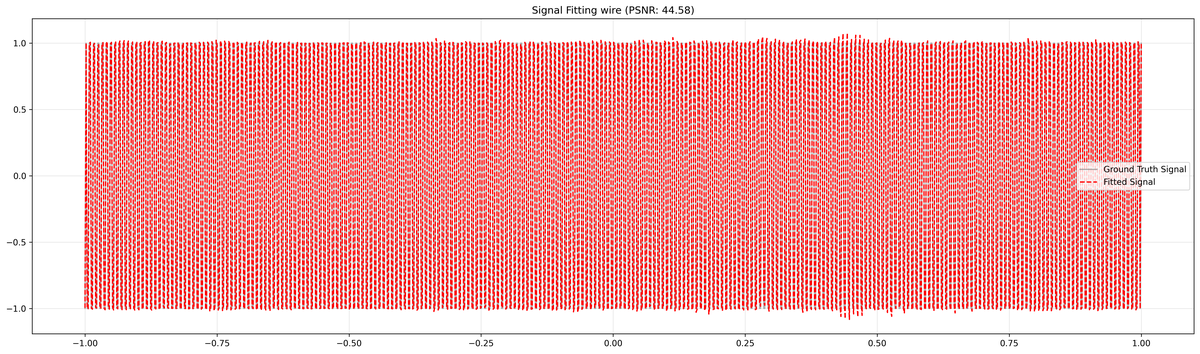} & 
        \includegraphics[width=0.16\linewidth, valign=c]{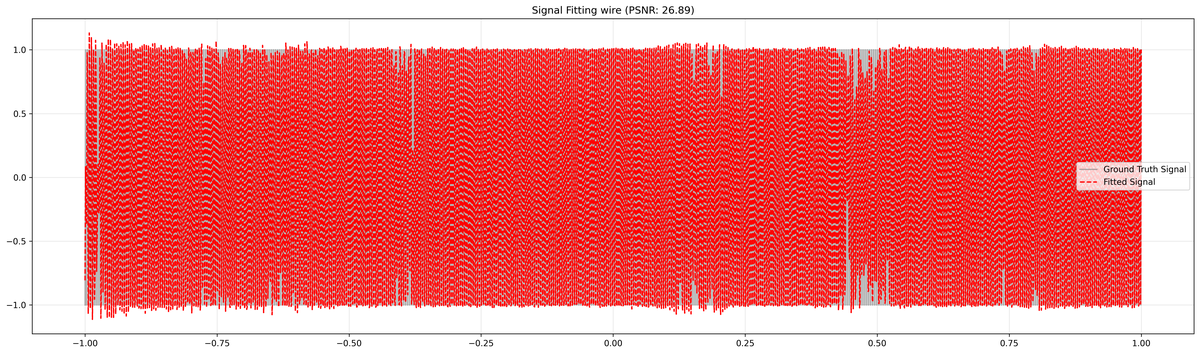} & 
        \includegraphics[width=0.16\linewidth, valign=c]{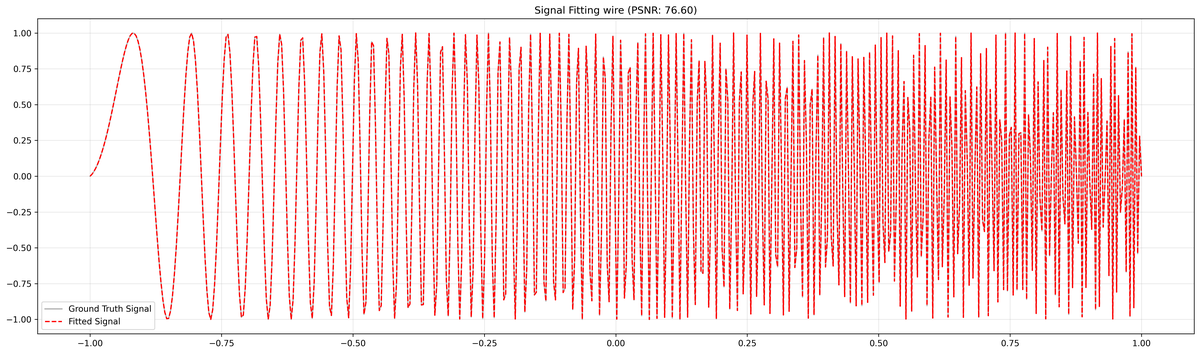} & 
        \includegraphics[width=0.16\linewidth, valign=c]{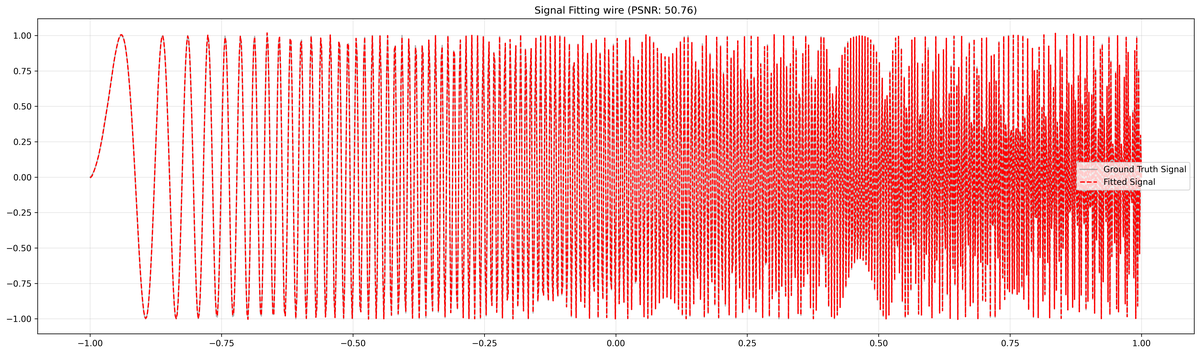} & 
        \includegraphics[width=0.16\linewidth, valign=c]{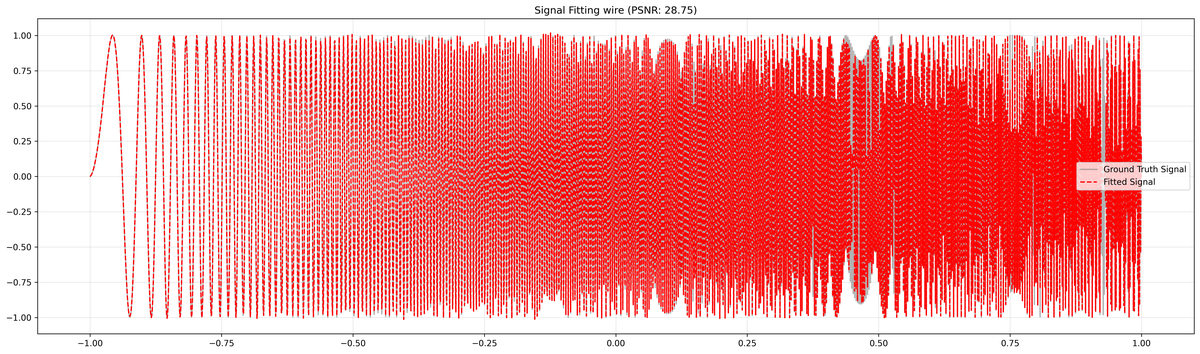} \\

        {\tiny BACON} &
        \includegraphics[width=0.16\linewidth, valign=c]{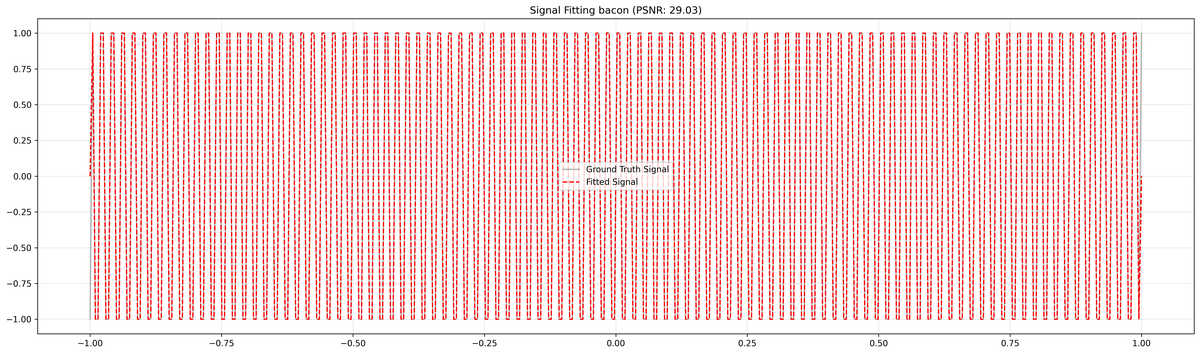} & 
        \includegraphics[width=0.16\linewidth, valign=c]{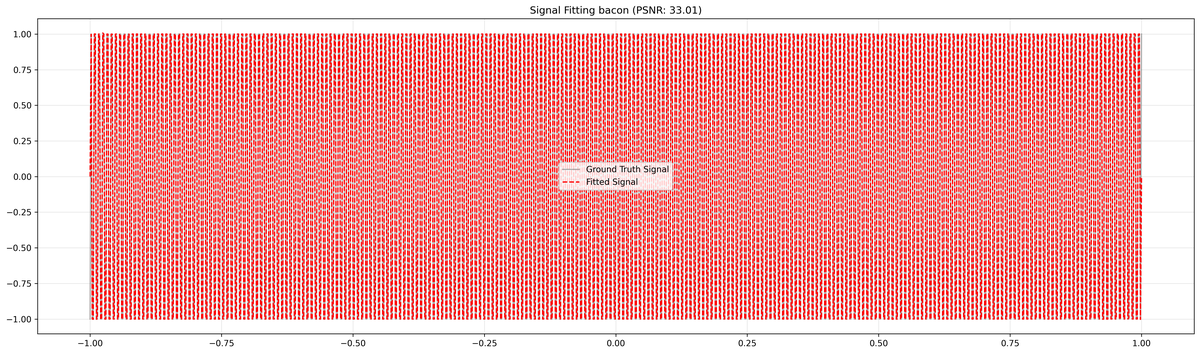} & 
        \includegraphics[width=0.16\linewidth, valign=c]{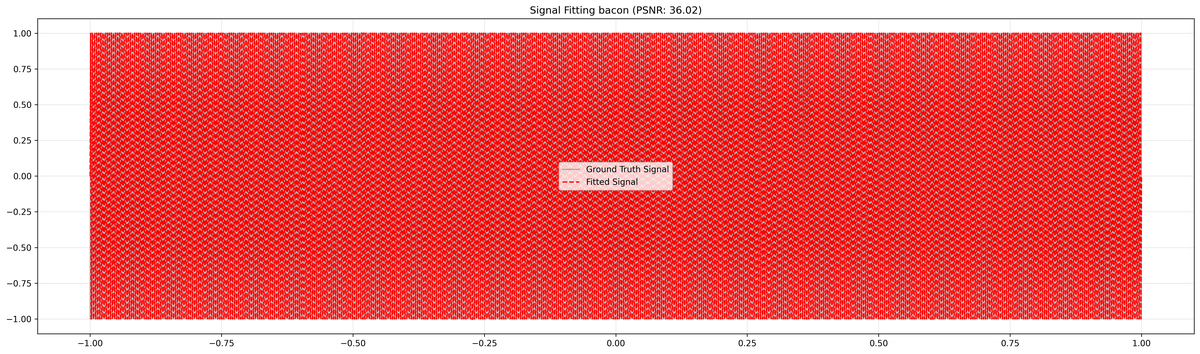} & 
        \includegraphics[width=0.16\linewidth, valign=c]{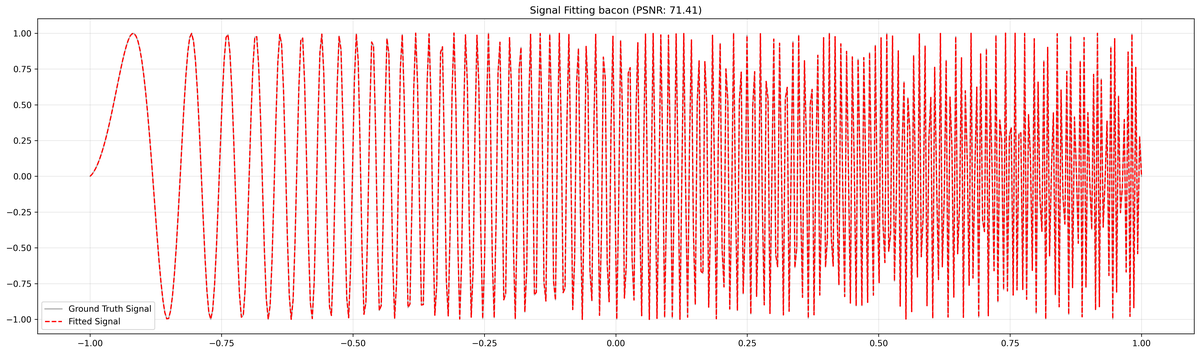} & 
        \includegraphics[width=0.16\linewidth, valign=c]{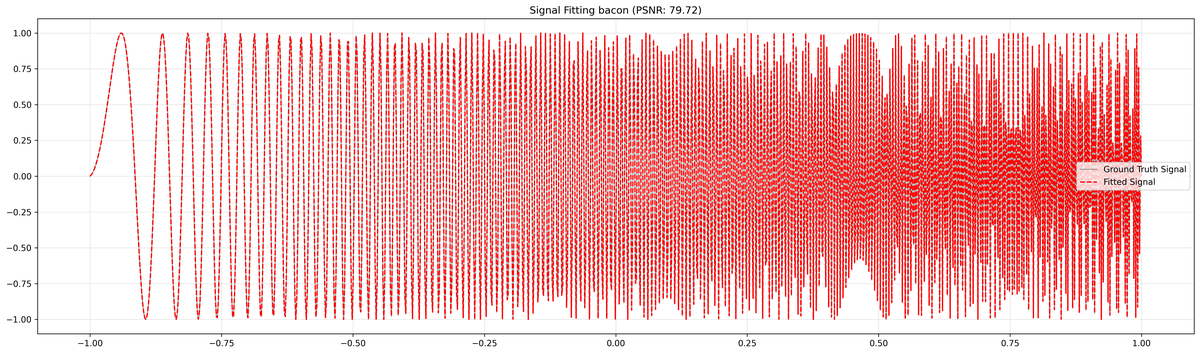} & 
        \includegraphics[width=0.16\linewidth, valign=c]{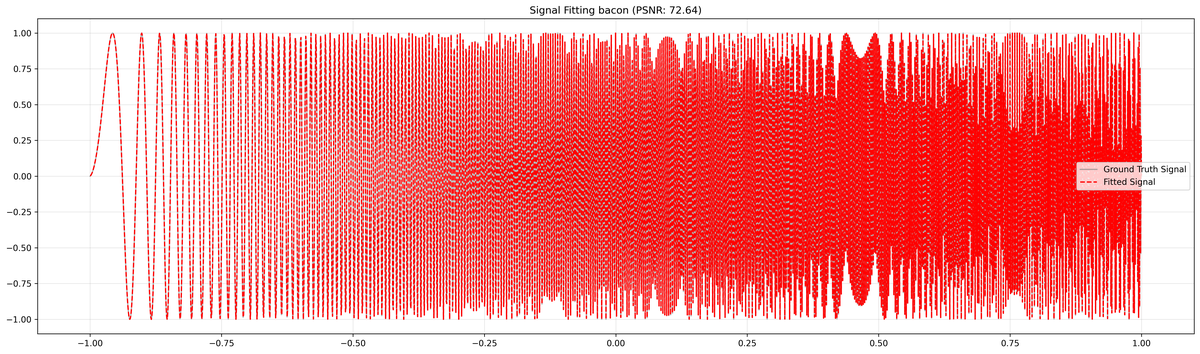} \\

        {\tiny FINER} &
        \includegraphics[width=0.16\linewidth, valign=c]{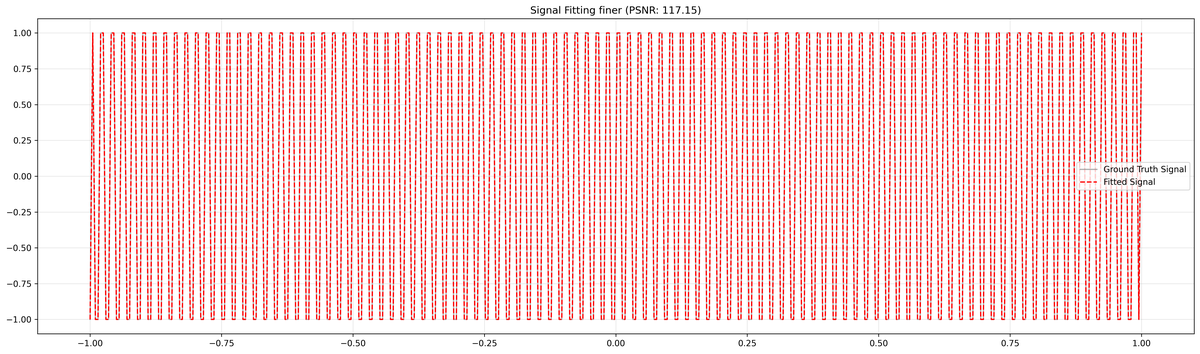} & 
        \includegraphics[width=0.16\linewidth, valign=c]{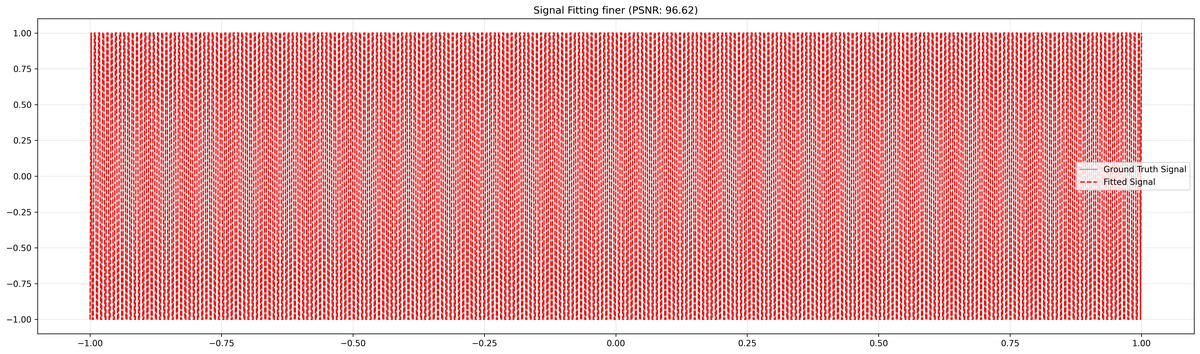} & 
        \includegraphics[width=0.16\linewidth, valign=c]{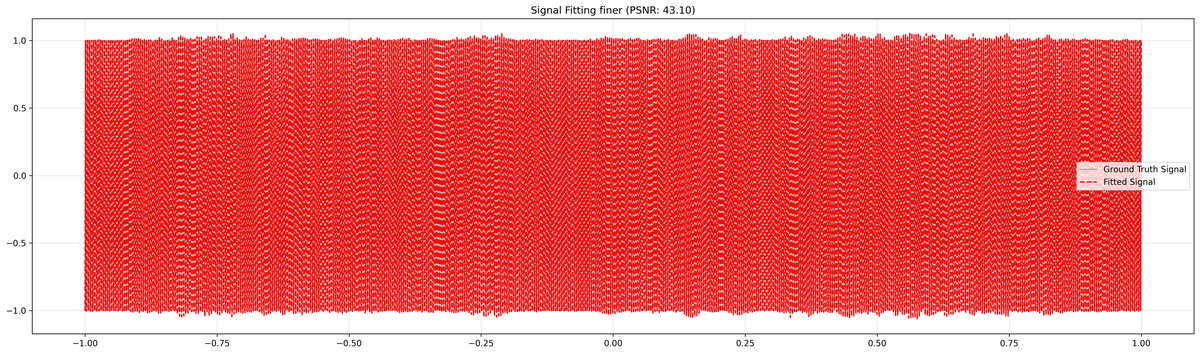} & 
        \includegraphics[width=0.16\linewidth, valign=c]{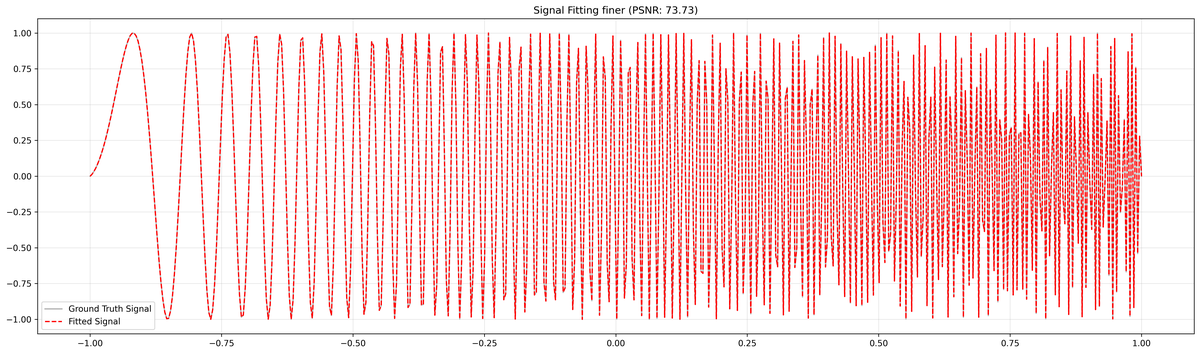} & 
        \includegraphics[width=0.16\linewidth, valign=c]{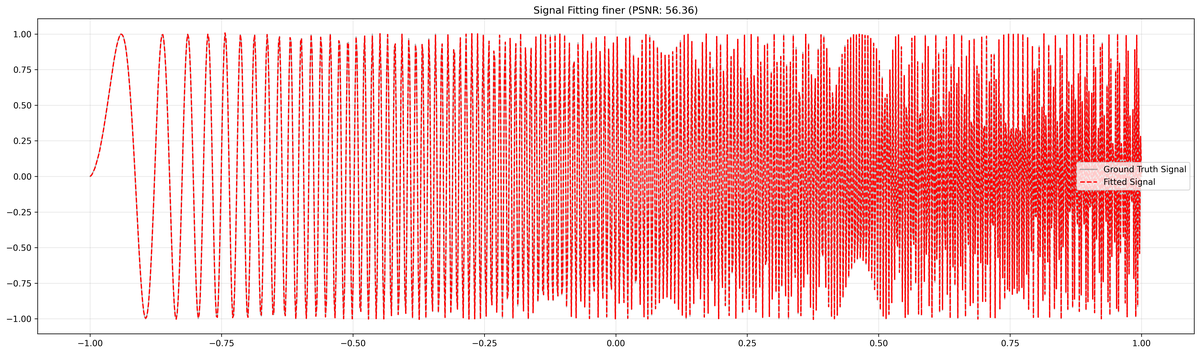} & 
        \includegraphics[width=0.16\linewidth, valign=c]{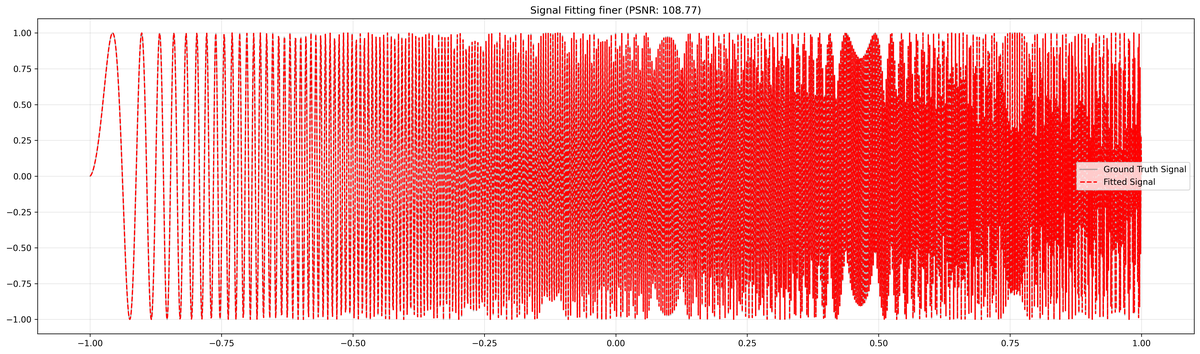} \\

        {\tiny MFN} &
        \includegraphics[width=0.16\linewidth, valign=c]{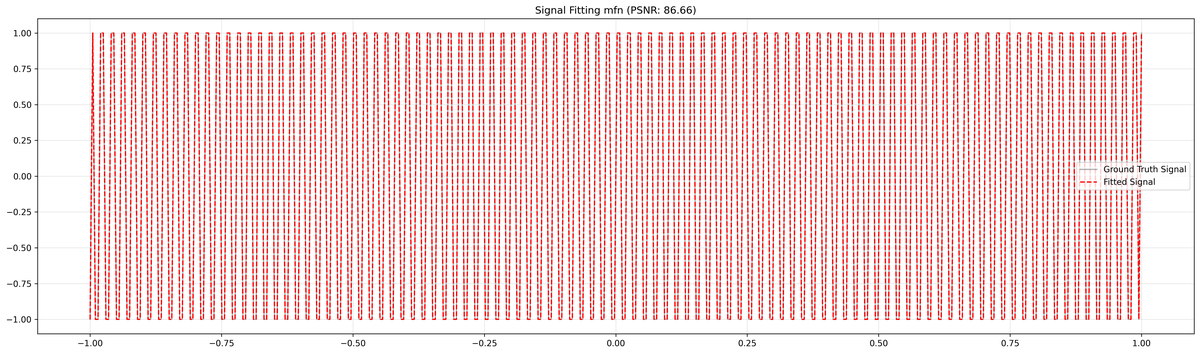} & 
        \includegraphics[width=0.16\linewidth, valign=c]{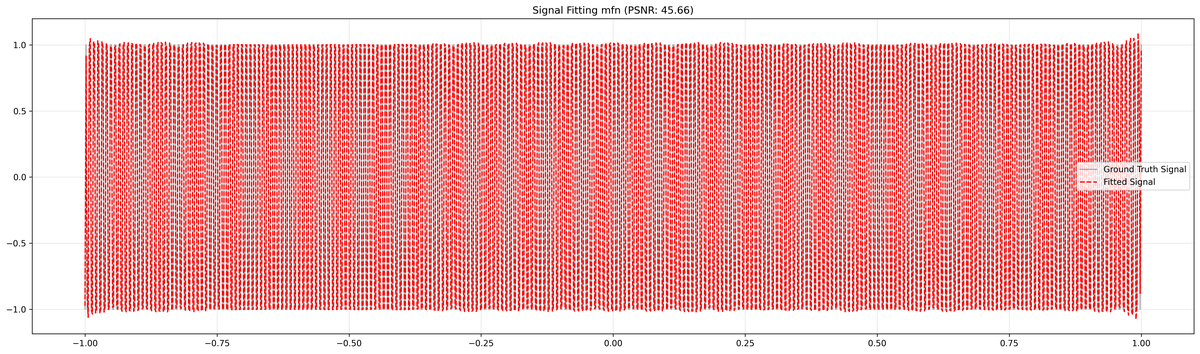} & 
        \includegraphics[width=0.16\linewidth, valign=c]{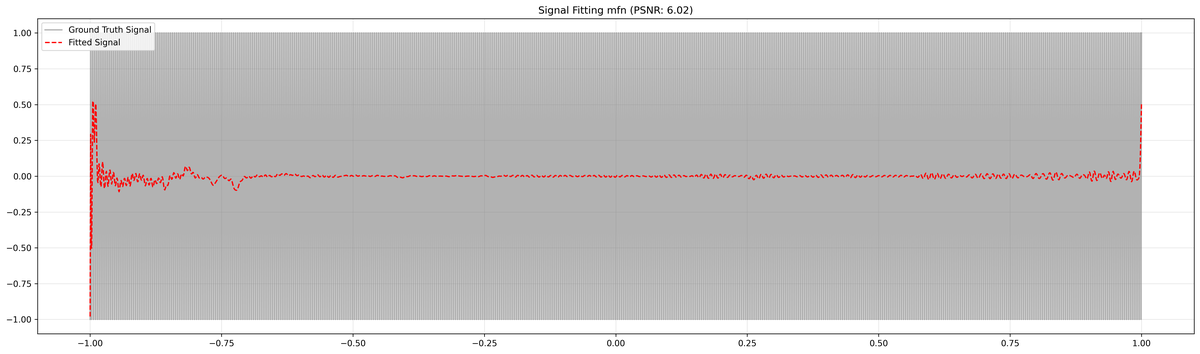} & 
        \includegraphics[width=0.16\linewidth, valign=c]{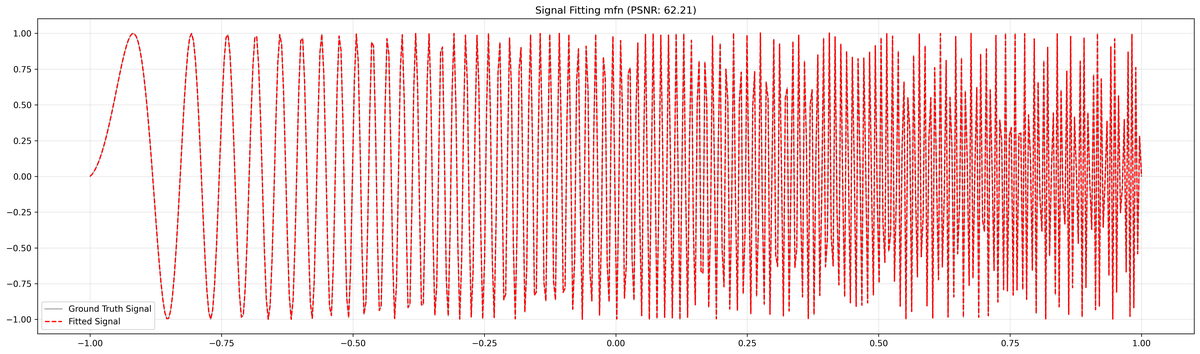} & 
        \includegraphics[width=0.16\linewidth, valign=c]{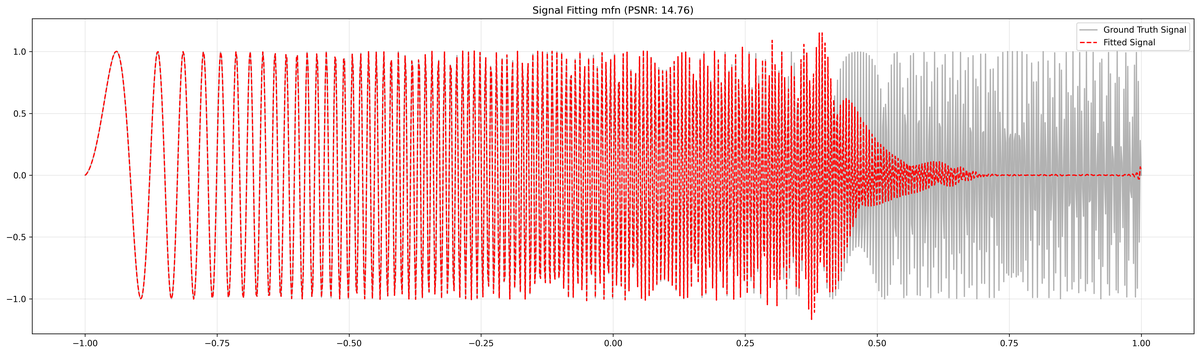} & 
        \includegraphics[width=0.16\linewidth, valign=c]{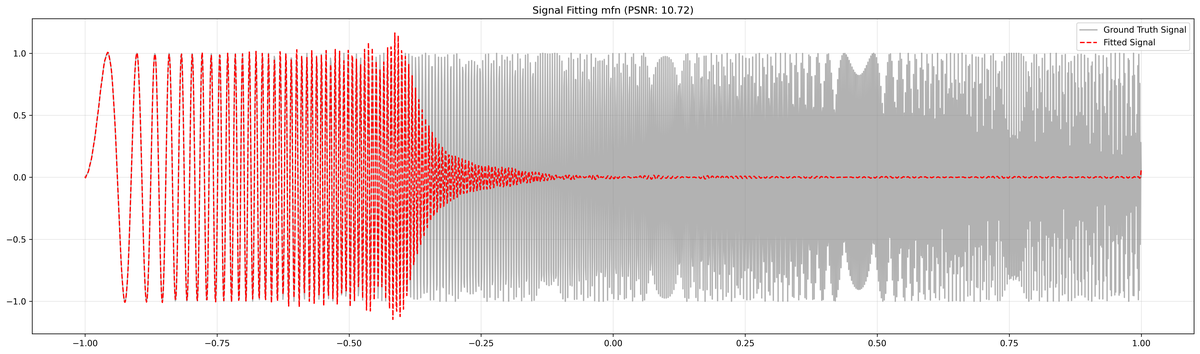} \\

        {\tiny Fourier} &
        \includegraphics[width=0.16\linewidth, valign=c]{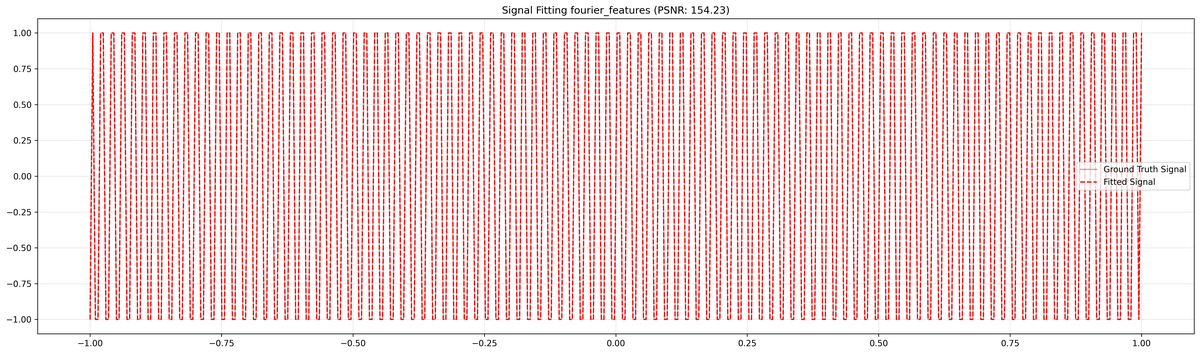} & 
        \includegraphics[width=0.16\linewidth, valign=c]{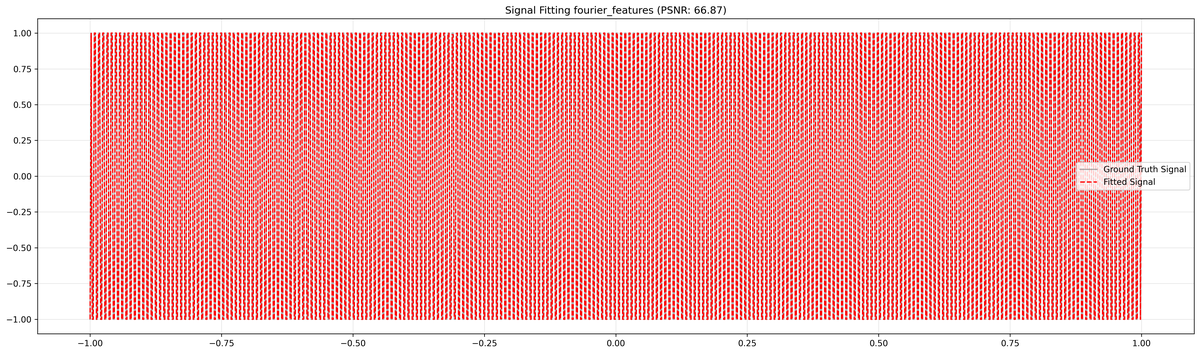} & 
        \includegraphics[width=0.16\linewidth, valign=c]{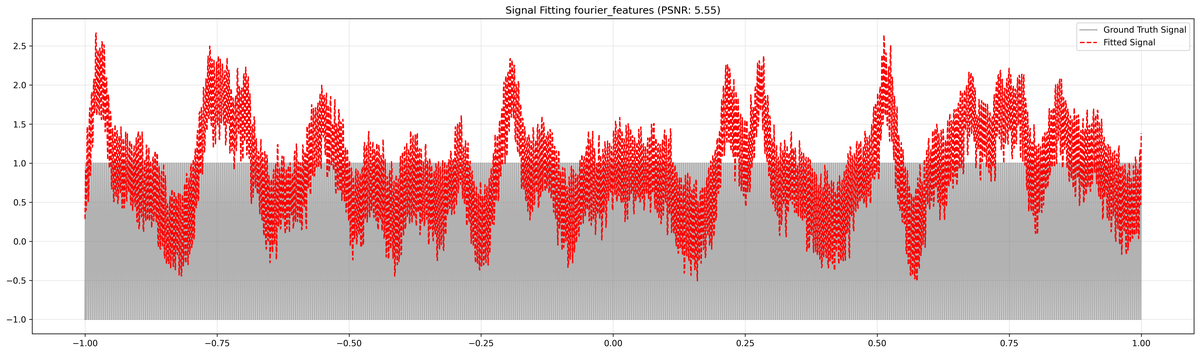} & 
        \includegraphics[width=0.16\linewidth, valign=c]{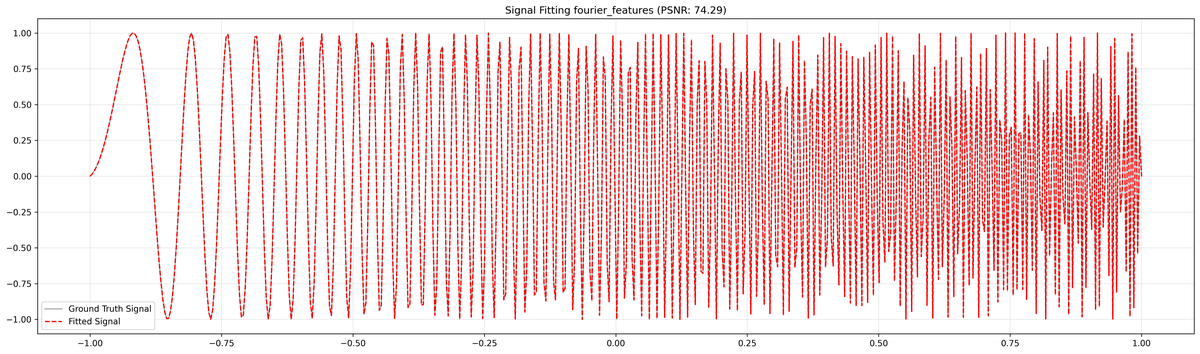} & 
        \includegraphics[width=0.16\linewidth, valign=c]{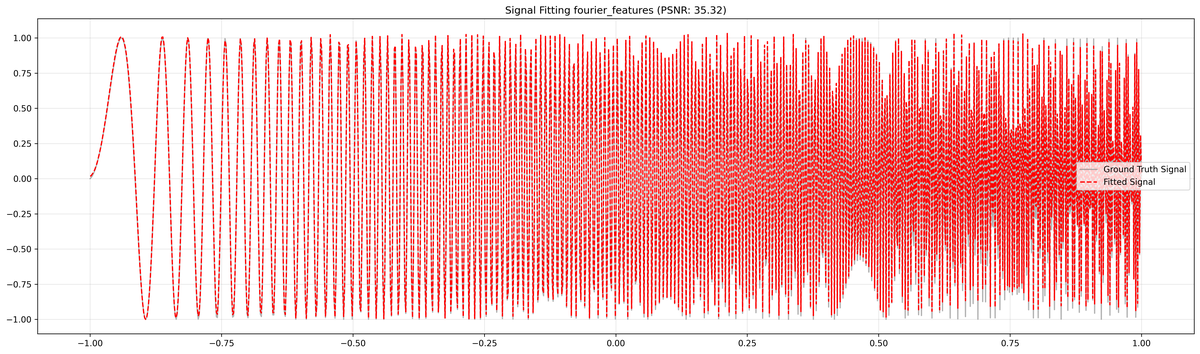} & 
        \includegraphics[width=0.16\linewidth, valign=c]{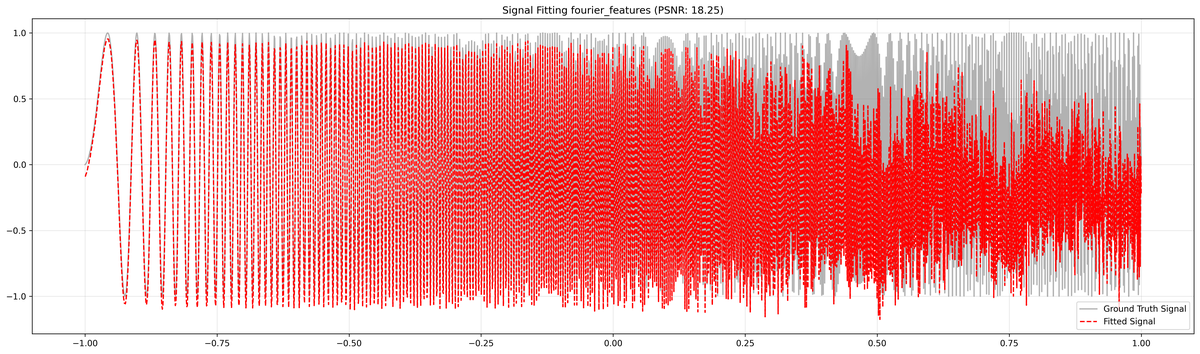} \\

        {\tiny FR} &
        \includegraphics[width=0.16\linewidth, valign=c]{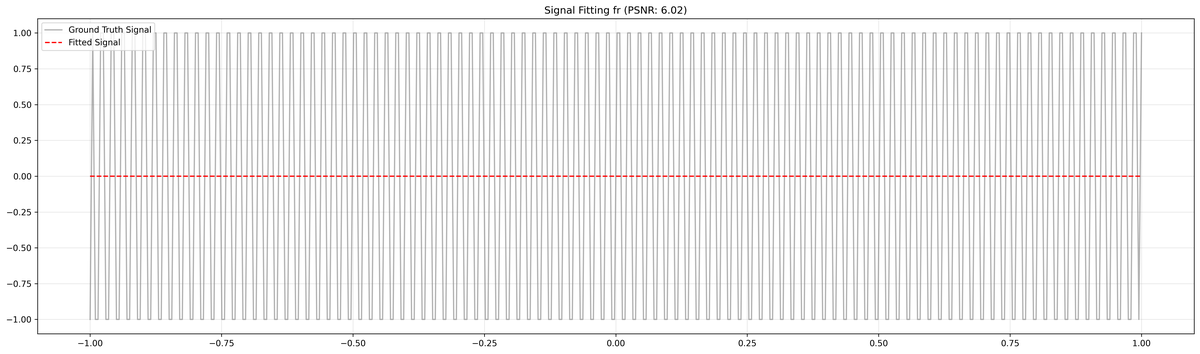} & 
        \includegraphics[width=0.16\linewidth, valign=c]{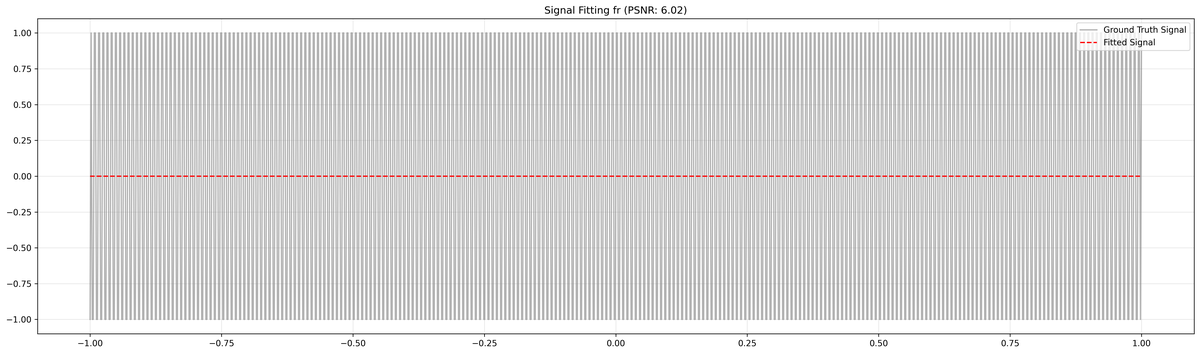} & 
        \includegraphics[width=0.16\linewidth, valign=c]{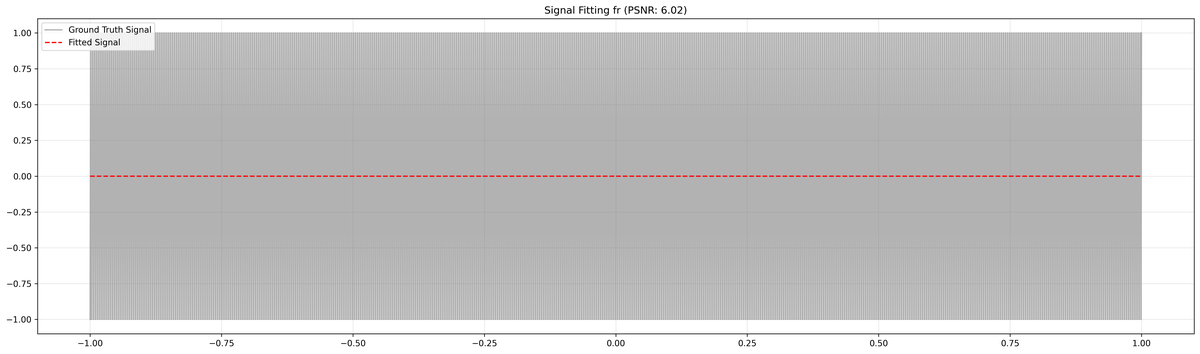} & 
        \includegraphics[width=0.16\linewidth, valign=c]{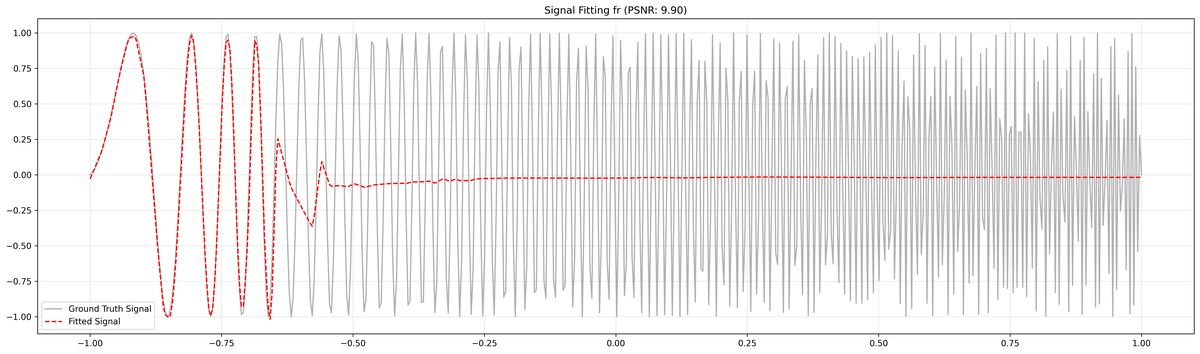} & 
        \includegraphics[width=0.16\linewidth, valign=c]{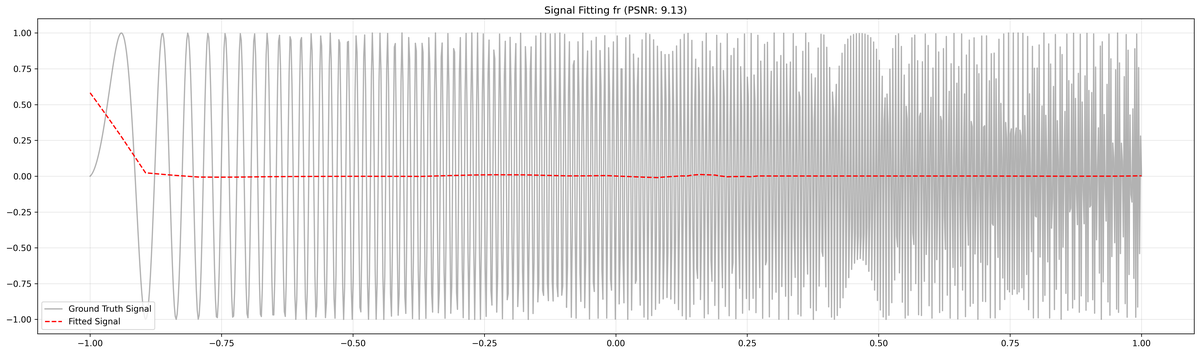} & 
        \includegraphics[width=0.16\linewidth, valign=c]{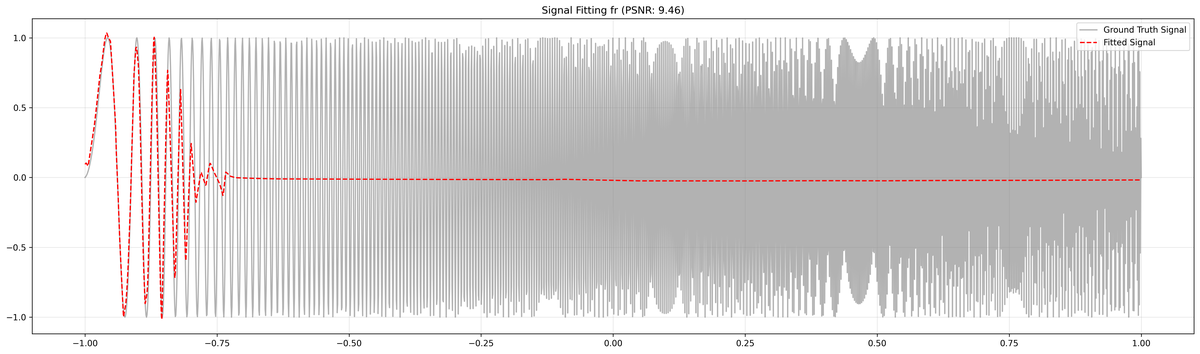} \\

    \end{tabular}
    }
    \caption{
    \textbf{Qualitative results on Signal Fitting.}
    The red line represents the fitted signal, while the grey line represents the ground-truth signal.
    Please zoom to appreciate the differences.
    }
    \label{fig:qual_signal_fitting}
\end{figure}

        \begin{figure}[t]
    \centering
    \resizebox{\linewidth}{!}{
    \begin{tabular}{ccc}        

        & \textbf{Bach} & \textbf{Counting} \\
                
        {\tiny \algoname{}} &
        \includegraphics[width=0.475\linewidth, valign=c]{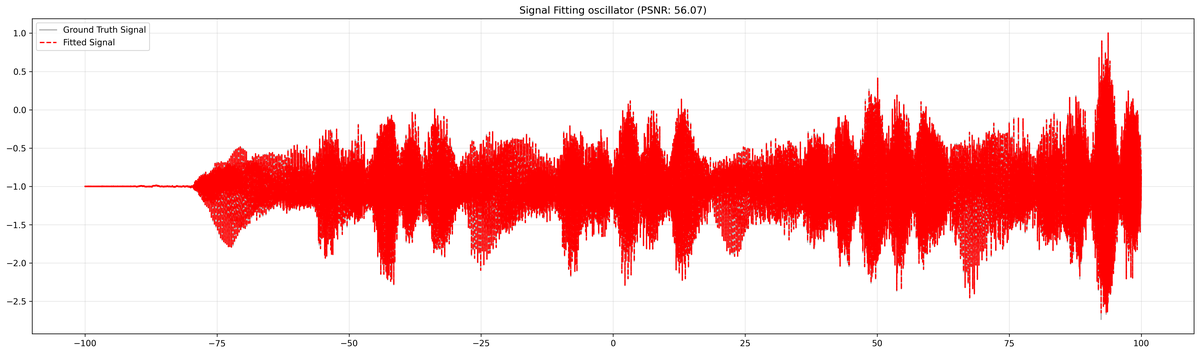} & 
        \includegraphics[width=0.475\linewidth, valign=c]{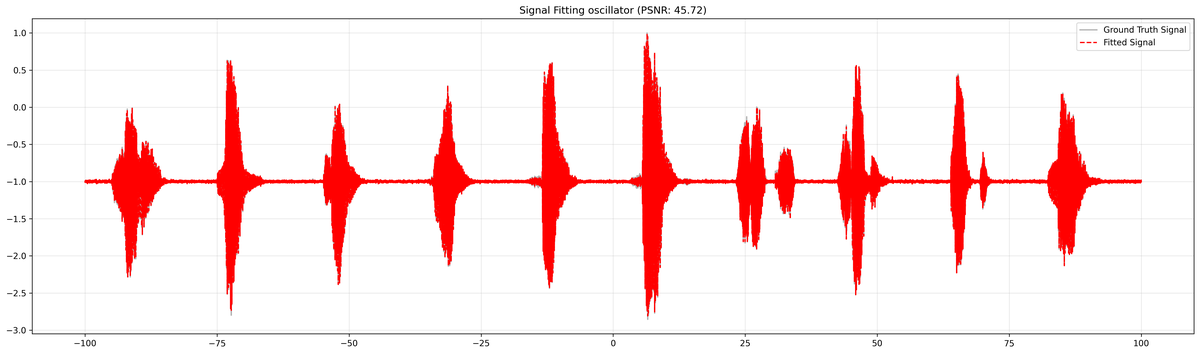} \\

        {\tiny SIREN} &
        \includegraphics[width=0.475\linewidth, valign=c]{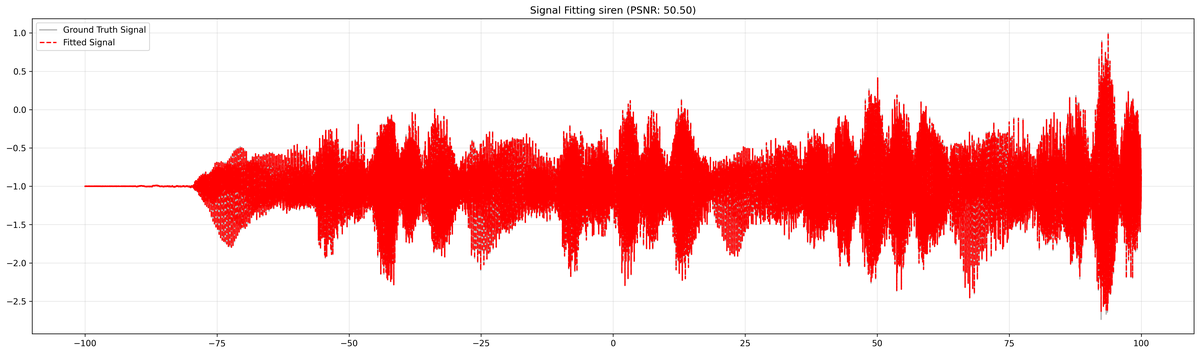} & 
        \includegraphics[width=0.475\linewidth, valign=c]{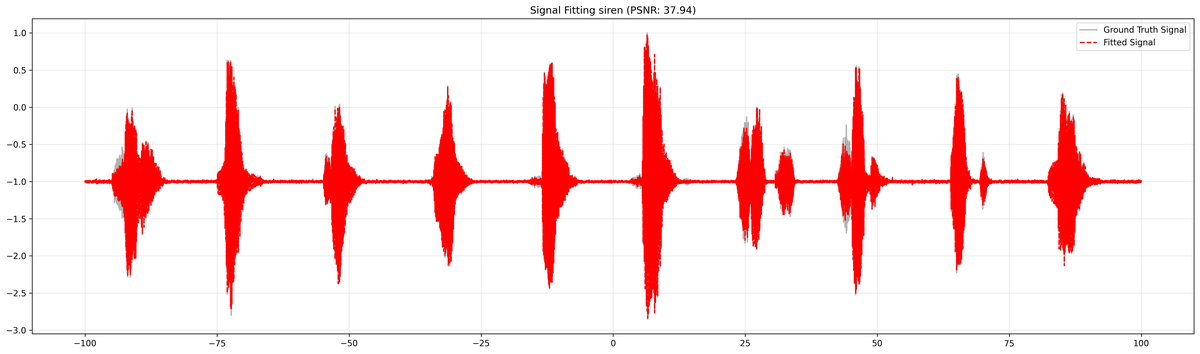} \\

        {\tiny Gauss} &
        \includegraphics[width=0.475\linewidth, valign=c]{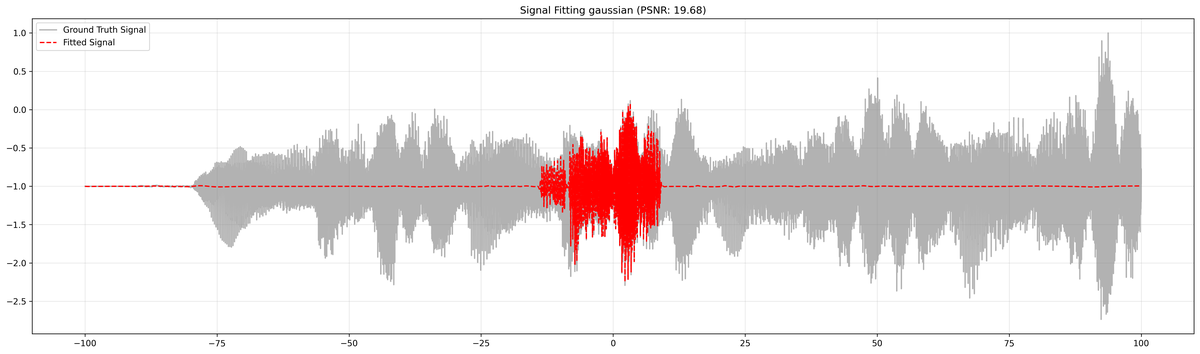} & 
        \includegraphics[width=0.475\linewidth, valign=c]{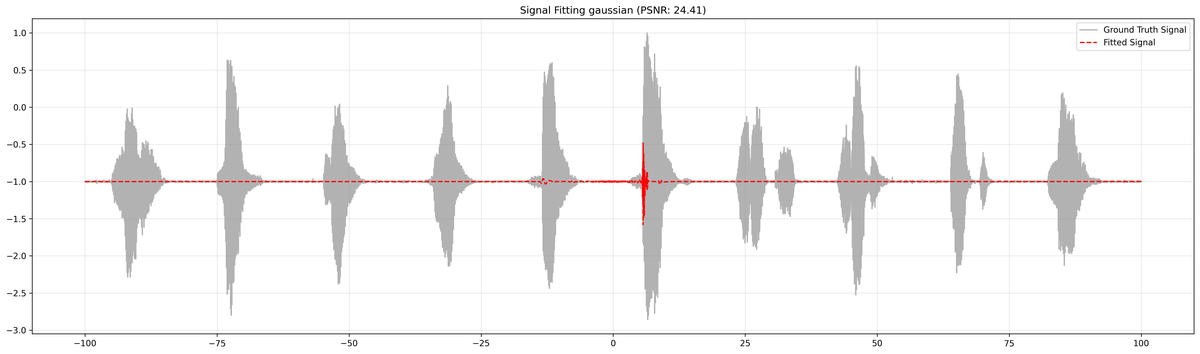} \\

        {\tiny WIRE} &
        \includegraphics[width=0.475\linewidth, valign=c]{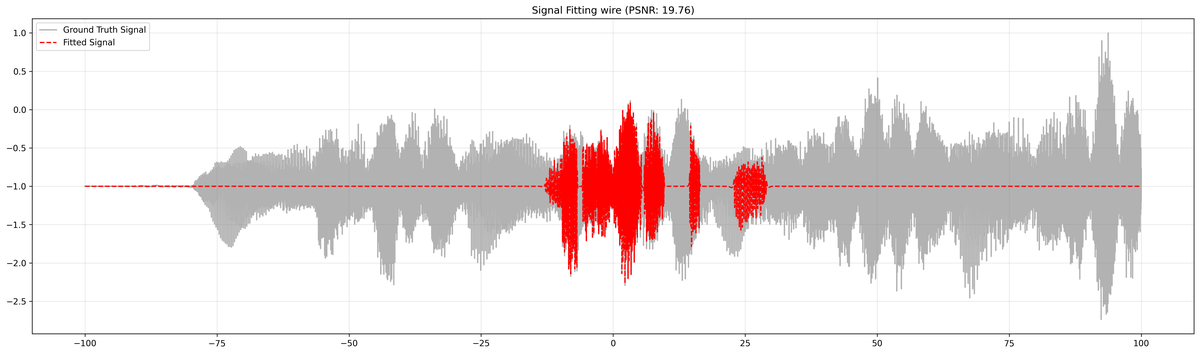} & 
        \includegraphics[width=0.475\linewidth, valign=c]{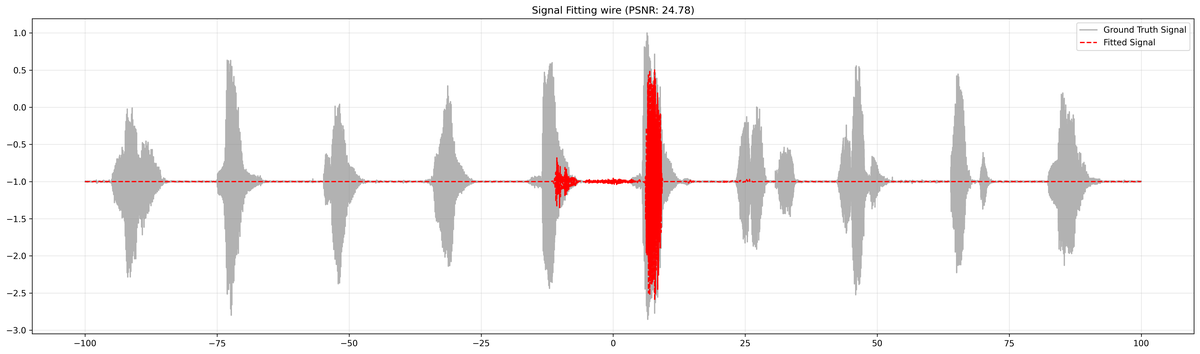} \\

        {\tiny BACON} &
        \includegraphics[width=0.475\linewidth, valign=c]{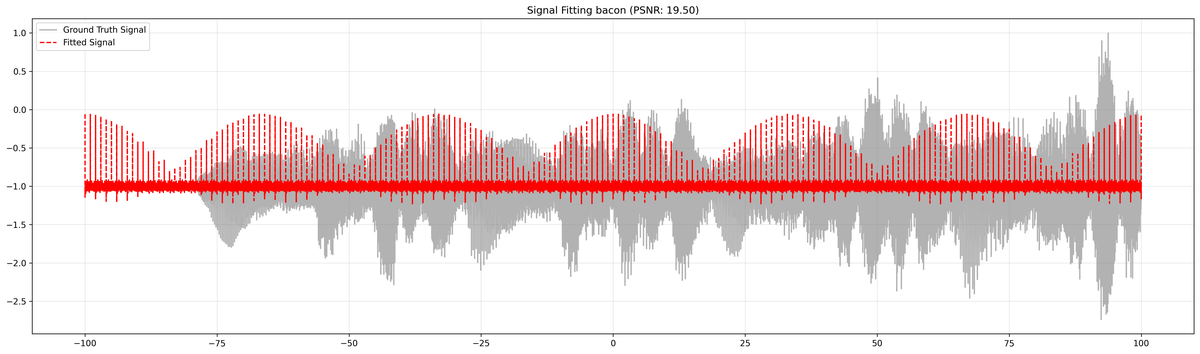} & 
        \includegraphics[width=0.475\linewidth, valign=c]{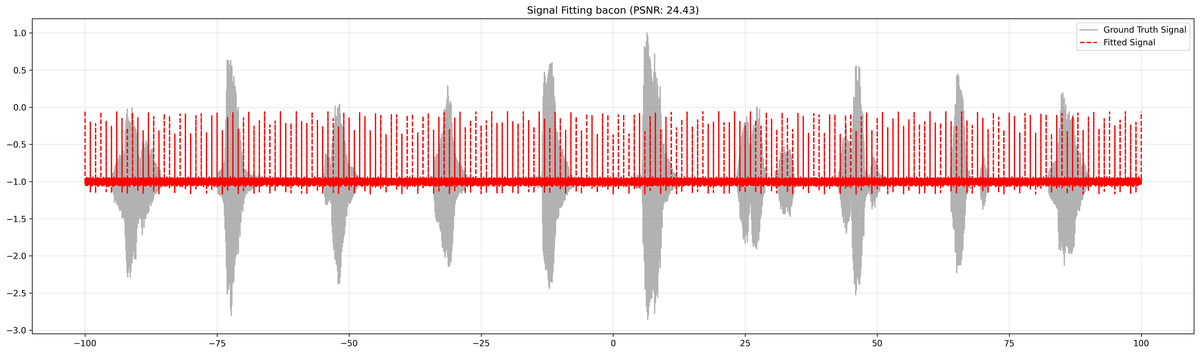} \\

        {\tiny FINER} &
        \includegraphics[width=0.475\linewidth, valign=c]{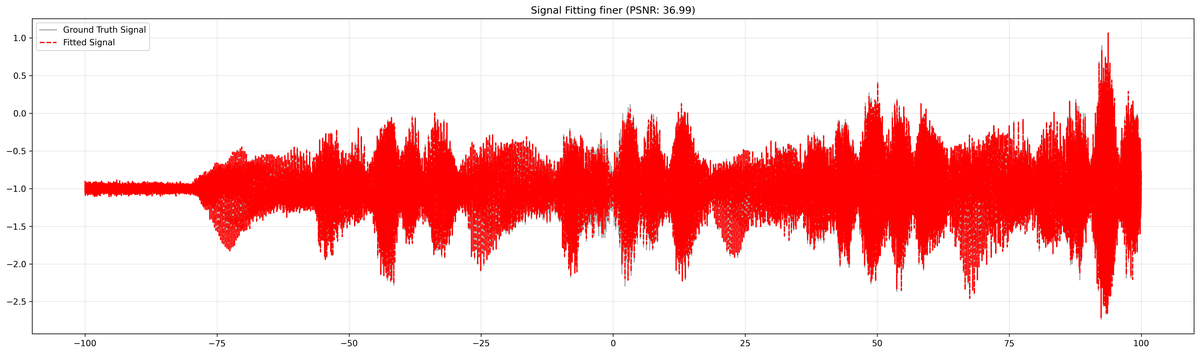} & 
        \includegraphics[width=0.475\linewidth, valign=c]{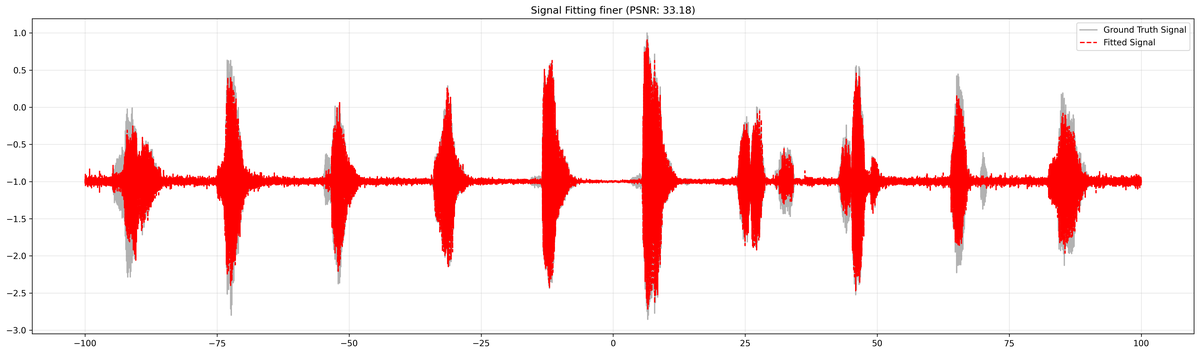} \\

        {\tiny MFN} &
        \includegraphics[width=0.475\linewidth, valign=c]{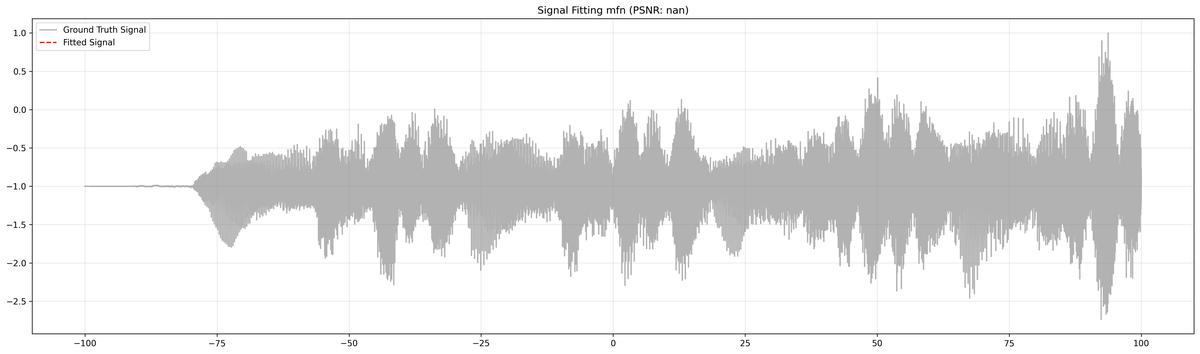} & 
        \includegraphics[width=0.475\linewidth, valign=c]{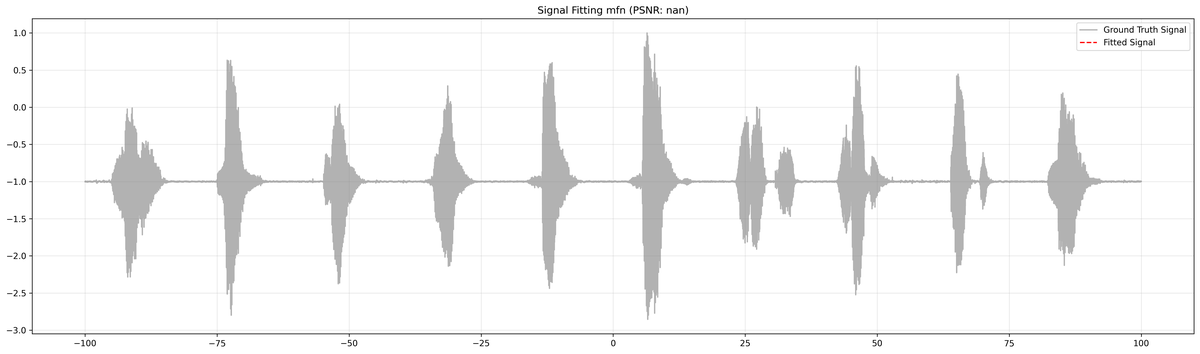} \\

        {\tiny Fourier} &
        \includegraphics[width=0.475\linewidth, valign=c]{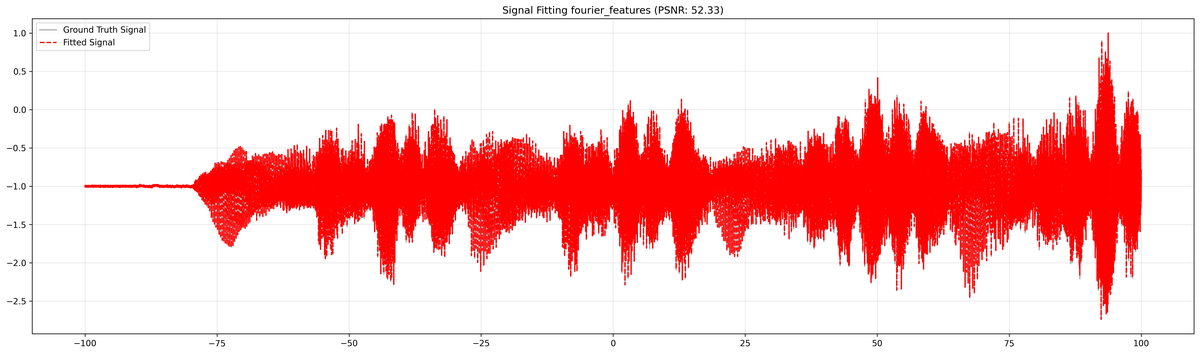} & 
        \includegraphics[width=0.475\linewidth, valign=c]{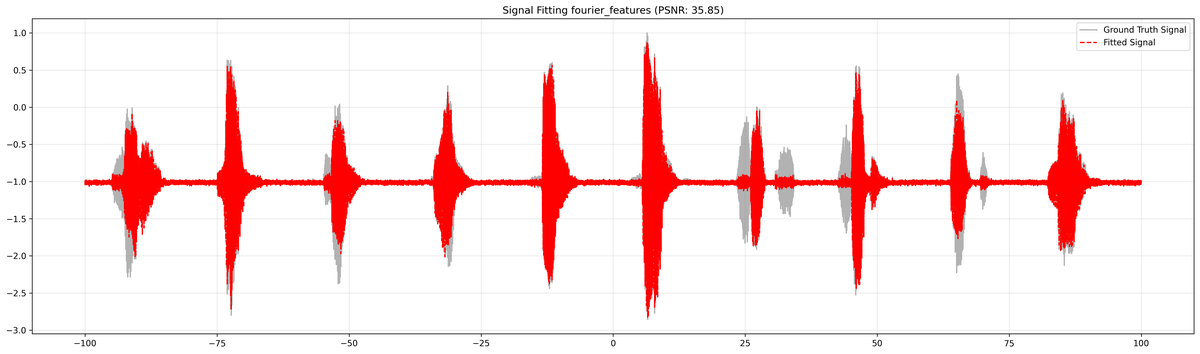} \\

        {\tiny FR} &
        \includegraphics[width=0.475\linewidth, valign=c]{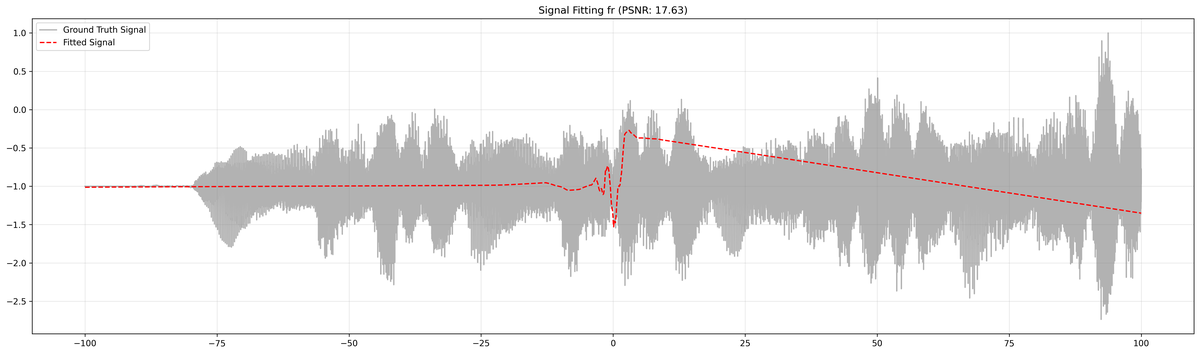} & 
        \includegraphics[width=0.475\linewidth, valign=c]{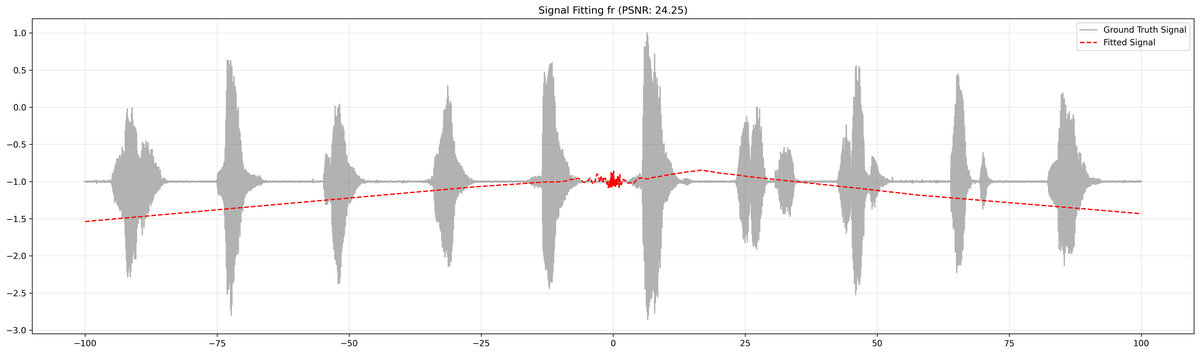} \\

    \end{tabular}
    }
    \caption{
    \textbf{Qualitative results on Audio Fitting.}
    The red line represents the fitted signal, while the grey line represents the ground-truth signal.
    Please zoom to appreciate the differences.
    }
    \label{fig:qual_audio_fitting}
\end{figure}

        We provide in \cref{fig:qual_tiger}, \cref{fig:qual_tiles}, \cref{fig:qual_bikers}, \cref{fig:qual_butterfly} and \cref{fig:qual_knot} all the qualitative results on the image-related tasks, already partially provided in \cref{fig:qualitatives_transposed} of \cref{sec:experiments}.
        
        \begin{figure}[t]
    \centering
    \resizebox{\linewidth}{!}{
    \begin{tabular}{lcc @{\hspace{0.75em}} ccccccccc}        

        & {\tiny Supervision} 
        & {\tiny GT} 
        & {\tiny \algoname{}} 
        & {\tiny SIREN} 
        & {\tiny Gauss} 
        & {\tiny WIRE} 
        & {\tiny BACON} 
        & {\tiny FINER} 
        & {\tiny MFN} 
        & {\tiny Fourier} 
        & {\tiny FR} \\
        
        \rotatebox{90}{\tiny \hspace{0.025em} \emph{Fitting}} &
        \includegraphics[width=0.09\linewidth]{figures/qualitatives/supervision/tiger_fitting.png} & 
        \includegraphics[width=0.09\linewidth]{figures/qualitatives/supervision/tiger_fitting.png} & 
        \includegraphics[width=0.09\linewidth]{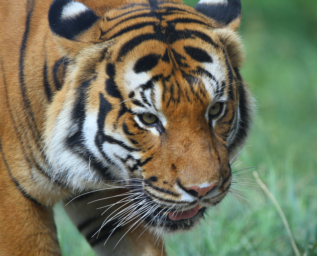} & 
        \includegraphics[width=0.09\linewidth]{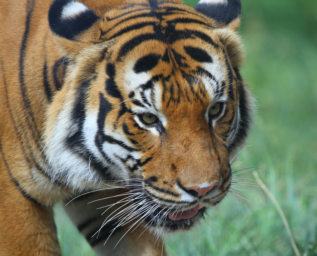} & 
        \includegraphics[width=0.09\linewidth]{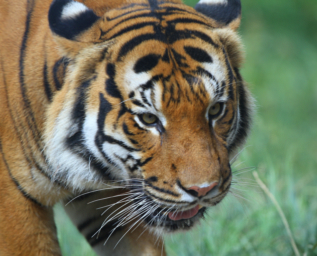} & 
        \includegraphics[width=0.09\linewidth]{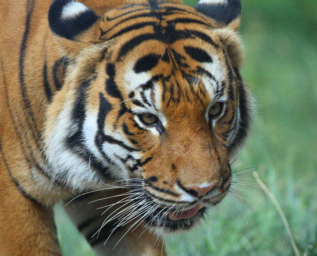} & 
        \includegraphics[width=0.09\linewidth]{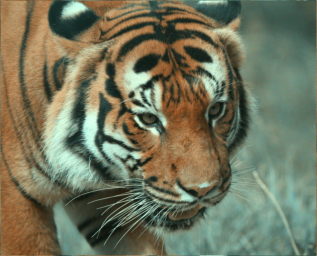} & 
        \includegraphics[width=0.09\linewidth]{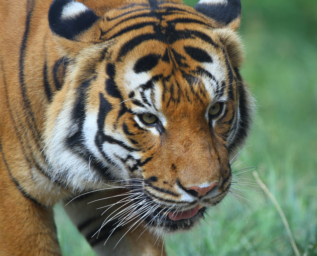} & 
        \includegraphics[width=0.09\linewidth]{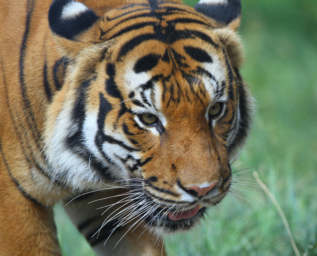} & 
        \includegraphics[width=0.09\linewidth]{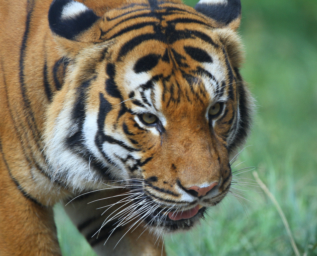} &
        \includegraphics[width=0.09\linewidth]{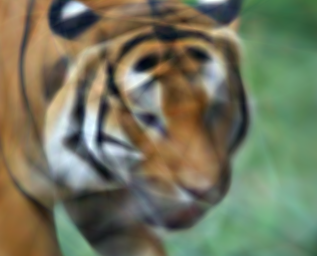} \\

        \rotatebox{90}{\tiny \hspace{0.25em} \emph{Den.}} &
        \includegraphics[width=0.09\linewidth]{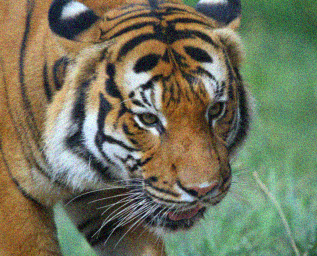} & 
        \includegraphics[width=0.09\linewidth]{figures/qualitatives/supervision/tiger_fitting.png} & 
        \includegraphics[width=0.09\linewidth]{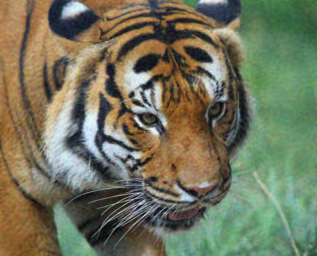} & 
        \includegraphics[width=0.09\linewidth]{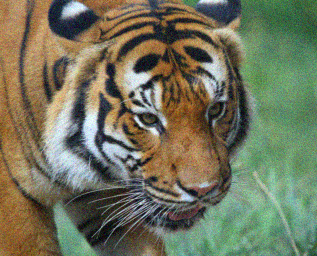} & 
        \includegraphics[width=0.09\linewidth]{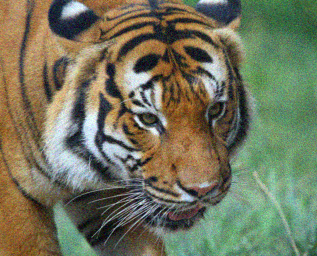} & 
        \includegraphics[width=0.09\linewidth]{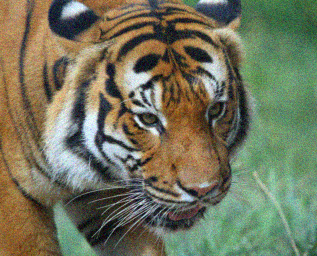} & 
        \includegraphics[width=0.09\linewidth]{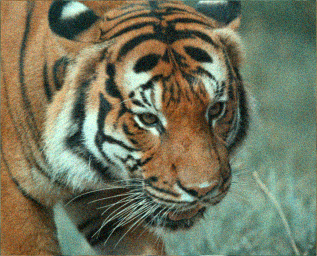} & 
        \includegraphics[width=0.09\linewidth]{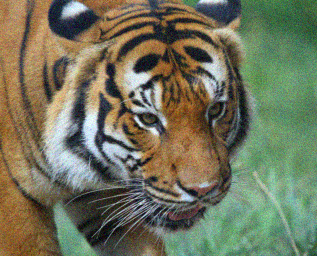} & 
        \includegraphics[width=0.09\linewidth]{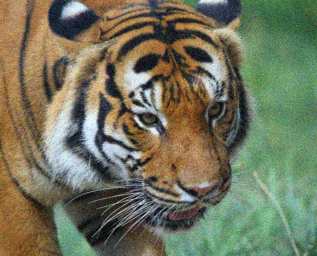} & 
        \includegraphics[width=0.09\linewidth]{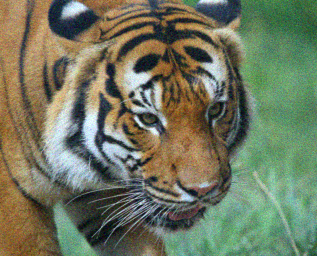} &
        \includegraphics[width=0.09\linewidth]{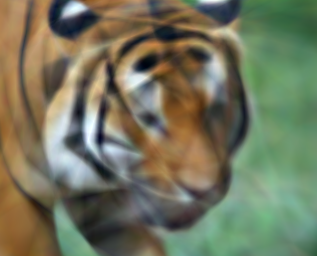} \\

        \rotatebox{90}{\tiny \hspace{0.4em} \emph{Inp.}} &
        \includegraphics[width=0.09\linewidth]{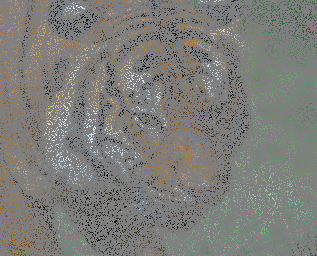} & 
        \includegraphics[width=0.09\linewidth]{figures/qualitatives/supervision/tiger_fitting.png} & 
        \includegraphics[width=0.09\linewidth]{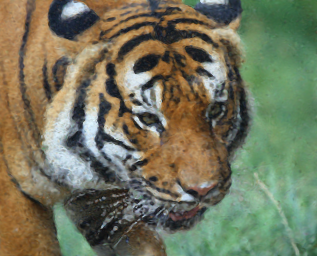} & 
        \includegraphics[width=0.09\linewidth]{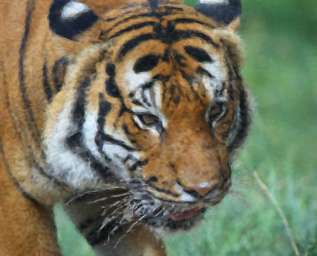} & 
        \includegraphics[width=0.09\linewidth]{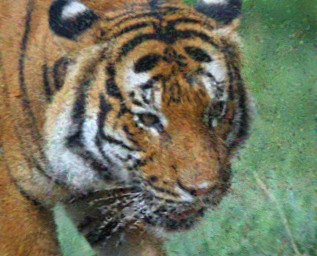} & 
        \includegraphics[width=0.09\linewidth]{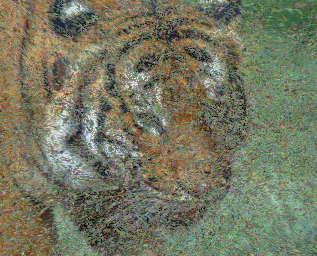} & 
        \includegraphics[width=0.09\linewidth]{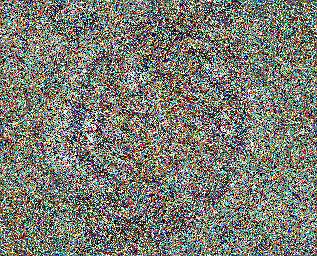} & 
        \includegraphics[width=0.09\linewidth]{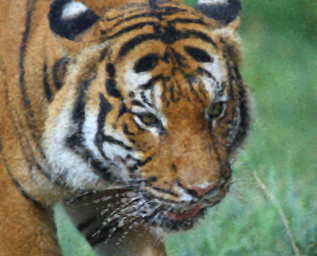} & 
        \includegraphics[width=0.09\linewidth]{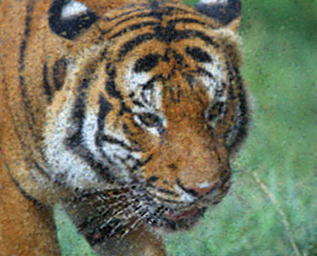} & 
        \includegraphics[width=0.09\linewidth]{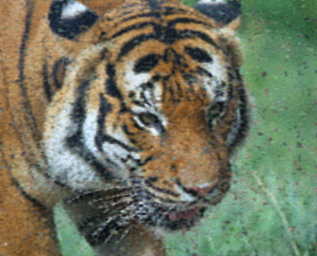} &
        \includegraphics[width=0.09\linewidth]{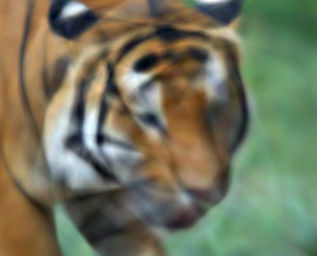} \\

        \rotatebox{90}{\tiny \emph{SR}} &
        \makebox[0.09\linewidth][c]{\includegraphics[width=0.05\linewidth, valign=m]{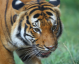}} &
        \includegraphics[width=0.09\linewidth, valign=m]{figures/qualitatives/supervision/tiger_fitting.png} & 
        \includegraphics[width=0.09\linewidth, valign=m]{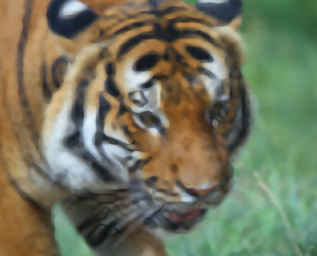} & 
        \includegraphics[width=0.09\linewidth, valign=m]{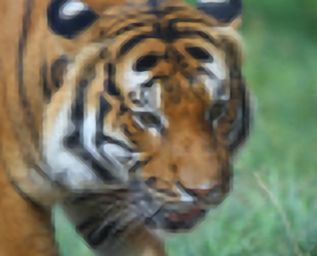} & 
        \includegraphics[width=0.09\linewidth, valign=m]{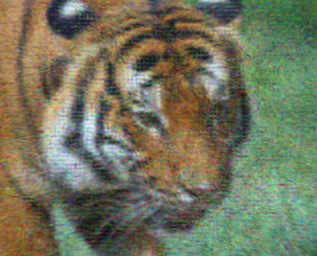} & 
        \includegraphics[width=0.09\linewidth, valign=m]{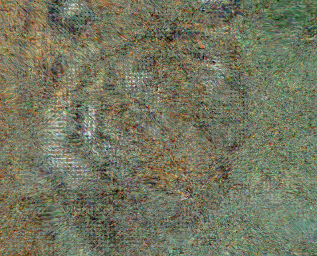} & 
        \includegraphics[width=0.09\linewidth, valign=m]{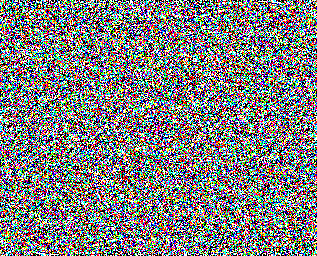} & 
        \includegraphics[width=0.09\linewidth, valign=m]{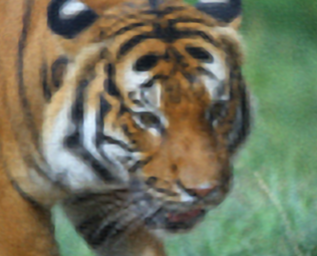} & 
        \includegraphics[width=0.09\linewidth, valign=m]{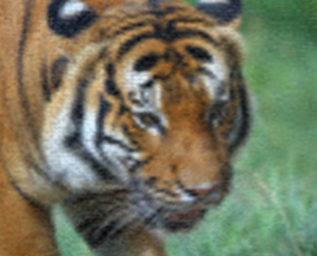} & 
        \includegraphics[width=0.09\linewidth, valign=m]{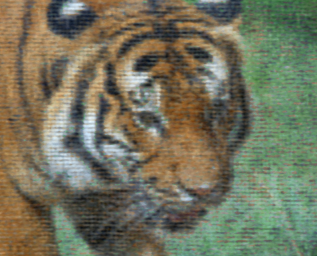} &
        \includegraphics[width=0.09\linewidth, valign=m]{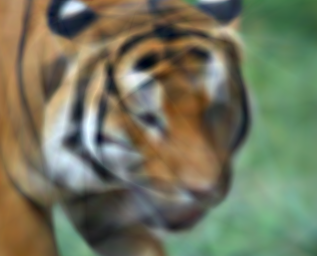} \\[2.25ex]
        
        \rotatebox{90}{\tiny \emph{Poisson}} &
        \includegraphics[width=0.09\linewidth]{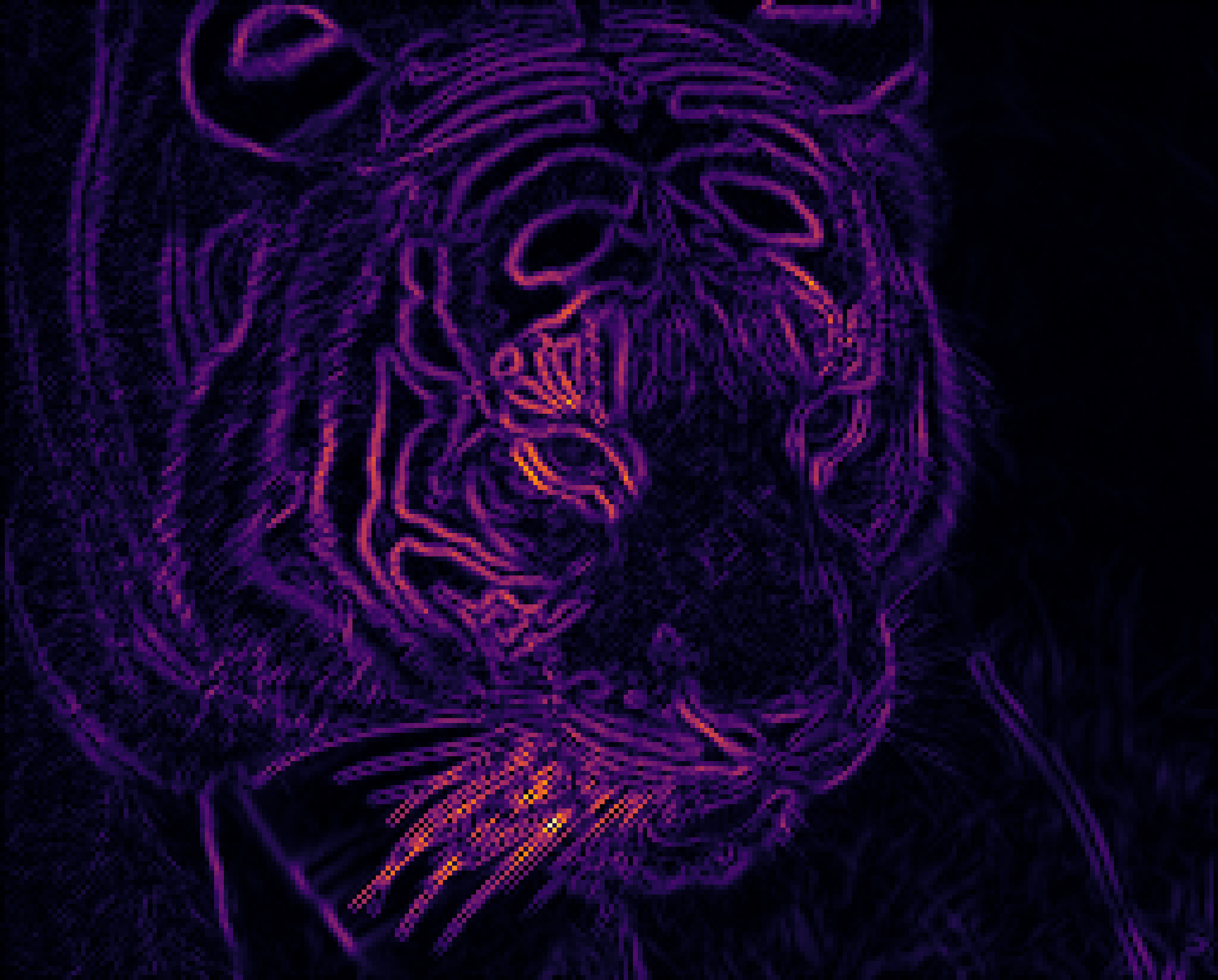} & 
        \includegraphics[width=0.09\linewidth]{figures/qualitatives/supervision/tiger_fitting.png} & 
        \includegraphics[width=0.09\linewidth]{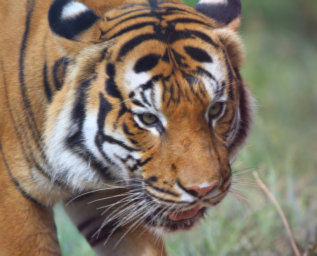} & 
        \includegraphics[width=0.09\linewidth]{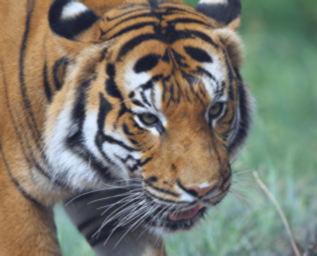} & 
        \includegraphics[width=0.09\linewidth]{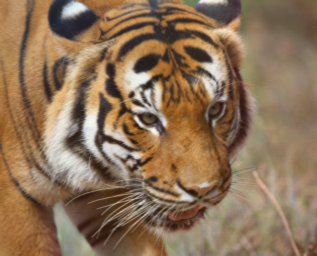} & 
        \includegraphics[width=0.09\linewidth]{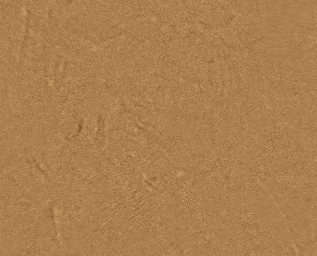} & 
        \includegraphics[width=0.09\linewidth]{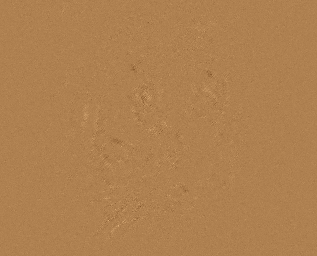} & 
        \includegraphics[width=0.09\linewidth]{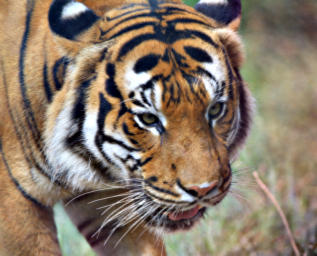} & 
        \includegraphics[width=0.09\linewidth]{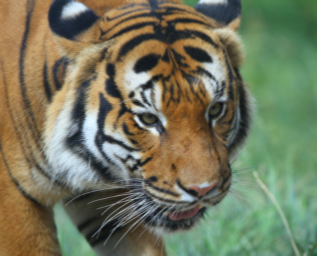} & 
        \includegraphics[width=0.09\linewidth]{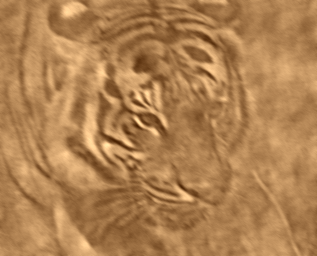} &
        \includegraphics[width=0.09\linewidth]{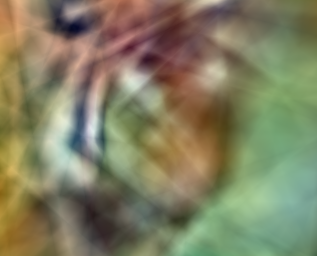} \\

    \end{tabular}
    }
    \caption{
    \textbf{Qualitative results on Tiger.}
    Please zoom to appreciate the differences.
    }
    \label{fig:qual_tiger}
\end{figure}

        \begin{figure}[t]
    \centering
    \resizebox{\linewidth}{!}{
    \begin{tabular}{lcc @{\hspace{0.75em}} ccccccccc}        

        & {\tiny Supervision} 
        & {\tiny GT} 
        & {\tiny \algoname{}} 
        & {\tiny SIREN} 
        & {\tiny Gauss} 
        & {\tiny WIRE} 
        & {\tiny BACON} 
        & {\tiny FINER} 
        & {\tiny MFN} 
        & {\tiny Fourier} 
        & {\tiny FR} \\
        
        \rotatebox{90}{\tiny \hspace{1.0em} \emph{Fitting}} &
        \includegraphics[width=0.09\linewidth]{figures/qualitatives/supervision/tiles_fitting.png} & 
        \includegraphics[width=0.09\linewidth]{figures/qualitatives/supervision/tiles_fitting.png} & 
        \includegraphics[width=0.09\linewidth]{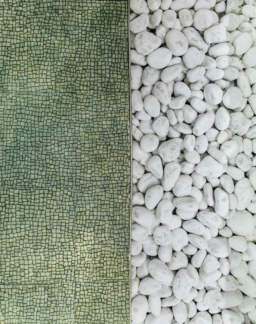} & 
        \includegraphics[width=0.09\linewidth]{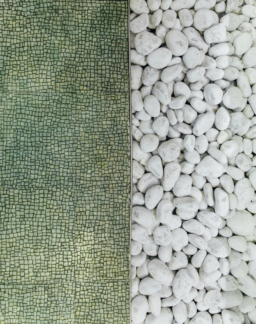} & 
        \includegraphics[width=0.09\linewidth]{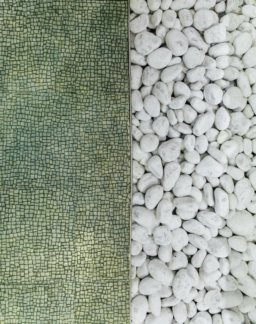} & 
        \includegraphics[width=0.09\linewidth]{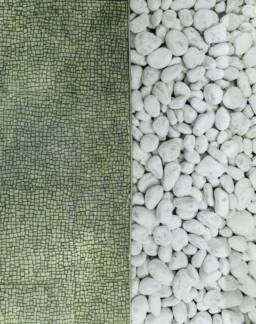} & 
        \includegraphics[width=0.09\linewidth]{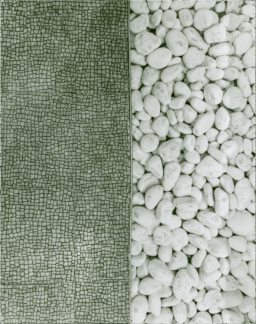} & 
        \includegraphics[width=0.09\linewidth]{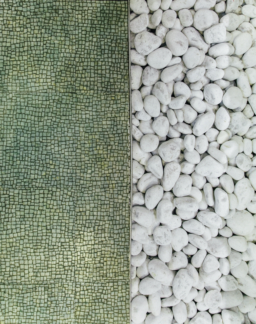} & 
        \includegraphics[width=0.09\linewidth]{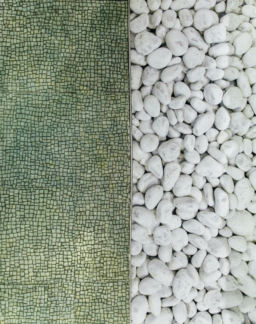} & 
        \includegraphics[width=0.09\linewidth]{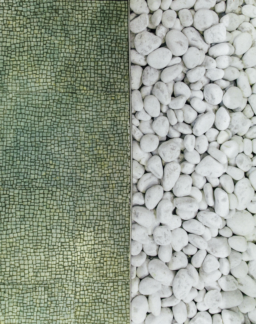} &
        \includegraphics[width=0.09\linewidth]{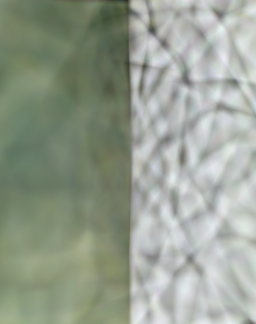} \\

        \rotatebox{90}{\tiny \hspace{0.25em} \emph{Denoising}} &
        \includegraphics[width=0.09\linewidth]{figures/qualitatives/supervision/tiles_denoising.png} & 
        \includegraphics[width=0.09\linewidth]{figures/qualitatives/supervision/tiles_fitting.png} & 
        \includegraphics[width=0.09\linewidth]{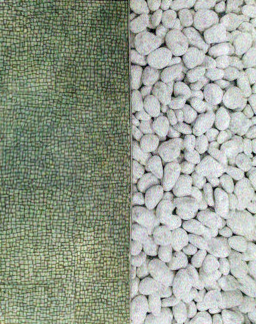} & 
        \includegraphics[width=0.09\linewidth]{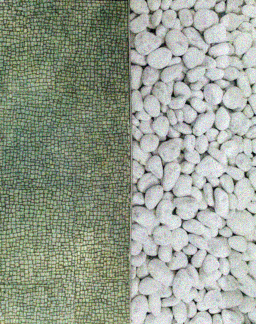} & 
        \includegraphics[width=0.09\linewidth]{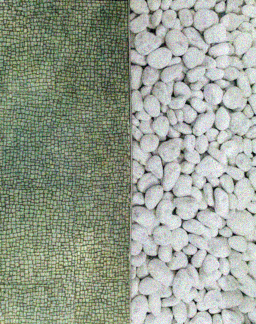} & 
        \includegraphics[width=0.09\linewidth]{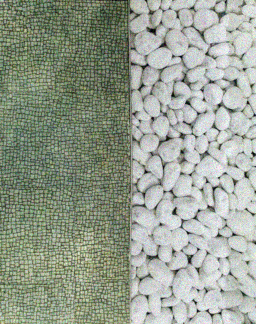} & 
        \includegraphics[width=0.09\linewidth]{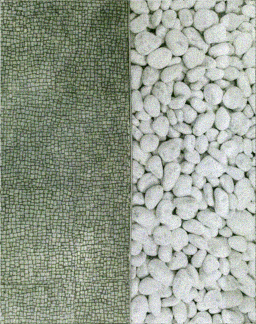} & 
        \includegraphics[width=0.09\linewidth]{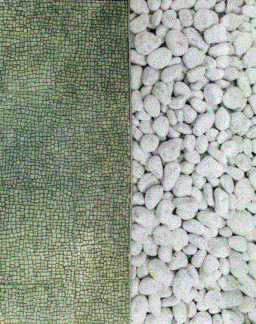} & 
        \includegraphics[width=0.09\linewidth]{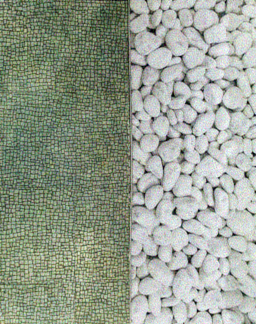} & 
        \includegraphics[width=0.09\linewidth]{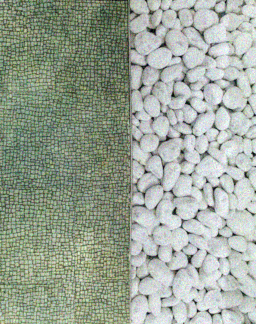} &
        \includegraphics[width=0.09\linewidth]{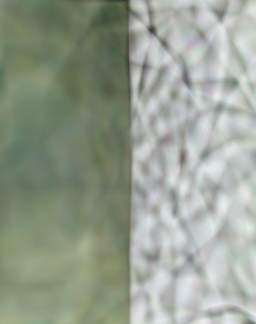} \\

        \rotatebox{90}{\tiny \hspace{0.1em} \emph{Inpainting}} &
        \includegraphics[width=0.09\linewidth]{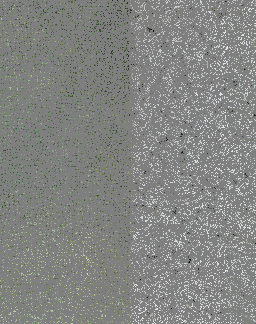} & 
        \includegraphics[width=0.09\linewidth]{figures/qualitatives/supervision/tiles_fitting.png} & 
        \includegraphics[width=0.09\linewidth]{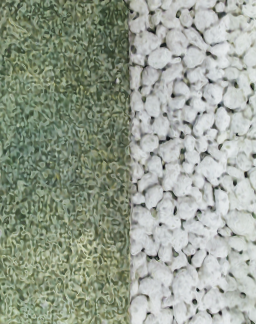} & 
        \includegraphics[width=0.09\linewidth]{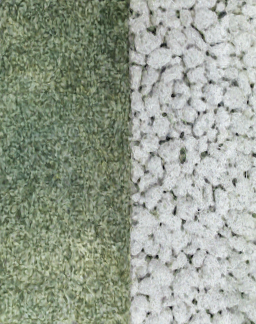} & 
        \includegraphics[width=0.09\linewidth]{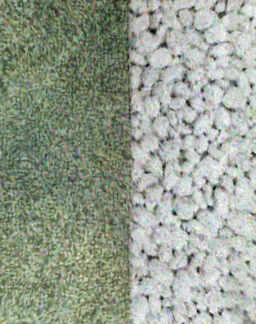} & 
        \includegraphics[width=0.09\linewidth]{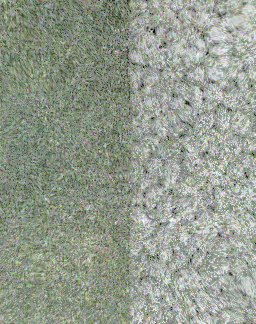} & 
        \includegraphics[width=0.09\linewidth]{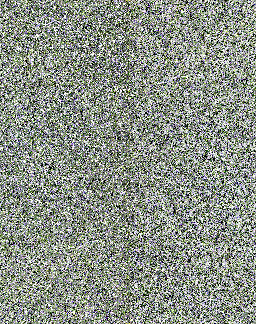} & 
        \includegraphics[width=0.09\linewidth]{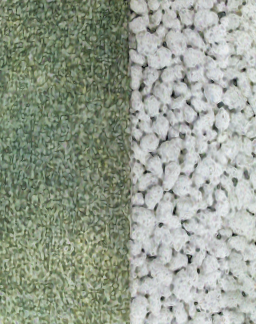} & 
        \includegraphics[width=0.09\linewidth]{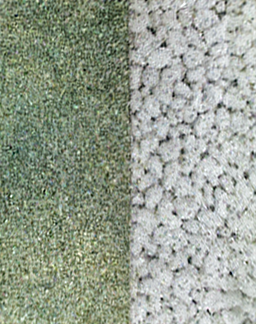} & 
        \includegraphics[width=0.09\linewidth]{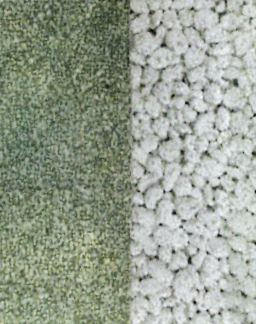} &
        \includegraphics[width=0.09\linewidth]{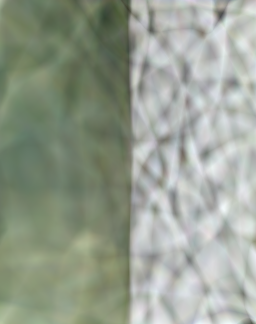} \\

        \rotatebox{90}{\tiny \emph{SR}} &
        \makebox[0.09\linewidth][c]{\includegraphics[width=0.05\linewidth, valign=m]{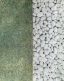}} &
        \includegraphics[width=0.09\linewidth, valign=m]{figures/qualitatives/supervision/tiles_fitting.png} & 
        \includegraphics[width=0.09\linewidth, valign=m]{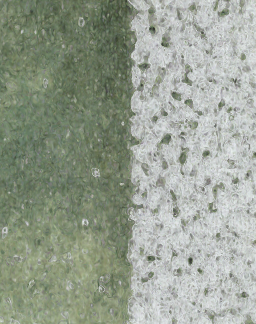} & 
        \includegraphics[width=0.09\linewidth, valign=m]{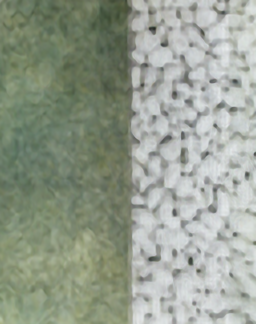} & 
        \includegraphics[width=0.09\linewidth, valign=m]{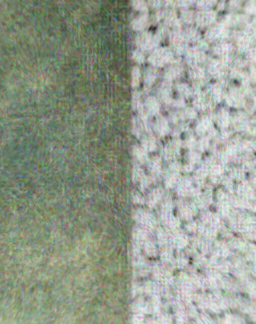} & 
        \includegraphics[width=0.09\linewidth, valign=m]{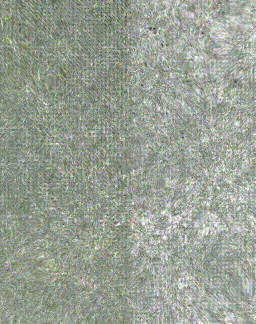} & 
        \includegraphics[width=0.09\linewidth, valign=m]{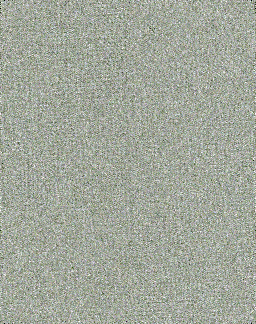} & 
        \includegraphics[width=0.09\linewidth, valign=m]{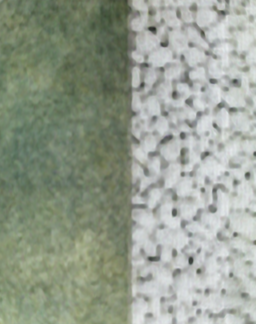} & 
        \includegraphics[width=0.09\linewidth, valign=m]{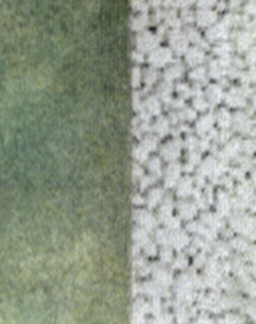} & 
        \includegraphics[width=0.09\linewidth, valign=m]{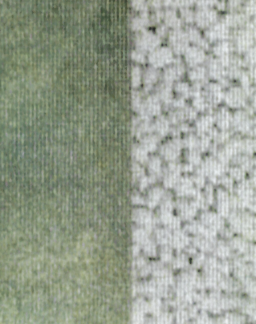} &
        \includegraphics[width=0.09\linewidth, valign=m]{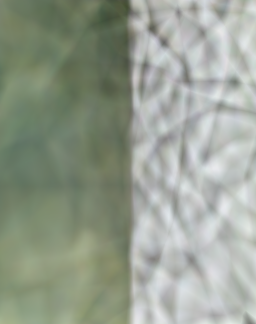} \\[4ex]
        
        \rotatebox{90}{\tiny \hspace{0.3em} \emph{Poisson}} &
        \includegraphics[width=0.09\linewidth]{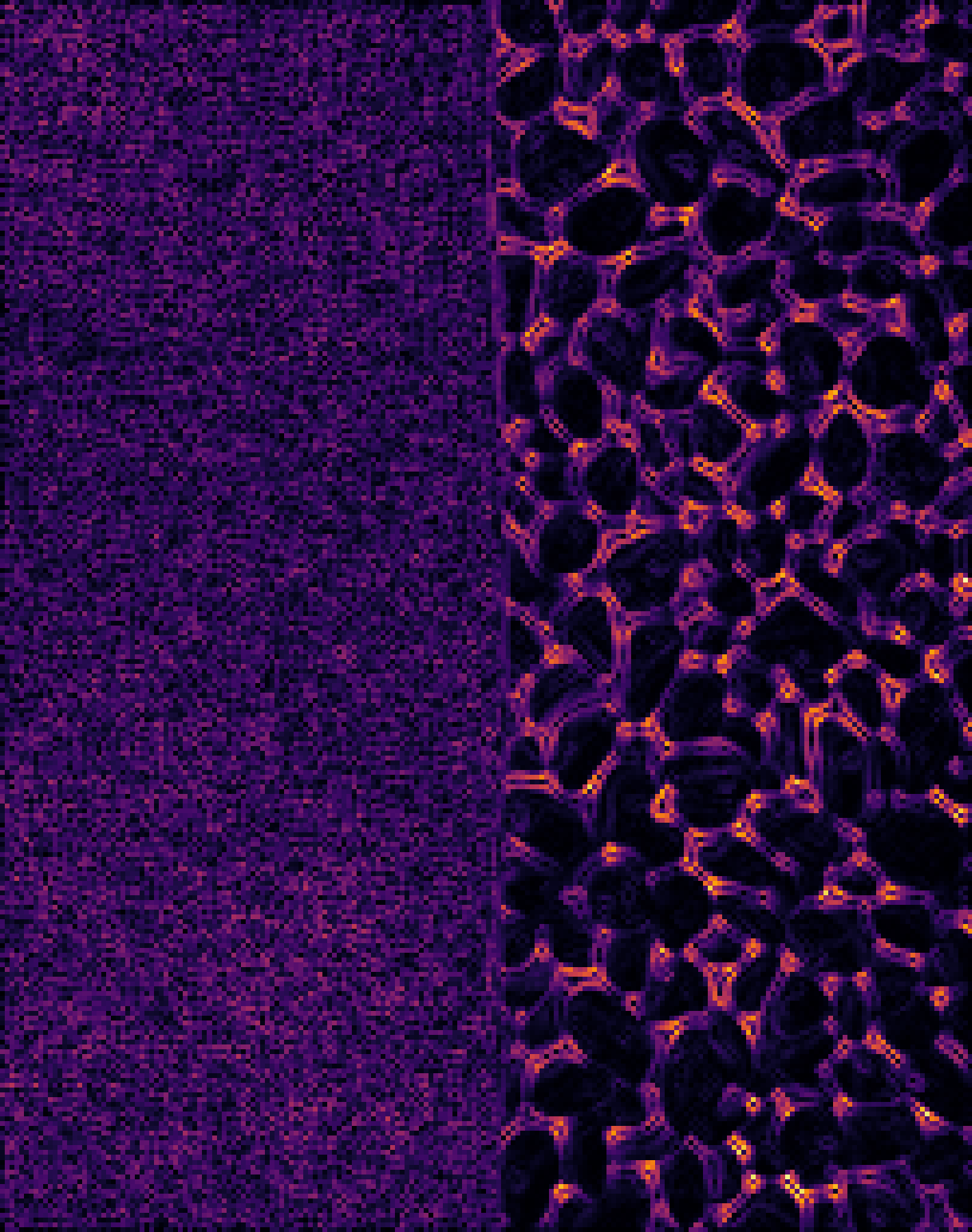} & 
        \includegraphics[width=0.09\linewidth]{figures/qualitatives/supervision/tiles_fitting.png} & 
        \includegraphics[width=0.09\linewidth]{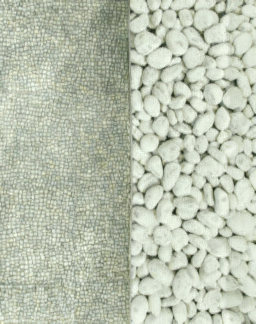} & 
        \includegraphics[width=0.09\linewidth]{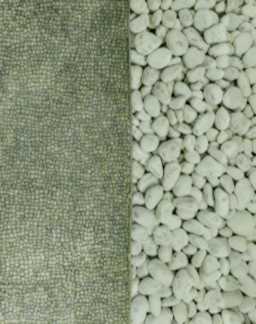} & 
        \includegraphics[width=0.09\linewidth]{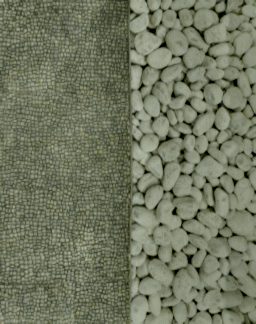} & 
        \includegraphics[width=0.09\linewidth]{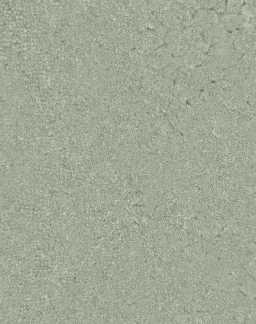} & 
        \includegraphics[width=0.09\linewidth]{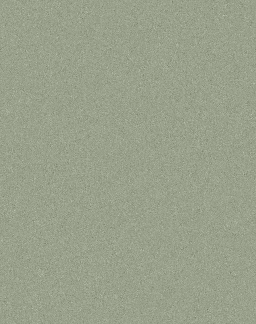} & 
        \includegraphics[width=0.09\linewidth]{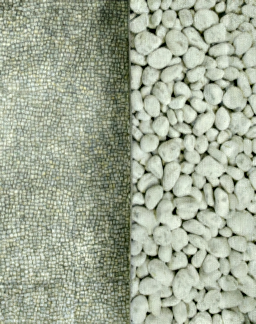} & 
        \includegraphics[width=0.09\linewidth]{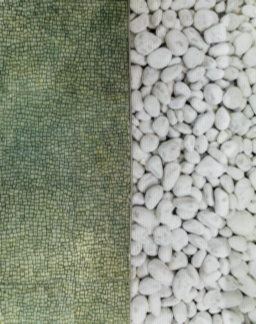} & 
        \includegraphics[width=0.09\linewidth]{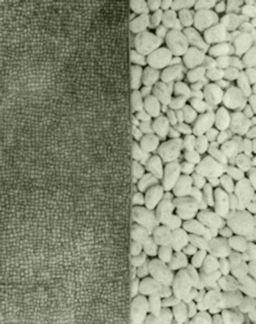} &
        \includegraphics[width=0.09\linewidth]{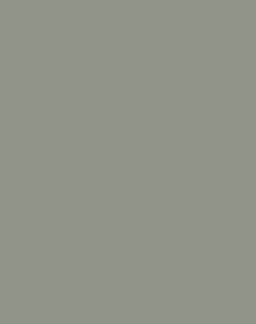} \\

    \end{tabular}
    }
    \caption{
    \textbf{Qualitative results on Tiles.}
    Please zoom to appreciate the differences.
    }
    \label{fig:qual_tiles}
\end{figure}

        \begin{figure}[t]
    \centering
    \resizebox{\linewidth}{!}{
    \begin{tabular}{lcc @{\hspace{0.75em}} ccccccccc}        

        & {\tiny Supervision} 
        & {\tiny GT} 
        & {\tiny \algoname{}} 
        & {\tiny SIREN} 
        & {\tiny Gauss} 
        & {\tiny WIRE} 
        & {\tiny BACON} 
        & {\tiny FINER} 
        & {\tiny MFN} 
        & {\tiny Fourier} 
        & {\tiny FR} \\
        
        \rotatebox{90}{\tiny \hspace{0.25em} \emph{Fit.}} &
        \includegraphics[width=0.09\linewidth]{figures/qualitatives/supervision/bikers_fitting.png} & 
        \includegraphics[width=0.09\linewidth]{figures/qualitatives/supervision/bikers_fitting.png} & 
        \includegraphics[width=0.09\linewidth]{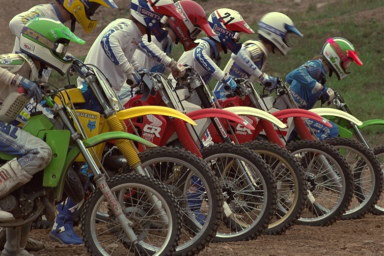} & 
        \includegraphics[width=0.09\linewidth]{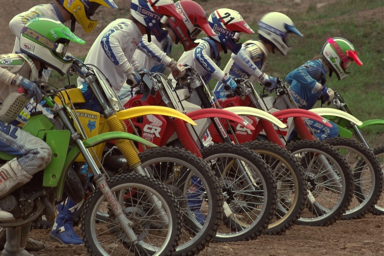} & 
        \includegraphics[width=0.09\linewidth]{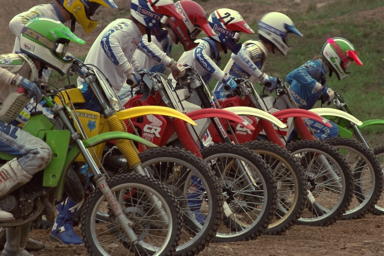} & 
        \includegraphics[width=0.09\linewidth]{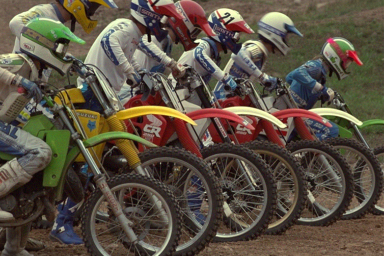} & 
        \includegraphics[width=0.09\linewidth]{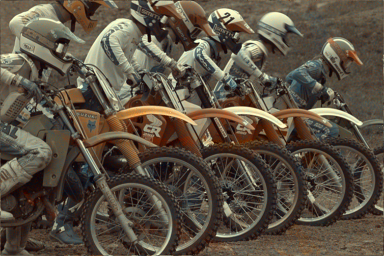} & 
        \includegraphics[width=0.09\linewidth]{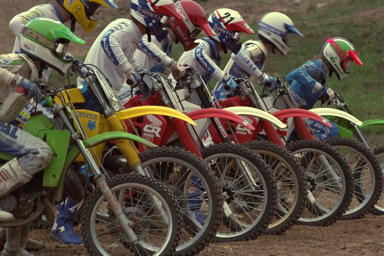} & 
        \includegraphics[width=0.09\linewidth]{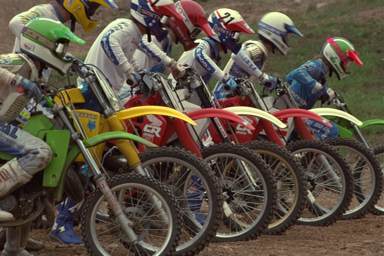} & 
        \includegraphics[width=0.09\linewidth]{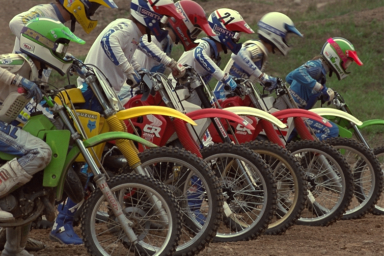} &
        \includegraphics[width=0.09\linewidth]{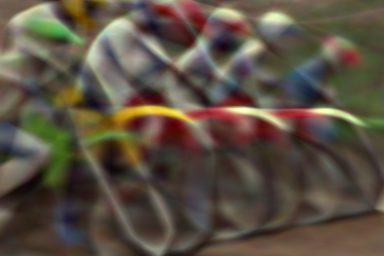} \\

        \rotatebox{90}{\tiny \hspace{0.25em} \emph{Den.}} &
        \includegraphics[width=0.09\linewidth]{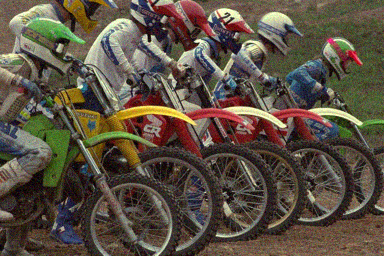} & 
        \includegraphics[width=0.09\linewidth]{figures/qualitatives/supervision/bikers_fitting.png} & 
        \includegraphics[width=0.09\linewidth]{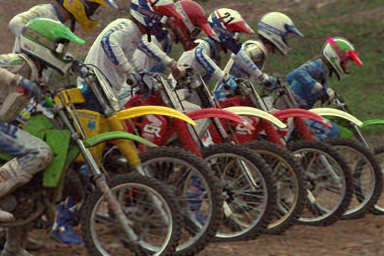} & 
        \includegraphics[width=0.09\linewidth]{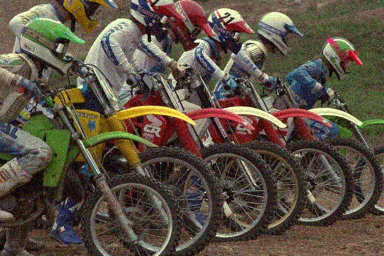} & 
        \includegraphics[width=0.09\linewidth]{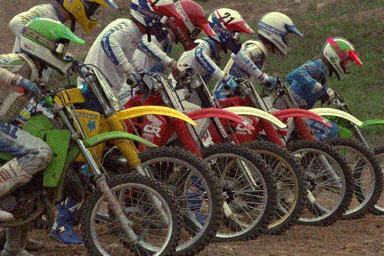} & 
        \includegraphics[width=0.09\linewidth]{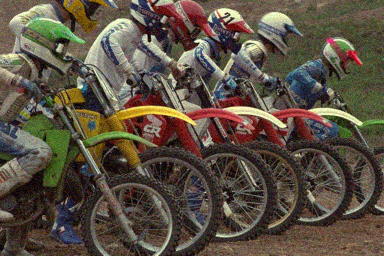} & 
        \includegraphics[width=0.09\linewidth]{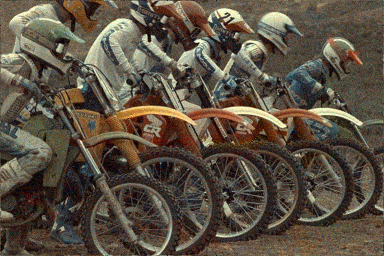} & 
        \includegraphics[width=0.09\linewidth]{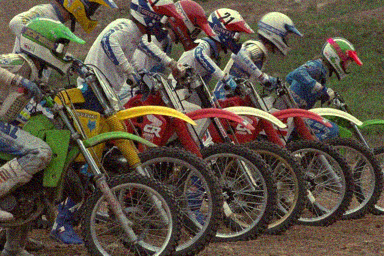} & 
        \includegraphics[width=0.09\linewidth]{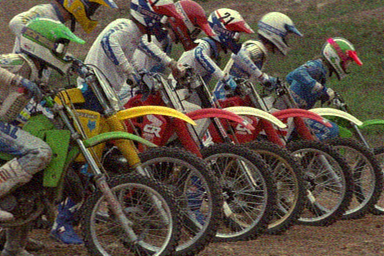} & 
        \includegraphics[width=0.09\linewidth]{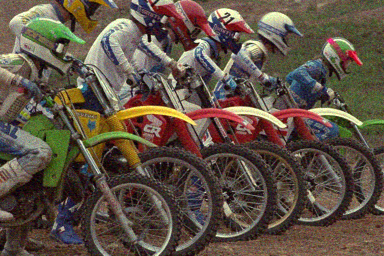} &
        \includegraphics[width=0.09\linewidth]{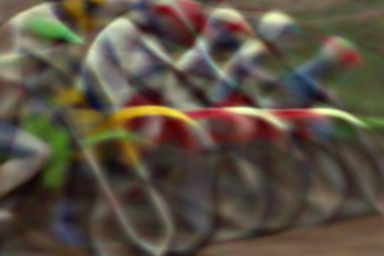} \\

        \rotatebox{90}{\tiny \hspace{0.4em} \emph{Inp.}} &
        \includegraphics[width=0.09\linewidth]{figures/qualitatives/supervision/bikers_inpainting.png} & 
        \includegraphics[width=0.09\linewidth]{figures/qualitatives/supervision/bikers_fitting.png} & 
        \includegraphics[width=0.09\linewidth]{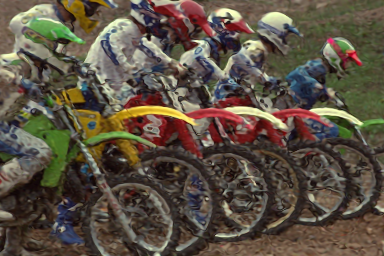} & 
        \includegraphics[width=0.09\linewidth]{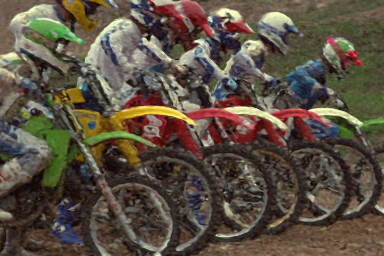} & 
        \includegraphics[width=0.09\linewidth]{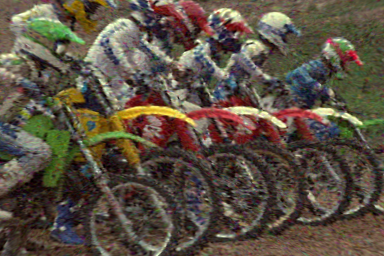} & 
        \includegraphics[width=0.09\linewidth]{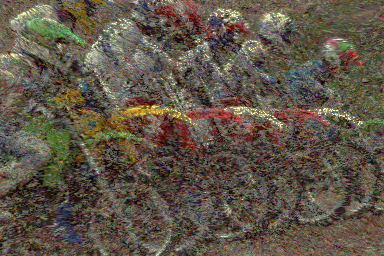} & 
        \includegraphics[width=0.09\linewidth]{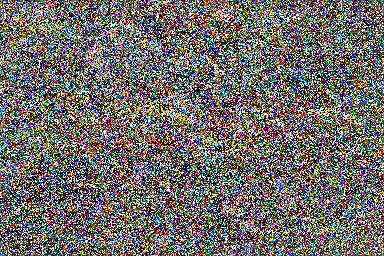} & 
        \includegraphics[width=0.09\linewidth]{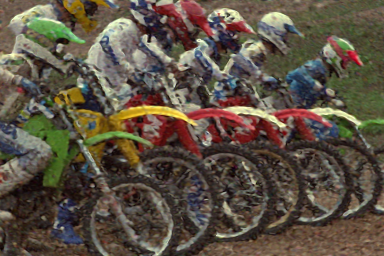} & 
        \includegraphics[width=0.09\linewidth]{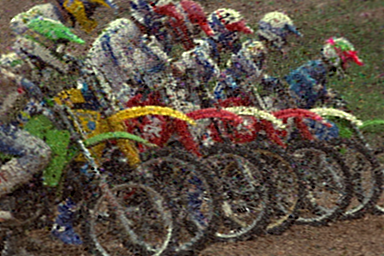} & 
        \includegraphics[width=0.09\linewidth]{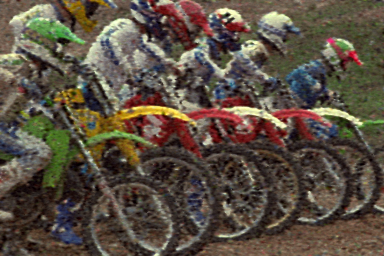} &
        \includegraphics[width=0.09\linewidth]{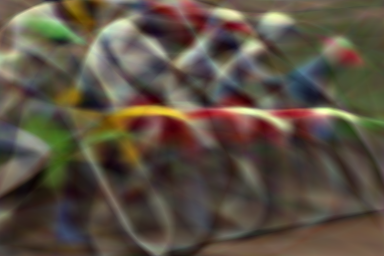} \\

        \rotatebox{90}{\tiny \emph{SR}} &
        \makebox[0.09\linewidth][c]{\includegraphics[width=0.05\linewidth, valign=m]{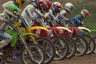}} &
        \includegraphics[width=0.09\linewidth, valign=m]{figures/qualitatives/supervision/bikers_fitting.png} & 
        \includegraphics[width=0.09\linewidth, valign=m]{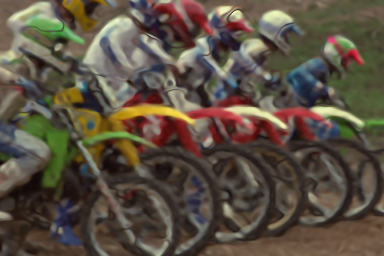} & 
        \includegraphics[width=0.09\linewidth, valign=m]{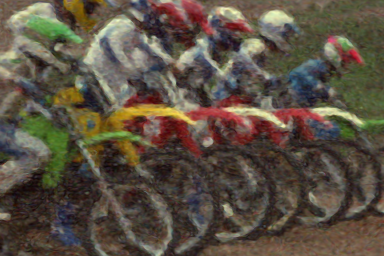} & 
        \includegraphics[width=0.09\linewidth, valign=m]{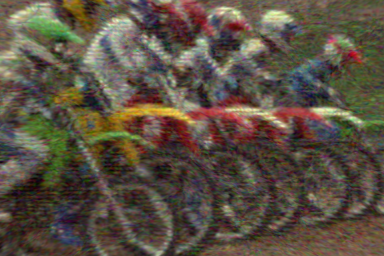} & 
        \includegraphics[width=0.09\linewidth, valign=m]{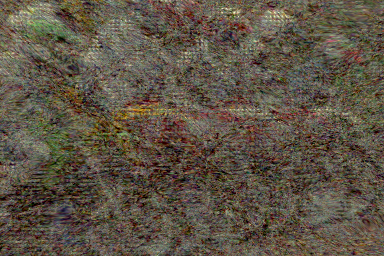} & 
        \includegraphics[width=0.09\linewidth, valign=m]{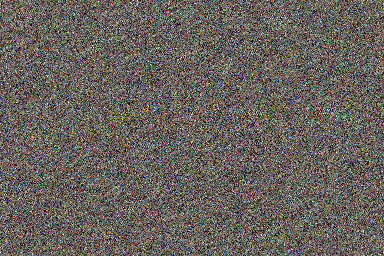} & 
        \includegraphics[width=0.09\linewidth, valign=m]{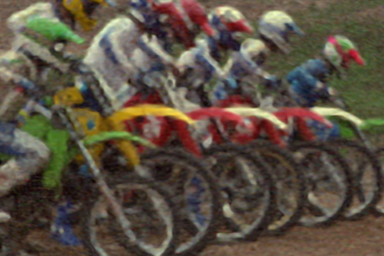} & 
        \includegraphics[width=0.09\linewidth, valign=m]{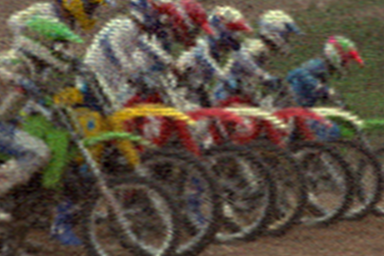} & 
        \includegraphics[width=0.09\linewidth, valign=m]{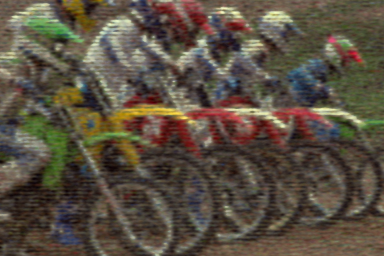} &
        \includegraphics[width=0.09\linewidth, valign=m]{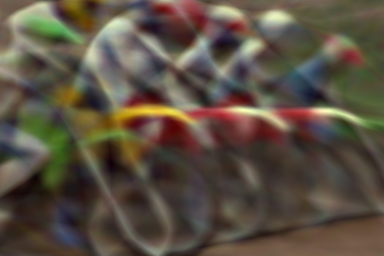} \\[1.75ex]
        
        \rotatebox{90}{\tiny \emph{Poiss.}} &
        \includegraphics[width=0.09\linewidth]{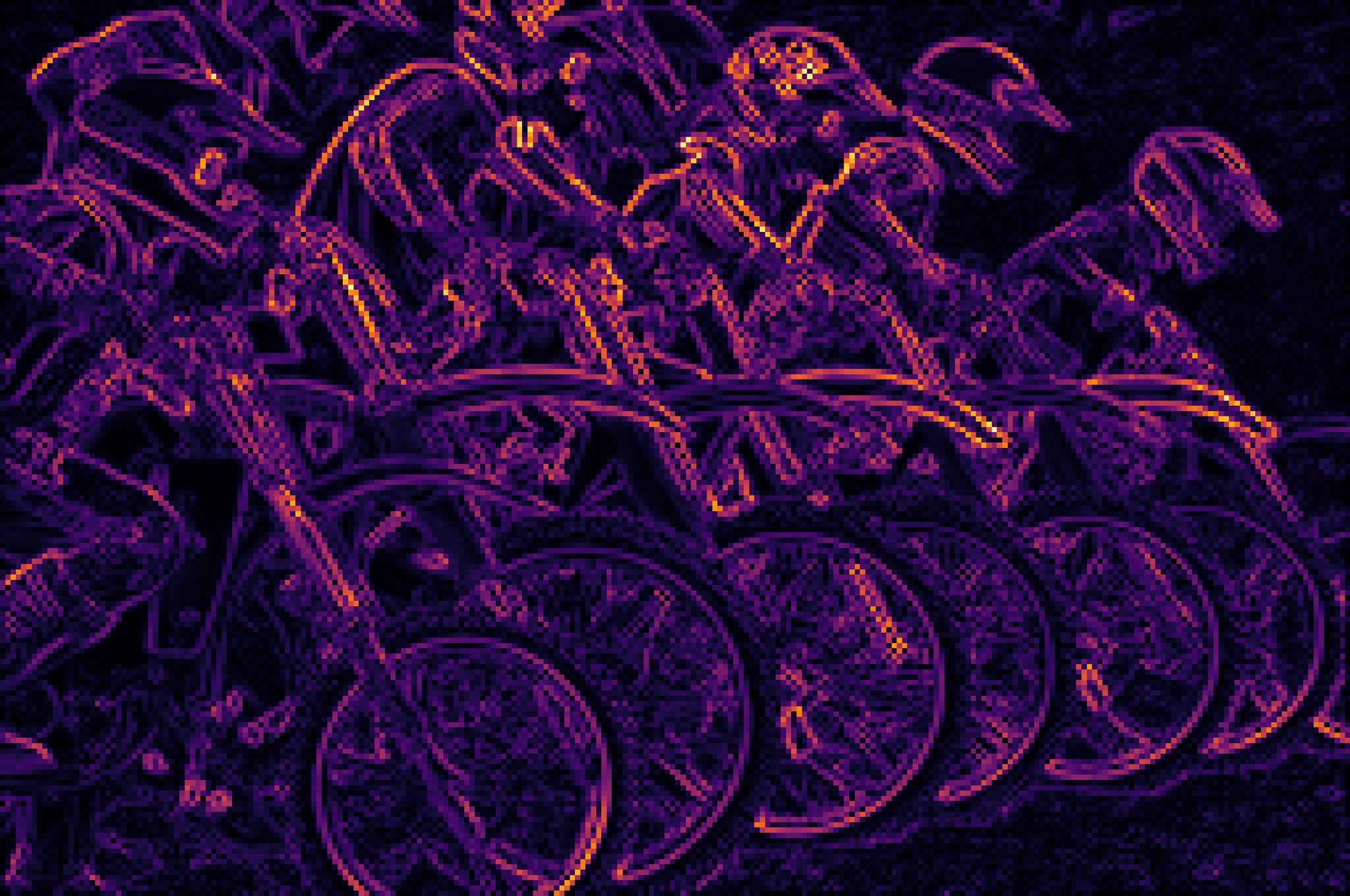} & 
        \includegraphics[width=0.09\linewidth]{figures/qualitatives/supervision/bikers_fitting.png} & 
        \includegraphics[width=0.09\linewidth]{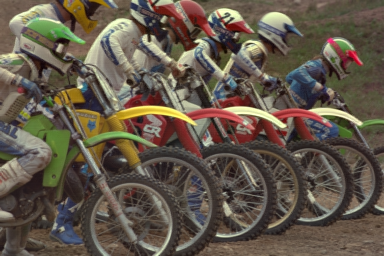} & 
        \includegraphics[width=0.09\linewidth]{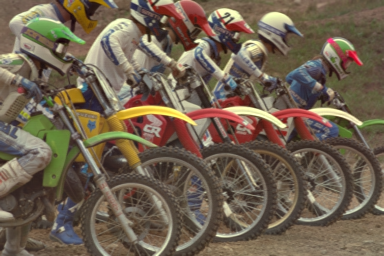} & 
        \includegraphics[width=0.09\linewidth]{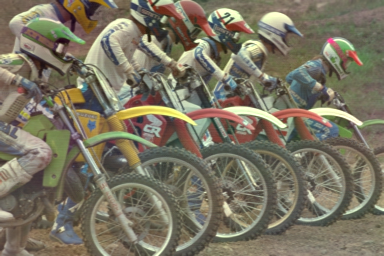} & 
        \includegraphics[width=0.09\linewidth]{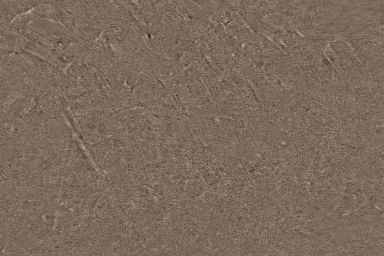} & 
        \includegraphics[width=0.09\linewidth]{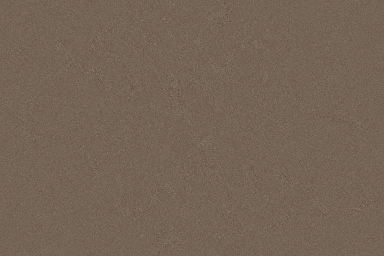} & 
        \includegraphics[width=0.09\linewidth]{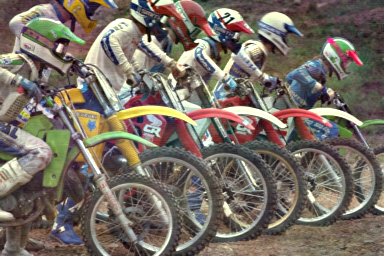} & 
        \includegraphics[width=0.09\linewidth]{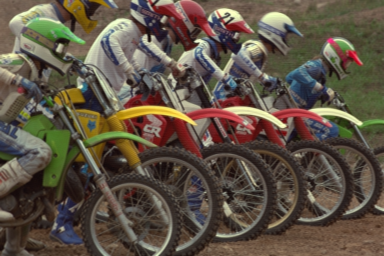} & 
        \includegraphics[width=0.09\linewidth]{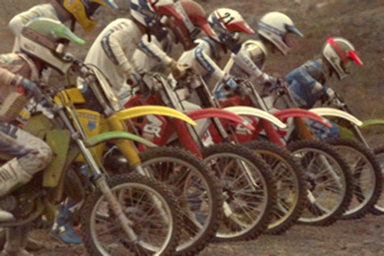} &
        \includegraphics[width=0.09\linewidth]{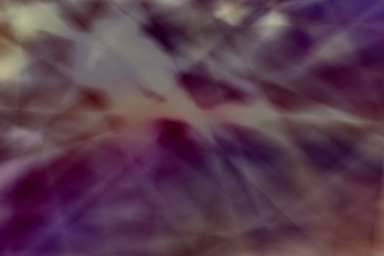} \\

    \end{tabular}
    }
    \caption{
    \textbf{Qualitative results on Bikers.}
    Please zoom to appreciate the differences.
    }
    \label{fig:qual_bikers}
\end{figure}

        \begin{figure}[t]
    \centering
    \resizebox{\linewidth}{!}{
    \begin{tabular}{lcc @{\hspace{0.75em}} ccccccccc}        

        & {\tiny Supervision} 
        & {\tiny GT} 
        & {\tiny \algoname{}} 
        & {\tiny SIREN} 
        & {\tiny Gauss} 
        & {\tiny WIRE} 
        & {\tiny BACON} 
        & {\tiny FINER} 
        & {\tiny MFN} 
        & {\tiny Fourier} 
        & {\tiny FR} \\
        
        \rotatebox{90}{\tiny \hspace{0.25em} \emph{Fit.}} &
        \includegraphics[width=0.09\linewidth]{figures/qualitatives/supervision/butterfly_fitting.png} & 
        \includegraphics[width=0.09\linewidth]{figures/qualitatives/supervision/butterfly_fitting.png} & 
        \includegraphics[width=0.09\linewidth]{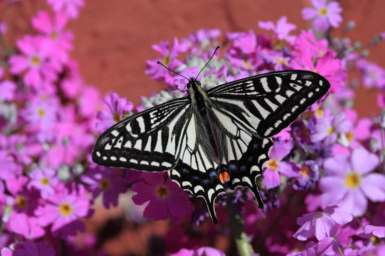} & 
        \includegraphics[width=0.09\linewidth]{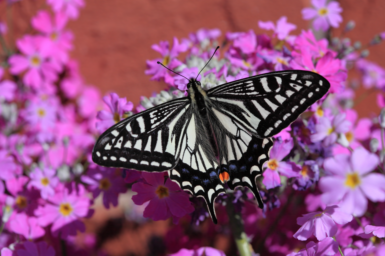} & 
        \includegraphics[width=0.09\linewidth]{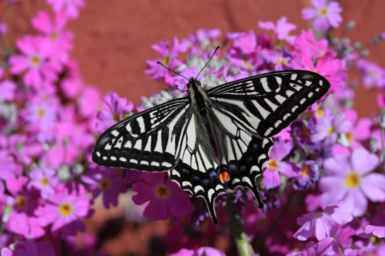} & 
        \includegraphics[width=0.09\linewidth]{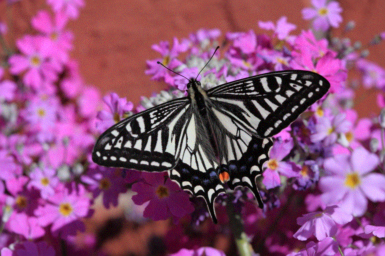} & 
        \includegraphics[width=0.09\linewidth]{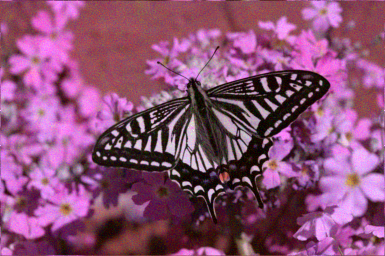} & 
        \includegraphics[width=0.09\linewidth]{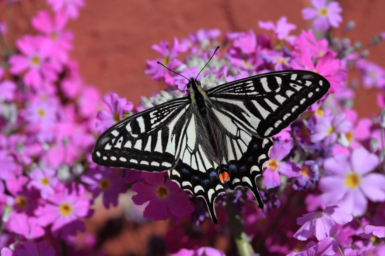} & 
        \includegraphics[width=0.09\linewidth]{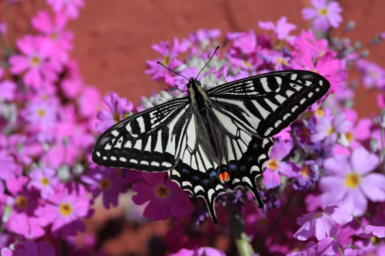} & 
        \includegraphics[width=0.09\linewidth]{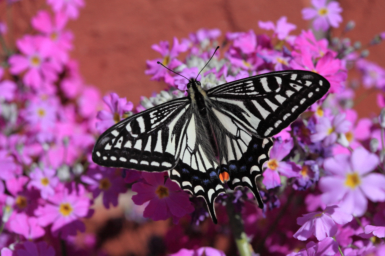} &
        \includegraphics[width=0.09\linewidth]{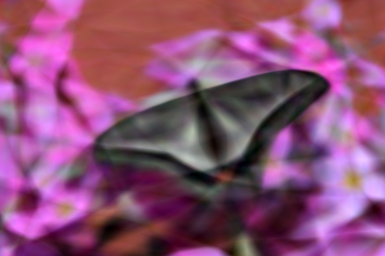} \\

        \rotatebox{90}{\tiny \hspace{0.25em} \emph{Den.}} &
        \includegraphics[width=0.09\linewidth]{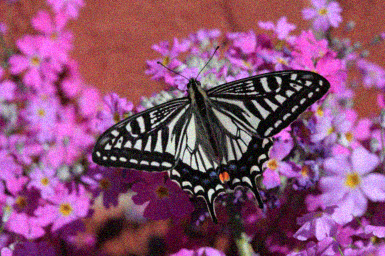} & 
        \includegraphics[width=0.09\linewidth]{figures/qualitatives/supervision/butterfly_fitting.png} & 
        \includegraphics[width=0.09\linewidth]{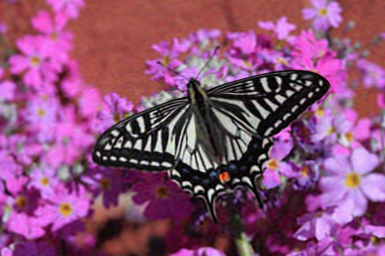} & 
        \includegraphics[width=0.09\linewidth]{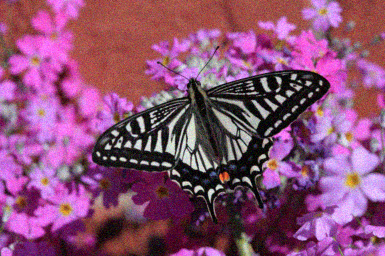} & 
        \includegraphics[width=0.09\linewidth]{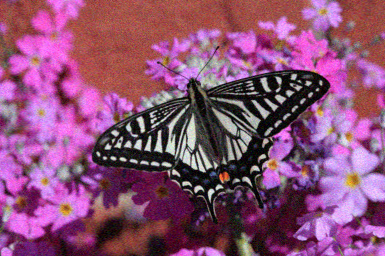} & 
        \includegraphics[width=0.09\linewidth]{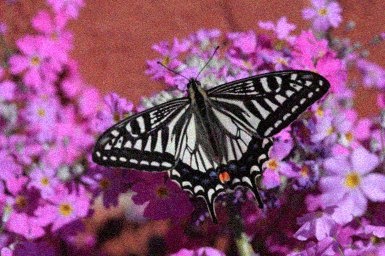} & 
        \includegraphics[width=0.09\linewidth]{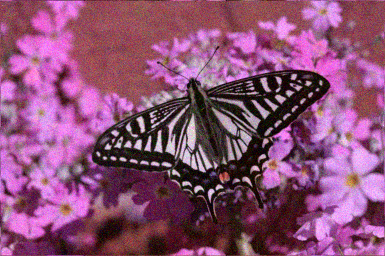} & 
        \includegraphics[width=0.09\linewidth]{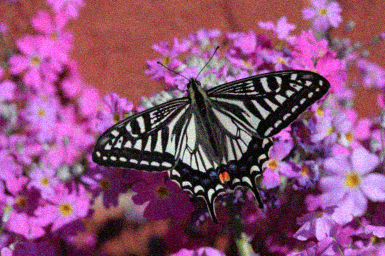} & 
        \includegraphics[width=0.09\linewidth]{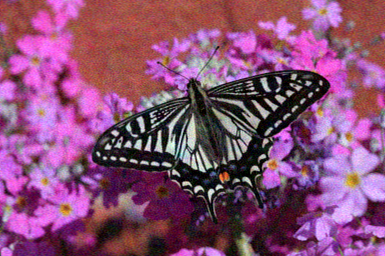} & 
        \includegraphics[width=0.09\linewidth]{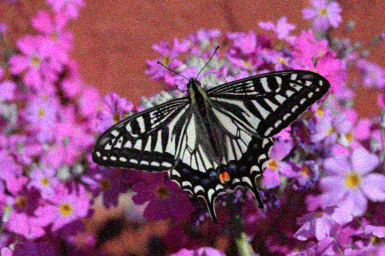} &
        \includegraphics[width=0.09\linewidth]{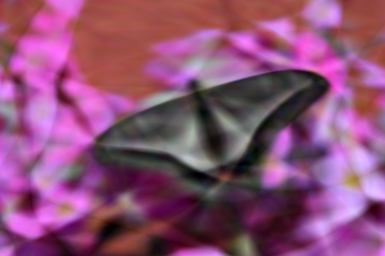} \\

        \rotatebox{90}{\tiny \hspace{0.4em} \emph{Inp.}} &
        \includegraphics[width=0.09\linewidth]{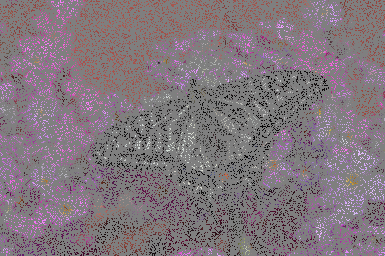} & 
        \includegraphics[width=0.09\linewidth]{figures/qualitatives/supervision/butterfly_fitting.png} & 
        \includegraphics[width=0.09\linewidth]{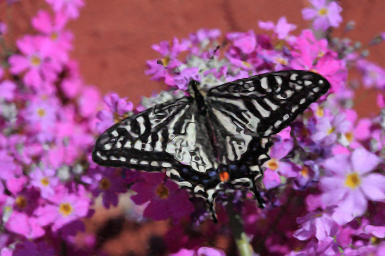} & 
        \includegraphics[width=0.09\linewidth]{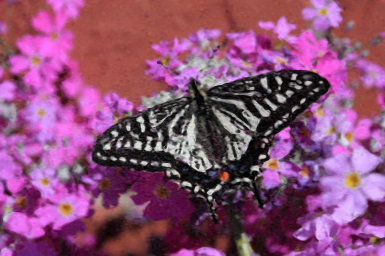} & 
        \includegraphics[width=0.09\linewidth]{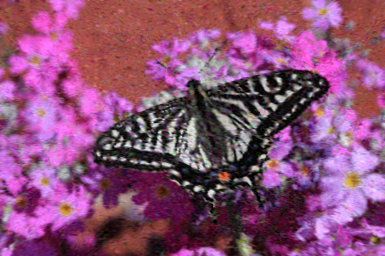} & 
        \includegraphics[width=0.09\linewidth]{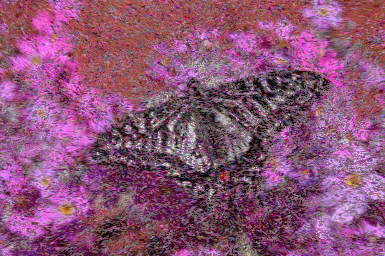} & 
        \includegraphics[width=0.09\linewidth]{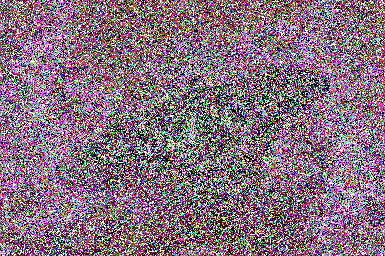} & 
        \includegraphics[width=0.09\linewidth]{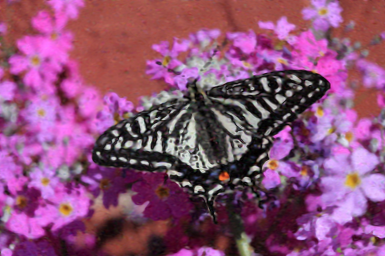} & 
        \includegraphics[width=0.09\linewidth]{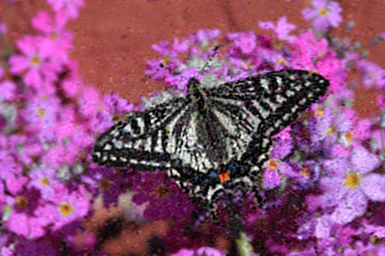} & 
        \includegraphics[width=0.09\linewidth]{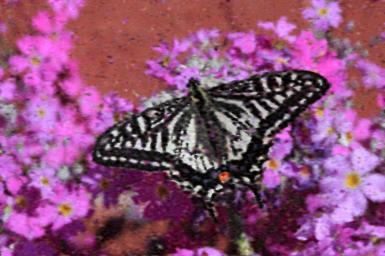} &
        \includegraphics[width=0.09\linewidth]{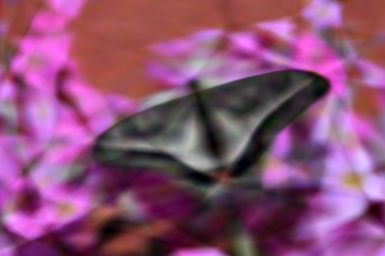} \\

        \rotatebox{90}{\tiny \emph{SR}} &
        \makebox[0.09\linewidth][c]{\includegraphics[width=0.05\linewidth, valign=m]{figures/qualitatives/supervision/butterfly_super_resolution.png}} &
        \includegraphics[width=0.09\linewidth, valign=m]{figures/qualitatives/supervision/butterfly_fitting.png} & 
        \includegraphics[width=0.09\linewidth, valign=m]{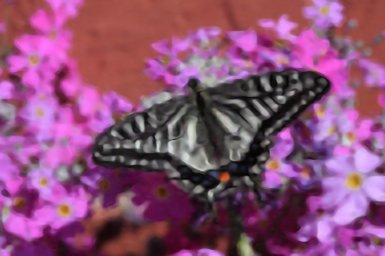} & 
        \includegraphics[width=0.09\linewidth, valign=m]{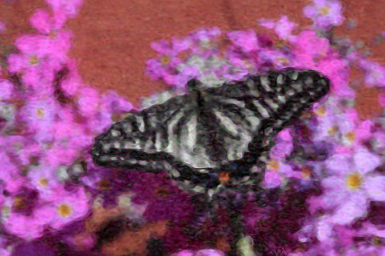} & 
        \includegraphics[width=0.09\linewidth, valign=m]{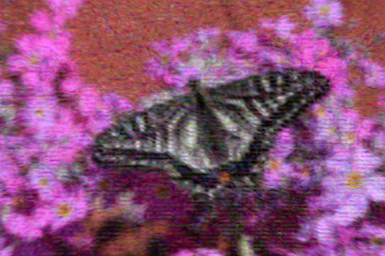} & 
        \includegraphics[width=0.09\linewidth, valign=m]{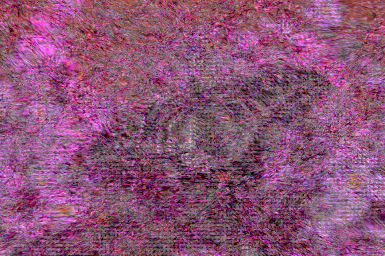} & 
        \includegraphics[width=0.09\linewidth, valign=m]{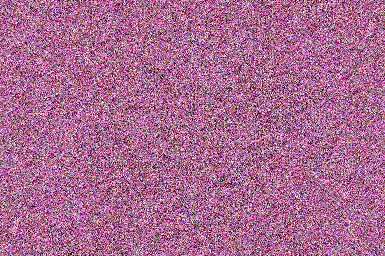} & 
        \includegraphics[width=0.09\linewidth, valign=m]{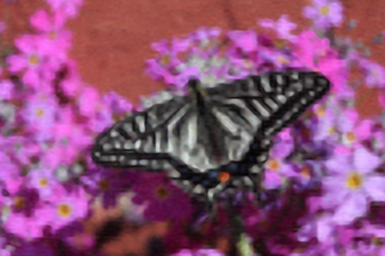} & 
        \includegraphics[width=0.09\linewidth, valign=m]{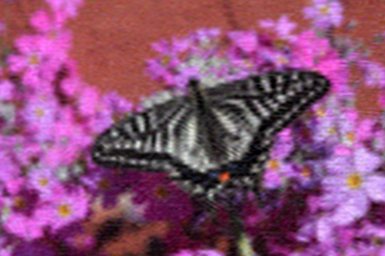} & 
        \includegraphics[width=0.09\linewidth, valign=m]{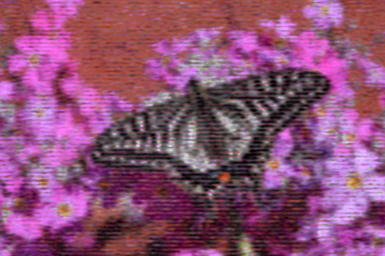} &
        \includegraphics[width=0.09\linewidth, valign=m]{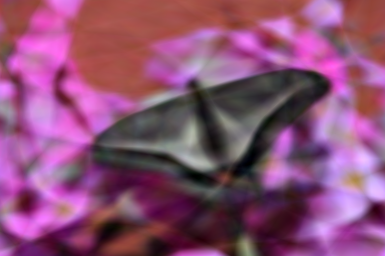} \\[1.75ex]
        
        \rotatebox{90}{\tiny \emph{Poiss.}} &
        \includegraphics[width=0.09\linewidth]{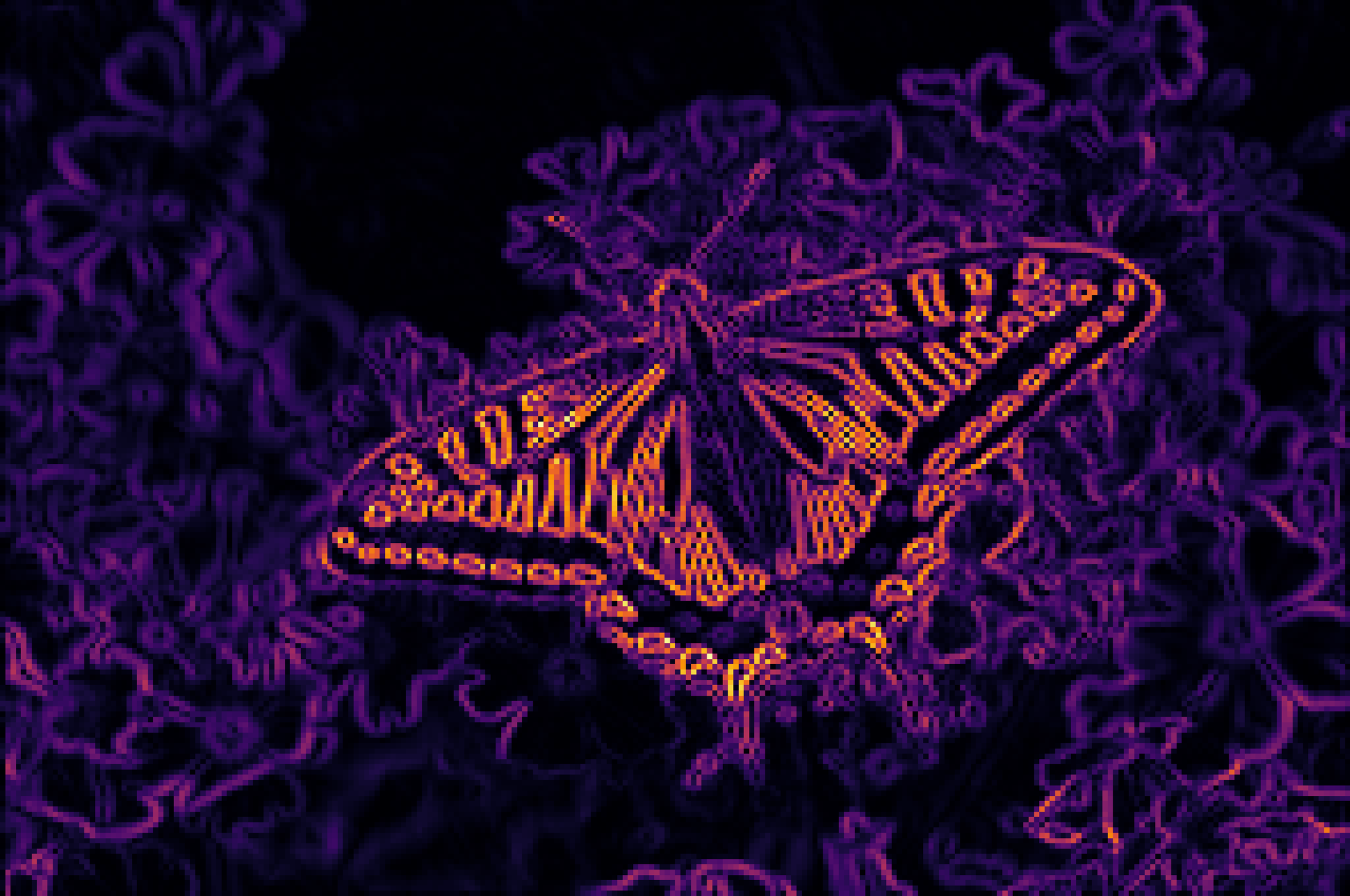} & 
        \includegraphics[width=0.09\linewidth]{figures/qualitatives/supervision/butterfly_fitting.png} & 
        \includegraphics[width=0.09\linewidth]{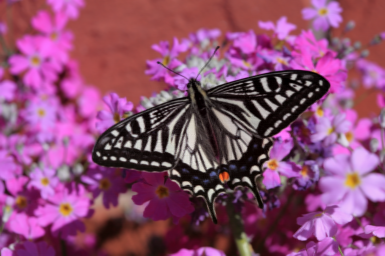} & 
        \includegraphics[width=0.09\linewidth]{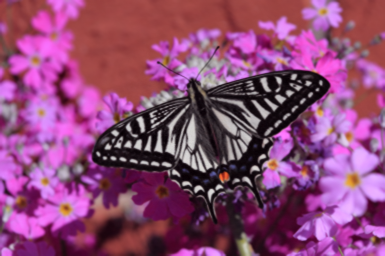} & 
        \includegraphics[width=0.09\linewidth]{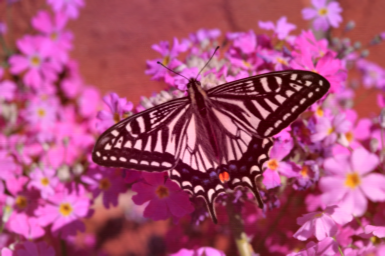} & 
        \includegraphics[width=0.09\linewidth]{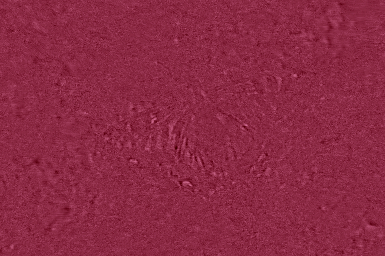} & 
        \includegraphics[width=0.09\linewidth]{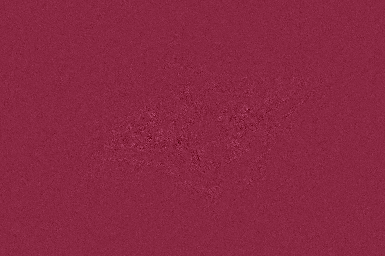} & 
        \includegraphics[width=0.09\linewidth]{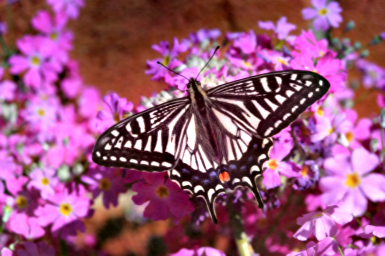} & 
        \includegraphics[width=0.09\linewidth]{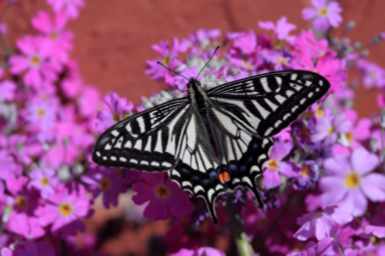} & 
        \includegraphics[width=0.09\linewidth]{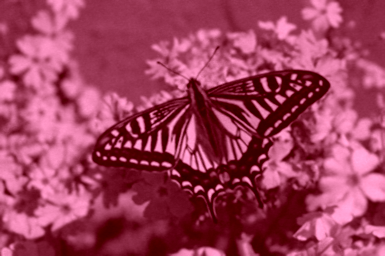} &
        \includegraphics[width=0.09\linewidth]{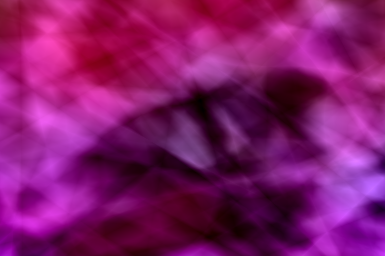} \\

    \end{tabular}
    }
    \caption{
    \textbf{Qualitative results on Butterfly.}
    Please zoom to appreciate the differences.
    }
    \label{fig:qual_butterfly}
\end{figure}

        \begin{figure}[t]
    \centering
    \resizebox{\linewidth}{!}{
    \begin{tabular}{lcc @{\hspace{0.75em}} ccccccccc}        

        & {\tiny Supervision} 
        & {\tiny GT} 
        & {\tiny \algoname{}} 
        & {\tiny SIREN} 
        & {\tiny Gauss} 
        & {\tiny WIRE} 
        & {\tiny BACON} 
        & {\tiny FINER} 
        & {\tiny MFN} 
        & {\tiny Fourier} 
        & {\tiny FR} \\
        
        \rotatebox{90}{\tiny \hspace{0.25em} \emph{Fitting}} &
        \includegraphics[width=0.09\linewidth]{figures/qualitatives/supervision/knot_fitting.jpg} & 
        \includegraphics[width=0.09\linewidth]{figures/qualitatives/supervision/knot_fitting.jpg} & 
        \includegraphics[width=0.09\linewidth]{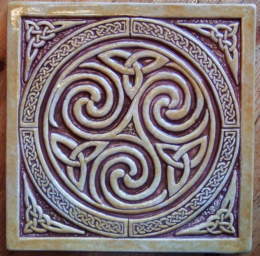} & 
        \includegraphics[width=0.09\linewidth]{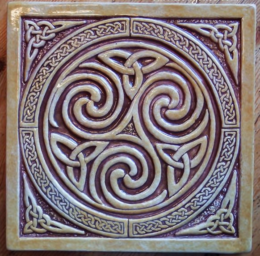} & 
        \includegraphics[width=0.09\linewidth]{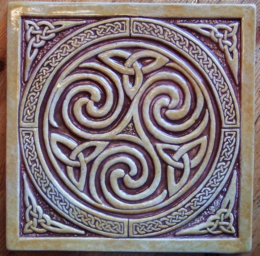} & 
        \includegraphics[width=0.09\linewidth]{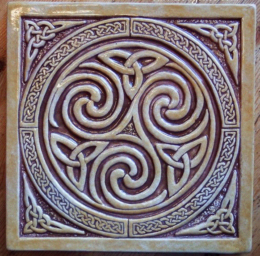} & 
        \includegraphics[width=0.09\linewidth]{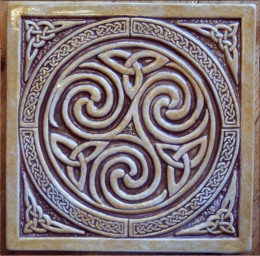} & 
        \includegraphics[width=0.09\linewidth]{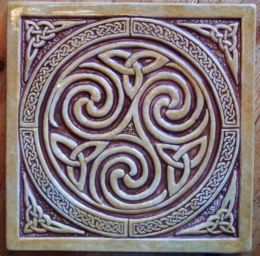} & 
        \includegraphics[width=0.09\linewidth]{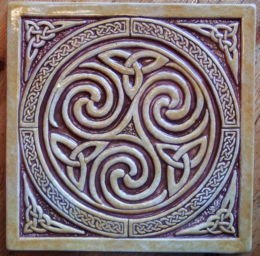} & 
        \includegraphics[width=0.09\linewidth]{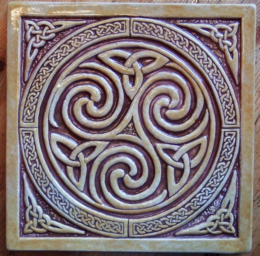} &
        \includegraphics[width=0.09\linewidth]{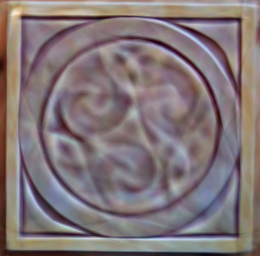} \\

        \rotatebox{90}{\tiny \emph{Denoising}} &
        \includegraphics[width=0.09\linewidth]{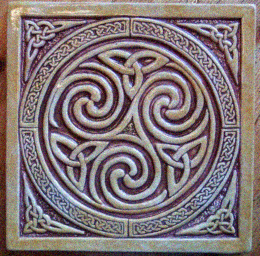} & 
        \includegraphics[width=0.09\linewidth]{figures/qualitatives/supervision/knot_fitting.jpg} & 
        \includegraphics[width=0.09\linewidth]{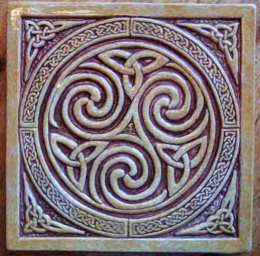} & 
        \includegraphics[width=0.09\linewidth]{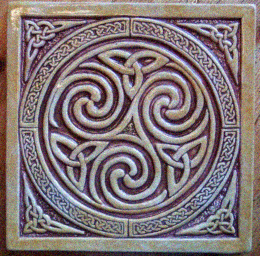} & 
        \includegraphics[width=0.09\linewidth]{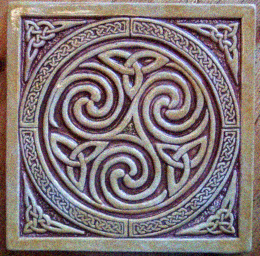} & 
        \includegraphics[width=0.09\linewidth]{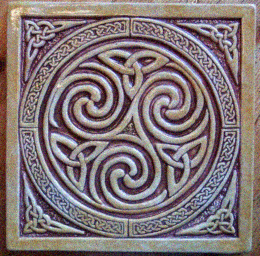} & 
        \includegraphics[width=0.09\linewidth]{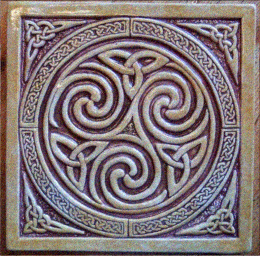} & 
        \includegraphics[width=0.09\linewidth]{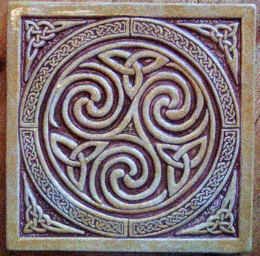} & 
        \includegraphics[width=0.09\linewidth]{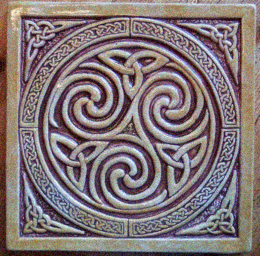} & 
        \includegraphics[width=0.09\linewidth]{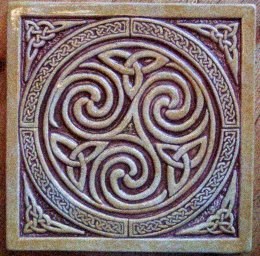} &
        \includegraphics[width=0.09\linewidth]{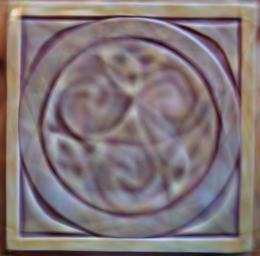} \\

        \rotatebox{90}{\tiny \emph{Inpainting}} &
        \includegraphics[width=0.09\linewidth]{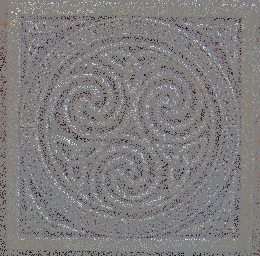} & 
        \includegraphics[width=0.09\linewidth]{figures/qualitatives/supervision/knot_fitting.jpg} & 
        \includegraphics[width=0.09\linewidth]{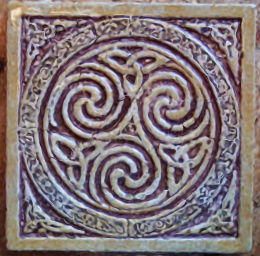} & 
        \includegraphics[width=0.09\linewidth]{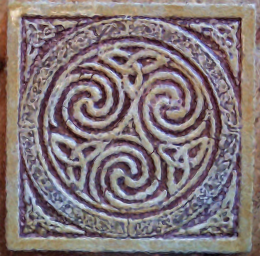} & 
        \includegraphics[width=0.09\linewidth]{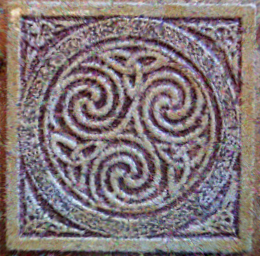} & 
        \includegraphics[width=0.09\linewidth]{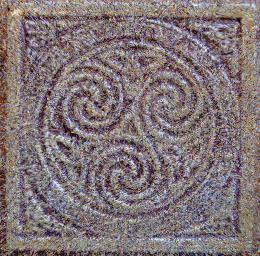} & 
        \includegraphics[width=0.09\linewidth]{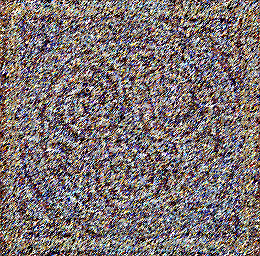} & 
        \includegraphics[width=0.09\linewidth]{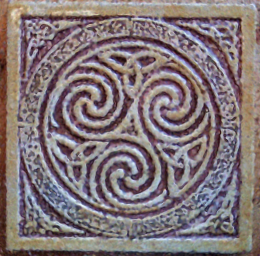} & 
        \includegraphics[width=0.09\linewidth]{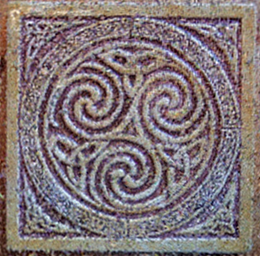} & 
        \includegraphics[width=0.09\linewidth]{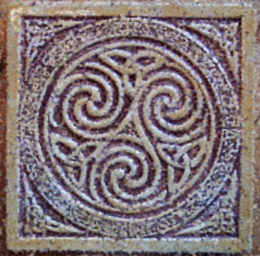} &
        \includegraphics[width=0.09\linewidth]{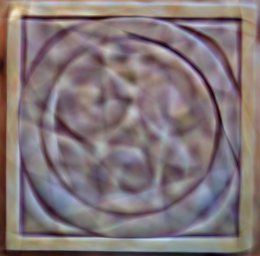} \\

        \rotatebox{90}{\tiny \emph{SR}} &
        \makebox[0.09\linewidth][c]{\includegraphics[width=0.05\linewidth, valign=m]{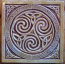}} &
        \includegraphics[width=0.09\linewidth, valign=m]{figures/qualitatives/supervision/knot_fitting.jpg} & 
        \includegraphics[width=0.09\linewidth, valign=m]{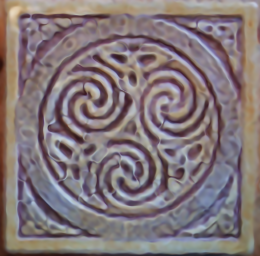} & 
        \includegraphics[width=0.09\linewidth, valign=m]{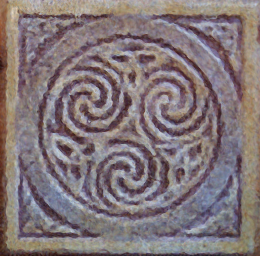} & 
        \includegraphics[width=0.09\linewidth, valign=m]{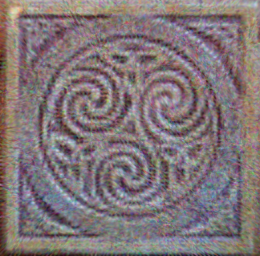} & 
        \includegraphics[width=0.09\linewidth, valign=m]{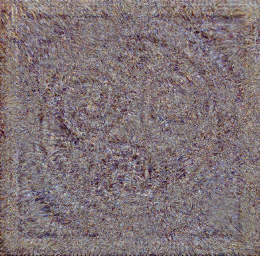} & 
        \includegraphics[width=0.09\linewidth, valign=m]{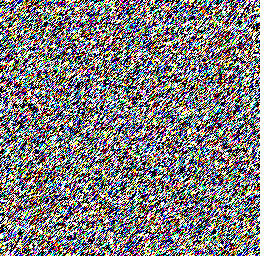} & 
        \includegraphics[width=0.09\linewidth, valign=m]{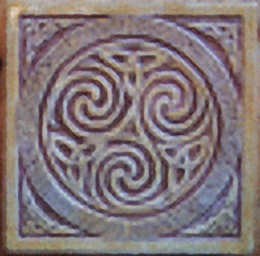} & 
        \includegraphics[width=0.09\linewidth, valign=m]{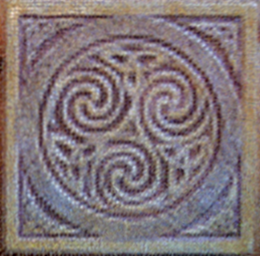} & 
        \includegraphics[width=0.09\linewidth, valign=m]{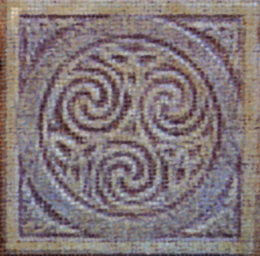} &
        \includegraphics[width=0.09\linewidth, valign=m]{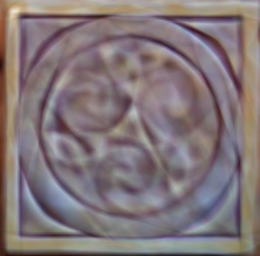} \\[3ex]
        
        \rotatebox{90}{\tiny \hspace{0.1em} \emph{Poisson}} &
        \includegraphics[width=0.09\linewidth]{figures/qualitatives/supervision/knot_poisson.png} & 
        \includegraphics[width=0.09\linewidth]{figures/qualitatives/supervision/knot_fitting.jpg} & 
        \includegraphics[width=0.09\linewidth]{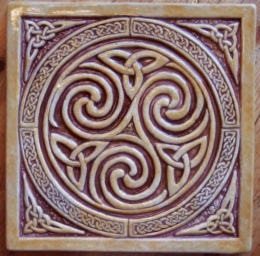} & 
        \includegraphics[width=0.09\linewidth]{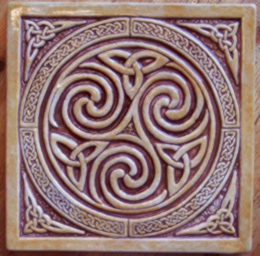} & 
        \includegraphics[width=0.09\linewidth]{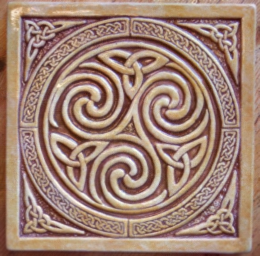} & 
        \includegraphics[width=0.09\linewidth]{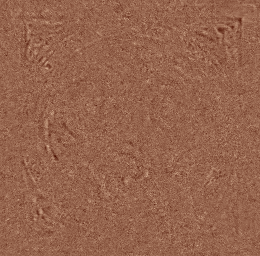} & 
        \includegraphics[width=0.09\linewidth]{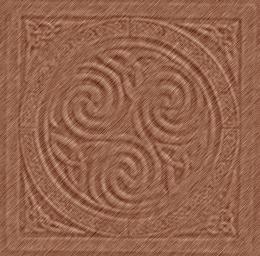} & 
        \includegraphics[width=0.09\linewidth]{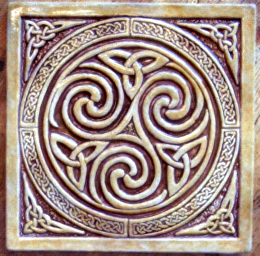} & 
        \includegraphics[width=0.09\linewidth]{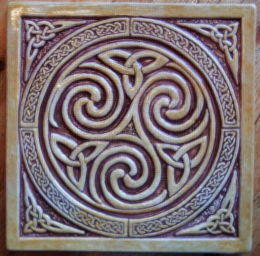} & 
        \includegraphics[width=0.09\linewidth]{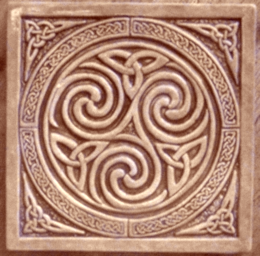} &
        \includegraphics[width=0.09\linewidth]{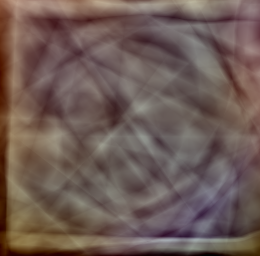} \\

    \end{tabular}
    }
    \caption{
    \textbf{Qualitative results on Knot.}
    Please zoom to appreciate the differences.
    }
    \label{fig:qual_knot}
\end{figure}

\end{document}